\documentclass[journal]{IEEEtran}

\usepackage{bm}
\usepackage{amsmath}
\usepackage{epsf}
\usepackage{graphics}
\usepackage{ amssymb }
\usepackage[dvips]{graphicx}
\usepackage{mathtools}% Loads amsmath
\usepackage{epsfig}
\usepackage{cite}
\usepackage[linesnumbered,ruled,vlined]{algorithm2e}
\usepackage{colortbl}
\usepackage{color}
\usepackage{enumitem}
\usepackage{soul,xcolor}
\usepackage{bm}
\usepackage{amsmath}
\usepackage{epsf}
\usepackage{graphics}
\usepackage{ amssymb }
\usepackage[dvips]{graphicx}
\usepackage{epsfig}
\usepackage{cite}
\usepackage[linesnumbered,ruled,vlined]{algorithm2e}
\usepackage{graphicx}
\usepackage{epsfig}
\usepackage{latexsym}
\usepackage{amsfonts}
\usepackage{here}
\usepackage{rawfonts}

\usepackage[utf8]{inputenc}
\usepackage[english]{babel}
\usepackage{amsmath}
\usepackage{amsfonts}
\usepackage{amssymb}
\usepackage{color} %May be necessary if you want to color links
\usepackage{bm}
\usepackage{listings}
\usepackage{caption}
    \usepackage{multicol}

\usepackage{tabularx,booktabs}
\newcolumntype{Y}{>{\centering\arraybackslash}X}
\makeatletter
\newcommand\notsotiny{\@setfontsize\notsotiny{6.31415}{7.1828}}
\makeatother
\usepackage{amssymb}
\usepackage{amsthm}
\usepackage{graphicx}
\usepackage{epstopdf}
\usepackage{listings}
\usepackage{float}
\usepackage{amsmath}
\usepackage{amssymb}
\usepackage{amsfonts}
\usepackage{epstopdf}

\usepackage{multirow}
\usepackage{amscd}
\usepackage{mathrsfs}
\usepackage{graphicx}
\usepackage{color}
\usepackage{url}
\usepackage{bm}
\usepackage{setspace}
\usepackage{footnote}
\usepackage{xcolor}
\DeclareMathOperator*{\argmin}{arg\,min}
\newtheorem{theorem}{Theorem}
\newtheorem{lemma}{Lemma}

\newtheorem{definition}{Definition}
\newtheorem{proposition}{Proposition}
\newtheorem{corollary}{Corollary}

\newtheorem{remark}{Remark}

\newtheorem{assumption}{Assumption}

\usepackage[]{algorithm2e}

\addto\captionsenglish{}
\usepackage{bbm}

\usepackage{multicol}
\usepackage{mathtools}

\usepackage{lipsum,graphicx,subcaption}

\usepackage{diagbox}
\usepackage{hyperref}
\usepackage[protrusion=true,expansion=true]{microtype}
\pdfoutput=1
\usepackage[font=footnotesize]{caption}
\allowdisplaybreaks[4]

\usepackage{graphicx}
\usepackage{grffile}
\usepackage{tabularx}
\newcounter{term}[section]

\renewcommand\theterm{\alph{term}}
\makeatletter
\newcommand{\vast}{\bBigg@{4}}

\newcommand{\Vast}{\bBigg@{5}}
\makeatother
\newcommand\semiHuge{\fontsize{22.7}{31.38}\selectfont}
\usepackage{soul,xcolor}

\usepackage[most]{tcolorbox}

\tcbset{colback=yellow!10!white, colframe=black!50!black, 
        highlight math style= {enhanced, %<-- needed for the ’remember’ options
            colframe=black,colback=red!10!white,boxsep=0pt}
        }
        \allowdisplaybreaks

\begin{document} 

\title{{\semiHuge Parallel Successive Learning for Dynamic Distributed Model Training over Heterogeneous Wireless Networks}}
\author{Seyyedali Hosseinalipour,~\IEEEmembership{Member,~IEEE}, Su Wang,~\IEEEmembership{Student~Member,~IEEE}, Nicol\`{o} Michelusi,~\IEEEmembership{Senior~Member,~IEEE}, Vaneet Aggarwal,~\IEEEmembership{Senior~Member,~IEEE}, Christopher G. Brinton,~\IEEEmembership{Senior~Member,~IEEE}, David~J.~Love,~\IEEEmembership{Fellow,~IEEE}, and Mung~Chiang,~\IEEEmembership{Fellow,~IEEE}
\vspace{-3mm}
\thanks{
This work was in part supported by Cisco, Inc., NSF CNS-2146171, ONR N000142212305, DARPA D22AP00168-00, and NSF CNS-2129615.

S. Hosseinalipour is with the Department of Electrical Engineering,
University at Buffalo-SUNY, Buffalo, NY 14228 USA (e-mail: alipour@buffalo.edu).

S. Wang, C. G. Brinton, D. J. Love, and M. Chiang are with the School of Electrical and Computer Engineering,
Purdue University, West Lafayette, IN 47906 USA (e-mail: \{wang2506,cgb,djlove,chiang\}@purdue.edu).

N. Michelusi is with the School of Electrical, Computer, and Energy
Engineering, Arizona State University, Tempe, AZ 85287 USA (e-mail:
nicolo.michelusi@asu.edu).

V. Aggarwal is with the School of Industrial Engineering, Purdue
University, West Lafayette, IN 47906 USA (e-mail: vaneet@purdue.edu).

(Corresponding author: Seyyedali Hosseinalipour)}}
\maketitle
\vspace{-9mm}
\setulcolor{red}
\setul{red}{2pt}
\setstcolor{red}

\begin{abstract}
Federated learning (FedL) has emerged as a popular technique for distributing model training over a set of wireless devices, via iterative local updates (at devices) and global aggregations (at the server). In this paper, we develop \textit{parallel successive learning} (PSL), which expands the FedL architecture along  three dimensions: (i) \textit{Network}, allowing decentralized cooperation among the devices via device-to-device (D2D) communications. (ii) \textit{Heterogeneity}, interpreted at three levels: (ii-a) Learning:  PSL considers heterogeneous number of stochastic gradient descent iterations with different mini-batch sizes at the devices; (ii-b) Data: PSL presumes a \textit{dynamic environment} with  data arrival and departure, where the distributions of local datasets evolve over time, captured via a new metric for \textit{model/concept drift}.  (ii-c) Device: PSL considers devices with different computation and communication capabilities. (iii) \textit{Proximity}, where devices have different distances to each other and the access point.  PSL considers the realistic scenario where global aggregations are conducted with \textit{idle times} in-between them for resource efficiency improvements, and incorporates \textit{data dispersion} and \textit{model dispersion with local model condensation} into FedL. Our analysis sheds light on the notion of \textit{cold} vs. \textit{warmed up} models, and model \textit{inertia} in distributed machine learning.  We then propose \textit{network-aware dynamic model tracking}  to optimize the model learning vs. resource efficiency tradeoff, which we show is an NP-hard signomial programming problem. We finally solve this problem through proposing a general optimization solver. Our numerical results reveal new findings on the interdependencies between the idle times in-between the global aggregations, model/concept drift, and D2D cooperation configuration.
\end{abstract}
\vspace{-.1mm}
\vspace{-6.5mm}
\section{Introduction}
\noindent 
Distributed machine learning (ML) over wireless networks has attracted significant attention recently, for applications ranging from keyboard next word prediction to autonomous driving~\cite{jiang2016machine,8865093}. Distributed ML is an alternative to centralized ML, which requires a central node, e.g., a server, coexisting with the dataset. This alternative is of particular interest since in many applications the dataset is  collected in a distributed fashion across a set of wireless devices, e.g., through their sensing equipment or the users' input, where the transfer of data to the cloud incurs significant energy consumption and long delays.

Federated learning (FedL) is the most recognized distributed ML technique, with the premise of keeping the devices' datasets local~\cite{mcmahan2017communication,zhu2020toward}. Its conventional architecture resembles a \textit{star} topology of device-server interactions (Fig.~\ref{fig:simpleFL}). Each  model training round  of FedL consists of (i) \textit{local updating}, where devices update their local models based on their datasets and the global model, e.g., via stochastic gradient descent (SGD), and (ii) \textit{global aggregation}, where the server aggregates the local models to a new global model and broadcasts it.

\vspace{-.1mm}
 \subsection{Related Work}
 \vspace{-.15mm}
 Researchers have considered the effects of limited/imperfect communications in wireless networks (e.g.,  channel fading, packet loss, and limited bandwidth) on the performance of FedL~\cite{chen2019joint,8737464,8664630,8870236,9014530}. Also, quantization~\cite{9054168} and sparsification~\cite{renggli2019sparcml} techniques have been studied to facilitate  FedL implementation. 

% Focusing on the case of FogL, our work generalizes the communication strategy of federated learning to include both distributed average consensus formation via D2D and multi-layer aggregations.
\begin{figure}[t]
\centering
\includegraphics[width=.46\textwidth]{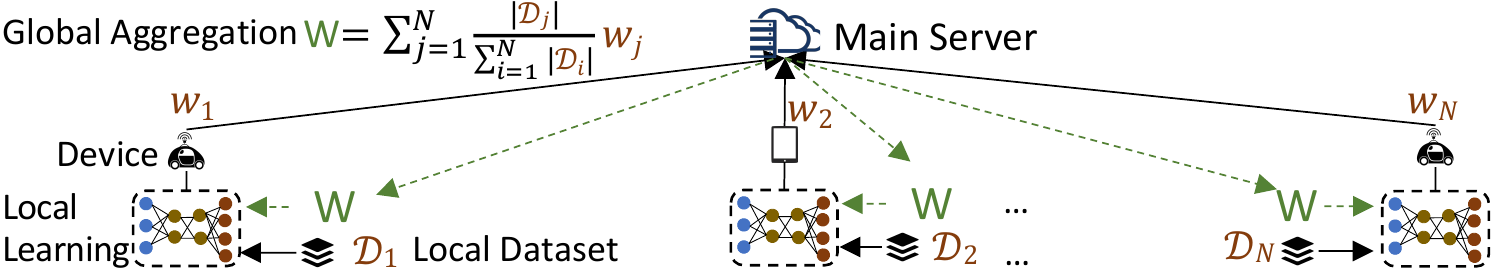}
\vspace{-1mm}
\caption{``Star" learning topology architecture of conventional FedL.}
\label{fig:simpleFL}
\vspace{-0.1mm}
\end{figure}
Researchers have also considered the computation aspects of implementing FedL over wireless networks~\cite{8737464,9261995,9024521,wang2021network,8664630}. Part of this literature has focused on the impact of straggler nodes, i.e., when some nodes have low computation capability, on model training~\cite{wang2021network,8737464,9261995,wang2018wide}. Another emphasis has been reducing the computation requirements of model training, e.g., through coding techniques~\cite{9024521}, intelligent data offloading between devices~\cite{wang2021network}, and efficient device sampling~\cite{ji2020dynamic}. 
% In our work, we study the computation requirements of devices in FogL in terms of the number of required training rounds.

%Given the limitations of processing equipment and varying availabilities of devices at the network edge, 

% Recent works have also studied privacy and security in federated learning~\cite{geyer2017differentially,hardy2017private}. Even though source data is never transmitted over the network in federated learning, it can be possible to extract sensitive information from transmitted model parameters.

% Recent works have also studied adapting security techniques such as differential privacy~\cite{geyer2017differentially} and homomorphic encryption~\cite{hardy2017private} to federated learning. Our focus is instead on the communication and computation aspects.

Other research has focused on extending the \textit{star} topology of FedL. Hierarchical FedL has proposed a \textit{tree} structure between edge devices and the main server, e.g.,~\cite{TobeAdded4}. This literature has mostly focused on specific use cases of two-tiered network structures above wireless cellular devices, e.g., with edge clouds at base stations.
% ~\cite{TobeAdded4} or small cell and macro cell base stations~\cite{TobeAdded2}. 
Additionally, there is a literature on \textit{fully decentralized} FedL over \textit{mesh} network architectures without a centralized server, where device-server communications are replaced with device-to-device (D2D) communications~\cite{8950073,9154332}. Building upon this,  \textit{semi-decentralized architectures} for FedL have also been proposed, where D2D communications are exploited in conjunction with device-server interactions to improve the model training performance~\cite{hosseinalipour2020multi,lin2021two}. In this literature, D2D communications are solely used for distributed model parameter aggregation across the nodes.

Different from existing works, we propose a new architecture for \textit{dynamic/online} distributed ML, which is motivated by the shortcomings of FedL, discussed next.
\vspace{-3.5mm}
    \subsection{FedL Shortcomings and Solution Overview}
    \vspace{-.7mm}
    \subsubsection{Limitations of FedL over heterogeneous wireless networks} Consider a wireless network consisting of a set of battery limited devices, e.g., smart phones or sensors. Assume that a base station (BS) aims to train an ML model using the data gathered by the devices.
    % , where the training is coordinated by a base station (BS). 
    There are multiple challenges faced in utilizing conventional FedL in this environment:
    % architecture:
    \begin{enumerate}[leftmargin=4.3mm]
        \item There might exist some devices with high data quality/quantity  that suffer from low computation capability, resulting in their data being neglected during training.
        \item There might exist some devices with high computation capability suffering from low data quantity/quality, resulting in idle processing resources.
        \item There might exist some devices with high computation capability and good data quality/quantity, but poor channel conditions to the BS, making their participation in the model aggregation step challenging.
        \item There might exist some devices with low computation capability and low data quality/quantity that have good channel conditions to the BS, which will be unused.
    \end{enumerate}
    Hence, there are several conditions under which conventional FedL will result in poor performance. Roughly speaking, FedL models are biased towards devices with good channel conditions that have large amounts of data, which might encompass only a small portion of the overall network.
      \subsubsection{Enabling Data Sharing in FedL} Most of the FedL literature to date has assumed that  users have strict data privacy concerns and never share their data. While this is true in some applications, e.g., healthcare systems, there are also applications where the data privacy is not strictly regulated, e.g., model training over sensor networks to detect abrupt environmental changes. Also, economic incentives (e.g., rewards, gas credit, or cash back) can be developed to encourage data sharing in many applications such as autonomous driving. Further, research on privacy preserving representation learning, which aims to obfuscate  sensitive attributes of raw data, can expand the types of data which can be shared~\cite{bertran2019adversarially,azam2020towards}. These factors have motivated recent initial investigations of data offloading in FedL. Specifically, the recent work \cite{9492062} 
proposes data offloading for edge-assisted FedL in vehicular networks, where the data collected at the cars is not  sharing restricted. Also, researchers in~\cite{zhao2018federated} use data offloading to alleviate the impact of non-i.i.d. data in FedL. Finally, works \cite{hosseinalipour2020federated,wang2021device} introduce data offloading
for federated and distributed learning in heterogeneous networks.
    \subsubsection{Parallel successive learning (PSL)}
     Motivated by the above challenges,
    % motivates us to look for possible alternative distributed ML frameworks.  
  we propose a novel methodology for distributing ML over wireless networks that leverages the following properties,   which constitute the pillars of \textit{parallel successive learning} (PSL). PSL is a foremost realization of \textit{fog learning} paradigm introduced by us in~\cite{hosseinalipour2020federated}, which enables both parameter and data offloading among the devices.
    \begin{enumerate}[leftmargin=4.3mm]
        \item Modern wireless devices are capable of device-to-device (D2D) communications~\cite{6815897,hosseinalipour2020federated,9311931}, which  are often considerably low power consuming. D2D  communications can also be carried out in the \textit{out-band} mode that does not occupy the licensed spectrum. There is an existing set of literature on D2D communications for various ad-hoc networks~\cite{abolhasan2004review,zeadally2012vehicular,bekmezci2013flying}.  Motivated by this, we exploit D2D communications among the edge devices in PSL.
        \item Devices with high data quantity/quality and low computation capability can transfer a portion of their data to those with better computation resources via D2D communications. Then, the devices with more computation resources will train on larger datasets, while the rest train on smaller datasets.  To accomplish this, in PSL, we introduce a \textit{data dispersion} mechanism among the devices.
        \item Devices with high computation capabilities and bad channel conditions to the BS can execute model training and offload their trained models/gradients to those with better channels. Thus, we introduce a \textit{model/gradient dispersion} mechanism in PSL to transfer the models/gradients among the devices through D2D communications. Further, given the heterogeneity among the devices and time-varying nature of datasets, in PSL, we consider local model training via \textit{non-uniform SGD with stratified data sampling} with various mini-batch sizes and number of iterations across the devices.
        \item Devices with good channel conditions to the BS and low computation capability can act as \textit{aggregators}, receiving models from neighboring devices and conducting a \textit{local aggregation} followed by uplink transmission to the BS.  Subsequently, we introduce a \textit{local condensation} method in PSL for conducting these device-side model aggregations.
    \end{enumerate}

    \vspace{-5mm}
    \subsection{Summary of Contributions}
    \vspace{-.5mm}
    Our main contributions can be summarized as follows:
    \begin{itemize}[leftmargin=4mm]
    \item We develop PSL, a distributed ML technique that introduces a new degree of freedom into model training, which is \textit{idle times} in between global aggregations. PSL further extends FedL in the following three dimensions: 
    % network, heterogeneity, and proximity: 
    \begin{enumerate}[label=\Roman*.,leftmargin=4mm]
        \item Network: PSL considers the local D2D network among the devices and allows direct device cooperation for data and parameter dispersion. This migrates away from the star topology of FedL and paves the road to more decentralized distributed ML architectures. This approach is complementary to recent works that utilize D2D in distributed ML to conduct model consensus~\cite{8950073,9154332,hosseinalipour2020multi,lin2021two}.
        \item Heterogeneity: PSL considers and addresses three types of heterogeneity in wireless distributed ML:
        \begin{itemize}[leftmargin=0mm]
        \item  Device: PSL assumes different computation and communication capabilities across devices. This is reflected in different CPU cycles ranges, chip-set coefficients, and transmit powers in D2D vs. uplink transmissions.
            \item Learning:  PSL adapts device participation in  model training according to their capabilities. In particular, it considers heterogeneous number of local SGD iterations with various mini-batch sizes at the devices. 
            \item Data: PSL considers a \textit{dynamic environment} with data arrival and departure at the devices, where the distribution of the local datasets are non-i.i.d and evolving over time. We interpret this as the \textit{drift} of the local loss over time. 
        \end{itemize}
        \item Proximity: PSL considers D2D and device-to-server proximities to determine efficient transfers of both ML models and data across the network.
        \end{enumerate}
\item  We introduce \textit{resource pooling}, where device resources are cooperatively orchestrated to facilitate ML model training. This is realized through a sequence of steps we develop and analyze: \textit{data dispersion}, \textit{local computation}, and \textit{model/gradient dispersion with local condensation}.
        \item We analytically characterize the convergence behavior of PSL, through which we quantify the joint effect of (i) SGD with non-uniform local data sampling at each device, (ii) non-uniform numbers of local SGD iterations across the devices, and (iii) dynamic data at the devices captured via \textit{model/concept drift}. We leverage this to formulate the \textit{network-aware PSL} problem, which optimizes over the tradeoffs between energy consumption, delay, and model performance for heterogeneous wireless networks.
        \item We show that network-aware the PSL problem is highly non-convex and NP-hard. We then propose a tractable approach based on posynomial approximation and constraint correction to solve the problem through a sequence of convex problems, which enjoys convergence guarantees. The proposed optimization transformation technique, given the generality of the analysis and the problem formulation, sheds light on the solution of a broader range of problems under the umbrella of \textit{network-aware distributed learning}.
    \end{itemize}
          
\vspace{-3mm}
    \section{System Model}
    \vspace{-.7mm}
    % \nm{I think the system model section needs tyo be organized more clearly: 1) Data model: talk about how data evolves at each node, the heterogeneity of data, and what assumptions you make.
    % 2) Learning model and objective (stratified SGD, aggregation etc.)
    % }
    % \nm{The notation is also quite heavy. Is it possible to simplify the model? For instance, e is the same for all nodes at all times? }
    % \nm{Also, in my opinion, there is no point in using a complicated model if your analysis shows that the simpler model is best, or 
    % if your analysis does not show clear advantages of using the complex model. For instance, where do you show the beneficial effect of using number of SGD steps $\neq 1$? Where do you show the advantage of using more than 1 stratum? This is not clear to me}\\
    
\noindent  \textbf{PSL in a nutshell.} PSL is a general distributed ML paradigm, with its cornerstone built on FedL.  It conducts distributed model training via \textit{resource pooling} and coordinates device resources to operate in a cooperative manner. It conducts each round of model training via five steps: (i) global  model broadcasting, (ii) data dispersion, (iii) local computation, (iv) model/gradient dispersion with local condensation, and (v) global model aggregation, which are illustrated in Fig.~\ref{fig:PSLbig}. 

Henceforth, we present the PSL model tracking in Sec.~\ref{subsec:ML}, describe the  data management of PSL in Sec.~\ref{subsec:data}, and discuss the local model training of PSL in~Sec.~\ref{subsec:data2}. We then introduce model training phases experienced through PSL in Sec.~\ref{subsec:cold}, and model the device orchestration in PSL in Sec.~\ref{subsec:DLM}.

\vspace{-3mm}
 \subsection{Dynamic/Online Model Tracking Problem in PSL}\label{subsec:ML}
 \vspace{-.5mm}
 We consider a network of $N$ devices $\mathcal{N}=\{1,2,\cdots, N\}$ coexisting with a server located at a BS, where ML model training is conducted through a series of global aggregations indexed by $k\in\mathbb{N}$.  In contrast to most  existing works in FedL that assume a stationary data distribution, PSL considers \textit{dynamic ML} characterized by \textit{data variations}. In particular, in PSL, the size and distribution of devices' datasets are assumed to be time-varying, i.e., changing across global aggregations. This is more realistic for real-world applications of distributed ML. For instance, in online product recommendation systems, user preferences may change from day to night and from season to season~\cite{zhang2019deep2}, and in keyword next word prediction, word choices are affected by trending news~\cite{LAtimes}. 
   
   At global iteration $k$, each device $n\in\mathcal{N}$ is associated with a dataset $\mathcal{D}^{(k)}_n$, which has $D^{(k)}_n\triangleq |\mathcal{D}^{(k)}_n|$ data points. Each data point $d \in \mathcal{D}^{(k)}_n$ contains a feature vector, denoted by $\bm{d}$, and a label. For example, in image classification, the feature may be the RGB colors of all pixels in the image, and the label may be the location where the image was taken.

    In PSL, the model training is started with an initial global model broadcast among the nodes, i.e., $\mathbf{w}^{(0)}\in\mathbb{R}^{M}$, where $M$ is the \textit{model dimension}. 
    %  At the beginning of each global aggregation $k$ of PSL, the server broadcasts a global model parameter $\mathbf{w}^{(k)}\in\mathbb{R}^{M}$, which is used to synchronize the local models at the devices.   Therefore,
%  At global iteration $k$ of PSL, each device $n\in\mathcal{N}$ is associated with a dataset $\mathcal{D}^{(k)}_n$, with $D^{(k)}_n\triangleq |\mathcal{D}^{(k)}_n|$ number of data points. Each data point $d \in \mathcal{D}^{(k)}_n$ contains a feature vector and a label for the ML task of interest. For example, in image classification, the feature may be the RGB colors of all pixels in the image, and the label may be the location in which the image was taken.  
 For each data point $d \in \mathcal{D}^{(k)}_n$, the ML model is associated with a \textit{loss function} ${f}(\mathbf{w},d)$
that quantifies the error of parameter $\textbf{{w}}\in\mathbb{R}^{M}$. We refer to Table 1 in~\cite{8664630} for a list of ML loss functions. 
    For an arbitrary $\mathbf w$, during the $k$-th global aggregation, let
 $F_{n}^{(k)}(\mathbf{w})\triangleq F_{n}(\mathbf{w}|\mathcal{D}^{(k)}_n)$ denote the local loss at node $n$,  $F_{n}(\mathbf{w}|\mathcal{D}^{(k)}_n) = \sum_{d \in \mathcal{D}^{(k)}_n} {f}(\mathbf{w},d)/{D}^{(k)}_n$.\footnote{Note that since the same ML model (e.g., neural network) architecture is trained across the node, $f$ is not indexed by $n$.} Then, the global loss of the ML model is given by
 % \vspace{-1mm}
\begin{equation}\label{eq:globlossinit}
  F^{(k)}(\mathbf{w})\triangleq  F(\mathbf{w}| \mathcal{D}^{(k)})= \frac{1}{{D}^{(k)}} \sum_{n \in \mathcal{N}} {D}^{(k)}_n F_{n}^{(k)}(\mathbf{w}),
\end{equation}
 where ${D}^{(k)}\triangleq |\mathcal{D}^{(k)}|$ is the cardinality of $\mathcal{D}^{(k)}$, the collection of data points across devices. We model the evolution of data through a new definition of \textit{model drift}  in Sec.~\ref{sec:conv} (Definition~\ref{def:cons}). 
% , i.e., the instantaneous total number of data points at the system.

% with is a multiset of global data distribution at the devices with cardinality ${D}^{(k)}=\sum_{n \in \mathcal{N}} {D}^{(k)}_n$, which is the instantaneous total number of data points at the system.

% \nm{this feels redundant: I will go straight and define $\mathbf{w}^{{(k)}^\star}$ right after the previous paragraph, and mention in a few words that it varies over k.}
 Due to the temporal variations in the distributions of local datasets, the optimal global model is time varying. In particular, for a training duration of $K$ global iterations, the optimal global models can be represented as a sequence {\small $\big\{\mathbf{w}^{{(k)}^\star}\big\}_{k=1}^{K}$},  where
%  \footnote{In the notations, we just use the index of time for the model parameter to avoid redundancy. For instance, $F(\mathbf{w}^{(k)})$ inherently considers the model loss under $\mathbf{w}^{(k)}$ for the dataset of the devices at time $k$.}
\vspace{-1mm}
\begin{equation}\label{eq:genForm}
\mathbf{w}^{{(k)}^\star}= \underset{\mathbf{w}\in \mathbb{R}^M}{\argmin} \;  F^{(k)}(\mathbf{w}), ~\forall k,
\vspace{-1mm}
\end{equation}
which may not be unique in the case of loss functions which are not strongly convex, e.g., neural networks.

  \vspace{-3mm}
  \subsection{Data Heterogeneity and Management in PSL}\label{subsec:data}
\vspace{-.4mm}
  We propose \textit{partitioning} the dataset of each device into a \textit{disjoint} set of sub-datasets, called stratum (see Fig.~\ref{fig:startData}).\footnote{We use ``stratum" as singular form and ``strata" as plural form.}
  At global iteration $k$, we assume that dataset of device $n$  consists of set {\small $\mathcal{S}^{(k)}_n\triangleq \{\mathcal{S}^{(k)}_{n,1},\mathcal{S}^{(k)}_{n,2},\cdots\}$}  of {\small ${S}^{(k)}_n\triangleq|\mathcal{S}^{(k)}_n|$} strata, each with size {\small ${S}^{(k)}_{n,j}\triangleq|\mathcal{S}^{(k)}_{n,j}|$}.
%   ${S}^{(k)}_n\triangleq|\mathcal{S}^{(k)}_n|$
%   , where $\mathcal{S}^{(k)}_{n,j} \cap S^{(k)}_{n,j'}=\emptyset$, $j\neq j'$, and $\sum_{j=1}^{{S}^{(k)}_{n}} {S}^{(k)}_{n,j} = D^{(k)}_n$, where ${S}^{(k)}_{n,j}\triangleq|\mathcal{S}^{(k)}_{n,j}|$. 
  We also let {\small $\widetilde{\sigma}^{(k)}_{n,j}=\sqrt{\frac{1}{{S}^{(k)}_{n,j}-1}\sum_{{d}\in \mathcal{S}^{(k)}_{n,j}} \Vert \bm{d}- \widetilde{\mu}^{(k)}_{n,j}\Vert^2_{_2}}$} and $\widetilde{\mu}^{(k)}_{n,j}=\sum_{{d}\in \mathcal{S}^{(k)}_{n,j}} \bm{d}/{S}^{(k)}_{n,j}$ denote the sample (total) standard deviation and the mean of data inside stratum $\mathcal{S}^{(k)}_{n,j}$. We will exploit the means of the strata for data management as explained in the following, and  the variance for optimal local data sampling for SGD in Sec.~\ref{sec:conv} (Proposition~\ref{prop:neyman}).
    We assume that data points in each stratum possess the same label. 
    
    This data  partitioning has three advantages: (i) it provides effective data management upon data arrival/departure, (ii) it leads to tractable techniques to track the local dataset evolution, (iii) it opens the door to effective non-uniform data sampling to reduce the noise of SGD. We describe the first two advantages below and defer the explanation of the third one to Sec.~\ref{subsec:data2}.

%  As compared to SGD with uniform sampling, commonly used in FedL literature, we exploit SGD with non-uniform data sampling. In particular, we propose partitioning the dataset of each node into a disjoint set of sub-datasets, called stratum\footnote{We use ``stratum" as singular form and ``strata" as plural form.}. Our technique, which is inspired by \textit{stratified sampling} in statistics literature~\cite{lohr2019sampling}, is in particular superior to common uniform sampling techniques when the distribution of data in each stratum is similar, while it is different from one stratum to another. The data partitioning and allocation across the strata at each device at the beginning of training  can be conducted as in centralized SGD literature~\cite{zhao2014accelerating}; in this work we are mostly focused on the benefits of this technique to distributed ML.

       \begin{figure}[t]
\vspace{-.5mm}
\includegraphics[width=.42\textwidth]{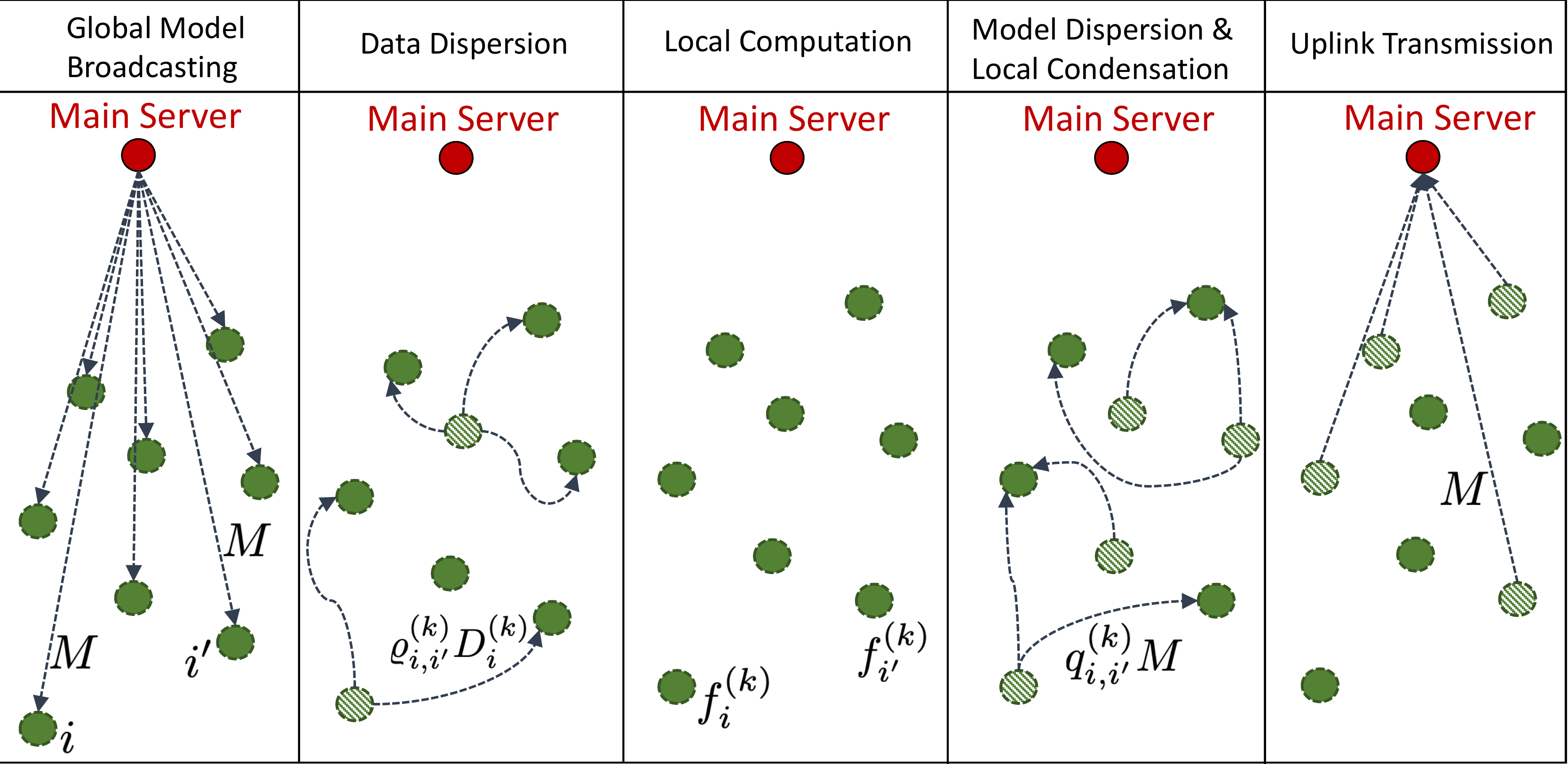}
\centering
\vspace{-1mm}
\caption{A schematic of the learning architecture of PSL with five steps: (i) global model broadcasting, where the devices receive the global parameter from the server, (ii) data dispersion, where the devices conduct partial dataset offloading, (iii) local computation, where the devices compute their local models, (iv) model/gradient dispersion with local condensation, where devices conduct partial model/gradient transfer and perform local aggregations, (v) uplink transmission, where some devices transmit their models to the BS.}
\label{fig:PSLbig}
\vspace{-.1mm}
\end{figure}
 PSL considers two types of data movement: (i) inter-node data arrival/departure and (ii) intra-node data arrival/departure. In case (i), the arrival/departure of data at the devices is due to local data offloading, facilitated by device-to-device (D2D) communications. In case (ii),  arrival/departure is caused by devices collecting and abandoning data. When new data arrives, the device has control on how to assign it to the existing strata. Also, in case (i), the devices have control on which data points from which strata to transfer (see Fig.~\ref{fig:startData}).
      
      \textbf{Data arrival.} If data points with new labels arrive, the device
       will form new strata and assign the respective data points to them. Otherwise, the device assigns each arriving data point to the stratum with the closest mean  among those with the same label. In particular, given the current means of the strata at device $n$, {\small $\widetilde{\mu}^{(k)}_{n,j}$}, $\forall j$, 
      the arriving data point $d$, with feature vector $\bm{d}$, gets assigned to stratum {\small $\mathcal{S}^{(k)}_{n,j^\star}$}, where {\small $\mathcal{S}^{(k)}_{n,j^\star}=\argmin_{\mathcal{S}^{(k)}_{n,j}\in \mathcal{S}^{(k)}_{n}}\Vert\widetilde{\mu}^{(k)}_{n,j}- \bm{d}\Vert$}. This promotes homogeneity among the data points within each stratum. After the size of a stratum reaches a predefined maximum size $s_{\textrm{max}}$, the stratum is further partitioned into two strata with equal sizes.\footnote{$s_{\textrm{max}}$ is assumed to be an even number without loss of generality.} In addition, if the size of the stratum falls bellow a threshold $s_{\textrm{min}}$, the stratum is assumed to merged with the strata containing data with the data points with the same label if such strata exists.

      \textbf{Data offloading.} 
    %   Devices face two strategies: (i) reducing local dataset variance via data offloading from strata with higher data variance: this results in a smaller sampling error for SGD, and (ii) keeping the local data diversity high via offloading the data points from crowded strata: this promotes local data diversity and reduces the local model bias. 
    When a device offloads data, we can imagine two potential strategies for data point selection: (i) choosing from strata with higher variances, which results in a smaller sampling error for SGD, and (ii) choosing from strata with more data points, which reduces local model bias. Since data across the devices is non-iid in PSL, strategy (i) can increase the divergence across local datasets, which can significantly reduce the performance of the global model. We thus rely on strategy (ii), where for offloading a data point,  device $n$ chooses the stratum $\mathcal{S}^{(k)}_{n,j\star}$ with largest size and offloads the data point $d$ that has the closest distance to the mean of strata, i.e., $d=\argmin_{d\in \mathcal{S}^{(k)}_{n,j\star}}\Vert \widetilde{\mu}^{(k)}_{n,j}- \bm{d}\Vert$. In this way, we minimize the impact an offloading device experiences on its local data distribution.
       \vspace{-1mm}
      \begin{remark}
      In general, the optimal data offloading strategy may be a hybrid of the two mentioned strategies, since the impact of data offloading on the divergence of the local models from the global model is difficult to quantify in environments with unknown local data distributions. In this work, we propose the first steps towards smart data management and leave further investigations to future work.
      \end{remark}
     
       \vspace{-2mm} 
     \begin{remark}
     The introduction of device dataset statistics and D2D data offloading in PSL opens the possibility of incorporating techniques such as generative adversarial networks (GANs) \cite{creswell2018generative} for dataset generation/augmentation. If a device shares a sufficiently detailed set of distribution statistics (i.e., more than mean and variance) with the server summarizing its local dataset, the server could potentially reproduce a portion of the local datasets to further train/refine the global model.
       Also, given the received data points from the neighboring devices, each device could potentially employ a GAN model to augment its local dataset with synthetically-generated data points that are similar to those in the neighboring devices, resulting in local model debiasing from reductions in data heterogeneity. We leave further investigations on this subject as future work.
      \end{remark}
      
          \vspace{-1.5mm}
       \textbf{Tracking of local data statistics.} To cope with the dynamics of local datasets, we exploit an \textit{online data statistics tracking} technique.  
        %  Our technique relies on tracking the mean and variance of the data in strata over time. 
        %  Note that these parameters are not affected by the change in the local ML model parameter and they are solely function of the initial dataset and data arrival and departure at the node.
  We assume that each device has computed or otherwise gained knowledge of the initial mean and variance of its local data strata. Upon arrival or departure of each data point (or groups of data points), the device updates the mean and variance of its strata in an online manner through the following lemma, which eliminates the need for recomputing these quantities from scratch over the entire set of local data points.
    %  upon data arrival/departure.
     
     \vspace{-1.5mm}
           \begin{lemma} [Online Tracking of Strata Mean and Variance]\label{lemma:trackMeanvar}
      Let $\mathcal{S}$ denote a set of $|\mathcal{S}|$ (vector) data points  with mean $\bm{\mu}_{\mathsf{old}}$ and sample variance $\sigma_{\mathsf{old}}^2$. Let $\mathcal{A}$ denote a set of new data points that are added to $\mathcal{S}$ with
       mean and variance of $\bm{\mu}_{\mathcal A}$ and $\sigma^2_{\mathcal A}$, respectively, and  $\mathcal{D}$ denote a set of data points departing from $\mathcal{S}$ with mean and variance of $\bm{\mu}_{\mathcal D}$ and $\sigma^2_{\mathcal D}$, respectively. Then, the new mean and variance of 
      $\mathcal{S}$ are given by
      \vspace{-1mm}
          \begin{align}
      &\hspace{-3mm}\bm{\mu}_{\mathsf{new}} =  \frac{|\mathcal{S}| \bm{\mu}_{\mathsf{old}} + |\mathcal{A}|\bm{\mu}_{_\mathcal{A}}-|\mathcal{D}|\bm{\mu}_{_\mathcal{D}}}{|\mathcal{S}|+|\mathcal{A}|-|\mathcal{D}|},
      \end{align}
      \vspace{-2mm}
   \begin{align}
      &\hspace{-20mm}\sigma^2_{\mathsf{new}}= \frac{(|\mathcal{S}|-1)\sigma^2_{old} + (|\mathcal{A}|-1)\sigma^2_{_\mathcal{A}}-(|\mathcal{D}|-1)\sigma^2_{_\mathcal{D}}
            }{|\mathcal{A}|+|\mathcal{S}|-|\mathcal{D}|-1}\hspace{-15mm}\nonumber\hspace{-5mm} \\&\hspace{-20mm}
             +\hspace{-.5mm}\frac{\hspace{-.5mm}\left(\hspace{-.5mm}\frac{|\mathcal{A}||\mathcal{S}|}{|\mathcal{A}|+|\mathcal{S}|-|\mathcal{D}|}\hspace{-.5mm}\right)\hspace{-.8mm}\Vert \bm{\mu}_{\mathsf{old}}-\bm{\mu}_{_\mathcal{A}}\Vert^2\hspace{-.5mm}-\hspace{-.5mm}
        \left(\hspace{-.5mm}\frac{|\mathcal{S}||\mathcal{D}|}{|\mathcal{A}|+|\mathcal{S}|-|\mathcal{D}|}\hspace{-.5mm}\right)\hspace{-.8mm}\Vert \bm{\mu}_{\mathsf{old}}-\bm{\mu}_{_\mathcal{D}}\Vert^2}{|\mathcal{A}|+|\mathcal{S}|-|\mathcal{D}|-1}\hspace{-15mm} \nonumber \hspace{-5mm}\\&\hspace{-20mm}-\hspace{-.5mm}\frac{         \left(\hspace{-.5mm}\frac{|\mathcal{A}||\mathcal{D}|}{|\mathcal{A}|+|\mathcal{S}|-|\mathcal{D}|}\hspace{-.5mm}\right)\Vert \bm{\mu}_{A}-\bm{\mu}_{_\mathcal{D}}\Vert^2}{|\mathcal{A}|+|\mathcal{S}|-|\mathcal{D}|-1}.\hspace{-5mm}
          \end{align}
      \end{lemma}
      \vspace{-2.3mm}
      \begin{proof}
     The proof is provided in Appendix~\ref{app:meanVar}.
      \end{proof}
      \vspace{-2mm}
 {The method described in Lemma \ref{lemma:trackMeanvar} requires the computation of mean and variance over the entire local dataset of each device only once at the beginning of ML model training. This is particularly useful for settings of large local datasets at heterogeneous devices, as the computational complexity of the instantaneous mean and variance of each strata is  dependent on the number of arriving and departing data points, as opposed to the full dataset.}

\vspace{-1mm}

% Furthermore,  due to data variations, given a fixed model parameter $\mathbf w$, the local loss at the nodes varies over time, i.e., in general $F_n(\mathbf{w}| \mathcal{D}_n^{(k)})\neq F_n(\mathbf{w}| \mathcal{D}_n^{(k')})$, $ k\neq k'$, $\forall n$, which also implies the time varying nature of the global loss.
% , i.e., $F(\mathbf{w}| \mathcal{D}^{(k)})\neq F(\mathbf{w}| \mathcal{D}^{(k')})$, $ k\neq k'$.

    %  The model is associated with a loss $\widetilde{f}(\mathbf{w},\textbf{x}_i,y_i)$, referred to as $\widetilde{f}_i(\mathbf{w})$
% \nm{sholdnt it be $\widetilde{f}_{n,i}$ for user n?} \ali{The model structures are the same at different nodes (the same NN or SVM, etc.) If you mean we add the index n to the data point as well, like d_{i,n}, that can be done but may make the notations more complicated.} 

      \vspace{-3mm}
      \subsection{Local Data Sampling and Model Training Iterations} \label{subsec:data2}
      \vspace{-.5mm}
      
      \begin{figure}[t]
\vspace{-1mm}
\includegraphics[width=.43\textwidth]{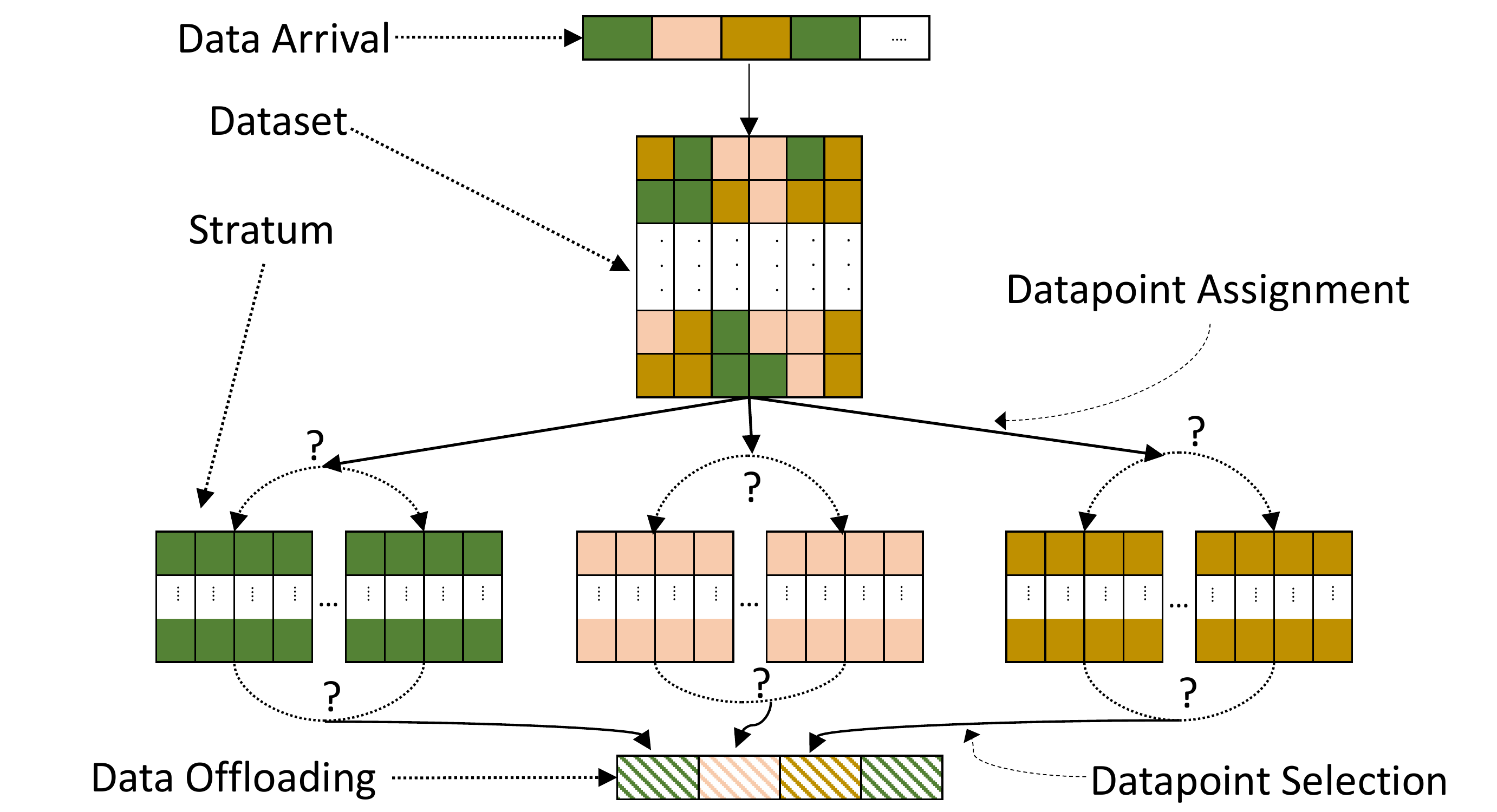}
\centering
\vspace{-1.mm}
\caption{A schematic of the data stratification for PSL. The figure represents the dataset at a device. The device faces two  dilemmas: (i) how to assign the arriving data to the strata, and (ii) which data points to choose to offload.}
\label{fig:startData}
% \vspace{-1.5mm}
\end{figure}

 As compared to SGD with uniform sampling, commonly used in FedL literature, we exploit SGD with non-uniform sampling. Our technique, inspired by \textit{stratified sampling} in statistics~\cite{lohr2019sampling}, is advantageous to uniform sampling techniques when the distribution of data in each stratum is homogeneous  while between strata is heterogeneous. The initial allocation of data points across the strata in each device at the beginning of training can be conducted as in centralized SGD~\cite{zhao2014accelerating}; we are focused on the benefits of this technique to distributed ML. 
 
% \subsubsection{Model Evolution}\label{subsec:ModelEvo} At global iteration $k$, we assume that dataset of device $n$  consists of set $\mathcal{S}^{(k)}_n\triangleq \{\mathcal{S}^{(k)}_{n,1},\mathcal{S}^{(k)}_{n,2},\cdots\}$ of strata, where $\mathcal{S}^{(k)}_{n,j} \cap S^{(k)}_{n,j'}=\emptyset$, $j\neq j'$, and $\cup_{\mathcal{S}^{(k)}_{n,j}\in \mathcal{S}^{(k)}_n} \mathcal{S}^{(k)}_{n,j} =\mathcal D^{(k)}_n$. We also let $\widetilde{\sigma}^{(k)}_{n,j}=\sqrt{\frac{1}{{S}^{(k)}_{n,j}-1}\sum_{\bm{d}\in \mathcal{S}^{(k)}_{n,j}} \Vert \bm{d}- \widetilde{\mu}^{(k)}_{n,j}\Vert^2}$ and $\widetilde{\mu}^{(k)}_{n,j}=\sum_{\bm{d}\in \mathcal{S}^{(k)}_{n,j}} \bm{d}/{S}^{(k)}_{n,j}$ denote the sample (total) standard deviation and the mean of data inside stratum $\mathcal{S}^{(k)}_{n,j}$
% (see Fig.~\ref{fig:startData}).

      To solve~\eqref{eq:genForm} in a distributed manner, within each global training round, devices conduct local model training through successive mini-batch SGD updates. However, device heterogeneity leads to varying contributions to model training. PSL assumes that devices utilize SGD with (i) different numbers of local iterations, (ii) different mini-batch sizes, and (iii) non-uniform data sampling. Formally, at global iteration $k$,
 device $n$ performs $e^{(k)}_n$ iterations of SGD over its local dataset. The evolution of local model parameters is then given by\footnote{The \textit{complexity} of ML model training using different choices of neural networks has been discussed in~\cite{bienstock2018principled}.}
 \vspace{-1.5mm}
%  \nm{this equation seems overly complicated. isnt it the standard GD over the batch selected? The only difference is that the batch is created by taking a number of samples from each stratum.is the first sum over $j$? I dont understand... notation is also quite heavy and difficult to follow what is going on.}
\begin{equation}\label{eq:WeightupdateStrat}
\hspace{-12mm}
% \begin{aligned}
   \mathbf{w}_n^{(k),e}\hspace{-0.5mm}=\mathbf{w}^{(k),e-1}_{n}\hspace{-0.5mm} - \hspace{-0.5mm}{\frac{\eta_{_k}}{{D}^{(k)}_n}} \hspace{-1mm}\sum_{j=1}^{{S}^{(k)}_{n}}\hspace{-0.6mm}\sum_{d\in \mathcal{B}^{(k),e}_{n,j}} \hspace{-3mm} {\frac{{S}^{(k)}_{n,j}\nabla  f(\mathbf{w}^{(k),e-1}_{n},d)}{{B}^{(k)}_{n,j}}}\hspace{-0.1mm},\
    % \end{aligned}
    \hspace{-9mm}
    \vspace{-2.5mm}
\end{equation}
where $e\in\{1,\cdots,e^{(k)}_n\}$
denotes the local iteration index, $\eta_{_k}$ denotes the step-size, and $\mathbf{w}^{(k),0}_{i}=\mathbf{w}^{(k)}$ is the previously received global parameter from the BS. The nested sum in~\eqref{eq:WeightupdateStrat} indicates the overall gradient is computed by calculating the gradient over samples from strata.
In particular, $\mathcal{B}^{(k),e}_{n,j}$ denotes the data batch, sampled at the $e$-th iteration from stratum $\mathcal{S}^{(k)}_{n,j}$. 

We conduct data sampling uniformly at random within each stratum, with the number of samples collected \textit{varying} by strata, leading to an overall \textit{non-uniform} sampling procedure (see Proposition~\ref{prop:neyman}).
We assume that the number of sampled data points from each stratum does not change during each training interval,
% \footnote{Due to the fixed statistics of local dataset during each global aggregation.} 
i.e., ${B}^{(k)}_{n,j}=|\mathcal{B}^{(k),e}_{n,j}|$, $\forall e$.
We also define $B^{(k)}_n=\sum_{j=1}^{S^{(k)}_n} {B}^{(k)}_{n,j}$ as the SGD mini-batch size. 
 Both $B^{(k)}_n$ and ${B}^{(k)}_{n,j}$ are design variables that should be tuned according to device capability and dataset heterogeneity, discussed in Sec.~\ref{sec:conv} and \ref{sec:PF}. Furthermore, the bounds in Sec.~\ref{sec:conv} and the subsequent optimization problem in Sec.~\ref{sec:PF} will differentiate between  conducting several SGD iterations with small mini-batches vs. a single SGD with a large mini-batch, and further reveal the advantage of conducting multiple local SGD iterations.

% In particular, $B^{(k)}_n$ denotes the total number of data points sampled at each instance of local SGD at node $k$.

After each device $n$ performs its last iteration of local model training, i.e., $e^{(k)}_n$, it computes the accumulated gradient {\small $\overline{\nabla {F}}_n^{(k)} = \big(\mathbf{w}^{(k)}-\mathbf{w}_n^{(k),e^{(k)}_n}\big)\big/\eta_{_k}$}, offloaded either to its neighbors or to the BS. The global update is carried out at the BS as
\vspace{-1.5mm}
\begin{equation}\label{eq:mainupdateWeight}
    \mathbf{w}^{(k+1)}=\mathbf{w}^{(k)}-\eta_{_k}\overline{\nabla {F}}^{(k)},
    \vspace{-2.5mm}
\end{equation}
where {\small $\overline{\nabla {F}}^{(k)}$} is the normalized accumulated gradient of the devices,  factoring in the heterogeneous number of local SGD
 iterations~\cite{rizk2020dynamic,wang2020tackling}, given by
 \vspace{-1mm}
\begin{equation}\label{eq:globAgg}
    \overline{\nabla {F}}^{(k)}=\sum_{n'\in \mathcal{N}}\frac{{D}^{(k)}_{n'}e^{(k)}_{n'}}{D^{(k)} }
    \sum_{n\in \mathcal{N}}\frac{{D}^{(k)}_n}{D^{(k)} e^{(k)}_n}\overline{\nabla {F}}_n^{(k)}.
     \vspace{-1.2mm}
\end{equation}
% with 
% \begin{equation}\label{eq:globAgg}
%   \overline{\nabla {F}}_n^{(k)} = \mathbf{w}_n^{(k),e^{(k)}_n}-\mathbf{w}^{(k)}.
% \end{equation}
% \begin{equation}\label{eq:globAgg}
%     \mathbf{w}^{(k+1)}=\sum_{n\in \mathcal{N}} \frac{|\mathcal{D}_n|E_n}{\sum_{n'\in\mathcal{N}} |\mathcal{D}_{n'}|} \mathbf{w}^{(k)}_{n,e^{(k)}_n},
% \end{equation}
Global model {\small $\mathbf{w}^{(k+1)}$} will be then used to synchronize the devices for the next round of local updates. The process of recovering $\overline{\nabla {F}}^{(k)}$ at the server through gradient dispersion and local condensation among the devices is discussed in Sec.~\ref{sec:GradCalc}.

      \vspace{-3mm}
\subsection{Dynamic Model Training with Idle Times: ``\textit{Cold}" vs. ``\textit{Warmed up}" Model, and Model ``Inertia"}  \label{subsec:cold}

Many distributed ML applications (e.g., keyboard next word prediction or online cloth recommendation system) call for model training to be executed across long time periods (e.g., multiple weeks or seasons). In such settings, it is unrealistic to assume that the global aggregations are conducted continuously back to back, where the devices are always engaged in local model training, since this will require prohibitively high resource consumption. Thus, as compared to current art, PSL further introduces a new design parameter $\Omega^{(k)}\hspace{-1mm}\in \hspace{-1mm} \mathbb{Z}^+ \hspace{-1mm}\cup\hspace{-0.7mm}\{0\}$ that captures the \textit{idle time} between the end of global aggregation $k-1$ and the beginning of $k$, during which the devices are not engaged in model training ($\Omega^{(0)}\triangleq 0$). This parameter captures the frequency of engagement of the devices in model training. 

Initially, when the global model is launched from a ``cold start", i.e., it is not well-trained, conducting model training results in significant improvements in model performance, calling for rapid global training rounds, i.e., small $\Omega^{(k)}$-s. After several global training rounds, the model training at the devices starts with a ``warmed up" model, which marginalizes the reward in terms of model performance gains. In this regime, if the data at the devices changes rapidly, to track the changes, fast global rounds are required; otherwise, model training can be delayed, i.e., large $\Omega^{(k)}$-s, to save energy and network resources. In particular, given a warmed up global model, the model training should be \textit{triggered} when sufficient changes in the local data distributions is occurred. We call this  phenomenon the \textit{inertia} of the global model, since it resembles the same notion in physics. 
Initially, the model has a lower inertia. During model training, the inertia of the global model increases (i.e., it becomes reluctant to changes) and it takes large shifts in the data distribution to trigger a new model training round.  A key contribution of our work is in characterizing this notion of inertia in distributed ML.
\vspace{-4mm}
      \subsection{Cooperative Resource Pooling in PSL}\label{subsec:DLM}
In PSL, the resource of the devices are utilized cooperatively, where both data and model parameters can be transferred in D2D mode,
 to improve model training.  Optimizing these transfers requires consideration of their resource requirements.
% In particular, the devices cooperatively operate in D2D mode, .

     At global aggregation $k$, for each device $n\in\mathcal{N}$, we let $h^{(k)}_n$
      denote the channel gain of the device to the BS. 
      Consequently, the data rate of device $n$ to the BS is given by\footnote{$\log(.)$ denotes logarithm with base $2$ unless otherwise stated and the data rate expressions are measured at the instance of data/gradient transmissions.}
      \vspace{-1mm}
      \begin{equation}\label{rateU}
          r^{(k)}_n=B^{\mathsf{U}} \log\left(1+{|h^{(k)}_n|^2 p^{\mathsf{U}}_n}\big /{N^{\mathsf{U}}}\right),
           \vspace{-1mm}
      \end{equation}
      where {\small $B^{\mathsf{U}}$} denotes the \underline{u}plink bandwidth given to each device, {\small $p^{\mathsf{U}}_n$} is the uplink transmit power of the device, and {\small $N^{\mathsf{U}}=N_0 B^{\mathsf{U}}$} is the noise power with {\small $N_0$} denoting the noise spectral density.\footnote{An alternative method is to use the average/expected data rate expressions obtained under average fading power. This would lead to closed form expressions for data rates (e.g., see~\cite{8424236})  that can be directly used in our subsequent optimization formulation. Note that upon using the instantaneous data rates as in \eqref{rateU}, if the data transmission time exceeds the channel coherence time, the instantaneous data rate may change during the data transmission. Using the average/expected data rate will also resolve this issue.} 
      
      Similarly, with D2D communications, for two devices $n,m\in\mathcal{N}$, we define the data rate at device $m$ achieved via transmission from device $n$ in D2D mode as follows:
       \vspace{-.7mm}
       \begin{equation}\label{rateD}
          r^{(k)}_{n,m}=B^{\mathsf{D}} \log\left(1+{|h^{(k)}_{n,m}|^2 p^{\mathsf{D}}_{n}}\big/{N^{\mathsf{D}}}\right),
           \vspace{-.7mm}
      \end{equation}
      where {\small $B^{\mathsf{D}}$} denotes the D2D bandwidth given to each user pair,
     {\small  $h^{(k)}_{n,m}$} denotes the channel gain among the respective nodes, {\small $p^{\mathsf{D}}_{n}$} denotes the D2D transmit power of device $n$, and {\small $N^{\mathsf{D}}=N_0 B^{\mathsf{D}}$}. 
      % \vspace{-5.5mm}
      \begin{remark}\label{rem:ofdm} Since the physical layer details are abstracted in this paper, for simplicity, it is assumed that the devices utilize a multi-access protocol such as {\small FDMA}~\cite{7110419} that avoids interference of the devices in uplink transmissions. The same holds for D2D communications. Investigating the effect of interference is left as a topic for future works.
      \end{remark}
      \vspace{-.5mm}
    %   \begin{remark} Note that a more realistic modeling is defining the rate between the devices $i$ and $j$ as
    %   \begin{equation}
    %       r_{n,j}= B^{\mathsf{D}} \log(1+SINR_{n,j}),
    %   \end{equation}
    %   where $B^{\mathsf{D}}$ denotes the BW available for D2D communications.  This is left out for now, we can later come back and see whether the problem can be solved with this realistic assumption or not. 
    %   Note that $SINR_{n,j}$ is not only a function of device $n,j$ decision, rather the decision of all the nodes in data transfers affects this term. In particular,
    %   \begin{equation}
    %       SINR_{n,j}= \frac{p_n |h_{n,j}|^2}{\sum_{k\neq n}p_k |h_{k,j}|^2+N_0},
    %   \end{equation}
    %   where $h_{n,j}$ denotes the channel gain between the users and $p_{n,j}$ denotes the transmit power of device $i$ to $j$.
    %   \begin{equation}
    %       h_{n,j}= \frac{g_{n,j}}{\sqrt{PL^{G2G}(d_{n,j})}},
    %   \end{equation}
    %   where $g_{n,j}$ denote the small scale fading gain (zero mean complex Gaussian), $PL^{G2G}$ denotes the ground to ground path loss model and $d_{n,j}$ denotes the Euclidean distance between the nodes. Note that here I can readily add the existence of UAVs, which enjoy the ground-to-air channel (G2A) that offers lower path loss (closer to air-to-air that is basically free space path loss) and solve the problem with the existence of UAVs as well, which I leave it here for now.
    %   \end{remark}
       \vspace{-1mm}
      We next formalize the data dispersion, local computations, and local parameter aggregations with local model condensation processes in PSL visualized in Fig.~\ref{fig:PSLbig}:
      
      \subsubsection{Data Dispersion} In PSL, each device can offload a portion of its data  in D2D mode. 
       We define a continuous variable {\small $\varrho^{(k)}_{n,m}\in [0,1]$}, where {\small $\sum_{m\in\mathcal{N}}\varrho^{(k)}_{n,m}=1$}, $\forall n,k$, as the fraction of data of node $n$ offloaded to node $m$ at global iteration $k$, with $\rho_{n,n}^{(k)}$ being the amount of data kept locally.\footnote{We assume that the offloaded data points of each device are no longer used by that device in its local model training. This assumption is made to alleviate the global model becoming biased towards the data distributions of devices offloading samples, i.e., to avoid double-counting them in the SGD process of the transmitting and the receiving devices.} We refer to {\small $\bm{\varrho}^{(k)}=\big[\varrho^{(k)}_{n,m}\big]_{1\leq n,m \leq N}$} as the \textit{data dispersion matrix}. 
       The \textit{\underline{d}ata \underline{r}eception time} at device $m$ from $n$ is thus given by
       \vspace{-1mm}
      \begin{equation}
          T^{\mathsf{DR},(k)}_{n,m}= \varrho^{(k)}_{n,m} D^{(k)}_n b^{\mathsf{D}} \big/r^{(k)}_{n,m},
              \vspace{-1mm}
      \end{equation}
      where $b^{\mathsf{D}}$ denotes the number of bits representing one \underline{d}ata point.
    %   Assuming sequential transfer of data in D2D mode, the \textit{time of data reception} at node $i$ is given by
    %  \begin{equation}
    %      T^{\mathsf{DR},{(k)}}_{m}=\sum_{n\in\mathcal{N}\setminus\{m\}}T^{\mathsf{DT},(k)}_{n,m}.
    %  \end{equation}
    %  Alternatively 
    Assuming data dispersion occurs in parallel (see Remark~\ref{rem:ofdm}), the total {d}ata {r}eception time at node $m$ is\footnote{Interference-avoidance multiple access techniques other than FDMA would also be compatible with our methodology, so long as the expression in~\eqref{eq:TDRDef} is changed accordingly. For example, with TDMA, which is used in \cite{9261995}, {\scriptsize $ \max_{n\in\mathcal{N}\setminus\{m\}} \Big\{T^{\mathsf{DR},(k)}_{n,m}\Big\}$} in~\eqref{eq:TDRDef} should be replaced with {\scriptsize $\sum_{n\in\mathcal{N}\setminus\{m\}} \Big\{T^{\mathsf{DR},(k)}_{n,m}\Big\}$}, which our optimization methodology in Sec.~\ref{sec:PF} can easily incorporate. The same holds for other parts of the paper, where simultaneous information exchange under FDMA is presumed.}
    \vspace{-1mm}
     \begin{equation}\label{eq:TDRDef}
         T^{\mathsf{DR},{(k)}}_{m}=\max_{n\in\mathcal{N}\setminus\{m\}} \Big\{T^{\mathsf{DR},(k)}_{n,m}\Big\}.
         \vspace{-1mm}
     \end{equation}
      \subsubsection{Local Computation} We assume that the devices are equipped with different processing units. Let us define $a_n$ as the number of CPU cycles that are used to process one data sample at device $n$.
    %   Also, let $b^{(k)}_n\in [0,1]$ denote the fraction of data points the nodes processes during each SGD iteration, where $B^{(k)}_n=b^{(k)}_n D^{\mathsf{F},(k)}_n$ with $D^{\mathsf{F},(k)}_n=\sum_{m\in \mathcal{N}} \varrho^{(k)}_{m,n}D^{(k)}_m $ denoting the size of the \underline{f}inal dataset at the node after data receptions.
     Subsequently, the \textit{\underline{c}omputation time} at device $n$ to execute $e^{(k)}_n$ local  SGD iterations based on~\eqref{eq:WeightupdateStrat} at global iteration $k$ is given by
     \vspace{-1mm}
     \begin{equation}\label{eq:TCdef}
     \begin{aligned}
     T^{\mathsf{C},{(k)}}_n&= e^{(k)}_n{a_n B^{(k)}_n }\big/{f^{(k)}_n},
        %  \hspace{-1mm}T^{\mathsf{C},{(k)}}_n&= e^{(k)}_n\frac{a_n b^{(k)}_n \left(\varrho^{(k)}_{n,n} D^{(k)}_n + 
        %  \sum_{m\in \mathcal{N}\setminus\{n\} }\varrho^{(k)}_{m,n}D^{(k)}_{m}\right)}{f^{(k)}_n} \\
        %  &= e^{(k)}_n\frac{a_n b^{(k)}_n\sum_{m\in \mathcal{N}}\varrho^{(k)}_{m,n}D^{(k)}_{m}}{f^{(k)}_n},
         \end{aligned}
        %  \hspace{-5mm}
           \vspace{-1mm}
          \end{equation}
     where $f^{(k)}_n$ denotes the CPU frequency of the device, and the mini-batch size satisfies $B_n^{(k)} \leq \sum_{m\in \mathcal{N}}\varrho^{(k)}_{m,n}D^{(k)}_{m}$.

      \subsubsection{Model Parameter/Gradient Dispersion and Local Condensation}\label{sec:GradCalc}
    %  It is assumed that during the resource pooling process, each device $i$ also transmits its instantaneous number of data points, i.e., $D^{(k)}_i$, to those devices that are the recipient of its data. At the end of resource pooling and local computation, two stages of model condensations are performed: (i) local model condensation, where each device aggregates the local updates that it calculated for the neighbors together via weighting each local update by the number of data points of the respective device:
    %  \begin{equation}
    %      w^{k}_{i}= D^{(k)}_i w^{(k)}_{i,E_i} + \sum_{i'\in\mathcal{N}} D^{(k)}_{i'} w^{(k)}_{i',E_{i'}}
    %  \end{equation}
     
    PSL introduces a scenario in which after performing local computations, each device can partition its vector of local model/accumulated gradient into different \textit{chunks} and disperse the chunks among the neighboring devices in D2D mode, which we call \textit{model/gradient dispersion} (see Fig.~\ref{fig:modelConds}). Considering the update rule in~\eqref{eq:mainupdateWeight}\&\eqref{eq:globAgg}, at each device $n$, this can be carried out via either dispersing the (normalized) latest local model {\small $\frac{{D}^{(k)}_n}{D^{(k)} e_n^{(k)}}\mathbf{w}^{(k),e_n^{(k)}}$}, 
    or via dispersing the (normalized accumulated) gradient {\small $\frac{{D}^{(k)}_n}{D^{(k)} e_n^{(k)}}\overline{\nabla {F}}_n^{(k)}$},
    % , where {\small $\overline{\nabla {F}}_n^{(k)} = \big(\mathbf{w}^{(k)}-\mathbf{w}_n^{(k),e^{(k)}_n}\big)\big/\eta_{_k}$}, 
    both of which are vectors with size $M$. In the following, we focus on the latter approach without loss of generality.
 % To simplify the modeling, it is assumed that the devices always transmit a full gradient vector (with size $M$) in uplink transmission and fractional gradient sharing/dispersion only occurs in D2D mode. 
    % From the energy and delay optimization perspective, this implies that the devices
    In this work, we consider that in each aggregation period, a device
    either (i) completely disperses its local gradient across its neighbors in D2D mode or (ii) keeps its local gradient, receives gradient chucks from its neighbors, and engages in uplink transmission.
    \vspace{-2mm}
    \begin{remark}
    One can imagine a hybrid framework where each device disperses a fraction of its model/gradient parameters in D2D mode and uploads the remaining portion to the BS. In this regime, to reduce the communication burden, the dispersion of chunks to different neighbors should be conducted to maximize the overlap between the chunks that a device receives and those it retains locally, studying of which is left to future work.
    \end{remark}
    \vspace{-2mm}
    To alleviate uplink transmission of multiple gradients, when a device receives gradient chunks, it conducts a local aggregation, i.e., summing the received chunks with its local gradient, and only sends the resulting vector to the BS. We call this \textit{local model/gradient condensation} since it results in uplink transmission of a vector with size $M$ at each node regardless of the number of received chunks from its neighbors (see Fig.~\ref{fig:modelConds}).
\begin{figure}[t]
\vspace{-5mm}
\centering
\includegraphics[width=.49\textwidth]{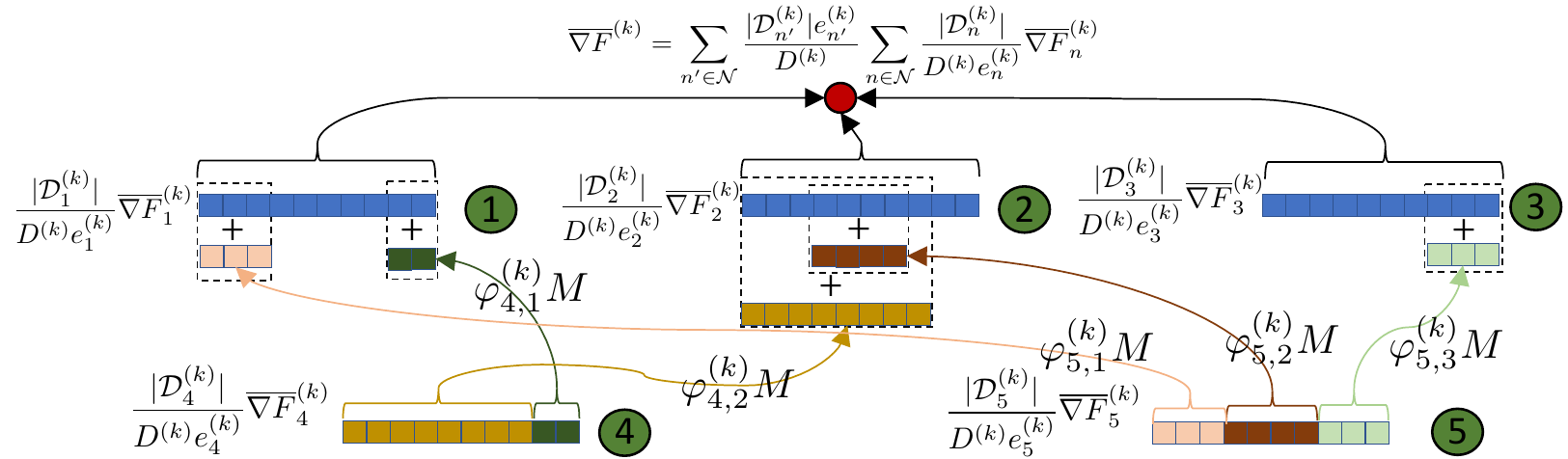}
\caption{A schematic of gradient dispersion with local condensation for a network of $5$ nodes. Each node $n$ computes {\scriptsize $\overline{\nabla {F}}_n^{(k)}$}
% = \big(\mathbf{w}^{(k)}-\mathbf{w}_n^{(k),e^{(k)}_n}\big)\big/\eta_{_k}$
(with size $M$), and normalizes it with respect to its number of data points and SGD iterations. Nodes $4$ and $5$ partition the resulting vector to multiple chunks and disperse the chunks across their neighbors. The recipients sum the received vectors with their own vectors and transmit the results (vector of size $M$) to the main server, which computes {\scriptsize $\overline{\nabla {F}}^{(k)}$} (see~\eqref{eq:globAgg}) and conducts global aggregation (see~\eqref{eq:mainupdateWeight}).}
\label{fig:modelConds}
\vspace{-1.2mm}
\end{figure}
    % {\textcolor{red} {what are the equations for this process? I don't fully understand tbh - so only recipient devices (denoted by some flag I guess?) receive model parameters + data qty at the Tx devices + calculates a modified global average from neighbor devices and only recipient devices send their model aggregates to the BS?}}
    
    \vspace{-4.5mm}
    To model this process, we define a continuous variable {\small $\varphi^{(k)}_{n,m}\in [0,1]$} to denote the fraction of gradient parameters (i.e., the fraction of indices of the gradient vector) offloaded from device $n$ to $m$ during global iteration $k$, where {\small $\sum_{m\in\mathcal{N}} \varphi^{(k)}_{n,m}=1$}, $\forall n,k$. If node $n$ does not share its gradient with any neighbor, {\small $\varphi^{(k)}_{n,n}=1$}.  We define {\small$\bm{\varphi}^{(k)} = [\varphi^{(k)}_{n,m}]_{1\leq n,m\leq N}$} as the \textit{gradient dispersion matrix}. 
    The time required for the \underline{r}eception of \underline{g}radient elements at node $m$ from node $n$ is given by
    \vspace{-1mm}
    \begin{equation}\label{eq:RP}
         {T}^{\mathsf{RG},(k)}_{n,m}= {\varphi_{n,m}^{(k)} {M}b^{\mathsf{G}} }/{r^{(k)}_{n,m}},
         \vspace{-1mm}
     \end{equation}
    where $b^{\mathsf{G}}$ denotes the number of bits required to represent one element of the \underline{g}radient vector with size $M$.
% Assuming the reception of model parameters in a sequential manner, the time required for the \underline{r}eception of \underline{m}odel parameters at node $m$ is given by
%       \begin{equation}\label{eq:RP}
%          {T}^{\mathsf{RG},(k)}_{m}=\sum_{n\in\mathcal{N}\setminus\{m\}} {T}^{\mathsf{RG},(k)}_{n,m}.
%      \end{equation}
%      Also, the total time taken for dispersion of model at node $n$ is given by
%     \begin{equation}
%         T_n^{\mathsf{DM}}=\sum_{m\in\mathcal{N}\setminus\{n\}}  {T}^{\mathsf{RG},(k)}_{n,m}.
%     \end{equation}
   Since the reception of gradient occurs in a parallel manner, the total time required for the \underline{r}eception of \underline{g}radient parameters at node $m$ is given by
   \vspace{-2mm}
   \begin{equation}\label{eq:RP}
         {T}^{\mathsf{RG},(k)}_{m}=\max_{n\in\mathcal{N}\setminus\{m\}}\left\{ {T}^{\mathsf{RG},(k)}_{n,m}\right\}.
   \vspace{-1.8mm}
   \end{equation}
    %   \ali{I will just have the second one here1!}
    %   {\textcolor{red}{is $T^{\mathsf{DM}}$ also going to be parallelized or is it embedded implicitly in $T^{\mathsf{RG}}$, e.g., $\max \{T^{RG},T^{DM}\}$? I'm not seeing it in the constraints}}
   Finally, the uplink transmission time at node $n$  is given by
   \vspace{-1mm}
   \begin{equation}\label{eq:MTtext}
       {T}^{\mathsf{GT},(k)}_{n} = {Mb^{\mathsf{G}}}/{r^{(k)}_n}.
   \vspace{-1mm}
   \end{equation}
%   If other devices upload their gradients to node $n$, the uplink transmission time remains the same since the node conducts local model condensation. 
   As explained earlier, each node either disperses its gradient or keeps it entirely local, and thus $\varphi^{(k)}_{n,n}$ is a binary variable, $\forall n$. Thus,~\eqref{eq:MTtext} can be expressed as {\small $ {T}^{\mathsf{GT},(k)}_{n} = {Mb^{\mathsf{G}}\varphi^{(k)}_{n,n}}/{r^{(k)}_n}$}. 
    % To enforce this situation, we have the constraint $\sum_{n\in\mathcal{N}} \varphi^{(k)}_{n,n} =1$ and
    % $\varphi_{n,n}^{(k)}\sum_{m \in \mathcal{N} \setminus \{n\}} \varphi_{n,m}^{(k)} \leq 0$, $\forall k$. Also, to ensure data transmission to the server, we have $(1-\varphi_{n,n}^{(k)})\sum_{m \in \mathcal{N} \setminus \{n\}} \varphi_{m,n}^{(k)} \leq 0$.

    % is the necessary condition for the reception of data from the neighbors, since if $q_{i,i}=0$ and node $i$ receives model parameters of other nodes, the model dispersion conducted by $i$ just adds extra power consumption since node $i$ has to engage in uplink transmission to the BS to transfer the received parameters.
    % {\textcolor{red}{Just to confirm then, that is why you have $T^{D}+T^{U}$ in eq(28), because its sequential, after data reception, those receiving devices would then transfer to the BS?}}
    
    % The main reason for conducting model dispersion and condensation are BW, delay and power savings. This is due to the fact that a set of devices can completely disperse their model parameters across their neighbors, which will not occupy the cellular BW for uplink transmission. Also, transmission of model parameter to the neighbors via short range D2D communication is less energy consuming as compared to uplink transmission to the BS.
    \vspace{-.2mm}
     \subsubsection{Energy Consumption}
     At global iteration $k$, 
       the  total energy consumption $E^{(k)}_n$ at each device $n$ is given by
    \vspace{-1mm}
\begin{equation}\label{eq:EoneGlob}
    E^{(k)}_n = E^{\mathsf{DD},(k)}_n + E_n^{\mathsf{C},(k)} + E_n^{\mathsf{GD},(k)}+ E_n^{\mathsf{GT},(k)},
    \vspace{-1mm}
\end{equation}
 where $E^{\mathsf{DD},(k)}_n$, $E_n^{\mathsf{C},(k)}$, $E_n^{\mathsf{GD},(k)}$, $E_n^{\mathsf{GT},(k)}$ denote the energy used for \underline{d}ata \underline{d}ispersion, local \underline{c}omputation, \underline{g}radient \underline{d}ispersion, and \underline{g}radient \underline{t}ransmission to the BS, given as follows:
 \vspace{-1mm}
 \begin{align}
 &E^{\mathsf{DD},(k)}_n= \sum_{m\in \mathcal{N}\setminus\{n\}} {p^{\mathsf{D}}_{n} \varrho^{(k)}_{n,m} D^{(k)}_nb^{\mathsf{D}}}/{r^{(k)}_{n,m}},\\[-.5em]
 & E_n^{\mathsf{C},(k)}= (\alpha_n/2) a_n e^{(k)}_n B^{(k)}_n \left(f^{(k)}_n\right)^2 ,\label{eq:Ec}\\ 
 & E_n^{\mathsf{GD},(k)}= \sum_{m\in\mathcal{N}\setminus \{n\}}  {p^{\mathsf{D}}_{n} \varphi_{n,m}^{(k)} M b^{\mathsf{G}}}/{r^{(k)}_{n,m}},\\
 & E_n^{\mathsf{GT},(k)}= {p^{\mathsf{U}}_n Mb^{\mathsf{G}}\varphi_{n,n}^{(k)}}/{r^{(k)}_n},
  \end{align}
  In~\eqref{eq:Ec}, $\alpha_n/2$ is the effective chipset capacitance of $n\in\mathcal{N}$~\cite{8737464}.
 
\subsubsection{PSL under Synchronization Signaling}\label{subsec:syncSig} To synchronize the processes involved in PSL, we assume that the BS ochestrates the devices through a sequence of signals. In particular, the BS starts each global training round via broadcasting signal $\mathsf{S^{D}}$, which begins the data dispersion phase among the devices. Afterward, the BS dictates the start of local model training via broadcasting signal  $\mathsf{S^{L}}$. Then, it starts the model/gradient dispersion phase with a signal $\mathsf{S^{M}}$. Finally, the global aggregation round ends with the BS broadcasting a signal $\mathsf{S^{U}}$, at which time the devices start uplink transmissions. Let ${T^{\mathsf D,(k)}}$, ${T^{\mathsf L,(k)}}$, ${T^{\mathsf M,(k)}}$, ${T^{\mathsf U,(k)}}$ denote the corresponding time interval associated with data dispersion, local model training, model/gradient dispersion, and uplink transmissions, respectively. The total delay of global training round $k$ is $T^{\mathsf{Tot},(k)} = T^{\mathsf{D},(k)}+T^{\mathsf{L},(k)}+T^{\mathsf{M},(k)}+T^{\mathsf{U},(k)}$, which we refer to as the \textit{device acquisition time}. Fig.~\ref{fig:PSLTimes} depicts a schematic of the timeline of PSL during a global training round.

 \vspace{-1.9mm}
\begin{remark}
PSL can be extended to an asynchronous regime, which is left as future work. 
For example, devices which finish the data dispersion faster can start the local computation and upload their resulting models/gradients to neighboring devices, while some others are still conducting data dispersion.
\end{remark}

\begin{table*}[t!]
% \vspace{-8mm}
\begin{minipage}{0.99\textwidth}
{\footnotesize
\begin{equation}\label{eq:gen_conv}
\hspace{-10mm}
% \resizebox{.6\linewidth}{!}{$
\begin{aligned}
 &\frac{1}{K} \sum_{k=0}^{K-1} \mathbb E\left[ \big\Vert \nabla F^{({k})}(\mathbf{w}^{({k})}) \big\Vert^2\right] \leq \frac{1}{K} \vast[
 \sum_{k=0}^{K-1}\underbrace{\frac{\mathbb E\left[{{F}^{({k-1})}(\mathbf{w}^{(k)})} - F^{({k})}(\mathbf{w}^{(k+1)})\right]}{\Gamma^{(k)}(1-\Lambda^{(k)})}}_{(a)}
 +\sum_{k=0}^{K-1}\underbrace{\frac{\Omega^{(k+1)} \Delta^{(k+1)} }{\Gamma^{(k)}(1-\Lambda^{(k)})}}_{(b)}
   \\[-0.2em]&+\sum_{k=0}^{K-1} \frac{1}{(1-\Lambda^{(k)})}\Bigg(\underbrace{{8\beta^2\Theta^2 \eta_k^2 }\sum_{n\in \mathcal{N}}\frac{\widehat{{D}}^{(k)}_n}{{D}^{(k)} }( e_n^{(k)}-1) 
   %%%%%%%%%%%%%
    \sum_{j=1}^{S^{(k)}_n} \left(1-\frac{{B}^{(k)}_{n,j}}{{S}^{(k)}_{n,j}} \right) \frac{{S}^{(k)}_{n,j}}{\left({D}^{(k)}_{n}\right)^2} \frac{{({S}^{(k)}_{n,j}-1)}
     \left(\widetilde{\sigma}_{n,j}^{(k)}\right)^2}{{B}^{(k)}_{n,j}}}_{(c)}\hspace{-14mm} \hspace{-6mm}\\[-0.3em]
     &
     +\underbrace{{8\zeta_2 \eta_k^2\beta^2 \left(e_{\mathsf{max}}^{(k)}\right)\left(e_{\mathsf{max}}^{(k)}-1\right)}
     }_{(d)}  +\underbrace{{8\Theta^2 {{\beta} \Gamma^{(k)} }}\sum_{n\in \mathcal{N}} \left(\frac{\widehat{{D}}^{(k)}_n}{{D}^{(k)} \sqrt{e^{(k)}_n}} \right)^2 \sum_{j=1}^{S^{(k)}_n} \left(1-\frac{{B}^{(k)}_{n,j}}{{S}^{(k)}_{n,j}} \right) \frac{{S}^{(k)}_{n,j}}{\left({D}^{(k)}_{n}\right)^2} \frac{{({S}^{(k)}_{n,j}-1)}
     \left(\widetilde{\sigma}_{n,j}^{(k)}\right)^2}{{B}^{(k)}_{n,j}}}_{(e)}
     \Bigg)
 \vast]
 \end{aligned}
%  $}
 \end{equation}
 }
 \vspace{-2mm}
 \hrule
 \end{minipage}
  \vspace{-2mm}
  \end{table*}
   \begin{table*}[t!]
\begin{minipage}{0.99\textwidth}
 \begin{equation}\label{eq:cor1}
 \footnotesize
     \begin{aligned}
 \hspace{-16mm}\frac{1}{K} \sum_{k=0}^{K-1} \mathbb E\left[ \big\Vert \nabla F^{({k})}(\mathbf{w}^{({k})}) \big\Vert^2\right] &\leq 
      2 \sqrt{\widehat{e}_{\mathsf{max}}} \frac{F^{({-1})}(\mathbf{w}^{(0)}) - F^{{(K)}^\star}}{\overline{e}_{\mathsf{min}}\alpha \sqrt{N K}(1-\Lambda_{\mathsf{max}})}
      +\frac{ 2 \sqrt{\widehat{e}_{\mathsf{max}}}}{\overline{e}_{\mathsf{min}}\alpha \sqrt{N K}}\sum_{k=0}^{K-1}\frac{\Omega^{(k+1)} \Delta^{(k+1)} }{1-\Lambda_{\mathsf{max}}}+\frac{1}{K}\sum_{k=0}^{K-1} \frac{1}{(1-\Lambda_{\mathsf{max}})}\Bigg({\frac{8\beta^2\Theta^2 \alpha^2N}{K e^{(k)}_{\mathsf{sum}}} }
   \hspace{-10mm} 
   \\&\hspace{-6mm}\times \sum_{n\in \mathcal{N}}\frac{\widehat{{D}}^{(k)}_n}{{D}^{(k)} }( e_n^{(k)}-1)    %%%%%%%%%%%%%
    \sum_{j=1}^{S^{(k)}_n} \left(1-\frac{{B}^{(k)}_{n,j}}{{S}^{(k)}_{n,j}} \right) \frac{{S}^{(k)}_{n,j}}{\left(\widehat{D}^{(k)}_{n}\right)^2} \frac{{({S}^{(k)}_{n,j}-1)}
     \left(\widetilde{\sigma}_{n,j}^{(k)}\right)^2}{{B}^{(k)}_{n,j}} 
     +{ \frac{8\zeta_2\alpha^2\beta^2 N}{Ke^{(k)}_{\mathsf{sum}}} \left(e_{\mathsf{max}}^{(k)}\right)\left(e_{\mathsf{max}}^{(k)}-1\right)}
     \Bigg)
     \hspace{-10mm} \\[-.2em]&\hspace{-6mm} +\frac{1}{K}\sum_{k=0}^{K-1}\frac{4\overline{e}_{\mathsf{max}}\alpha \Theta^2 {\beta \sqrt{N} }}{(1-\Lambda_{\mathsf{max}})\sqrt{{e}^{(k)}_{\mathsf{sum}}}\sqrt{K}}\sum_{n\in \mathcal{N}} \left(\frac{\widehat{{D}}^{(k)}_n}{{D}^{(k)} \sqrt{e^{(k)}_n}} \right)^2 \sum_{j=1}^{S^{(k)}_n} \left(1-\frac{{B}^{(k)}_{n,j}}{{S}^{(k)}_{n,j}} \right) \frac{{S}^{(k)}_{n,j}}{\left(\widehat{D}^{(k)}_{n}\right)^2} \frac{{({S}^{(k)}_{n,j}-1)}
     \left(\widetilde{\sigma}_{n,j}^{(k)}\right)^2}{{B}^{(k)}_{n,j}}\hspace{-14mm}
     \end{aligned}
     \vspace{-1.2mm}
 \end{equation}
\hrule
\end{minipage}
\vspace{-6.8mm}
\end{table*}
  
\begin{figure}[t]
% \vspace{-4mm}
\centering
\includegraphics[width=.48\textwidth]{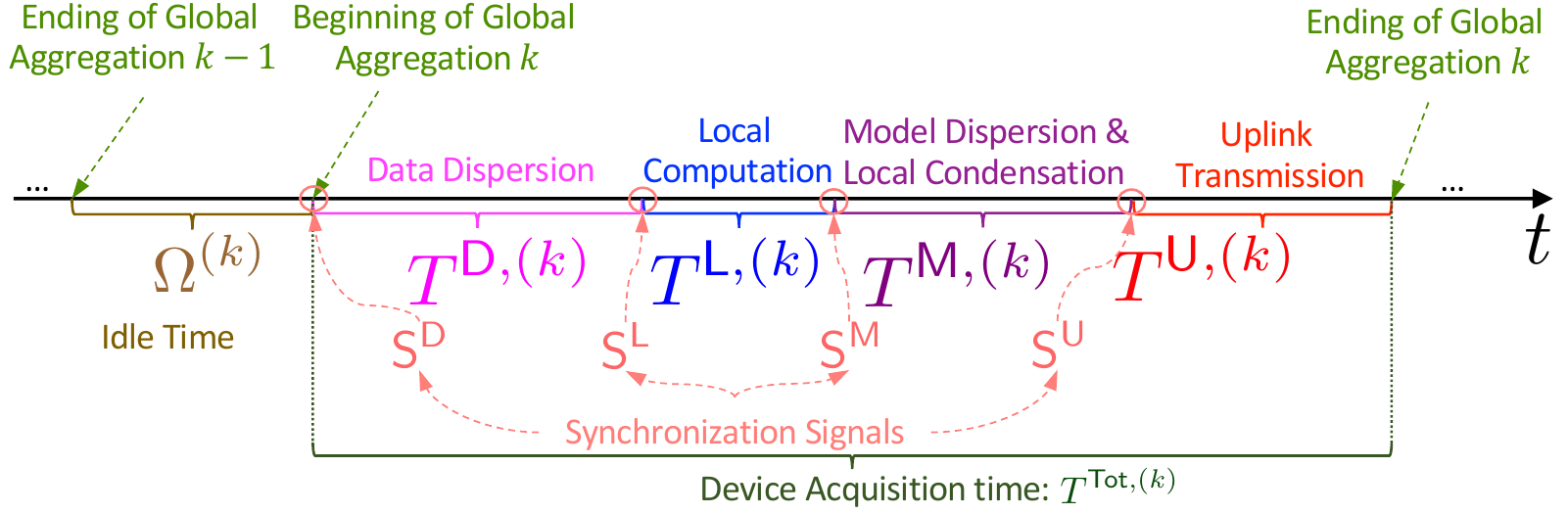}
\vspace{-2.1mm}
\caption{A schematic of PSL operations  during global round $k$. There is an \textit{idle time} in between global aggregations. During the \textit{device acquisition time},  a series of synchronization signals are broadcast by the BS.}
\label{fig:PSLTimes}
\vspace{1mm}
\end{figure}
   \vspace{-4mm}
\section{Convergence Analysis of PSL}\label{sec:conv}
 \noindent  
%  In this section, we analyze the convergence of PSL.
%  We use the notation $K$ to denote the total number of global aggregations and $T^{\mathsf{ML}}$ to denote the time (in seconds), in which the PSL is deployed. 
%  We assume that each global aggregation takes at least one second. Thus, we always have $K\leq T^{\mathsf{ML}}$, and $K= T^{\mathsf{ML}}$ if and only if all the global aggregations are back to back, i.e., with zero idle time.
%  Henceforth, notation $\Vert .\Vert$ denotes 2-norm of the indexed vector. We make two standard assumptions regarding the smoothness of the local loss functions and the dissimilarity among them. 
We first make two standard assumptions,  which are common in literature~\cite{8737464,8664630,dinh2019federated,wang2021network,wang2020tackling}, to conduct convergence analysis. Henceforth, notation $\Vert .\Vert$ denotes the 2-norm.
\vspace{-2mm}
\begin{assumption}[Smoothness of the Global and Local Loss Functions]\label{Assup:lossFun}
    Local loss function $F^{({k})}_n$  is  $\beta$-smooth,~$\forall n\in\mathcal{N},k$:
     \vspace{-2.mm}
    \begin{equation}
       \hspace{-1.5mm} \Vert \nabla F^{({k})}_n(\mathbf{w})- \nabla F^{({k})}_n(\mathbf{w}') \Vert \hspace{-.4mm}\leq \hspace{-.4mm} \beta \Vert \mathbf{w}-\mathbf{w}' \Vert,~ \hspace{-.8mm}\forall \mathbf{w},\mathbf{w}'\hspace{-.4mm}\in\hspace{-.4mm}\mathbb{R}^M\hspace{-.7mm}, \hspace{-2.3mm} 
        \vspace{-1.5mm}
    \end{equation}
    which implies the $\beta$-smoothness of the global loss function $F$. 
    %   \begin{equation}
    %     \Vert \nabla f_i(\mathbf{w};x)- \nabla f_i(\mathbf{w}';x) \Vert \leq L \Vert \mathbf{w}-\mathbf{w}' \Vert,~\forall \mathbf{w},\mathbf{w}',x.
    % \end{equation}
\end{assumption}
 \vspace{-3mm}
\begin{assumption}[Bounded Dissimilarity of Local Loss Functions]\label{Assup:Dissimilarity}
There exist finite constants $\zeta_1 \geq 1$, $\zeta_2 \geq 0$, for any set of coefficients $\{a_n\geq 0\}_{n\in \mathcal{N}}$, where $\sum_{n\in \mathcal{N}} a_n=1$, such that
%   For any set of coefficients $\{a_n\}_{n\in \mathcal{N}}$, where $\sum_{n\in \mathcal{N}} a_n=1$, there exist two finite constants $\zeta_1 \geq 1$, $\zeta_2 \geq 0$, such that
    \vspace{-1mm}
   \begin{equation}
       \hspace{-3mm} \sum_{n\in \mathcal{N}}\hspace{-1mm} a_n \Vert \nabla F^{({k})}_n (\mathbf{w}) \Vert^2 \hspace{-0.8mm}\leq \hspace{-0.5mm} \zeta_1  \Big\Vert \hspace{-1mm} \sum_{n\in \mathcal{N}}\hspace{-1mm}a_n    \hspace{-.5mm} \nabla F^{({k})}_n (\mathbf{w})   \hspace{-.5mm} \Big\Vert^2\hspace{-2mm} +\zeta_2,\hspace{-1mm}~\forall k,\mathbf{w}.    \hspace{-3mm}
         % \vspace{-mm}
   \end{equation}
\end{assumption}
The two parameters $\zeta_1$ and $\zeta_2$ introduced in Assumption~\ref{Assup:Dissimilarity} measure
the level of heterogeneity (i.e., non-i.i.d-ness) across the devices' datasets. If the device's datasets are homogeneous (i.e., i.i.d. across the devices), we will have $\zeta_1=1,\zeta_2=0$, and these values will increase as the heterogeneity of data across the devices increases. These two parameters will later play a key role in our convergence analysis through determining the acceptable ranges for the step size of SGD and the rate of convergence of ML model training. Roughly speaking, as these two parameters increase, there would be a higher chance of local model bias across the devices, which often leads to more stringent conditions on step size and a worse ML performance.

We next define \textit{local data variability} to further measure the level of heterogeneity at the devices' local datasets: 
 \vspace{-1mm}
\begin{definition}[Local Data Variability]\label{Assump:DataVariabilit}
    The local data variability at each device $n$ is measured via $\Theta_n\geq 0$, which $\forall \mathbf{w} , k$ satisfies
    % \nm{isnt this just a smoothness of f wrt to the data points?}
    \vspace{-2mm}
    \begin{equation}
    \hspace{-5mm}
    \resizebox{.93\linewidth}{!}{$
   \Vert \nabla f\hspace{-.4mm}(\mathbf{w},d) \hspace{-.4mm}-\hspace{-.4mm} \nabla f\hspace{-.4mm}(\mathbf{w},d')\Vert  \hspace{-.4mm} \leq \hspace{-.4mm} \Theta_n \Vert d\hspace{-.4mm}-\hspace{-.4mm}d' \Vert,\hspace{-1.2mm}~\forall d,d'\hspace{-.7mm}\in\hspace{-.7mm}\mathcal{D}^{(k)}_n\hspace{-.9mm}.$}\hspace{-5mm}
    \end{equation}
    % \nm{what norm are you using here? (and earlier)
    We further define $\Theta = \max_{n\in\mathcal{N}}\{\Theta_n \}$.
\end{definition}

 \vspace{-1mm}
We next quantify the dynamics of data variations at the devices via introducing a versatile measure that connects the variation in the data to the performance of the ML model:
\vspace{-1.9mm}
\begin{definition} [Model/Concept Drift]\label{def:cons}
Let $D_n(t)$ and $D(t)$ denote the instantaneous number of data points at device $n$ and total number of data points at wall clock time $t$ (measured in seconds) during PSL. For each device $n$, we measure the online model/concept drift for two consecutive time instances $t-1$ and $t$ during which the device does not conduct ML model training by $\Delta_n(t)\in\mathbb{R}$, which captures the variation of the local loss for any model parameter, $ \forall \mathbf{w}\in\mathbb{R}^M$:
% \vspace{-1.5mm}
\begin{equation}\label{eq:conceptDrift}
\hspace{-16.9mm}
\resizebox{.94\linewidth}{!}{$
% \begin{aligned}
 \frac{D_n(\hspace{-.1mm}t\hspace{-.1mm})}{D(\hspace{-.1mm}t\hspace{-.1mm})}\hspace{-.4mm}  F_n \hspace{-.5mm} \left(\hspace{-.2mm} \mathbf w\big|{\mathcal{D}_n}(t)\hspace{-.2mm} \right)\hspace{-.7mm} -\hspace{-.7mm}  \frac{D_n(\hspace{-.1mm}t-1\hspace{-.2mm})}{D(\hspace{-.1mm}t-1\hspace{-.2mm})}\hspace{-.4mm} F_n  \hspace{-.5mm}  \left(\hspace{-.2mm} \mathbf w\big|{\mathcal{D}_n}(\hspace{-.1mm}t-1\hspace{-.1mm})\hspace{-.2mm} \right)\hspace{-.7mm} \leq\hspace{-.5mm}  \Delta_n(t)\hspace{-.1mm}.
%  \end{aligned}
 $}
 \hspace{-11mm}
 % \vspace{-1mm}
\end{equation}
$\Delta_n^{(t)}$ is assumed to be measured only at discrete time instances $t\in \mathbb{Z}^+$ and is presumed to be fixed in the continues time interval $(t-1,t]$ (i.e., for the duration of $1$ second).
% \nm{shouldnt there be an absolute value on the LHS?}
\end{definition}
A larger value for the model/concept drift, i.e., $\Delta_n \gg 0$, implies a larger local loss variation and a harder tracking of the optimal model parameters for the ML training. Also, the above definition encompasses the case where due to model drift the old model becomes more fit to the current data (when $\Delta_n<0$). Our definition of model drift is  different from other few definitions in current art~\cite{rizk2020dynamic,ruan2021flexible} from two aspects. First, our definition connects the data variations to the model performance and can be used in scenarios where major variations in some dimensions of data (i.e., some of the features) does not affect the performance. Second, our proposed drift is estimable (i.e., it is not defined with respect to the variations in the optimal model as in~\cite{rizk2020dynamic,ruan2021flexible} which is by itself unknown a priori).

We next obtain the convergence behavior of PSL for non-convex loss functions.  We use $\widehat{\mathcal{D}}_n^{(k)}$ to denote the set of data points at device $n$ used during global round $k$ obtained after the model dispersion phase, {\small $\widehat{{D}}_n^{(k)}=|\widehat{\mathcal{D}}_n^{(k)}|= \sum_{m\in\mathcal{N}}\varrho_{m,n}^{(k)}D_m^{(k)}$}. Also, with slight abuse of notations, we use $S^{(k)}_{n,j}$ in the bounds to denote the number of data points in respective stratum and {\small $\widetilde{\sigma}_{n,j}^{(k)}$} to denote its variance during local model training of global round $k$  which are obtained via Lemma~\ref{lemma:trackMeanvar}.
\vspace{-2mm}
\begin{theorem}[General Convergence Behavior of PSL]\label{th:main}
Assume 
\vspace{-1.5mm}
\begin{equation*}
\resizebox{.99\linewidth}{!}{$
    \eta_k \leq \min \Big\{\frac{1}{2\beta} \sqrt{ \frac{\Lambda^{(k)}}{\zeta_1(\beta^2+\Lambda^{(k)})\left( e_{\mathsf{max}}^{(k)}\left(e_{\mathsf{max}}^{(k)}-1\right)\right)}}, \left(2\beta\sum_{n\in \mathcal{N}}\frac{{D}^{(k)}_{n}e^{(k)}_{n}}{{D}^{(k)} } \right)^{-1}\Big\},$}
    \vspace{-1.2mm}
\end{equation*}
where $\Lambda^{(k)}<1$ is a constant and  $e^{(k)}_{\mathsf{max}}=\max_{n\in\mathcal{N}}\{e^{(k)}_n\}$. Also, define  $\Gamma^{(k)}=\frac{\eta_{_k}}{2}\sum_{n\in \mathcal{N}}\frac{{D}^{(k)}_{n}e^{(k)}_{n}}{{D}^{(k)} }$, and $\Delta^{(k)}=\sum_{n\in \mathcal{N}}\widehat{\Delta}_n^{(k)}$, where $\widehat{\Delta}_n^{(k)}=\max_{t\in T^{\mathsf{Idle},(k)}} \Delta_n(t)$. Then,
the cumulative average of the gradient of the global loss function over the training period of PSL satisfies the upper bound in~\eqref{eq:gen_conv}.
% \nm{The equations are impossible to read..would it be worth to use the big O notation? I.e. to hide the dependece on the parameters behind constants? It seems to me that the details of these equations are unimportant, and with the big O notation you can just focus on what matters most.
% }
% \nm{what is $T^{\mathsf{ML}}$? I though that K is the learning period but now I see both $T^{\mathsf{ML}}$ and K, Im confused.}
% \begin{table*}[t]
% \begin{minipage}{0.99\textwidth}

% \end{minipage}
% \end{table*}
 
\end{theorem}
\vspace{-3.7mm}
\begin{proof}
Refer to Appendix~\ref{app:th:main} for the detailed proof.
\end{proof}
\vspace{-1mm}
In~\eqref{eq:gen_conv}, ${{F}^{({k-1})}(\mathbf{w}^{(k)})}={{F}(\mathbf{w}^{(k)}|\mathcal{D}^{(k-1)}})$ denotes the loss under which the $k$-th global model training ends where $\mathbf w^{(k)}$ is obtained and $F^{({k})}(\mathbf{w}^{(k+1)})={{F}(\mathbf{w}^{(k+1)}|\mathcal{D}^{(k)}})$ denotes the loss under which the $k+1$-th global model training ends where $\mathbf{w}^{(k+1)}$ is obtained. Also, $F^{(-1)}(\mathbf{w}^{(0)})$ denotes the \textit{initial} loss of the algorithm before the start of any model training, i.e., $\Omega^{(1)}$ seconds before the first model training round.

\textbf{Main Takeaways.}
The bound in~\eqref{eq:gen_conv} captures the effect of the ML-related parameters on the performance of PSL. Term $(a)$ captures the effect of consecutive loss function gains 
% (via {\footnotesize $F(\mathbf{w}^{(k)})-F(\mathbf{w}^{(k+1)})$}) 
during ML training. Term $(b)$ captures the effect of model drift (via $\Delta^{(k+1)}$). Terms $(c)$ and $(e)$ are both concerned with the data stratification and mini-batch sizes used at the devices. Term $(d)$ captures the effect of model dissimilarity (via $\zeta_2$) and the number of SGD iterations (via $e^{(k)}_{\mathsf{max}}$). Larger model dissimilarity leads to smaller step size choices ($\zeta_1$ incorporated in the condition on $\eta_{_k}$) and larger upper bound ($\zeta_2$ in $(d)$). Also, larger local data variability implies a larger bound ($\Theta$ in $(c)$ and $(e)$). Further, terms $(a)$ and $(b)$ are inversely proportional to $\eta_k$ (incorporated in $\Gamma^{(k)}$), terms $(c)$ and $(d)$ are quadratically proportional to $\eta_k$, and term $(e)$ is linearly proportional to it (incorporated in $\Gamma^{(k)}$).
The bound also provides further insights: (i) upon having $e^{(k)}_n=1$, $\forall n,k$, terms $(c)$ and $(d)$ become zero and the bound demonstrate 1-epoch distributed ML with non-uniform SGD sampling; (ii) upon sampling all the data points, i.e., $B^{(k)}_{n,j}=S^{(k)}_{n,j}$, $\forall j,n,k$, terms $(c)$ and $(e)$ become zero and the bound reveals the convergence of full-batch local gradient descents, (iii) upon having uniform sampling across  strata, $B^{(k)}_{n,j}=B^{(k)}_n S^{(k)}_{n,j}/D^{(k)}_{n}$, the bound  demonstrates the convergence upon uniform data sampling using SGD with mini-batch size $B^{(k)}_n$ at each device $n$; (iv) the effect of offloading is reflected in $S_{n,j}^{(k)}$ and $\widehat{D}_{n}^{(k)}$. In particular, considering terms $(c)$ and $(e)$, increasing  $S_{n,j}^{(k)}$-s (i.e.,  data reception at device $n$), with everything else being constant, needs to be met via increasing $B^{(k)}_{n,j}$ (i.e., increasing the mini-batch size) to keep the bound value fixed. Similarly, upon data offloading, the device can use a smaller mini-batch size. Thus, the bound promotes offloading data from devices with low computation capability to those with higher computation capability.

 \begin{table*}[t]
\begin{minipage}{0.99\textwidth}
    % \vspace{-7mm}
     \begin{equation}\label{eq:cor2}
 \footnotesize
         \begin{aligned}
        \frac{1}{K} \sum_{k=0}^{K-1} \mathbb E\left[ \big\Vert \nabla F^{({k})}(\mathbf{w}^{({k})}) \big\Vert^2\right] \leq&  2 \sqrt{\widehat{e}_{\mathsf{max}}}  \frac{F^{({-1})}(\mathbf{w}^{(0)}) - F^{{(K)}^\star}}{\overline{e}_{\mathsf{min}}\alpha \sqrt{N K}(1-\Lambda_{\mathsf{max}})}
      +\frac{ 2\sqrt{\widehat{e}_{\mathsf{max}}} \gamma   }{\overline{e}_{\mathsf{min}}\alpha \sqrt{N K}(1-\Lambda_{\mathsf{max}})}
     +\frac{4\overline{e}_{\mathsf{max}}\alpha \Theta^2 {\beta \sqrt{N} }}{(1-\Lambda_{\mathsf{max}})\sqrt{\widehat{e}_{\mathsf{min}}}\sqrt{K}}  \sigma_{\mathsf{max}} 
    \\[-.2em]&+\frac{1}{K}\frac{1}{(1-\Lambda_{\mathsf{max}})}\Bigg({8\beta^2\Theta^2 {\alpha^2} N}( e_{\mathsf{max}}-1) \sigma_{\mathsf{max}}/\widehat{e}_{\mathsf{min}}+{8\zeta_2 {\alpha^2 N}\beta^2 \left(e_{\mathsf{max}}\right)\left(e_{\mathsf{max}}-1\right)}/\widehat{e}_{\mathsf{min}} \Bigg)
   %%%%%%%%%%%%%
         \end{aligned}
         \vspace{-1.8mm}
     \end{equation}
     \hrulefill
      \end{minipage}
      \vspace{-4mm}
\end{table*}

The bound in~\eqref{eq:gen_conv} reveals that reaching convergence is attainable under certain circumstances, which we aim to obtain:
\vspace{-6mm}
\begin{corollary}[Convergence under Proper Choice of Step Size and Bounded Local Iterations]\label{cor:1} 
% Assume that the global aggregations are equally separated in time, i.e., $T^{\mathsf{Tot},(k)}+\Omega^{(k)}=T^{\mathsf{Tot},(k')}+\Omega^{(k')}=\Upsilon$, $1\leq k\neq k'\leq K$. 
In addition to the conditions in Theorem~\ref{th:main}, further assume that (i) {\small $\eta_k = \alpha \big /{\sqrt{K e^{(k)}_{\mathsf{sum}}/N}}$} with a finite positive constant $\alpha$ chosen to satisfy the condition on $\eta_k$ in Theorem~\ref{th:main}, where $e^{(k)}_{\mathsf{sum}}=\sum_{n\in\mathcal{N}} e_n^{(k)}$, (ii) $ \widehat{e}_{\mathsf{min}} \leq e^{(k)}_{\mathsf{sum}}\leq  \widehat{e}_{\mathsf{max}}$ for two finite positive constants $\widehat{e}_{\mathsf{min}}$ and $\widehat{e}_{\mathsf{max}}$, $\forall k$, (iii)  $\max_k \left\{\Lambda^{(k)}\right\} \leq \Lambda_{\mathsf{max}}< 1$, (iv)  $\overline{e}_{\mathsf{min}} \leq e^{(k)}_{\mathsf{avg}}\leq  \overline{e}_{\mathsf{max}}$ for two finite positive constants $\overline{e}_{\mathsf{min}}$ and  $\overline{e}_{\mathsf{max}}$, $\forall k$, where $e^{(k)}_{\mathsf{avg}}=\sum_{n\in \mathcal{N}}{{D}^{(k)}_{n}e^{(k)}_{n}}/{D^{(k)} }$.
% and $\max_{k}\left\{\left(\sum_{n\in \mathcal{N}}\frac{{D}^{(k)}_{n}e^{(k)}_{n}}{D^{(k)} }\right)^{-1}\right\}\leq \overline{e}_{\mathsf{max}}$. 
Then, the cumulative average of the global loss function gradient for PSL satisfies~\eqref{eq:cor1}.
\end{corollary} 
\vspace{-4mm}
\begin{proof}
Refer to Appendix~\ref{app:cor:1} for the detailed proof.
% See Appendix~\ref{app:cor:1}.
\end{proof}
\vspace{-2mm}
Bound~\eqref{eq:cor1} no longer depends on consecutive loss gains ($(a)$ in~\eqref{eq:gen_conv}), rather it depends on the initial error $F^{(-1)}(\mathbf{w}^{(0)}) - F^{{(K)}^\star}$. 
We next obtain the conditions under which  the  cumulative  average of the global gradient converges to zero:

% Also, the effect of training time $T^{\mathsf{ML}}$ is explicit in~\eqref{eq:cor1}, where all the terms on the right hand side are inversely proportional either to it or its square root.  

\begin{table*}[t!]
\begin{minipage}{0.99\textwidth}
% \vspace{-3mm}
 \begin{equation}\label{eq:gen_conv_neyman_main}
 \footnotesize
 \hspace{-13mm}
 \begin{aligned}
 & \frac{1}{K}  \sum_{k=0}^{K-1}\mathbb E\left[\Vert \nabla F^{({k})}(\mathbf{w}^{({k})}) \Vert^2\right]\leq  \Xi\left(\hspace{-.6mm}\widehat{\bm{D}}^{(k)}\hspace{-.7mm},\bm{B}^{(k)}\hspace{-.7mm},\Omega^{(k)}\hspace{-.7mm},\Delta^{(k)}\hspace{-.6mm} \right)\triangleq 
 \underbrace{\frac{2\sqrt{\widehat{e}_{\mathsf{max}}}\left(\hspace{-.3mm}{F}^{({-1})}(\mathbf{w}^{(0)}) - F^{{(K)}^\star}\hspace{-.3mm}\right)}{\alpha \overline{e}_{\mathsf{min}}\sqrt{N K}(1-\Lambda_{\mathsf{max}})}}_{(a)}
 \hspace{-.6mm}+\sum_{k=0}^{K-1}\hspace{-.2mm}\frac{2\sqrt{e^{(k)}_{\mathsf{sum}}}\Omega^{(k+1)}\Delta^{(k+1)} }{\alpha e^{(k)}_{\mathsf{avg}}\sqrt{N K}(1-\Lambda^{(k)})}\hspace{-26mm}
  \\[-.27em]&+\sum_{k=0}^{K-1} \frac{1}{(1-\Lambda^{(k)})}\vast({ \frac{8\beta^2\Theta^2\alpha^2N}{{e^{(k)}_{\mathsf{sum}} K^2}} }\sum_{n\in \mathcal{N}}\frac{\widehat{{D}}^{(k)}_n}{{D}^{(k)} }( e_n^{(k)}-1)
   %%%%%%%%%%%%%
    \frac{1}{\left(\widehat{D}_n^{(k)}\right)^2}\vast[ \frac{1}{{B}^{(k)}_{n}} \left(\sum_{j=1}^{S^{(k)}_n} \widetilde{\sigma}^{(k)}_{n,j}{S}^{(k)}_{n,j} \right)^2  
    \hspace{-2mm}-\underbrace{\sum_{j=1}^{S^{(k)}_n}{S}^{(k)}_{n,j} {\left(\widetilde{\sigma}^{(k)}_{n,j}\right)^2}}_{(b)}\vast]
    \\[-0.22em]& +{ \frac{8\zeta_2\alpha^2\beta^2 N}{{e^{(k)}_{\mathsf{sum}} K^2}}\hspace{-1mm} \left(e_{\mathsf{max}}^{(k)}\right)\hspace{-1mm}\left(e_{\mathsf{max}}^{(k)}-1\right)}
     \vast)
     \hspace{-0.5mm}+\hspace{-0.5mm}\sum_{k=0}^{K-1}\frac{4 e^{(k)}_{\mathsf{avg}} \alpha \Theta^2 {\beta  \sqrt{N} }}{2\sqrt{e^{(k)}_{\mathsf{sum}}}K\sqrt{K}(1-\Lambda^{(k)})}\hspace{-0.1mm}\sum_{n\in \mathcal{N}}\hspace{-0.5mm} \frac{1}{\left({D}^{(k)}\sqrt{e^{(k)}_n} \right)^2} 
\vast[ \frac{1}{{B}^{(k)}_{n}} \left(\sum_{j=1}^{S^{(k)}_n} \widetilde{\sigma}^{(k)}_{n,j}{S}^{(k)}_{n,j} \right)^2\hspace{-2mm}-\underbrace{\sum_{j=1}^{S^{(k)}_n}{S}^{(k)}_{n,j} {\left(\widetilde{\sigma}^{(k)}_{n,j}\right)^2}}_{(c)}\vast]\hspace{-4mm}
 \end{aligned} 
 \hspace{-14mm}
 % \vspace{-1.2mm}
 \end{equation}
 \vspace{-1mm}
 \hrule
 \end{minipage}
 \vspace{-2mm}
  \end{table*}

\vspace{-1.2mm}
\begin{corollary}[Convergence under Unified Upperbounds on the Sampling Noise]\label{cor:2}
Under the conditions specified in Corollary~\ref{cor:1}, further assume (i) bounded stratified sampling noise {\small $\displaystyle\max_{k,n}\left\{\sum_{j=1}^{S^{(k)}_n} \left(1-\frac{{B}^{(k)}_{n,j}}{{S}^{(k)}_{n,j}} \right) \frac{{S}^{(k)}_{n,j}}{\left({D}^{(k)}_{n}\right)^2} \frac{{({S}^{(k)}_{n,j}-1)}
     \left(\widetilde{\sigma}_{n,j}^{(k)}\right)^2}{{B}^{(k)}_{n,j}}\right\}$} $\leq \sigma_{\mathsf{max}}$, $\forall n,k$,  (ii) bounded local iterations $\max_{k}\{ e^{(k)}_{\mathsf{max}}\}\leq e_{\mathsf{max}} $ for positive constant $e_{\mathsf{max}}$, 
    %  (iii)  $ \widehat{e}_{\mathsf{min}} \leq e^{(k)}_{\mathsf{sum}}\leq  \widehat{e}_{\mathsf{max}}$, where $e^{(k)}_{\mathsf{sum}}=\sum_{n\in\mathcal{N}} e_n^{(k)}$, for two finite positive constants $\widehat{e}_{\mathsf{min}}$ and $\widehat{e}_{\mathsf{max}}$, $\forall k$
     and (iii) bounded idle period as $\Omega^{(k)}\leq \left[\frac{\gamma}{K\Delta^{(k)}}\right]^+$, $\forall k$, for a finite non-negative $\gamma$, where $[a]^+\triangleq \max\{a,0\}$, $\forall a\in\mathbb{R}$. Then,  the  cumulative  average of the global loss function gradient for PSL satisfies~\eqref{eq:cor2}, implying $\frac{1}{K} \sum_{k=0}^{K-1} \left\Vert \nabla F^{({k})}(\mathbf{w}^{({k})}) \right\Vert^2 \leq \mathcal{O}(1/\sqrt{K})$.
\end{corollary}
\vspace{-2mm}
\begin{proof}
Refer to Appendix~\ref{app:cor:2} for the detailed proof.
% See Appendix~\ref{app:cor:2}.
\end{proof}
\vspace{-2mm}
One of the key findings of Corollary~\ref{cor:2} is that, to have a guaranteed convergence behavior, the idle times must be inversely proportional to the model drift, i.e., $\Omega^{(k)}\leq \left[\frac{\gamma}{K\Delta^{(k)}}\right]^+$, $\forall k$. This implies that larger model drift requires rapid global aggregations (i.e., small idle times), while smaller model drift requires less frequent global aggregations (i.e., large idle times). 

We next obtain the data sampling technique that needs to be utilized under a certain data sampling budget at each device:
% \vspace{-2mm}
\begin{proposition}[PSL under Optimal Local Data Sampling]\label{prop:neyman} Considering the assumptions made in Theorem~\ref{th:main}.
% , let $Q_{max} = \max_{1\leq k\leq K} \{T^{\mathsf{Tot},(k)}+\Omega^{(k)} \}$ denote the maximum of the total idle and ML training time across the global aggregation sequences of PSL.
For a given mini-batch size at each device $n$, i.e., $B_n^{(k)}$, tuning the number of sampled points from strata according to Neyman's sampling technique represented with respect to the variance of the data described by {\small ${B}^{(k)}_{n,j} = \left(B^{(k)}_n\widetilde{\sigma}^{(k)}_{n,j}{S}^{(k)}_{n,j} \right)\Big(\sum_{i=1}^{S^{(k)}_n} \widetilde{\sigma}^{(k)}_{n,i}{S}^{(k)}_{n,i}\Big)^{-1}$} minimizes the bound in~\eqref{eq:gen_conv}. Further, if {\small $\eta_k = \alpha\big/{\sqrt{Ke^{(k)}_{\mathsf{sum}}/N}}$} with a finite positive constant $\alpha$ chosen to satisfy the condition on $\eta_k$ in Theorem~\ref{th:main} the cumulative average of global gradient under PSL satisfies the bound in~\eqref{eq:gen_conv_neyman_main}.
% \nm{bounds impossible to follow..}
\end{proposition}
\vspace{-3.5mm}
\begin{proof}
Refer to Appendix~\ref{app:prop:neyman} for the detailed proof.
% See Appendix~\ref{app:prop:neyman}.
\end{proof}
\vspace{-1.5mm}
The choice of ${B}^{(k)}_{n,j}$ in Proposition~\ref{prop:neyman} advocates sampling more data points from those strata with higher variance to reduces the SGD noise. This technique is particularly effective when the local datasets are non-i.i.d., e.g., devices possess unbalanced number of data points from different labels.

The bound obtained in Proposition~\ref{prop:neyman} is rather general. In particular, the bound is not obtained under conditions in Corollary~\ref{cor:2}, which were considered to prove the convergence. This is intentionally done to give a generalized
bound that describes the PSL convergence under arbitrary choice of idle times and  sampling errors. We thus build our optimization with respect to this bound, making our optimization solver general and applicable to a wide variety of scenarios. We obtain the convergence of PSL when the conditions of Corollaries~\ref{cor:1} and~\ref{cor:2} are imposed on Proposition~\ref{prop:neyman} in Appendix~\ref{app:furtherOptSample} (Corollary~\ref{cor:neyman1}~and~\ref{cor:furtherOptSam2}), which are omitted here for brevity since the results are qualitatively similar. 

\vspace{-3mm}
 \section{Network-aware Parallel Successive Learning}\label{sec:PF}
 \vspace{-.1mm}
\noindent  
In~\textit{network-aware PSL}, we aim to jointly optimize the \textit{macro decisions} of the system, e.g., timing of the synchronization signals and idle times between global aggregations, and the \textit{micro decisions} of the system, e.g., local mini-batch sizes and data/parameter offloading ratios, which is among the most general formulations in literature. We
 formulate the {network-aware PSL} as the following optimization problem:
 % \vspace{-1.5mm}
 \begin{align}
     &(\bm{\mathcal{P}}): ~~\min \frac{1}{K}\left[\sum_{k=0}^{K-1} c_1 E^{\mathsf{Tot},(k)}+ c_2 T^{\mathsf{Tot},(k)}\right]\nonumber \\[-.12em] &~~~~~~~~~~~+ c_3 \hspace{-.1mm}\underbrace{ \frac{1}{K}\sum_{k=0}^{K-1}\mathbb E\left[ \big\Vert \nabla F^{({k})}(\mathbf{w}^{({k})}) \big\Vert^2\right]}_{\small = \Xi\left(\widehat{\bm{D}}^{(k)},\bm{B}^{(k)},\Omega^{(k)},\Delta^{(k)} \right) ~\textrm{given by}~\eqref{eq:gen_conv_neyman_main}}\hspace{-40mm} \\[-.45em]
     & \textrm{s.t.}\nonumber\\ 
     &  T^{\mathsf{Tot},(k)} = T^{\mathsf{D},(k)}+T^{\mathsf{L},(k)}+T^{\mathsf{M},(k)}+T^{\mathsf{U},(k)},\label{prob:Ttot}\\[-.2em]
     & E^{\mathsf{Tot},(k)}= \sum_{n\in \mathcal{N}}E^{(k)}_n,\label{prob:Etot}\\[-.45em]
     & \sum_{k=1}^{K} T^{\mathsf{Tot},(k)}+\Omega^{(k)} = T^{\mathsf{ML}},\label{prob:Tml}\\[-.2em]
    %  &  \widehat{D}^{(k)}_n=\sum_{m\in\mathcal{N}} \varrho_{m,n}^{(k)} D_m^{(k)}, ~~n\in\mathcal{N}, \label{prob:Dhat}\\
     & \max_{n\in\mathcal{N}} \left\{  T^{\mathsf{DR},{(k)}}_{n} \right\} \leq T^{\mathsf{D},(k)},\label{prob:Td}\\
     & \max_{n\in\mathcal{N}} \left\{T^{\mathsf{C},{(k)}}_{n}\right\} \leq T^{\mathsf{L},(k)},\label{prob:Tl}\\
     & \max_{n\in\mathcal{N}} \left\{  T^{\mathsf{RG},{(k)}}_{n} \right\} \leq T^{\mathsf{M},(k)},\label{prob:Tm}\\
     & \max_{n\in\mathcal{N}} \left\{  T^{\mathsf{GT},{(k)}}_{n} \right\} \leq T^{\mathsf{U},(k)},\label{prob:Tu}\\
    %  &D^{\mathsf{F},(k)}_n=\sum_{m\in \mathcal{N}} \varrho^{(k)}_{m,n}\widehat{D}^{(k)}_m ,~~ n\in\mathcal{N},\\
      & \sum_{m\in \mathcal{N}}\varrho^{(k)}_{n,m} = 1,~ n\in\mathcal{N},\label{prob:varrho}\\
      & \sum_{m\in \mathcal{N}} \varphi^{(k)}_{n,m}=1,~~n\in\mathcal{N}, \label{prob:varphi}\\
        &\varphi_{n,n}^{(k)}\sum_{m \in \mathcal{N} \setminus \{n\}} \varphi_{n,m}^{(k)} \leq 0, ~n\in\mathcal{N}, \label{eq:varphi1}\\
        &(1-\varphi_{n,n}^{(k)})\sum_{m \in \mathcal{N} \setminus \{n\}} \varphi_{m,n}^{(k)} \leq 0, ~n\in\mathcal{N},\label{eq:varphi2}\\
        & f^{\mathsf{min}}_n\hspace{-.5mm}\leq\hspace{-.5mm} f^{(k)}_n\hspace{-.5mm}\leq \hspace{-.5mm} f^{\mathsf{max}}_n,\hspace{-.5mm}~1\leq \hspace{-.5mm} B_n^{(k)} \hspace{-1mm}\leq \hspace{-1mm}\sum_{m\in \mathcal{N}}\varrho^{(k)}_{m,n}D^{(k)}_{m}, n\in\mathcal{N},\label{prob:freq}\\[-0.7em]
     & \varrho^{(k)}_{n,m},\varphi^{(k)}_{n,m} \geq 0, ~~n,m\in \mathcal{N},\label{prob:feas}\\[-.1em]
     &\hspace{-1mm}\textrm{\textit{\textbf{Variables}}:}\nonumber\\[-0.4em]
     &\hspace{-1mm}\footnotesize K,\big\{\mathbf{e}^{(k)}\hspace{-.5mm},\mathbf{f}^{(k)},\mathbf{B}^{(k)}\hspace{-.5mm},\bm{\varrho}^{(k)}\hspace{-.5mm}, \bm{\varphi}^{(k)}\hspace{-.5mm},T^{\mathsf{D},(k)}\hspace{-.5mm},T^{\mathsf{L},(k)}\hspace{-.5mm},T^{\mathsf{M},(k)}\hspace{-.5mm},T^{\mathsf{U},(k)}\hspace{-.5mm},\Omega^{(k)}\big\}_{k=1}^{K} \nonumber \hspace{-10mm} 
 \end{align}
 \vspace{-7mm}

 \noindent \textbf{Objective and variables.} $\bm{\mathcal{P}}$ aims to identify the number of global aggregations $K$, and the value of the following variables at each global aggregation $k$: the number of SGD iterations {\small $\mathbf{e}^{(k)}=[{e}_n^{(k)}]_{n\in\mathcal{N}}$}, frequency cycles of the devices {\small $\bm{f}^{(k)}=[{f}_n^{(k)}]_{n\in\mathcal{N}}$},  mini-batch sizes {\small$\bm{B}^{(k)}=[{B}_n^{(k)}]_{n\in\mathcal{N}}$} (given the mini-bath size, the sample size of strata, i.e., {\small$[{B}_{n,j}^{(k)}]_{n\in\mathcal{N}}, \forall j$, is given by Proposition~\ref{prop:neyman})},  data offloading ratios {\small $\bm{\varrho}^{(k)}=[{\varrho}_{n,m}^{(k)}]_{n,m\in\mathcal{N}}$},  model parameter offloading ratios {\small $\bm{\varphi}^{(k)}=[{\varphi}_{n,m}^{(k)}]_{n,m\in\mathcal{N}}$}, synchronization  periods (i.e., {\small $T^{\mathsf{D},(k)},T^{\mathsf{L},(k)},T^{\mathsf{M},(k)},T^{\mathsf{U},(k)}$} defined in Sec.~\ref{subsec:syncSig}), and  idle times  between the global aggregations $\Omega^{(k)}$.
 The objective function of $\bm{\mathcal{P}}$ draws a tradeoff between the total energy consumption, device acquisition time/cost, and the ML training performance of the global model. The latter is captured via $\Xi$, which is characterized by the bound in~\eqref{eq:gen_conv_neyman_main}. These (possibly) competing objectives are weighted via non-negative coefficients $c_1,c_2,c_3$.\footnote{We trivially  defined $E^{\mathsf{Tot},(0)}= T^{\mathsf{Tot},(0)}=0$.}
 
 \textbf{Constraints.}
 $T^{\mathsf{Tot},(k)}$ in~\eqref{prob:Ttot} denotes the device acquisition time and $E^{\mathsf{Tot},(k)}$ in~\eqref{prob:Etot} denotes the total energy consumption during global round $k$, where $E_n^{(k)}$ is given by~\eqref{eq:EoneGlob}. Constraint~\eqref{prob:Tml} ensures that the accumulated time spent during the model training and the idle times equals to the ML model training time $T^{\mathsf{ML}}$. 
%  Constraint~\eqref{prob:Dhat} determines the number of final data points at the nodes, which appear in the last term of the objective function.
 Constraints~\eqref{prob:Td},~\eqref{prob:Tl},~\eqref{prob:Tm},~\eqref{prob:Tu} ensure that the time interval used for data dispersion (see~\eqref{eq:TDRDef}), local computation (see~\eqref{eq:TCdef}), model dispersion (see~\eqref{eq:RP}), and uplink transmission (see~\eqref{eq:MTtext}) are chosen to ensure the operation of the system without conflict. Constraints~\eqref{prob:varrho} and~\eqref{prob:varphi} ensure the proper dispersion of the data and model parameters. Constraints~\eqref{eq:varphi1} and~\eqref{eq:varphi2} together ensure that each device either disperses its model parameter or keeps it local and engages in model condensation, i.e., $\phi_{n,n}$ is binary and takes the value of $0$ if any portion of the local model is offloaded; and $1$ otherwise. Finally,~\eqref{prob:freq},~\eqref{prob:feas} are the feasibility constraints.

\textbf{Main takeaways.} In~$\bm{\mathcal{P}}$,
if $c_3=0$, and $c_1,c_2>0$, model training never occurs ($K=0$) and  devices always remain in the idle state. As $c_3$ increases, the solution favors a lower  model loss via increasing the number of global rounds $K$ and decreasing the idle times $\Omega^{(k)}$-s. $\Xi$ is a function of concept drift  ($\Delta^{(k)}_{n}$-s), especially 
% $\bm{\mathcal{P}}$ to capture both transient and steady state behavior of model training.
upon having small concept drifts, once the global model reaches a relatively low loss (i.e., it is \textit{warmed up} and has high \textit{inertia}), performing global aggregations result in marginal performance gains, and thus the frequency of global rounds will be decreased. However, upon having large concept drifts, since the global model obsoletes fast, to track the model variations, more global rounds with lower idle times are preferred. 
% This implies frequent \textit{triggering} of model training at the devices, detaching them from the idle state.
Also, considering $\Xi$ behavior in~\eqref{eq:gen_conv_neyman_main}, the optimization favors larger mini-batch sizes at the devices with higher number of data points and larger data variance. Also, the data is often offloaded from the devices with limited computation resource to those with abundant resources, while the model parameters are often transferred from the devices with poor BS channel conditions to those with better channels.

\textbf{Behavior of~$\bm{\mathcal{P}}$.} Except integer {\small$K$}, all the variables are continuous. Given that {\small$1\leq K\leq T^{\mathsf{ML}}$}, with all the rest of the variables known as a function of $K$, $K$ can be obtained with an exhaustive search. We thus focus on obtaining the rest of variables for a given {\small$K$}. In~$\bm{\mathcal{P}}$, multiplication between optimization variables appear in multiple places. For example, in {\small $E^{\mathsf{C},{(k)}}_n$}, encapsulated in {\small$E^{(k)}_n$} (see~\eqref{eq:EoneGlob}) in~\eqref{prob:Etot}, multiplication between {\small$e_n^{(k)}$, $B_n^{(k)}$},  and {\small$f_n^{(k)}$} exist (see~\eqref{eq:Ec}). Similar phenomenon exist in {\small$T_n^{\mathsf{C},(k)}$} (see~\eqref{eq:TCdef}) in~\eqref{prob:Tl}. More importantly, the definition of $\Xi$ in~\eqref{eq:gen_conv_neyman_main} consists of multiple terms with multiplication of variables some of which with negative coefficients.
% where in some of the terms the multiplication of variables appear with a negative sign.
In particular, the problem belongs to the category of \textit{Signomial Programming}, which are highly non-convex and NP-hard~\cite{chiang2005geometric}. This is expected given the generality of the formulation and the complex behavior of bound~\eqref{eq:gen_conv_neyman_main}.

In the following, we provide a tractable technique to solve $\bm{\mathcal{P}}$. Our approach relies on a set of approximations and constraint modifications to solve the problem in an iterative manner, which enjoys strong convergence guarantees.\footnote{To be able to solve the problem, we rely on strict positive optimization variables and replace constraint~\eqref{prob:feas} with  $\varrho_{n,m}^{(k)},\varphi_{n,m}^{(k)}>0, \forall n,m$. Accordingly, inequalities in~\eqref{eq:varphi1} and~\eqref{eq:varphi2} in the form of $A(\bm{x})\leq 0$ are replaced with  $A(\bm{x})< \vartheta$, where $\vartheta>0$ is an optimization variable. $\vartheta$ is then added to the objective function with a penalty term to ensure $\vartheta \downarrow 0$ at the final solution.} Although our approach is developed for~$\bm{\mathcal{P}}$, it can be applied to a broader category of problems that we call \textit{network-aware distributed ML}, where the formulations are mostly concerned with optimizing the performance of the ML training under network constraints. We are among the first to introduce these highly versatile optimization techniques to distributed ML literature.
  
   \begin{table*}[tbp]
% \vspace{-4.5mm}
\begin{minipage}{0.99\textwidth}
{\scriptsize
      \begin{equation}\label{approx:H}
      \begin{aligned}
    H(\bm{x})=&\sum_{k=1}^{K} T^{\mathsf{Tot},(k)}+\Omega^{(k)} \geq \widehat{H}(\bm{x};\ell) \triangleq \prod_{k=1}^{K} \left(\frac{T^{\mathsf{D},(k)}  H([\bm{x}]^{\ell-1})}{\left[T^{\mathsf{D},(k)} \right]^{\ell-1}}\right)^{\frac{\left[T^{\mathsf{D},(k)} \right]^{\ell-1} }{H([\bm{x}]^{\ell-1})}} 
    \left(\frac{T^{\mathsf{L},(k)}  H([\bm{x}]^{\ell-1})}{\left[T^{\mathsf{L},(k)} \right]^{\ell-1}}\right)^{\frac{\left[T^{\mathsf{L},(k)} \right]^{\ell-1} }{H([\bm{x}]^{\ell-1})}} \\[-.4em]&~~~~~~
    \left(\frac{T^{\mathsf{M},(k)}  H([\bm{x}]^{\ell-1})}{\left[T^{\mathsf{M},(k)} \right]^{\ell-1}}\right)^{\frac{\left[T^{\mathsf{M},(k)} \right]^{\ell-1} }{H([\bm{x}]^{\ell-1})}} 
    \left(\frac{T^{\mathsf{U},(k)}  H([\bm{x}]^{\ell-1})}{\left[T^{\mathsf{U},(k)} \right]^{\ell-1}}\right)^{\frac{\left[T^{\mathsf{U},(k)} \right]^{\ell-1} }{H([\bm{x}]^{\ell-1})}} 
    \left(\frac{\Omega^{(k)} H([\bm{x}]^{\ell-1})}{\left[\Omega^{(k)} \right]^{\ell-1}}\right)^{\frac{\left[\Omega^{(k)} \right]^{\ell-1} }{H([\bm{x}]^{\ell-1})}}
    \end{aligned}
    \vspace{-1mm}
\end{equation}
 \vspace{-1mm}
% \begin{equation}\label{approx:G}
%     G(\bm{x})=\sum_{m\in \mathcal{N}}\varrho^{(k)}_{n,m} \geq \widehat{G}(\bm{x};\ell)\triangleq
%     \prod_{m\in \mathcal{N}} \left(\frac{\varrho^{(k)}_{n,m} G([\bm{x}]^{\ell-1})}{\left[\varrho^{(k)}_{n,m} \right]^{\ell-1}}\right)^{\frac{\left[\varrho^{(k)}_{n,m} \right]^{\ell-1} }{G([\bm{x}]^{\ell-1})}} 
% \end{equation}
% \vspace{-3mm}
% \begin{equation}\label{approx:j}
%     J(\bm{x})=\sum_{m\in \mathcal{N}}\varphi^{(k)}_{n,m} \geq \widehat{J}(\bm{x};\ell)\triangleq
%     \prod_{m\in \mathcal{N}} \left(\frac{\varphi^{(k)}_{n,m} J([\bm{x}]^{\ell-1})}{\left[\varphi^{(k)}_{n,m} \right]^{\ell-1}}\right)^{\frac{\left[\varphi^{(k)}_{n,m} \right]^{\ell-1} }{J([\bm{x}]^{\ell-1})}} 
% \end{equation}
\hrulefill
\vspace{-5mm}
\begin{multicols}{2}
\begin{equation}\label{approx:G}
    \hspace{-12mm} G(\bm{x})=\sum_{m\in \mathcal{N}}\varrho^{(k)}_{n,m} \geq \widehat{G}(\bm{x};\ell)\triangleq
    \prod_{m\in \mathcal{N}} \left(\frac{\varrho^{(k)}_{n,m} G([\bm{x}]^{\ell-1})}{\left[\varrho^{(k)}_{n,m} \right]^{\ell-1}}\right)^{\frac{\left[\varrho^{(k)}_{n,m} \right]^{\ell-1} }{G([\bm{x}]^{\ell-1})}} 
    \hspace{-8mm}
\end{equation}\;
% \vspace{-3mm}
% \vspace{2.5mm}
\begin{equation}\label{approx:j}
\hspace{-12mm}
    J(\bm{x})=\sum_{m\in \mathcal{N}}\varphi^{(k)}_{n,m} \geq \widehat{J}(\bm{x};\ell)\triangleq
    \prod_{m\in \mathcal{N}} \left(\frac{\varphi^{(k)}_{n,m} J([\bm{x}]^{\ell-1})}{\left[\varphi^{(k)}_{n,m} \right]^{\ell-1}}\right)^{\frac{\left[\varphi^{(k)}_{n,m} \right]^{\ell-1} }{J([\bm{x}]^{\ell-1})}} \hspace{-8mm}
\end{equation}
\end{multicols}
\vspace{-3.5mm}
\hrulefill
\vspace{-1mm}

\begin{minipage}{.49\linewidth}
\begin{equation}\label{approx:i}
\hspace{-1mm}
      I(\bm{x})=\hspace{-.8mm}\sum_{m\in\mathcal{N}} \hspace{-.8mm}\varrho^{(k)}_{m,n} D^{(k)}_m \geq \widehat{I}(\bm{x};\ell)\triangleq \hspace{-.8mm} \prod_{m\in\mathcal{N}}\hspace{-.8mm}\left(\frac{\varrho^{(k)}_{m,n} I([\bm{x}]^{\ell-1})}{[\varrho^{(k)}_{m,n}]^{\ell-1}}\right)^{\hspace{-.8mm}\frac{D^{(k)}_m[\varrho^{(k)}_{m,n} ]^{\ell-1}}{I([\bm{x}]^{\ell-1})}}
\hspace{-5mm}
 \end{equation}
 \end{minipage}
 \;\;
 \begin{minipage}{.49\linewidth}
%   \vspace{-2mm}
% \hspace{31mm}
\begin{equation}\label{approx:R}
       \hspace{3mm} R(\bm{x})= \sum_{n\in \mathcal{N}}e^{(k)}_{n}  \geq \widehat{R}(\bm{x};\ell)\triangleq \prod_{n\in\mathcal{N}}\left(\frac{e^{(k)}_{n}  R([\bm{x}]^{\ell-1})    }{[e^{(k)}_{n}]^{[\ell-1]}} \right)^{\frac{[e^{(k)}_{n} ]^{[\ell-1]}}{R([\bm{x}]^{\ell-1})   }}
% \vspace{-12mm}
     \end{equation}
     \end{minipage}
    %  \end{multicols}
    %  \vspace{-5mm}
    
    \vspace{1mm}
    \hrulefill
    \vspace{-2mm}
    
  \begin{equation}\label{approx:V}
        V(\bm{x})= \sum_{n\in \mathcal{N}}\widehat{{D}}^{(k)}_{n}e^{(k)}_{n}  \geq \widehat{V}(\bm{x};\ell)
        \triangleq 
        \prod_{n\in\mathcal{N}} \prod_{m \in \mathcal{N}}
        \left(\frac{ \varrho_{m,n}^{(k)} e^{(k)}_{n} {V}([\bm{x}]^{\ell-1})} 
        {\left[ \varrho_{m,n}^{(k)} e^{(k)}_{n} \right]^{[\ell-1]}} \right)
        ^{\frac{{D}_m^{(k)}\left[   \varrho_{m,n}^{(k)} e^{(k)}_{n} \right]^{[\ell-1]}}{{V}([\bm{x}]^{\ell-1}) }}
     \end{equation}
        }
        
        \vspace{-1mm}
        \hrulefill
        % {\footnotesize
% \begin{multicols}{2}
% \begin{equation}\label{approx:G}
%     G(\bm{x})=\sum_{m\in \mathcal{N}}\varrho^{(k)}_{n,m} \geq \widehat{G}(\bm{x};\ell)\triangleq
%     \prod_{m\in \mathcal{N}} \left(\frac{\varrho^{(k)}_{n,m} G([\bm{x}]^{\ell-1})}{\left[\varrho^{(k)}_{n,m} \right]^{\ell-1}}\right)^{\frac{\left[\varrho^{(k)}_{n,m} \right]^{\ell-1} }{G([\bm{x}]^{\ell-1})}} 
%     \hspace{-8mm}
% \end{equation}\;
% % \vspace{-3mm}
% \vspace{2.5mm}
% \begin{equation}\label{approx:j}
%     J(\bm{x})=\sum_{m\in \mathcal{N}}\varphi^{(k)}_{n,m} \geq \widehat{J}(\bm{x};\ell)\triangleq
%     \prod_{m\in \mathcal{N}} \left(\frac{\varphi^{(k)}_{n,m} J([\bm{x}]^{\ell-1})}{\left[\varphi^{(k)}_{n,m} \right]^{\ell-1}}\right)^{\frac{\left[\varphi^{(k)}_{n,m} \right]^{\ell-1} }{J([\bm{x}]^{\ell-1})}} \hspace{-8mm}
% \end{equation}
% \end{multicols}
% }
\hrulefill
\end{minipage}
\vspace{-2mm}
\end{table*} 

A related problem format in literature to $\bm{\mathcal{P}}$ is \textit{geometric programming} (GP), to understand which the knowledge of monomial and posynomials is necessary.
\vspace{-2mm}
 \begin{definition}
         A \textbf{{monomial}} is defined as a function {\small $f: \mathbb{R}^n_{++}\rightarrow \mathbb{R}$}:
         {\small$f(\bm{y})=z y_1^{\alpha_1} y_2^{\alpha_2} \cdots y_n ^{\alpha_n}$}, where {\small$z\geq 0$}, {\small$\bm{y}=[y_1,\cdots,y_n]$}, and {\small$\alpha_j\in \mathbb{R}$}, $\forall j$. Further, a \textbf{posynomial} $g$ is defined as a sum of monomials: {\small$g(\bm{y})\hspace{-.5mm}= \hspace{-.5mm}\sum_{m=1}^{M} z_m y_1^{\alpha^{(1)}_m} y_2^{\alpha^{(2)}_m} \cdots y_n ^{\alpha^{(n)}_m}\hspace{-1mm}$, $z_m\hspace{-1mm}\geq 0,\forall m$}.
\end{definition}
\vspace{-2mm}
A GP in its standard form admits a posynomials objective function subject to inequality constraints on posynomials and equality constraints on monomials (see Appendix~\ref{sec:GPtransConv}). With a logarithmic change of variables, GP in its standard form can be transformed into a convex optimization that can be efficiently solved using well-known software, e.g., CVXPY~\cite{diamond2016cvxpy}. Nevertheless, $\bm{\mathcal{P}}$ does not admit the format of GP. In particular, the bound in~\eqref{eq:gen_conv_neyman_main} appearing in the objective function consists of terms with negative sign that violate the definition of posynomials. Furthermore, constraints~\eqref{prob:Tml},~\eqref{prob:varrho} and~\eqref{prob:varphi} are equalities on  posynomials, which GP does not admit. To tackle these violating terms, we first consider the constraints in the form of equality on posynomials, and use the method of penalty functions and auxiliary variables~\cite{SignomialGlobal}. To this end, we consider each equality on a posynomial in the format of $g(\bm{x})=c$ as two inequality constraints: (c-i) $g(\bm{x})\leq c$, and (c-ii)~${1}/({A g(\bm{x})})\leq c$, where $A\geq 1$ is an auxiliary variable, which will later be forced $A\downarrow 1$ via being added to the objective function with a penalty coefficient.\footnote{$A\downarrow 1$ is an equivalent  representation of $A\rightarrow 1^+$.} Inequality (c-i) is an inequality on a posynomial, which GP admits. However, (c-ii) is an inequality on a non-posynomial (division of a monomial/posynomial by a posynomial is not a posynomial). One way to transform (c-ii) to an inequality on a posynomial is to approximate its denominator by a monomial (division of a posynomial by a monomial is a posynomial). To this end, we exploit the arithmetic-geometric mean inequality, which upper bounds a posynomial with a larger monomial in value.

\vspace{-2mm}
  \begin{lemma}[\textbf{Arithmetic-geometric mean inequality}~\cite{duffin1972reversed}]\label{Lemma:ArethmaticGeometric}
         Consider a posynomial function $g(\bm{y})=\sum_{i=1}^{i'} u_i(\bm{y})$, where $u_i(\bm{y})$ is a monomial, $\forall i$. The following inequality holds:
         \vspace{-1mm}
         \begin{equation}\label{eq:approxPosMonMain}
             g(\bm{y})\geq \hat{g}(\bm{y})\triangleq \prod_{i=1}^{i'}\left( {u_i(\bm{y})}/{\alpha_i(\bm{z})}\right)^{\alpha_i(\bm{z})},
             \vspace{-1mm}
         \end{equation}
         where $\alpha_i(\bm{z})=u_i(\bm{z})/g(\bm{z})$, $\forall i$, and $\bm{z}>0$ is a fixed point.
         \end{lemma}
         
         \vspace{-1mm}
In~Appendix~\ref{app:optTransform}, we explain all the steps taken to solve the optimization problem $\bm{\mathcal{P}}$. Due to space limitations, we provide a high level discussion in the following.

  Our technique to solve $\bm{\mathcal{P}}$ is an iterative approach, where at each iteration $\ell$, after using the aforementioned method of penalty functions based on (c-i) and (c-ii). The corresponding posynomials in (c-ii) for~\eqref{prob:Tml},~\eqref{prob:varrho} and~\eqref{prob:varphi} are approximated using~\eqref{approx:H},~\eqref{approx:G}, and~\eqref{approx:j}, respectively. Furthermore, since $\widehat{D}^{(k)}_n$, $e^{(k)}_{\mathsf{sum}}$ and $e^{(k)}_{\mathsf{avg}}$ appear in multiple places in~\eqref{eq:gen_conv_neyman_main}, for tractability, we treat them as optimization variables and add the following constraints to $\bm{\mathcal{P}}$:
  $(\widehat{D}^{(k)}_n)^{-1}\sum_{m\in\mathcal{N}} \varrho_{m,n}^{(k)} D_m^{(k)}=1, n\in\mathcal{N}$,
     ${\sum_{n\in\mathcal{N}} e_n^{(k)}}/e^{(k)}_{\mathsf{sum}}=1, \sum_{n\in \mathcal{N}}{\widehat{{D}}^{(k)}_{n}e^{(k)}_{n}}/(e^{(k)}_{\mathsf{avg}}{{D}^{(k)} })=1$, which are all equality constraint on posynomials. We thus use the method of penalty functions with approximations given in~\eqref{approx:i},~\eqref{approx:R} and~\eqref{approx:V} to transform them. It is easy to verify that~\eqref{approx:H}-\eqref{approx:V} are in fact the best local monomial approximations to their corresponding posynomials
near fixed point $\bm{x}^{[\ell-1]}$ in terms of the first-order Taylor approximation (vector $\bm{x}$ encapsulates all the optimization variables in all the expressions). 

\begin{algorithm}[t]
 	\caption{Optimization solver for problem~$\bm{\mathcal{P}}$}\label{alg:cent}
 	\SetKwFunction{Union}{Union}\SetKwFunction{FindCompress}{FindCompress}
 	\SetKwInOut{Input}{input}\SetKwInOut{Output}{output}
 	 	{\footnotesize
 	\Input{Convergence criterion, model training duration $T^{\mathsf{ML}}$.}
 	\For{$K=1$ to $T^{\mathsf{ML}}$}{
 	 Set the iteration count $\ell=0$.\\
 	 Choose a feasible point $\bm{x}^{[0]}$.\\
 	 Obtain the monomial approximations~\eqref{approx:H}-\eqref{approx:V},\eqref{eq:posObj} given $\bm{x}^{[\ell]}$.\label{midAlg1}\\
 	 Replace the results in the approximation of Problem~$\bm{\mathcal{P}}$ (i.e., $\bm{\mathcal{P}'}$ given by~\eqref{prob:INITIAL}-\eqref{prob:FINAL} in Appendix~\ref{app:optTransform}).\\
 	 With logarithmic change of variables, transform the resulting GP problem to a convex problem (as described in Appendix~\ref{sec:GPtransConv}).\\
 	 $\ell=\ell+1$\\
 	 Obtain the solution of the convex problem using  current art solvers (e.g., CVXPY~\cite{diamond2016cvxpy}) to determine  $\bm{x}^{[\ell]}$.\label{Alg:Gpconvexste}\\
 	 %$\norm{\bm{x}^{[m]}-\bm{x}^{[m-1]}}> \sigma$ 
 	 \If{two consecutive solutions $\bm{x}^{[\ell-1]}$ and $\bm{x}^{[\ell]}$ do not meet the specified convergence criterion}{
 	\textrm{Go to line~\ref{midAlg1} and redo the steps using $\bm{x}^{[\ell]}$.}\\\Else{Set the solution of the iterations as $\bm{x}_K^{\star}=\bm{x}^{[\ell]}$.\label{Alg:point2}\\
 	}}
 	  }
 	  $\bm{x}^\star= \min_{\{\bm{x}^\star_K \}_{1\leq K\leq T^{\mathsf{ML}}}}\{\textrm{Objective of $\bm{\mathcal{P}}$ evaluated at $\bm{x}^\star_K$}\}$
 	  }
 	  \vspace{-.1mm}
  \end{algorithm}
  
   \begin{table*}[tbp]
   % \vspace{-7mm}
\begin{minipage}{0.99\textwidth}
{\footnotesize
 \begin{align}
    W(\bm{x})&= \chi^{(k)}+ \frac{1}{(1-\Lambda^{(k)})}\frac{8\beta^2\Theta^2 \alpha^2 N}{K^2} (e^{(k)}_{\mathsf{sum}})^{-1} \sum_{n\in \mathcal{N}}\frac{e_n^{(k)}}{{D}^{(k)}} \widehat{Z}_n^{(k)}+\frac{4 \Theta^2 {\beta \alpha \sqrt{N}}}{K\sqrt{K}(1-\Lambda^{(k)})} 
     \left(e^{(k)}_{\mathsf{avg}}\right)\left(e^{(k)}_{\mathsf{sum}}\right)^{-1/2}\sum_{n\in \mathcal{N}} \frac{\widehat{D}_n^{(k)}}{\left({D}^{(k)}\right)^2 e^{(k)}_n} \widehat{Z}_n^{(k)}
    \nonumber \\[-.2em]&
     \geq  \widehat{W}(\bm{x};\ell)\triangleq \left(\frac{\chi^{(k)}  {W}([\bm{x}]^{\ell-1})    }{[\chi^{(k)}]^{\ell-1} } \right)^{\frac{[\chi^{(k)}]^{\ell-1}}{{W}([\bm{x}]^{\ell-1})   }}\times \prod_{n\in\mathcal{N}}  \prod_{q=1}^{2} \left(\frac{\delta_{q}(\bm{x},n)  {W}([\bm{x}]^{\ell-1})    }{\delta_{q}([\bm{x}]^{\ell-1},n) } \right)^{\frac{\delta_{q}([\bm{x}]^{\ell-1},n)}{{W}([\bm{x}]^{\ell-1})   }},\label{eq:posObj}
     \end{align}
     }\vspace{-4mm}
     {\footnotesize
     \begin{equation}
      \delta_{1}(\bm{x},n)= \frac{1 } {(1-\Lambda^{(k)})}\frac{8\beta^2\Theta^2 \alpha^2 N }{K^2} \frac{e_n^{(k)}}{e^{(k)}_{\mathsf{sum}}{D}^{(k)} } \widehat{Z}_n^{(k)},~~
     \delta_{2}(\bm{x},n)= \frac{4 \Theta^2 {\beta \alpha \sqrt{N}}}{K\sqrt{K}(1-\Lambda^{(k)})}
 \frac{e^{(k)}_{\mathsf{avg}} \widehat{D}_n^{(k)}}{  \sqrt{e^{(k)}_{\mathsf{sum}}}\left({D}^{(k)}\right)^2 e^{(k)}_n} \widehat{Z}_n^{(k)},~\widehat{Z}^{(k)}_n=\sum_{j=1}^{S^{(k)}_n}{s}_{n,j} {\left(\widetilde{\sigma}^{(k)}_{n,j}\right)^2}\nonumber 
 \vspace{-2mm}
     \end{equation}
     }
            \hrulefill
     \end{minipage}
     \vspace{-2mm}
     \end{table*}
   We next tackle the complex term~$\Xi$ in the objective function of~$\bm{\mathcal{P}}$, i.e.,~\eqref{eq:gen_conv_neyman_main}. In~\eqref{eq:gen_conv_neyman_main}, we first upper bound  $F(\mathbf{w}^{(0)})-F^{{(K)}^\star}$ (inside the term in the first line) with $F(\mathbf{w}^{(0)})$ and $e^{(k)}_n-1$ with $e^{(k)}_n$ (inside the term in the second line), and $e^{(k)}_{\mathsf{max}}-1$ with $e^{(k)}_{\mathsf{max}}$ (inside the term in the third line) since these do not impose a notable difference in the optimization solution.  Also, to have a tractable solution, we assume that (i) the relative size of the strata to the size of the local dataset is upper bounded throughout the learning period, and let $s_{n,j}\leq 1$ denote the upper bound of the relative size of  stratum $\mathcal{S}^{(k)}_{n,j}$ to the local dataset, i.e., ${{S}^{(k)}_{n,j}}/{\widehat{D}^{(k)}_{n}}\leq s_{n,j}, ~\forall k$, and (ii) the optimizer optimizes the ML bound for an upper bound on the variance of the local strata (i.e., $\widetilde{\sigma}_{n,j}^{(k)}$, $\forall k,n,j$, is upper bounded via historical data for the optimizer); however, during the PSL model training each node uses Lemma~\ref{lemma:trackMeanvar} to track the variance of its strata, based on which it conducts non-uniform data sampling according to the rule obtained in Proposition~\ref{prop:neyman}. These upper bounds and assumptions are inherently assumed in Proposition~\ref{propKKT2}. 
     Note that in~\eqref{eq:gen_conv_neyman_main}, all the terms contain the summation over global aggregation index $k$, i.e., $\sum_{k=1}^{K}$, expect the first term (term $(a)$), which in turn can be upper bounded as {\small $\sum_{k=1}^{K} \frac{2\sqrt{\widehat{e}_{\mathsf{max}}} {F}(\mathbf{w}^{(0)}) }{\alpha \overline{e}_{\mathsf{min}}K\sqrt{K}(1-\Lambda_{\mathsf{max}})}$}. We thus consider $\Xi$ in~\eqref{eq:gen_conv_neyman_main} as {\small $\Xi=\sum_{k=1}^{K}\sigma^{(k)}_+(\bm{x})-\sigma^{(k)}_{-,1}(\bm{x})-\sigma^{(k)}_{-,2}(\bm{x})$}, where {\small $\sigma^{(k)}_+(\bm{x})$} contains all the terms with positive coefficients, while {\small $\sigma^{(k)}_{-,1}(\bm{x})$, $\sigma^{(k)}_{-,2}(\bm{x})$} are two terms with negative sign (terms $(b)$ and $(c)$ and their coefficients). 
   We next replace the term {\small $\gamma \Xi $} in the objective function of~$\bm{\mathcal{P}}$ with {\small $\gamma \sum_{k=0}^{K-1} \chi^{(k)}$}, which auxiliary variable {\small $\chi^{(k)}$} is the upperbound of summand in {\small $\Xi$}, which we  add it to the constraints as  {\small $\sigma^{(k)}_+(\bm{x})-\sigma^{(k)}_{-,1}(\bm{x})-\sigma^{(k)}_{-,2}(\bm{x}) \leq \chi^{(k)}$, $\forall k$}.
%   \footnote{Note that in~\eqref{eq:gen_conv_neyman_main}, all the terms all contain the summation over global aggregation index $k$, i.e., $\sum_{k=1}^{K}$, expect the first term, which in turn can be upper bounded as $\sum_{k=1}^{K} \frac{2\sqrt{\widehat{e}_{\mathsf{max}}} {F}(\mathbf{w}^{(0)}) }{\alpha \overline{e}_{\mathsf{min}}K\sqrt{K}(1-\Lambda_{\mathsf{max}})}$.} 
%   In this expression, $\sigma^{(k)}_+(\bm{x})$ contains all the terms in~\eqref{eq:gen_conv_neyman_main}, which appear with positive sign, while $\sigma^{(k)}_{-,1}(\bm{x})$, $\sigma^{(k)}_{-,2}(\bm{x})$ are two terms with negative sign (two terms with $\widehat{Z}^{(k)}_n\triangleq\sum_{j=1}^{S^{(k)}_n}{S}^{(k)}_{n,j} {\left(\widetilde{\sigma}^{(k)}_{n,j}\right)^2}$ and their coefficients).
   Now we focus on this constraint, which can be written as follows:
   \vspace{-1.5mm}
     \begin{equation}\label{eq:finalObjFracMain}
  {\sigma^{(k)}_+(\bm{x})}\Big/\left({\chi^{(k)}+\sigma^{(k)}_{-,1}(\bm{x})+\sigma^{(k)}_{-,2}(\bm{x}) }\right) \leq 1.
  \vspace{-1.75mm}
     \end{equation}
 Considering the fraction in~\eqref{eq:finalObjFracMain}, all the terms encapsulated in $\sigma^{(k)}_+(\bm{x})$ are posynomials, making its numerator a posynomial. However, its denominator is also a posynomial, making it a non-posynomial fraction.  We thus focus on the approximation of the denominator with a monomial, for which we exploit Lemma~\ref{Lemma:ArethmaticGeometric} (see~Appendix~\ref{app:optTransform} for the detailed steps), the result of which is given by~\eqref{eq:posObj}. After the conducted approximations, we obtain a problem in which the objective function is a posynomial and all the constraints are inequalities on posynomials admitting the GP format (see the problem in~\eqref{prob:INITIAL}-\eqref{prob:FINAL} in Appendix~\ref{app:optTransform}). It is easy to verify that with a logarithmic change of the optimization variables, the problem then becomes a convex optimization, which is easy to solve using existing software. We provide the pseudo-code of our optimization solver in Algorithm~\ref{alg:cent}. Our optimization solver, for each given $K$, optimizes over the variable set {\small $\big\{\mathbf{e}^{(k)}\hspace{-.5mm},\mathbf{f}^{(k)},\mathbf{B}^{(k)}\hspace{-.5mm},\bm{\varrho}^{(k)}\hspace{-.5mm}, \bm{\varphi}^{(k)}\hspace{-.5mm},T^{\mathsf{D},(k)}\hspace{-.5mm},T^{\mathsf{L},(k)}\hspace{-.5mm},T^{\mathsf{M},(k)}\hspace{-.5mm},T^{\mathsf{U},(k)}\hspace{-.5mm},\Omega^{(k)}\big\}_{k=1}^{K}$}, which leads to a space complexity of $\mathcal{O}(3|\mathcal{N}|\times K + 2|\mathcal{N}|^2\times K + 4\times K)$. Specifically, to solve the convex optimization problem obtained in Alg. 1, gradient descent using CVXPY~\cite{diamond2016cvxpy} is conducted over this set of variables. We next show the convergence guarantees of our solver.

 \begin{figure}[t]
\vspace{-.5mm}
\centering
\includegraphics[width=.40\textwidth]{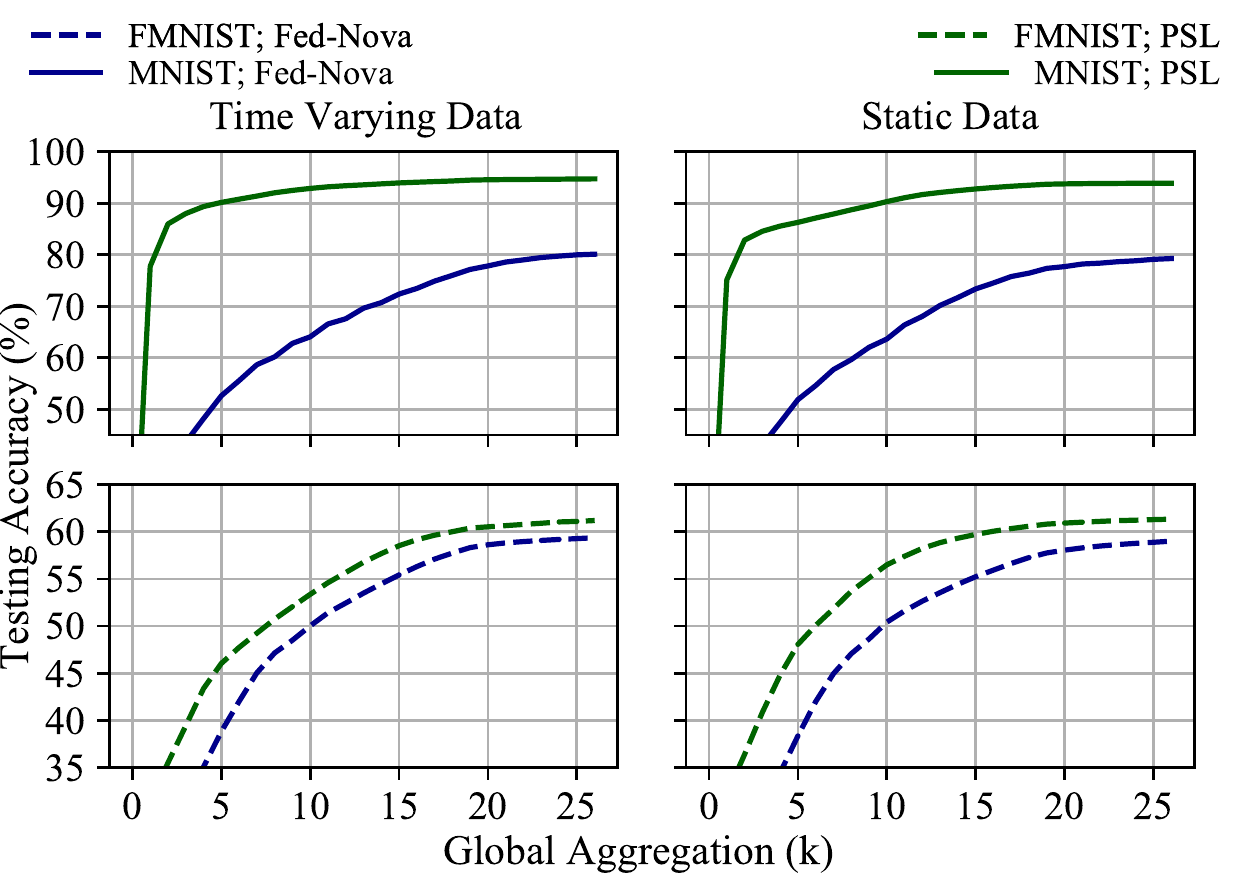}
\vspace{-1.8mm}
\caption{Accuracy obtained using PSL vs. the baseline method (Fed-Nova~\cite{wang2020tackling}). Left subplot: time varying local datasets across the global aggregations. Right subplot: static local datasets across the global aggregations. The results are obtained using moving average with window size of $10$.}
\label{fig:mvavg_comparison}
\vspace{-.1mm}
\end{figure}

%   \begin{fact}\label{fact:2}
%         At the optimal point of problem~$\bm{\mathcal{P}'}$, the auxiliary variables will be forced to take the following values: $A^{\mathsf{MO}}=A^{\mathsf{DO}}=A^{\mathsf{ML}}=1$.
%         \end{fact}
    
        % \begin{fact}\label{fact:1}
        % Assuming $A^{(i)}_{k}=1$, and $B_i=1$, $\forall k,i$, the solution of~$\bm{\mathcal{P}'}$ coincides with the solution of $\bm{\mathcal{P}}$.
        % \end{fact}
        \vspace{-2mm}
        \begin{proposition}[Convergence of the Optimization Solver]\label{propKKT2}
        For each $K$, Algorithm~\ref{alg:cent} generates a sequence of
 solutions for $\bm{\mathcal{P}}'$ using the approximations~\eqref{approx:H}-\eqref{approx:V},\eqref{eq:posObj} that converge to $\bm{x}^\star_K$
satisfying the Karush–Kuhn–Tucker (KKT) conditions of $\bm{\mathcal{P}}$.
        \end{proposition}
         \vspace{-4.2mm}
      \begin{proof}
     It can be shown that~$\bm{\mathcal{P}'}$ (given by~\eqref{prob:INITIAL}-\eqref{prob:FINAL}) is an \textit{inner approximation}~\cite{GeneralInnerApp} of~$\bm{\mathcal{P}}$. Thus, it is sufficient to examine the three characteristics mentioned in~\cite{GeneralInnerApp} for Algorithm~\ref{alg:cent} to prove the convergence, which can be done using the methods in~\cite{SignomialGlobal,wang2021uav} and omitted for brevity.
      \end{proof}
      \vspace{-3mm}
    %   As explained earlier, Algorithm~\ref{alg:cent} solves the problem for a given $K$. Assuming each global aggregation takes at least one second, Since $1\leq K\leq T^{\mathsf{ML}}$ and thus the overall solution to $\bm{\mathcal{P}}$ via iterating over the possible range of $K$.

\vspace{-3mm}
\section{Simulation Results}
\noindent We next evaluate the effectiveness of PSL.
% and thereafter perform a detailed ablation study into the characteristics of our proposed network-aware PSL.
We incorporate the effect of fading in {\small $h^{(k)}_{n}$},  {\small$h^{(k)}_{n,m}$} (in~\eqref{rateU} and~\eqref{rateD}). For the uplink channel {\small
$
     h^{(k)}_{n}= \sqrt{\beta_{n}^{(k)}} u_{n}^{(k)},
$}
 where {\small$ u_{n}^{(k)} \sim \mathcal{CN}(0,1)$} captures Rayleigh fading, and {\small
$
    \beta_{n}^{(k)} = \beta_0 - 10\widetilde{\alpha}\log_{10}(d^{(k)}_{n}/d_0)
$}~\cite{tse2005fundamentals}.
Here, {\small$\beta_0=-30$}dB, $d_0=1$m, {\small$\widetilde{\alpha}=3$}, and $d^{(k)}_{n}$ is the instantaneous distance between node $n$ and the BS. We use the same formula to describe the D2D channels but choose $\widetilde{\alpha}=3.2$ since D2D are more prone to excessive loss. Channels are realized with coherence time of $50$ ms. The devices are randomly placed in a circle area with radius of $25$m with a BS in the center. Our simulations were implemented using Pytorch~\cite{paszke2019pytorch} and run on three Nvidia Tesla V100 GPUs with 32 GB VRAM each. 
We used CVXPY~\cite{diamond2016cvxpy} to obtain the solutions of the convex problems obtained through Alg.~\ref{alg:cent}.
\vspace{-5mm}
\subsection{Dynamic Model Training Under PSL: Proof-of-Concept}
\vspace{-0.5mm}
 We compare the ML performance of PSL against the baseline method Fed-Nova~\cite{wang2020tackling}, a state-of-the-art FedL method that also accounts for varying SGD iterations across the devices. Fed-Nova aims to obtain unbiased global models for FedL by normalizing devices' received gradients at the server with respect to their number of conducted local SGD iterations, which is shown to significantly  outperform existing FedL methods including FedAvg, FedProx, and VRLSGD~\cite{wang2020tackling}.
We consider classification tasks over MNIST~\cite{MNIST} and Fashion-MNIST \cite{xiao2017} datasets for a network of $10$ devices in Fig.~\ref{fig:mvavg_comparison} and Table~\ref{tab:poc_energy_time}.\footnote{The results in Fig.~\ref{fig:mvavg_comparison} and Table~\ref{tab:poc_energy_time} are the averaged results obtained from $10$ Monte-Carlo iterations of independent network realizations.}
Both 
datasets consist of $60$K training samples and $10$k testing samples, where each datum belongs to one of $10$ labels. 
We consider both time-varying and static device datasets. In the dynamic case, devices obtain a \textit{new} dataset of size drawn from a normal distribution $\mathcal{N}(1000,125)$ after each global aggregation. In the static case, devices use a fixed dataset with size  drawn from $\mathcal{N}(1000,125)$. Device datasets are only composed of data from $3$ distinct labels. To have a fair comparison, we isolate the ML performance and consider both methods with \textit{no data offloading}. The maximum and minimum number of local SGDs are considered $25$ and $1$ respectively.
% , with $10\%$ of the devices at each extreme, while the rest are uniformly distributed. 
Fig.~\ref{fig:mvavg_comparison} shows that PSL outperforms the baseline method in both cases owed to its local data management and non-uniform data sampling. We quantify the corresponding resource savings of PSL in Table~\ref{tab:poc_energy_time}. 
 The jumps in energy consumption when moving from $50$\% to $60$\% accuracy on FMNIST and $70$\% to $80$\% accuracy on MNIST are due to the natural saturation in training improvement of the federated ML methods upon reaching higher accuracies. 
For example, to improve from $40$\% to $50$\% accuracy on FMNIST, PSL
uses $3$ additional aggregations while Fed-Nova requires $5$ additional aggregations. Meanwhile, to
improve from $50$\% to $60$\% classification accuracy on FMNIST, PSL needs $10$ further aggregations while
Fed-Nova uses $16$ further aggregations.
\vspace{-2mm}

%additionally characterize the energy and time savings of PSL 

% Fig 0 --> train for 10 global aggreagtions, fix the model, change the data at the nodes, and measure the local loss.... (we wanna show the concept drift!!!!)

% 60K --- spread 20 K among the nodes, after 10 ends, start adding new data points to the nodes, and measure the loss at the nodes for the fixed parameter...
% [x - qty of new data; y - accuracy]

% Fig 1 - proof of concept: We wanna show that our non-uniform SGD outf performs all the other state-of-the-art, MNIST, FMNIST ....,  STATIC and DYNAMIC data ....
% epochs from 1 to 25
% loss, accuracy, energy

\begin{table}[t]
 \vspace{1mm}
\caption{PSL network savings vs the baseline method (Fed-Nova~\cite{wang2020tackling}). We measure network savings to reach accuracy thresholds by the combination of energy and device acquisition time (DAT).}
% Dashes indicate accuracy thresholds that are only attained by PSL.}
\vspace{-1.5mm}
\centering
\label{tab:poc_energy_time} 
{\footnotesize
% \begin{tabularx}{0.48\textwidth}{m{4em} m{11em} m{12em}} %m{3em} m{5em}}
% \toprule[.2em]
% \textbf{Accuracy} & \textbf{Energy Consumption (J)} & \textbf{Device Utilization Time (s)} \\
\begin{tabularx}{0.48\textwidth}{m{0.5em} m{1.5em} m{5.2em} m{4.8em} m{5.2em} m{4.8em}} %*{4}{Y}} %m{2em} m{3em} m{3em} m{3em} m{3em} m{2em} m{3em} m{3em} m{3em} m{3em}}
%c *{4}{Y}}%m{5em} m{5em} m{5em} m{5em}} 
\toprule[.2em] 
% \multicolumn{5}{c}{\bf{FMNIST}} & \multicolumn{5}{c}{\bf{MNIST}} \\
% \cmidrule(lr){1-5} \cmidrule{6-10}
% \multicolumn{5}{c}{\bf{FMNIST}} \\
% \cmidrule(lr){1-5}
& \multirow{2}{*}{\textbf{Acc}} & \multicolumn{2}{c}{\bf{Static Datasets}} & \multicolumn{2}{c}{\bf{Time Varying Datasets}} \\ 
\cmidrule(lr){3-4} \cmidrule{5-6} 
& & \textbf{Energy (kJ)} & \textbf{DAT (s)} & \textbf{Energy (kJ)} & \textbf{DAT (s)} \\ 
\hline 
\multirow{3}{*}{\rotatebox[origin=c]{90}{{\parbox[c]{7.8mm}{\notsotiny \bf{FMNIST}} }}} & 40\% & 11.7 (50\%) & 730 (50\%) & 11.7 (33\%) & 730 (33\%) \\
& 50\% & 23.3 (40\%) & 1461 (40\%) & 17.5 (30\%) & 1095 (30\%) \\
& 60\% & 64.1 (42\%) & 4017 (42\%) & 58.3 (38\%) & 3652 (38\%)\\
% 55\% & & & \rule{5 mm}{0.2pt} & \rule{5 mm}{0.2pt} \\ 
\hline 
\multirow{3}{*}{\rotatebox[origin=c]{90}{{\parbox[c]{6mm}{\notsotiny \bf{MNIST}} }}} & 60\% & 35.0 (86\%) & 2191 (86\%) & 35.0 (86\%) & 2191 (86\%) \\
& 70\% & 64.1 (85\%) & 4017 (85\%) & 64.1 (85\%) & 4017 (85\%) \\
& 80\% & 140 (92\%) & 8765 (92\%) & 117 (91\%) & 7304 (91\%) \\
% & 90\% & \rule{5 mm}{0.2pt} & \rule{5 mm}{0.2pt} &  \rule{5 mm}{0.2pt} & \rule{5 mm}{0.2pt}\\
\midrule
\end{tabularx}
}
\vspace{-1.5mm}
\end{table}

\vspace{-2mm}
\subsection{Network-Aware PSL: Ablation Study}
Direct examination of $\boldsymbol{\mathcal{P}}$ is difficult due to the complexity and entanglement of the optimization variables 
% (e.g., local SGD iterations $\boldsymbol{e}$ closely affects the time for training $T^{L}$) 
and elements of the bound in~\eqref{eq:gen_conv_neyman_main}.
% (e.g., the value of $\zeta_2$ scales the machine learning convergence term of $(\boldsymbol{\mathcal{P}})$). 
As a result, we perform an ablation study -- systematically evaluating the impacts of important optimization and scaling variables in isolation -- to characterize $\boldsymbol{\mathcal{P}}$ in detail. We use the set of network characteristics in Table~\ref{tab:sim_network_params}.
 The results in Sections \ref{subsubsub:1}, \ref{subsubsub:3}, \ref{subsubsub:5}, \ref{subsubsub:7} are the averaged results over $10$ Monte-Carlo iterations of independent network realizations. Sections \ref{subsubsub:2}, \ref{subsubsub:4}, \ref{subsubsub:6} focus on showing discretized numerical values and, in order to preserve the key intuition behind the results, we show the result for a single network realization.
% first evaluate the overall performance of $(\boldsymbol{\mathcal{P}})$ followed by 

% \multicolumn{2}{c}{\bf{Ratio}} & \multicolumn{3}{c}{\bf{MNIST} (kJ)} & \multicolumn{2}{c}{\bf{Ratio}} & \multicolumn{3}{c}{\bf{MNIST} (kJ)} \\
% \cmidrule(lr){1-2} \cmidrule(lr){3-5} \cmidrule(lr){6-7} \cmidrule(lr){8-10} 
% $\centering \tau_s^{\mathsf{L}}$ & $\centering \tau_s^{\mathsf{G}}$ & \bf{{\scriptsize HFL}} & \centering \bf{{\scriptsize HNPFL}} & \bf{{\scriptsize Savings}} & $\centering \tau_s^{\mathsf{L}}$ & $\centering \tau_s^{\mathsf{G}}$ & \bf{{\scriptsize HFL}} & \centering \bf{{\scriptsize HNPFL}} & \bf{{\scriptsize Savings}}\\
% \midrule
% \centering 1 & \centering 1 & 4.64 & 2.32 & 50.0\% %4643 & 2321 & a% 
% & 1 & 1 & 4.64 & 2.32 & 50.0\% \\ %4643 & 2321 & a\\
% \centering 1 & \centering 2 & 8.13 & 4.06 & 50.1\% %6964 & 3482 & a 
% & 2 & 1 & 9.29 & 4.64 & 50.1\% \\ %6964 & 3482 & a\\
% \centering 1 & \centering 4 & 14.51 & 8.13 & 44.0\% %12188 & 6964 & a 
% & 4 & 1 & 15.10 & 8.13 & 46.2\% \\ 
% \centering 1 & \centering 8 & 19.73 & 17.99 & 8.8 \% 
% & 8 & 1 & 18.57 & 14.51 & 21.9\% \\ 
% \bottomrule

\subsubsection{Optimization Solver and Network Size}\label{subsubsub:1}
We first investigate the convergence of our optimization solver for various network sizes in Fig.~\ref{fig:test_obj_value}, where {\small $T^{\mathsf{ML}} = 1000$}s, {\small$K=2$}, and {\small$N\in\{5,10,15,20\}$}. As can be seen, larger number of devices leads to slower convergence but a better final solution since it causes (i) processing higher number of data points across the devices that leads to a better ML performance, and (ii) more efficient D2D data/model transfer opportunity.
% The second result of Fig.~\ref{fig:test_obj_value} is that, as the network increases in size, the system obtains better results. The size-performance relationship is due to larger networks having (i) more total training data, which reduces the ML bound, and (ii) more efficient d2d data and model transfer, which allow for cheaper data processing and model aggregations.
Furthermore, Fig.~\ref{fig:test_obj_value} reveals the diminishing rewards of increasing the number of devices, where an initial increase from {\small $N=5$} to {\small $N=10$} results in a notable performance improvement; however, this effect is less notable as the number of devices increase.
% : when the $N$ doubles in size from 5 to 10, the network experiences savings exceeding $5e4$. Yet, from 10 to 20 devices, the network only saves roughly $3e4$. 
This implies that engaging more devices may allow for more efficient model training but there is a point after which the network energy consumption (due to data processing and model aggregations) overshadows the  ML performance gains.

\begin{figure}[t]
% \vspace{-5mm}
\centering
\includegraphics[width=.45\textwidth]{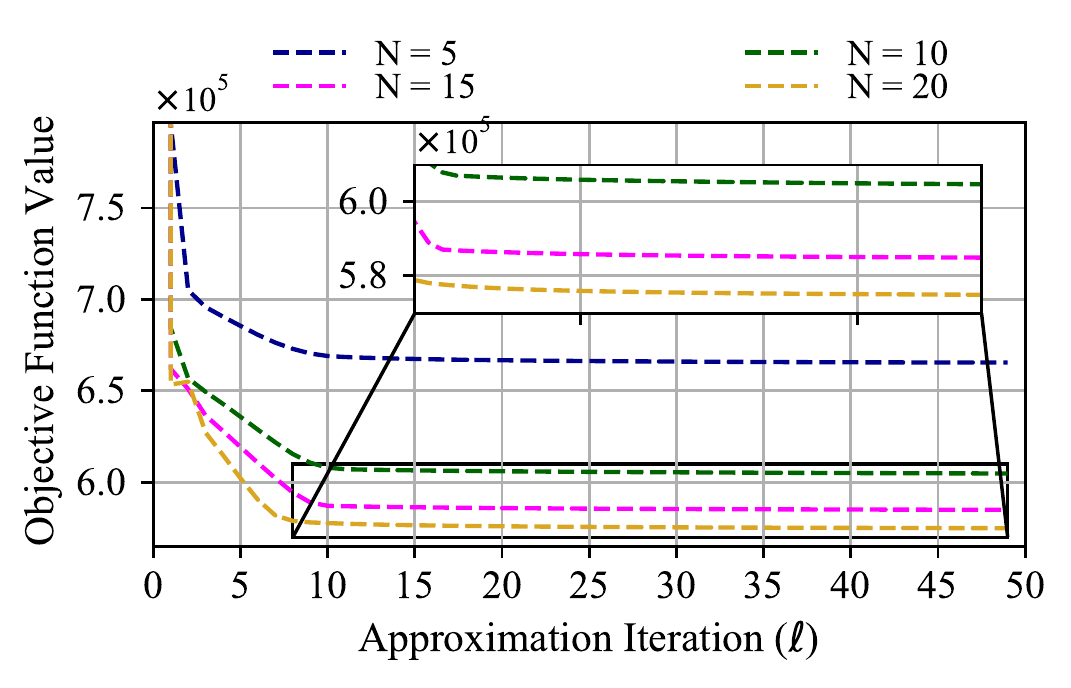}
\vspace{-3.3mm}
\caption{Convergence of the objective function of $\bm{\mathcal{P}}$ upon having varying number of devices $N$ for $T^{\mathsf{ML}}=1000s$, where the ML performance term in the objective function is prioritized over the energy and delay terms. For a fixed training period, PSL achieves a lower objective function value upon having higher number of devices due to achieving a lower ML loss value.}
\label{fig:test_obj_value}
\vspace{-1mm}
\end{figure}

\begin{figure}[t]
\vspace{-3mm}
\centering
\includegraphics[width=.48\textwidth]{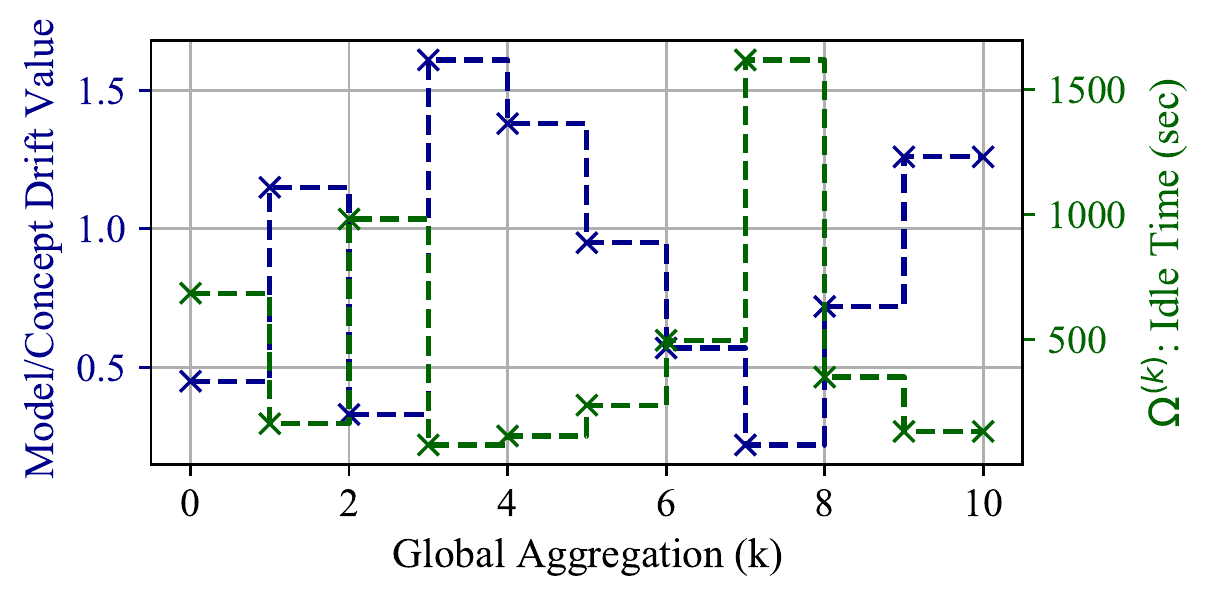}
\vspace{-3.4mm}
\caption{Demonstration of the variation of the idle time in between the consecutive global aggregations (right y-axis) for a given configuration of model/concept drift (left y-axis). The higher value of concept drift leads to smaller idle times in between the global aggregations, i.e., the global aggregations are triggered faster with higher concept drifts. }
\label{fig:MD_omega_K}
\vspace{-1mm}
\end{figure}

\begin{table}[!b]
\caption{Network characteristics used for ablation study. The experiments all use the network values herein, unless indicated otherwise.}
\vspace{-1.5mm}
\centering
\label{tab:sim_network_params} 
{\footnotesize
\begin{tabularx}{0.48\textwidth}{m{2.9em} m{5em} | m{2.5em} m{5.6em} | m{2.5em} m{5.1em}} %m{3em} m{3em}|m{3em} m{3em}}
\toprule[.2em]
\textbf{Param} & \textbf{Value} & \textbf{Param} & \textbf{Value} & \textbf{Param} & \textbf{Value} \\
% \multicolumn{2}{c}{\bf{Network Optimization Features}}  & \multicolumn{2}{c}{\bf{Data Transfer Features}} \\
\hline
$b^{\mathsf{D}}$ & $32\times 32\times 4$ & $p_n^{\mathsf{U}}$ & $250$mW & $\Delta$ & $0.1$  \\
$M \hspace{-0.5mm}\times\hspace{-0.5mm} b^{\mathsf{G}}$ & $72000$ & $p_n^{\mathsf{D}}$ & $100$mW & $\Lambda$  & $0.9$\\
$f^{\mathsf{max}}$ & $2.3$GHz & $N_0$ & $-174$dBm/Hz & $\alpha$ & $0.1$ \\
$f^{\mathsf{min}}$ & $100$KHz & $B^\mathsf{U}$ & $1$MHz & $\zeta_2$ & $1e-5$ \\
  $\alpha_n$& $2e-13$ & $B^\mathsf{D}$ & $100$kHz &  $\Theta$& $3$ \\
% &  & $\beta$ & $100$ & &  \\
%  \\
% \\
%  \\
% \\
\midrule
\end{tabularx}
}
\vspace{-2mm}
\end{table}
\subsubsection{Model/Concept Drift vs. Idle Time}\label{subsubsub:2}
We investigate the effect of the model drift {\small $\Delta^{(k)}$} on the  system idle time in Fig.~\ref{fig:MD_omega_K}, where {\small$T^{\mathsf{ML}}=5000$}s, and {\small$K=10$}.
% , and, for each $k$, model drift $\Delta^{k} \in [0.1,2]$.
The figure shows that our solution promotes rapid global aggregations when the value of model drift increases. Further, when the value of concept drift is low, the global aggregations are carried out via higher idle times: since the data variations at devices is small, they can stay in idle state to save energy and device acquisition cost.
% As model drift $\Delta^{k}$ captures the variance of new data at the $k$-th global aggregation, large model drifts require more data processing and less idle time $\Omega^{(k)}$ in order to lower the ML bound in~\eqref{eq:gen_conv_neyman_main}, and our optimization procedure is able to adapt to varying $\Delta^{(k)}$ in Fig.~\ref{fig:MD_omega_K}.

%Consequently, large model drift leads to low idle time in Fig.~\ref{fig:MD_omega_K}.

\begin{figure}[t]
\centering
\vspace{-5mm}
\includegraphics[width=.42\textwidth]{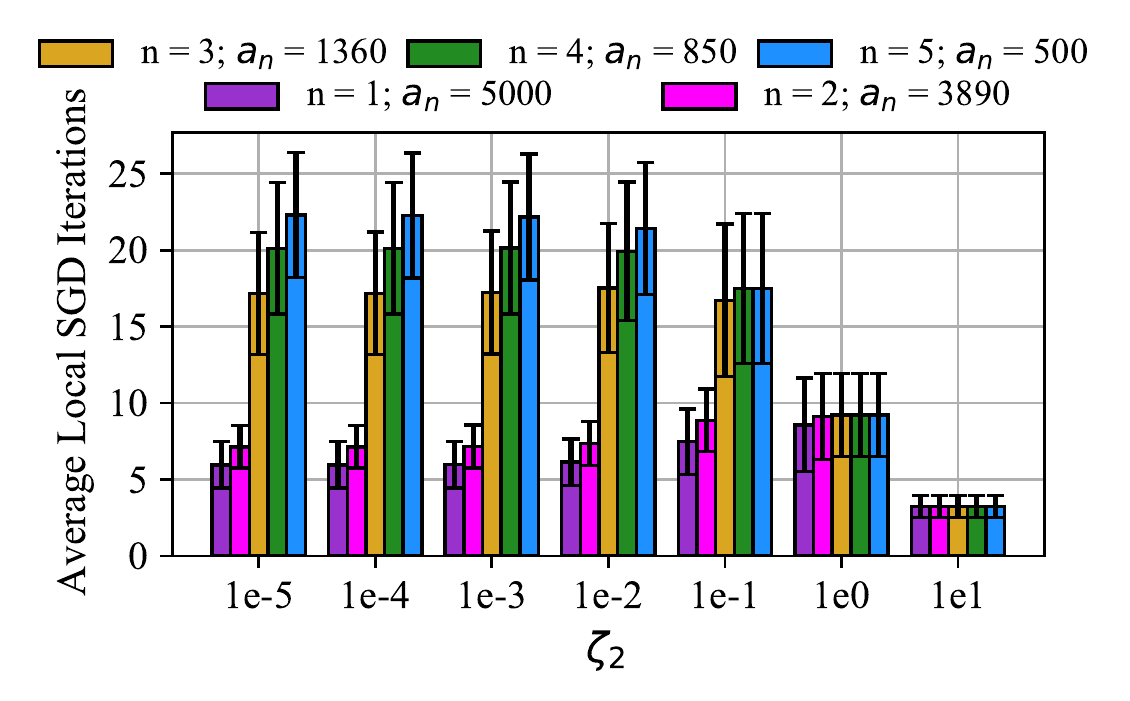}
\vspace{-3mm}
\caption{Average local SGD iterations for different devices through the model training period upon having different degree of model dissimilarity $\zeta_2$. Upon having smaller $\zeta_2$, the variance across the local number of SGD iterations across the devices is preferable, where the devices with lower processing cost have higher number of local SGD iterations. Upon increasing the model dissimilarity, both the variance and the mean of the SGD iterations across the devices is decreased. For larger values of $\zeta_2$, i.e., extreme non-iid data, all the devices conduct the same number of local SGD iterations. Each bar depicts the averaged result with standard deviations obtained over $10$ independent network realizations.}
\label{fig:zeta2_epochs}
\vspace{-2mm}
\end{figure}

\subsubsection{Model Dissimilarity vs. Local SGD Iterations}\label{subsubsub:3}
% While model drift quantifies the degree of loss/data variations between global aggregations for the devices, 
$\zeta_2$ in~\eqref{eq:gen_conv_neyman_main} quantifies data dissimilarity, where higher data dissimilarity increases the chance of local model bias with increased local SGD iterations. 
In Fig.~\ref{fig:zeta2_epochs}, we vary dissimilarity and depict the number of SGD iterations across the devices, where {\small$T^{\mathsf{ML}} = 100$}s. 
% $\zeta_2$ from $1e-5$ to $1e1$ for $T^{\mathsf{ML}} = 100$. 
Each device $n$ requires $a_n$ CPU cycles to process each datum, thus a large $a_n$ indicates a higher data processing cost. Fig.~\ref{fig:zeta2_epochs} shows that when $\zeta_2$ is small, i.e., data is homogeneous across the devices, PSL maximizes ML performance by having more local SGD iterations at devices with small $a_n$ ({\small$n\in\{3,4,5\}$)}. 
As $\zeta_2$ increases, uneven local SGD iterations can favor local models at devices with higher SGD iterations, so PSL reduces the variance of the number SGD iterations among the devices. 
% Finally, when $\zeta_2$ gets large, PSL reduces the local SGD iterations at all devices.
% (as fewer local SGD iterations counteracts large $\zeta_2$ scaling in the ML bound given by~\eqref{eq:gen_conv_neyman_main}).

% , and standardizes the SGD iterations at all devices in order to maximize the data variance during ML model training.

\begin{figure}[t]
\vspace{-.01mm}
\centering
\includegraphics[width=.48\textwidth]{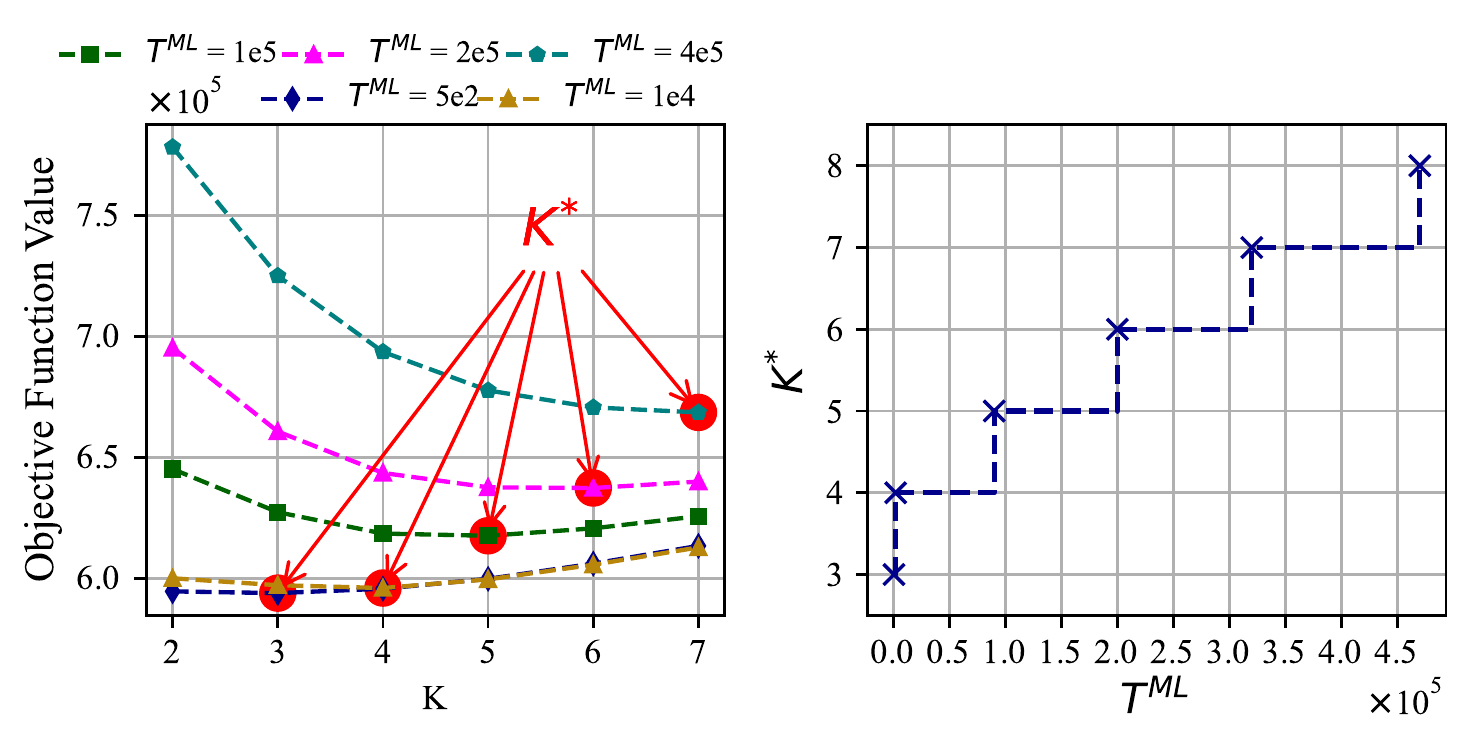}
\vspace{-3.4mm}
\caption{Optimal number of global aggregations $K^\star$ for different period of model training intervals $T^{\mathsf{ML}}$. Global aggregations trigger faster for a \textit{cold model}. During model training, the \textit{inertia} of the model increases and it becomes harder to trigger a new global aggregation once the model is \textit{warmed up}.}
\label{fig:TML_vs_K}
\vspace{-1mm}
\end{figure}

\vspace{-.1mm}
\subsubsection{Cold vs. Warmed Up Model, and Model Inertia}\label{subsubsub:4}
Model training interval {\small $T^{\mathsf{ML}}$} limits the time for all aspects of PSL (i.e., data/model transfer, local processing, and uplink transmissions). In Fig.~\ref{fig:TML_vs_K}, we depict the optimal number of global aggregations for a wide range of {\small$T^{\mathsf{ML}}$} values.
% and show that $T^{\mathsf{ML}}$ has a diminishing returns effect, which we call aggregation inertia, on the optimal number of aggregations $K^{*}$.
The left subplot shows the objective function value as a function of the number of global aggregations {\small$K$}. It can be seen that increasing the number of global aggregations {\small$K$} is not suitable for all scenarios. Large {\small$K$} implies that the system must spend more \textit{time} on the data and model transmission processes and therefore has less time to run model training. As a result, for smaller {\small$T^{\mathsf{ML}}$}, {\small$K^\star$} is smaller, and {\small$K^\star$} increases as {\small$T^{\mathsf{ML}}$} increases.
% , and, as the system has more time available (i.e., $T^{\mathsf{ML}}$ increases), $K^*$ correspondingly increases. 
In the right subplot of  Fig.~\ref{fig:TML_vs_K}, we analyze the relationship between $T^{\mathsf{ML}}$ and $K^\star$.
% in the, which starts at $T^{\mathsf{ML}} = 5e2$ and increments until $8e5$. 
Initially, i.e., having a cold model, $K^\star$ increases rapidly as a function of $T^{\mathsf{ML}}$, however, as the model gets warmed up increasing $K^\star$ calls for larger increments in $T^{\mathsf{ML}}$ which signifies the model inertia (i.e., for a warmed up model, to trigger a new model training round larger data variations are needed so that ML model gains outweighs the network costs). 
% We refer to this phenomenon as model inertia, where $T^{\mathsf{ML}}$ in order to produce the necessary model training gains needed to outweigh the increased energy and time costs of data and model offloading.

We next sequentially focus on the scaling of energy, acquisition time, and ML bound terms in problem~$\bm{\mathcal{P}}$.
\begin{figure}[t]
\centering
\vspace{-8mm}
\includegraphics[width=.45\textwidth]{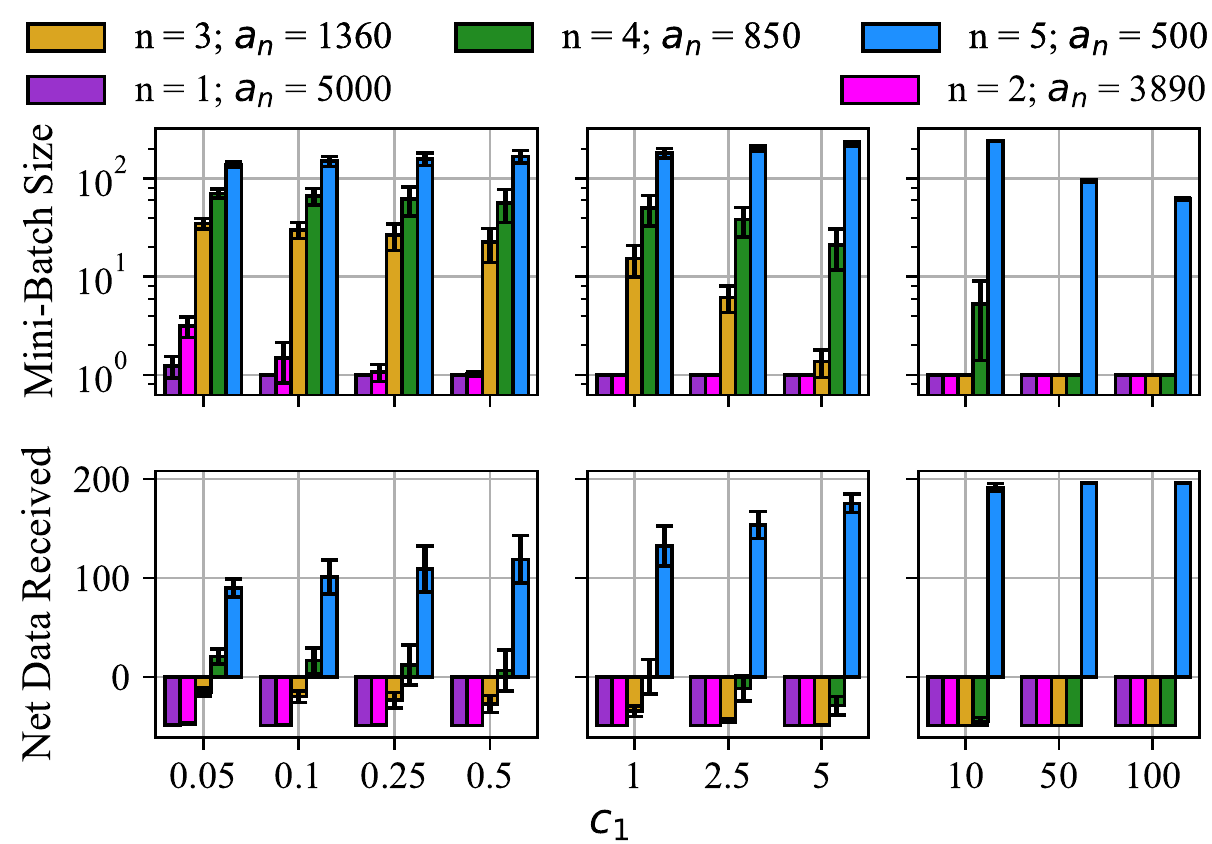}
\vspace{-2.5mm}
\caption{Top row: the average mini-batch size through model training period.  Bottom row: the net number of received data at the devices (negative values imply data offloading) through model training period. Different ranges of energy importance $c_1$ in $\bm{\mathcal{P}}$ are considered. 
The devices with higher processing costs have smaller mini-batch sizes and offload their data. 
As $c_1$ increases, the system offloads the data of devices with high processing costs, and, to compensate for the impact on the ML loss, the mini-batch at the device with the cheapest processing cost ($n=5$) is increased. 
% This trend continues in the middle subplot where the mini-batches at the devices with medium processing costs ($n=3,4$) are decreasing and they offload more data, while the mini-batch of $n=5$ is increasing and it receives more data. 
The right subplot shows the extreme case, where the importance of the energy overshadows the ML performance and the mini-batch size even at the cheapest device decreases.  Each bar depicts the averaged result with standard deviations obtained over $10$ independent network realizations.
}
% The first subplot demonstrates that as the energy importance initially increases, the system decreases the mini-batch size of the next device with the high processing cost ($n=2$), where to compensate for the impact on the ML loss the mini-batch at the device with the cheapest processing cost ($n=5$) is increased. This trend is continued in the middle subplot where the mini-batch at the devices with medium processing cost ($n=3,4$) is decreasing, while that of $n=5$ is increasing. The right subplot shows the extreme case, where the importance of the energy overshadows the ML performance, where the mini-batch even at the cheapest device would start to decrease after a certain point.
% Net number of received data at the devices (a negative value implies data offloading) for different ranges of the importance factor of the energy term $c_1$ in the objective of $\bm{\mathcal{P}}$. 
% This plot complements Fig.~\ref{fig:minibatch}, and shows that the devices offload data to those with better computation capability. When $c_1$ reaches a plateau, all the devices offload their data to the cheapest device.  
\label{fig:c1_minibatch_netd}
\vspace{-1.2mm}
\end{figure}

\vspace{-.2mm}
\subsubsection{Energy Scaling in Optimization Objective}\label{subsubsub:5}
% We next sequentially focus on the scaling terms in the objective function, 
$c_1$ in the objective of~$\bm{\mathcal{P}}$ controls  the importance of the energy components. %The energy components consists of the network energy consumed for data processing, data offloading, and ML model transmissions. 
To demonstrate the effect of $c_1$, we focus on the data processing (measured via mini-batch size) and offloading (measured via net data received) in Fig.~\ref{fig:c1_minibatch_netd}, where {\small $T^{\mathsf{ML}} = 100$}s and {\small$\zeta_2 = 1e-3$}. 
In the first column of Fig.~\ref{fig:c1_minibatch_netd}, $c_1$ increases while is still small, which leads to less data processing at devices with high processing cost (i.e., $n=2$) and more data offloading from them. The middle column accelerates the increment of $c_1$, showing rapid increase of data reception and processing at device $n=5$. Here, even lower processing cost devices ($n=4$) begin offloading data and processing less data. In the last column, data processing becomes almost prohibitive, and consequently, even device $n=5$ processes only a subset of its data. Even though device $n=5$ processes only a subset of its data, it is still beneficial for all devices to transfer data to device $n=5$ in order to reduce their SGD noise (they choose mini-batch sizes of $1$ so lower number of local data implies a lower SGD noise) and thereby reduce the ML term $\Xi$.
% in the objective function. 

\subsubsection{Temporal Importance in Optimization Objective}\label{subsubsub:6}
The cost of device acquisition, $c_2$, determines the balance of device acquisition and idle times. We investigate its influence in Fig.~\ref{fig:idle_vs_active}, where {\small $T^{\mathsf{ML}} = 1000$}s.
% and {\small $K=3$}.
It can be seen that a decrease in the device acquisition/active times corresponds to an increase in the idle times. This trade-off occurs in three regimes, segmented into three subplots. 
Initially, increasing $c_2$ has no effect until a network dependent threshold (in this case $c_2 = 40$) is reached.
After this threshold, we enter a regime in which incremental increase of $c_2$ results in a notable increase in idle times and a proportional decrease in active times. 
% As $c_2$ gets larger, the network would favor to engage devices in model training for shorter periods of times. 
Finally, in the third regime, the marginal reward of reducing device acquisition, which translates to processing less data in the ML bound, become prominent. To outweigh this effect, $c_2$ must increase by orders of magnitude to reduce the system active time.

% The final regime requires the order of magnitude increases in $c_2$ in order to produce significant increases in idle time (and corresponding decrease in active time). The behavior for large $c_2$ is due to the ML model training process. which requires data processing 

\subsubsection{ML Performance in Optimization Objective}\label{subsubsub:7}
Finally, we characterize the ML model training importance $c_3$ by its effect on the devices mini-batch size in Fig.~\ref{fig:c3_minib} for {\small $T^{\mathsf{ML}} = 1000$s}. With  increasing $c_3$, the importance of the model performance increases, and the solution prioritizes the third term in the objective by increasing the mini-batch size at the cost-efficient devices: we see that $n=5$ begins processing more data, and, as $c_3$ increases, other devices also begin processing more data based on a combination of data offloading and processing cost.

% \begin{figure}[t]
% \centering
% \includegraphics[width=.49\textwidth]{plots/netd_avgfinal_5devices.pdf}
% \caption{Net number of received data at the devices (a negative value implies data offloading) for different ranges of the importance factor of the energy term $c_1$ in the objective of $\bm{\mathcal{P}}$. This plot complements Fig.~\ref{fig:minibatch}, and shows that the devices offload data to those with better computation capability. When $c_1$ reaches a plateau, all the devices offload their data to the cheapest device.  }
% \label{fig:minibatch_netd}
% \end{figure}

\begin{figure}[t]
\vspace{-7mm}
\centering
\includegraphics[width=.47\textwidth]{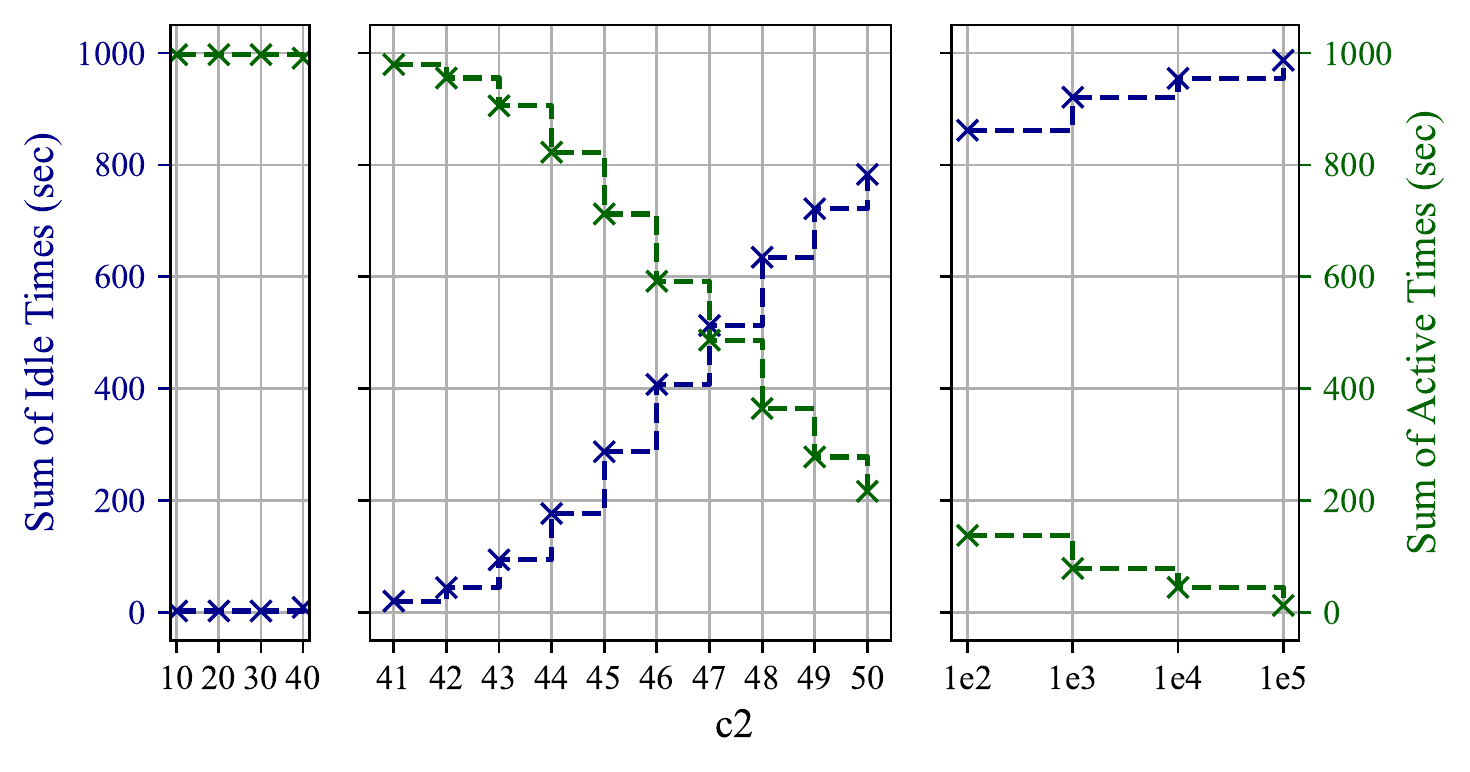}
\vspace{-4mm}
\caption{Device acquisition time (right y-axis) and device idle times (left y-axis) for varying values of cost of device utilization $c_2$ in~$\bm{\mathcal{P}}$. As $c_2$ increases the device utilization time decreases and devices remain idle for longer times.}
\label{fig:idle_vs_active}
\vspace{-2mm}
\end{figure}

\begin{figure}[t]
% \vspace{-.9mm}
\centering
\includegraphics[width=.45\textwidth]{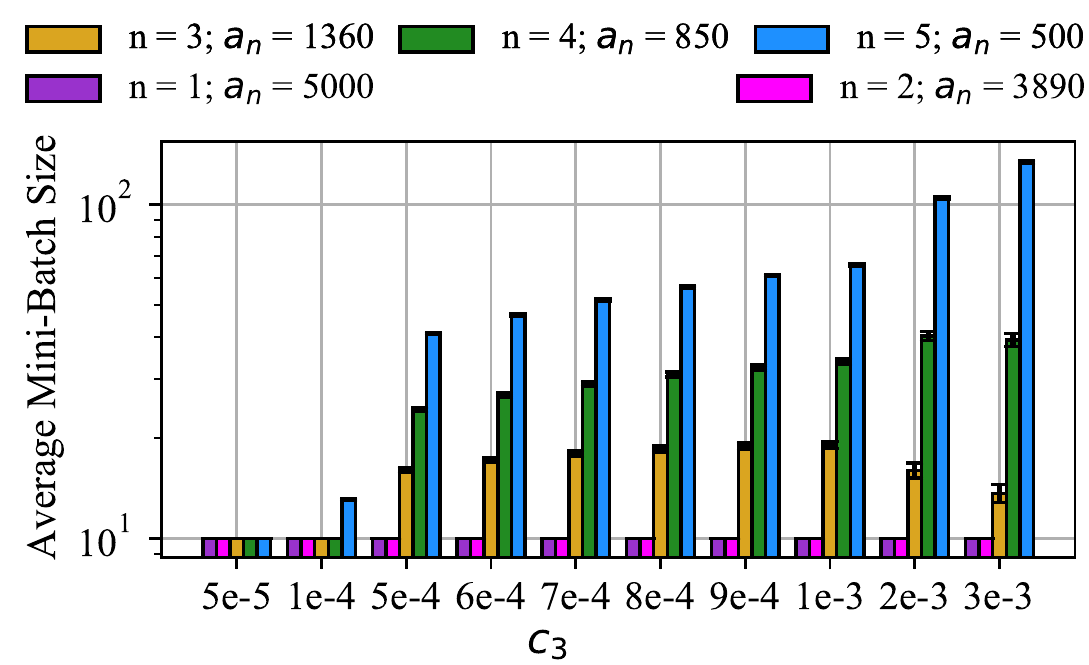}
\vspace{-2mm}
\caption{Average mini-batch size of the devices for various values of ML model performance importance $c_3$ in the objective function of $\bm{\mathcal{P}}$. 
As $c_3$ increases, the devices start processing more data locally, where more data is being processed in the cheapest devices in terms of computation cost.  Each bar depicts the averaged result with standard deviations obtained over $10$ independent network realizations.}
\label{fig:c3_minib}
\vspace{-1mm}
\end{figure}

\vspace{-1mm}
\section{Conclusion}
\noindent We introduced dynamic distributed learning via PSL. PSL extends federated learning through network, heterogeneity, and proximity. 
Also, it considers a cooperative distributed learning paradigm via incorporating \textit{data dispersion} and \textit{model/gradient dispersion with local condensation} across the devices to battle the innate shortcomings of
federated learning. PSL further considers a realistic scenario in which the model training rounds are conducted with \textit{idle times} in between, during which the data evolves across the devices. We modeled the processes conducted during PSL, introduced a new definition for concept/model drift, obtained the convergence characteristics of PSL, and formulated the network-aware PSL problem. We then proposed a general optimization solver to solve the problem with convergence guarantee.
Multiple future work directions are also discussed in the paper.
\vspace{-3mm}

\bibliographystyle{IEEEtran}
\bibliography{PSL}
\vspace{-14mm}
\begin{IEEEbiographynophoto}{Seyyedali Hosseinalipour (M'20)}
received his Ph.D. in EE from NCSU in 2020. He is currently an assistant professor of EE at University at Buffalo-SUNY. 
\end{IEEEbiographynophoto}
\vspace{-13.5mm}
\begin{IEEEbiographynophoto}{Su Wang (S'20)} is a Ph.D. student at Purdue University. He received his B.Sc. in ECE from Purdue University in 2019.
\end{IEEEbiographynophoto}
\vspace{-13.5mm}
\begin{IEEEbiographynophoto}{Nicol\`{o} Michelusi (SM'19)} 
received his Ph.D. in EE from University of Padova, Italy, in 2013. He is an Assistant Professor of ECEE at Arizona State University.
\end{IEEEbiographynophoto}
\vspace{-13.5mm}
% \vspace{-0.15in}
\begin{IEEEbiographynophoto}{Vaneet Aggarwal (SM'15)} received his Ph.D. in EE from Princeton University in 2010. He is currently an Associate Professor at Purdue University.
\end{IEEEbiographynophoto}
% \vspace{-0.15in}
\vspace{-13.5mm}
\begin{IEEEbiographynophoto}{Christopher G. Brinton (SM'20)}
is Elmore Chaired  Assistant Professor of ECE at Purdue University. He received his Ph.D. in EE from Princeton University in 2016.
\end{IEEEbiographynophoto}
\vspace{-13.5mm}
\begin{IEEEbiographynophoto}{David Love (F'15)} is Nick Trbovich Professor of ECE at Purdue University. He received his Ph.D. in EE from University of Texas at Austin in 2004.
\end{IEEEbiographynophoto}
\vspace{-13.5mm}
\begin{IEEEbiographynophoto}{Mung Chiang (F'12)} is the President of Purdue University. He is also Roscoe H. George Distinguished Professor of ECE at Purdue University. He received his Ph.D. from Stanford University in  2003.
\end{IEEEbiographynophoto}
% \vspace{-18mm}
% \vspace{-0.55in}

\begingroup
\onecolumn

\appendices

\section{Proof of Lemma~\ref{lemma:trackMeanvar}}\label{app:meanVar}

\noindent Letting $\mathcal{A}$ denote the set of new datapoints that are added to stratum $\mathcal{S}$, where $\bm{\mu}_{_\mathcal{A}}= \frac{1}{{|\mathcal{A}|}}\sum_{d\in \mathcal{A}} \bm{d},~~~~
    \sigma^2_{_\mathcal{A}}= \frac{1}{|\mathcal{A}|-1}\sum_{d\in \mathcal{A}} \left\Vert\bm{d} -\bm{\mu}_{_\mathcal{A}} \right\Vert^2
$, with $\bm{d}$ denoting the feature vector of datapoint $d$. It is straightforward to verify that the new mean  of the stratum is given by
\begin{align}
      &\hspace{-3mm}\bm{\mu}_{\mathsf{new}} =  \frac{|\mathcal{S}| \bm{\mu}_{\mathsf{old}} + |\mathcal{A}|\bm{\mu}_{_\mathcal{A}}-|\mathcal{D}|\bm{\mu}_{_\mathcal{D}}}{|\mathcal{S}|+|\mathcal{A}|-|\mathcal{D}|}.
      \end{align}
      Also, the new variance of the stratum is given by

\begin{equation}
\hspace{-0.3mm}
    \begin{aligned}
      &\sigma^2_{\mathsf{new}}=\frac{\sum_{d\in\mathcal{A}}\Vert\bm{d} - \bm{\mu}_{\mathsf{new}}\Vert^2+\sum_{d\in\mathcal{S}}\hspace{-.5mm}\Vert\bm{d} - \bm{\mu}_{\mathsf{new}}\Vert^2-\sum_{d\in\mathcal{D}}\hspace{-.5mm}\Vert\bm{d} - \bm{\mu}_{\mathsf{new}}\Vert^2}{|\mathcal{A}|+|\mathcal{S}|+|\mathcal{D}|-1}
      \\&=\frac{\sum_{d\in\mathcal{A}}\hspace{-.5mm}\Vert \bm{d}\Vert^2\hspace{-.5mm}+\hspace{-.5mm}\sum_{d\in\mathcal{S}}\hspace{-.5mm}\Vert \bm{d}\Vert^2\hspace{-.5mm}-\hspace{-.5mm}\sum_{d\in\mathcal{D}}\hspace{-.5mm}\Vert \bm{d}\Vert^2\hspace{-.5mm}+\hspace{-.5mm}|\mathcal{A}|\Vert\bm{\mu}_{\mathsf{new}}\Vert^2\hspace{-.5mm}+|\mathcal{S}|\Vert\bm{\mu}_{\mathsf{new}}\Vert^2 \hspace{-.5mm}-\hspace{-.5mm}|\mathcal{D}|\Vert\bm{\mu}_{\mathsf{new}}\Vert^2\hspace{-.5mm}-2\bm{\mu}_{\mathsf{new}}^\top(\sum_{d\in\mathcal{S}}\hspace{-.5mm} \bm{d} \hspace{-.5mm}+\hspace{-.5mm}\sum_{d\in\mathcal{A}} \hspace{-.5mm}\bm{d}\hspace{-.5mm}-\hspace{-.5mm}\sum_{d\in\mathcal{D}}\hspace{-.5mm} \bm{d})}{|\mathcal{A}|+|\mathcal{S}|+|\mathcal{D}|-1}\\&
      =\frac{\sum_{d\in\mathcal{A}}\Vert \bm{d}\Vert^2+\sum_{d\in\mathcal{S}}\Vert \bm{d}\Vert^2-\sum_{d\in\mathcal{D}}\Vert \bm{d}\Vert^2+|\mathcal{A}|\Vert\bm{\mu}_{\mathsf{new}}\Vert^2+|\mathcal{S}|\Vert\bm{\mu}_{\mathsf{new}}\Vert^2 -|\mathcal{D}|\Vert\bm{\mu}_{\mathsf{new}}\Vert^2-2\bm{\mu}_{\mathsf{new}}^\top(|\mathcal{S}|\bm{\mu}_{\mathsf{old}}+|\mathcal{A}|\bm{\mu}_{_\mathcal{A}}-|\mathcal{D}|\bm{\mu}_{_\mathcal{D}})}{|\mathcal{A}|+|\mathcal{S}|+|\mathcal{D}|-1}
      \\&
          =\frac{\sum_{d\in\mathcal{A}}\Vert \bm{d}\Vert^2+\sum_{d\in\mathcal{S}}\Vert \bm{d}\Vert^2-\sum_{d\in\mathcal{D}}\Vert \bm{d}\Vert^2+(|\mathcal{A}|+|\mathcal{S}|-|\mathcal{D}|)\left\Vert\frac{ |\mathcal{S}| \bm{\mu}_{\mathsf{old}} + |\mathcal{A}|\bm{\mu}_{_\mathcal{A}}-|\mathcal{D}|\bm{\mu}_{_\mathcal{D}} }{|\mathcal{S}|+|\mathcal{A}|-|\mathcal{D}|}\right\Vert^2} {|\mathcal{A}|+|\mathcal{S}|+|\mathcal{D}|-1}   \\&
          ~~~~-\frac{
          2\left(\frac{ |\mathcal{S}| \bm{\mu}_{\mathsf{old}} + |\mathcal{A}|\bm{\mu}_{_\mathcal{A}} -|\mathcal{D}|\bm{\mu}_{_\mathcal{D}}}{|\mathcal{S}|+|\mathcal{A}|-|\mathcal{D}|}\right)^\top(|\mathcal{S}|\bm{\mu}_{\mathsf{old}}+|\mathcal{A}|\bm{\mu}_{_\mathcal{A}}-|\mathcal{D}|\bm{\mu}_{_\mathcal{D}})}{|\mathcal{A}|+|\mathcal{S}|+|\mathcal{D}|-1}
          \\&
             =\frac{\sum_{d\in\mathcal{A}}\Vert \bm{d}\Vert^2+\sum_{d\in\mathcal{S}}\Vert \bm{d}\Vert^2-\sum_{d\in\mathcal{D}}\Vert \bm{d}\Vert^2+ \Vert\bm{\mu}_{\mathsf{old}}\Vert^2 \left(\frac{-|\mathcal{S}|^2}{|\mathcal{A}|+|\mathcal{S}|-|\mathcal{D}|}\right)  +\Vert\bm{\mu}_{_\mathcal{A}}\Vert^2\left(\frac{-|\mathcal{A}|^2}{|\mathcal{A}|+|\mathcal{S}|-|\mathcal{D}|}\right)+\Vert\bm{\mu}_{D}\Vert^2 \left(\frac{-|\mathcal{D}|^2}{|\mathcal{A}|+|\mathcal{S}|-|\mathcal{D}|}\right)}{|\mathcal{A}|+|\mathcal{S}|-|\mathcal{D}|-1}
                 \\&
            ~~~~ +\frac{-2\bm{\mu}_{\mathsf{old}}\bm{\mu}^\top_{_\mathcal{A}}\left(\frac{|\mathcal{A}||\mathcal{S}|}{|\mathcal{A}|+|\mathcal{S}|-|\mathcal{D}|}\right)  +2\bm{\mu}_{\mathsf{old}}\bm{\mu}^\top_{_\mathcal{D}}\left(\frac{|\mathcal{S}||\mathcal{D}|}{|\mathcal{A}|+|\mathcal{S}|-|\mathcal{D}|}\right) 
            +2\bm{\mu}_{_\mathcal{A}}\bm{\mu}^\top_{_\mathcal{D}}\left(\frac{|\mathcal{A}||\mathcal{D}|}{|\mathcal{A}|+|\mathcal{S}|-|\mathcal{D}|}\right) }{|\mathcal{A}|+|\mathcal{S}|-|\mathcal{D}|-1}
             \\&
             =\frac{\sum_{d\in\mathcal{A}}\Vert \bm{d}\Vert^2-|\mathcal{A}|\Vert\bm{\mu}_{_{\mathcal{A}}}\Vert^2+\sum_{d\in\mathcal{S}}\Vert \bm{d}\Vert^2-|\mathcal{S}|\Vert\bm{\mu}_{\mathsf{old}}\Vert^2-\sum_{d\in\mathcal{D}}\Vert \bm{d}\Vert^2+|\mathcal{D}|\Vert\bm{\mu}_{_{\mathcal{D}}}\Vert^2}{|\mathcal{A}|+|\mathcal{S}|-|\mathcal{D}|-1}
            \\&
            ~~~~+\frac{ \Vert\bm{\mu}_{\mathsf{old}}\Vert^2 \left(\frac{-|\mathcal{S}|^2}{|\mathcal{A}|+|\mathcal{S}|-|\mathcal{D}|}+|\mathcal{S}|\right)  +\Vert\bm{\mu}_{_\mathcal{A}}\Vert^2\left(\frac{-|\mathcal{A}|^2}{|\mathcal{A}|+|\mathcal{S}|-|\mathcal{D}|}+|\mathcal{A}|\right)+\Vert\bm{\mu}_{D}\Vert^2 \left(\frac{-|\mathcal{D}|^2}{|\mathcal{A}|+|\mathcal{S}|-|\mathcal{D}|}-|\mathcal{D}|\right)}{|\mathcal{A}|+|\mathcal{S}|-|\mathcal{D}|-1}
                 \\&
            ~~~~ +\frac{-2\bm{\mu}_{\mathsf{old}}\bm{\mu}^\top_{_\mathcal{A}}\left(\frac{|\mathcal{A}||\mathcal{S}|}{|\mathcal{A}|+|\mathcal{S}|-|\mathcal{D}|}\right)  +2\bm{\mu}_{\mathsf{old}}\bm{\mu}^\top_{_\mathcal{D}}\left(\frac{|\mathcal{S}||\mathcal{D}|}{|\mathcal{A}|+|\mathcal{S}|-|\mathcal{D}|}\right) 
            +2\bm{\mu}_{_\mathcal{A}}\bm{\mu}^\top_{_\mathcal{D}}\left(\frac{|\mathcal{A}||\mathcal{D}|}{|\mathcal{A}|+|\mathcal{S}|-|\mathcal{D}|}\right) }{|\mathcal{A}|+|\mathcal{S}|-|\mathcal{D}|-1}
            %%%%%%%%%%%%%%%%%%
            %%%%%%%%%%%%%%%%%%
            %%%%%%%%%%%%%%%%%%
            \\&
             =\frac{(|\mathcal{A}|-1)\sigma^2_{_\mathcal{A}}+(|\mathcal{S}|-1)\sigma^2_{old}-(|\mathcal{D}|-1)\sigma^2_{_\mathcal{D}}}{|\mathcal{A}|+|\mathcal{S}|-|\mathcal{D}|-1}
            \\&
            ~~~~+\frac{ \Vert\bm{\mu}_{\mathsf{old}}\Vert^2 \left(\frac{|\mathcal{A}||\mathcal{S}|-|\mathcal{S}||\mathcal{D}|}{|\mathcal{A}|+|\mathcal{S}|-|\mathcal{D}|}\right)  +\Vert\bm{\mu}_{_\mathcal{A}}\Vert^2\left(\frac{|\mathcal{A}||\mathcal{S}|-|\mathcal{A}||\mathcal{D}|}{|\mathcal{A}|+|\mathcal{S}|-|\mathcal{D}|}\right)+\Vert\bm{\mu}_{D}\Vert^2 \left(\frac{-|\mathcal{A}||\mathcal{D}|-|\mathcal{S}||\mathcal{D}|}{|\mathcal{A}|+|\mathcal{S}|-|\mathcal{D}|}\right)}{|\mathcal{A}|+|\mathcal{S}|-|\mathcal{D}|-1}
                 \\&
            ~~~~ +\frac{-2\bm{\mu}_{\mathsf{old}}\bm{\mu}^\top_{_\mathcal{A}}\left(\frac{|\mathcal{A}||\mathcal{S}|}{|\mathcal{A}|+|\mathcal{S}|-|\mathcal{D}|}\right)  +2\bm{\mu}_{\mathsf{old}}\bm{\mu}^\top_{_\mathcal{D}}\left(\frac{|\mathcal{S}||\mathcal{D}|}{|\mathcal{A}|+|\mathcal{S}|-|\mathcal{D}|}\right) 
            +2\bm{\mu}_{_\mathcal{A}}\bm{\mu}^\top_{_\mathcal{D}}\left(\frac{|\mathcal{A}||\mathcal{D}|}{|\mathcal{A}|+|\mathcal{S}|-|\mathcal{D}|}\right) }{|\mathcal{A}|+|\mathcal{S}|-|\mathcal{D}|-1}\\&
            =  \frac{(|\mathcal{S}|-1)\sigma^2_{old} + (|\mathcal{A}|-1)\sigma^2_{_\mathcal{A}}-(|\mathcal{D}|-1)\sigma^2_{_\mathcal{D}}
            }{|\mathcal{A}|+|\mathcal{S}|-|\mathcal{D}|-1}\\&~~~~
             +\frac{\left(\frac{|\mathcal{A}||\mathcal{S}|}{|\mathcal{A}|+|\mathcal{S}|-|\mathcal{D}|}\right)\Vert \bm{\mu}_{\mathsf{old}}-\bm{\mu}_{_\mathcal{A}}\Vert^2-
        \left(\frac{|\mathcal{S}||\mathcal{D}|}{|\mathcal{A}|+|\mathcal{S}|-|\mathcal{D}|}\right)\Vert \bm{\mu}_{\mathsf{old}}-\bm{\mu}_{_\mathcal{D}}\Vert^2-
         \left(\frac{|\mathcal{A}||\mathcal{D}|}{|\mathcal{A}|+|\mathcal{S}|-|\mathcal{D}|}\right)\Vert \bm{\mu}_{A}-\bm{\mu}_{_\mathcal{D}}\Vert^2}{|\mathcal{A}|+|\mathcal{S}|-|\mathcal{D}|-1}.
\end{aligned}
\hspace{-17mm}
\end{equation}
With minor rearranging of the terms, the result of the lemma is obtained.

\newpage
\section{Proof of Theorem~\ref{th:main}}\label{app:th:main}
\noindent Revisiting the parameter updating rule of the method, the evolution of the global parameter can be described as follows:
\begin{equation}
    \mathbf{w}^{(k+1)}=\mathbf{w}^{(k)}-\eta_{k}\overline{\nabla {F}}^{(k)},
\end{equation}
where $\overline{\nabla {F}}^{(k)}$ is the normalized accumulated gradient of the devices with respect to their number of local descents given by
\begin{equation}\label{eq:globAggApp}
    \overline{\nabla {F}}^{(k)}=\sum_{n'\in \mathcal{N}}\frac{\widehat{{D}}^{(k)}_{n'}e^{(k)}_{n'}}{{D}^{(k)} }
    \sum_{n\in \mathcal{N}}\frac{\widehat{{D}}^{(k)}_n}{{D}^{(k)} e^{(k)}_n}\overline{\nabla {F}}_n^{(k)},
\end{equation}
where we have used the facts that (i) the model training is conducted on the final datasets of the nodes at each global model training round, i.e., $\widehat{{D}}^{(k)}_{n}$, $\forall n\in\mathcal{N}$, which is obtained after data arrival and departure at the nodes prior to the start of local model training, and (ii) during data offloading, no data points disappears, and thus $\sum_{n\in \mathcal{N}}\widehat{{D}}^{(k)}_{n}= \sum_{n\in \mathcal{N}}{{D}}^{(k)}_{n}={D}^{(k)}$.

Using the $\beta$-smoothness of the global loss function (Assumption~\ref{Assup:lossFun}), we have 
\begin{equation}
 F^{({k})}(\mathbf{w}^{(k+1)}) \leq F^{({k})}(\mathbf{w}^{(k)}) +  \nabla{F^{({k})}(\mathbf{w}^{(k)})}^\top \left( \mathbf{w}^{(k+1)} - \mathbf{w}^{(k)}\right)+ \frac{\beta}{2} \left\Vert \mathbf{w}^{(k+1)} - \mathbf{w}^{(k)}\right\Vert^2.
\end{equation}

% {\color{blue}
\begin{remark}\label{remark:modelDisc} Throughout, the super index on the loss functions, e.g.,  $F^{(k)}(\mathbf{w})$ or $F^{(k)}_n(\mathbf{w})$ stands for the time index under which the losses are measured and the corresponding dataset at the network after model dispersion phase, i.e., $F^{(k)}(\mathbf{w})=F(\mathbf{w}|\mathcal{D}^{(k)})=F_n(\mathbf{w}|\widehat{\mathcal{D}}^{(k)})$ or $F^{(k)}_n(\mathbf{w})=F_n(\mathbf{w}|\widehat{\mathcal{D}}_n^{(k)})$, $\forall \mathbf{w}$.
\end{remark}
% stands for the loss under which the iteration $k+1$ of PSL starts, i.e., $F^{(k)}(\mathbf{w}^{(k)})=F(\mathbf{w}^{(k)}|\widehat{\mathcal{D}}^{(k+1)})$, which is in general different that the loss that PSL observes at the end of the $k$th iteration, i.e., $F(\mathbf{w}^{(k)}|\widehat{\mathcal{D}}^{(k)})$ due to model/concept drift. Also, $F^{(k)}(\mathbf{w}^{(k+1)})$ denotes the loss observed at the end of iteration $k+1$, i.e., $F^{(k)}(\mathbf{w}^{(k+1)})=F(\mathbf{w}^{(k+1)}|\widehat{\mathcal{D}}^{(k+1)})$. We will eliminate this discrepancy in the notation via explicitly considering the definition of model/concept in the later part of the proof.\end{remark} REVISE!!!}

Replacing the updating rule for $\mathbf{w}^{(k+1)}$ and taking the conditional expectation (with respect to randomized data sampling at the last aggregation of the PSL) from both hand sides yields
\begin{equation}\label{eq:res1}
\begin{aligned}
 \mathbb E_k \left[F^{(k)}(\mathbf{w}^{(k+1)})\right] \leq& F^{(k)}(\mathbf{w}^{(k)}) -  \nabla{F^{(k)}(\mathbf{w}^{(k)})}^\top \mathbb{E}_k\left[ \eta_{k}\sum_{n'\in \mathcal{N}}\frac{\widehat{{D}}^{(k)}_{n'}e^{(k)}_{n'}}{{D}^{(k)} }
    \sum_{n\in \mathcal{N}}\frac{\widehat{{D}}^{(k)}_n}{{D}^{(k)} e^{(k)}_n}\overline{\nabla {F}}_n^{(k)}\right]\\&+ \frac{\beta}{2} \mathbb{E}_k\left[\left\Vert  \eta_{k}\sum_{n'\in \mathcal{N}}\frac{\widehat{{D}}^{(k)}_{n'}e^{(k)}_{n'}}{{D}^{(k)} }
    \sum_{n\in \mathcal{N}}\frac{\widehat{{D}}^{(k)}_n}{{D}^{(k)} e^{(k)}_n}\overline{\nabla {F}}_n^{(k)}\right\Vert^2\right].
    \end{aligned}
\end{equation}
\begin{remark}
Note that as mentioned in Remark~\ref{remark:modelDisc}, $F^{(k)}(\mathbf{w}^{(k+1)})=F(\mathbf{w}^{(k+1)}|\mathcal{D}^{(k)})$ is the loss of the final model obtained at global iteration $k$ based on conducting SGD across the devices on the global data distribution $\mathcal{D}^{(k)}$. This loss is different than the loss under which the global iteration $k+1$ starts, i.e., $F^{(k+1)}(\mathbf{w}^{(k+1)})=F(\mathbf{w}^{(k+1)}|\mathcal{D}^{(k+1)})$ due to the existence of model/concept drift. In the later parts of the proof, we will reveal the connection between $F^{(k)}(\mathbf{w}^{(k+1)})$ and  $F^{(k+1)}(\mathbf{w}^{(k+1)})$.
\end{remark}
Since $\overline{\nabla {F}}_n^{(k)} =-\left( \mathbf{w}_n^{(k),e^{(k)}_n}-\mathbf{w}^{(k)}\right) \Big/ {\eta_{k}}$, via recursive expansion of the updating rule in~\eqref{eq:WeightupdateStrat}, we get
\begin{equation}\label{eq:nablabarF}
    \overline{\nabla {F}}_n^{(k)} = {\frac{1}{\widehat{{D}}^{(k)}_n}} \sum_{e=1}^{e^{(k)}_n} \sum_{j=1}^{{S}^{(k)}_{n}}\sum_{d\in \mathcal{B}^{(k),e}_{n,j}} \hspace{-1mm} {\frac{{S}^{(k)}_{n,j}\nabla  f(\mathbf{w}^{(k),e-1}_{n},d)}{{B}^{(k)}_{n,j}}}.
\end{equation}
Taking the expectation from both hand sides of~\eqref{eq:nablabarF}, using the law of total expectation (across the local mini-batches) and the fact that data points selection inside each stratum is conducted uniformly at random without replacement,  we get
\begin{equation}\label{eq:avrgNablas}
    \begin{aligned}
   \mathbb{E}_k \left[\overline{\nabla {F}}_n^{(k)}\right] =&{\frac{1}{\widehat{{D}}^{(k)}_n}} \sum_{e=1}^{e^{(k)}_n} \sum_{j=1}^{{S}^{(k)}_{n}}{S}^{(k)}_{n,j}\mathbb{E}_k  \left[ \sum_{d\in \mathcal{B}^{(k),e}_{n,j}}  {\frac{\nabla  f(\mathbf{w}^{(k),e-1}_{n},d)}{{B}^{(k)}_{n,j}}}\right]\\&={\frac{1}{\widehat{{D}}^{(k)}_n}} \sum_{e=1}^{e^{(k)}_n} \sum_{j=1}^{{S}^{(k)}_{n}}{S}^{(k)}_{n,j} \sum_{d\in \mathcal{S}^{(k)}_{n,j}}  {\frac{\nabla  f(\mathbf{w}^{(k),e-1}_{n},d)}{{S}^{(k)}_{n,j}}}
   \\&= \sum_{e=1}^{e^{(k)}_n} \sum_{j=1}^{{S}^{(k)}_{n}} \sum_{d\in \mathcal{S}^{(k)}_{n,j}}  {\frac{\nabla  f(\mathbf{w}^{(k),e-1}_{n},d)}{{\widehat{{D}}^{(k)}_n}}}
   = \sum_{e=1}^{e^{(k)}_n}  {\nabla  F^{(k)}_n(\mathbf{w}^{(k),e-1}_{n})}.
    \end{aligned}
\end{equation}
By replacing~\eqref{eq:avrgNablas} back in~\eqref{eq:res1}, we get
\begin{equation}
\begin{aligned}
 \mathbb E_k \left[F^{(k)}(\mathbf{w}^{(k+1)})\right] \leq& F^{(k)}(\mathbf{w}^{(k)}) -  \nabla{F^{(k)}(\mathbf{w}^{(k)})}^\top \mathbb{E}_k\left[ \eta_{k}\sum_{n'\in \mathcal{N}}\frac{\widehat{{D}}^{(k)}_{n'}e^{(k)}_{n'}}{{D}^{(k)} }
    \sum_{n\in \mathcal{N}}\frac{\widehat{{D}}^{(k)}_n}{{D}^{(k)} e^{(k)}_n} \sum_{e=1}^{e^{(k)}_n}  {\nabla  F^{(k)}_n(\mathbf{w}^{(k),e-1}_{n})}\right]\\&+ \frac{\beta}{2} \mathbb{E}_k\left[\left\Vert  \eta_{k}\sum_{n'\in \mathcal{N}}\frac{\widehat{{D}}^{(k)}_{n'}e^{(k)}_{n'}}{{D}^{(k)} }
    \sum_{n\in \mathcal{N}}\frac{\widehat{{D}}^{(k)}_n}{{D}^{(k)} e^{(k)}_n}\overline{\nabla {F}}_n^{(k)}\right\Vert^2\right].
    \end{aligned}
\end{equation}
Using the fact that for any two real valued vectors $\bm{a}$ and $\bm{b}$ with the same length, we have: $2\bm{a}^\top\bm{b}=\Vert \bm{a}\Vert^2+\Vert \bm{b}\Vert^2-\Vert \bm{a}-\bm{b}\Vert^2$, we further get
\begin{equation}\label{eq:firstlossIneq}
\begin{aligned}
 \mathbb E_k \left[F^{(k)}(\mathbf{w}^{(k+1)})\right] \leq& F^{(k)}(\mathbf{w}^{(k)}) - \frac{\eta_{k}}{2}\sum_{n'\in \mathcal{N}}\frac{\widehat{{D}}^{(k)}_{n'}e^{(k)}_{n'}}{{D}^{(k)} } \mathbb{E}_k\vast[
 \left\Vert\nabla{F^{(k)}(\mathbf{w}^{(k)})}\right\Vert^2 +\left\Vert 
    \sum_{n\in \mathcal{N}}\frac{\widehat{{D}}^{(k)}_n}{{D}^{(k)} e^{(k)}_n} \sum_{e=1}^{e^{(k)}_n}  {\nabla  F^{(k)}_n(\mathbf{w}^{(k),e-1}_{n})}\right\Vert^2\\&-  \left\Vert
    \nabla{F^{(k)}(\mathbf{w}^{(k)})}-
    \sum_{n\in \mathcal{N}}\frac{\widehat{{D}}^{(k)}_n}{{D}^{(k)} e^{(k)}_n} \sum_{e=1}^{e^{(k)}_n}  {\nabla  F^{(k)}_n(\mathbf{w}^{(k),e-1}_{n})}
    \right\Vert^2
    \vast]\\&+ \frac{\beta \eta^2_{k}}{2} \mathbb{E}_k\left[\left\Vert \sum_{n'\in \mathcal{N}}\frac{\widehat{{D}}^{(k)}_{n'}e^{(k)}_{n'}}{{D}^{(k)} }
    \sum_{n\in \mathcal{N}}\frac{\widehat{{D}}^{(k)}_n}{{D}^{(k)} e^{(k)}_n}\overline{\nabla {F}}_n^{(k)}\right\Vert^2\right].
    \end{aligned}
\end{equation}
We bound the last term on the right hand side of above inequality as follows:
\begin{equation}\label{eq:res2}
\hspace{-1mm}
    \begin{aligned}
    &\frac{\beta \eta^2_{k}}{2} \mathbb{E}_k\left[\left\Vert \sum_{n'\in \mathcal{N}}\frac{\widehat{{D}}^{(k)}_{n'}e^{(k)}_{n'}}{{D}^{(k)} }
    \sum_{n\in \mathcal{N}}\frac{\widehat{{D}}^{(k)}_n}{{D}^{(k)} e^{(k)}_n}\overline{\nabla {F}}_n^{(k)}\right\Vert^2\right]
    \\&=\frac{\beta \eta^2_{k}}{2} \mathbb{E}_k\left[\left\Vert \sum_{n'\in \mathcal{N}}\frac{\widehat{{D}}^{(k)}_{n'}e^{(k)}_{n'}}{{D}^{(k)} }
    \sum_{n\in \mathcal{N}}\frac{\widehat{{D}}^{(k)}_n}{{D}^{(k)} e^{(k)}_n} \left({\frac{1}{\widehat{{D}}^{(k)}_n}} \sum_{e=1}^{e^{(k)}_n} \sum_{j=1}^{{S}^{(k)}_{n}}\sum_{d\in \mathcal{B}^{(k),e}_{n,j}} \hspace{-1mm} {\frac{{S}^{(k)}_{n,j}\nabla  f(\mathbf{w}^{(k),e-1}_{n},d)}{{B}^{(k)}_{n,j}}}\right)\right\Vert^2\right]
    \\&
    =\frac{\beta \eta^2_{k}}{2} \mathbb{E}_k\Bigg[\Bigg\Vert \sum_{n'\in \mathcal{N}}\frac{\widehat{{D}}^{(k)}_{n'}e^{(k)}_{n'}}{{D}^{(k)} }
    \sum_{n\in \mathcal{N}}\frac{\widehat{{D}}^{(k)}_n}{{D}^{(k)} e^{(k)}_n} \Bigg({\frac{1}{\widehat{{D}}^{(k)}_n}} \sum_{e=1}^{e^{(k)}_n} \sum_{j=1}^{{S}^{(k)}_{n}}\sum_{d\in \mathcal{B}^{(k),e}_{n,j}} \hspace{-1mm} {\frac{{S}^{(k)}_{n,j}\nabla  f(\mathbf{w}^{(k),e-1}_{n},d)}{{B}^{(k)}_{n,j}}}\Bigg)
    \\&-\sum_{n'\in \mathcal{N}}\frac{\widehat{{D}}^{(k)}_{n'}e^{(k)}_{n'}}{{D}^{(k)} }
    \sum_{n\in \mathcal{N}}\frac{\widehat{{D}}^{(k)}_n}{{D}^{(k)} e^{(k)}_n}\sum_{e=1}^{e^{(k)}_n} \nabla F^{(k)}_n(\mathbf{w}_n^{(k),e-1})+\sum_{n'\in \mathcal{N}}\frac{\widehat{{D}}^{(k)}_{n'}e^{(k)}_{n'}}{{D}^{(k)} }
    \sum_{n\in \mathcal{N}}\frac{\widehat{{D}}^{(k)}_n}{{D}^{(k)} e^{(k)}_n}\sum_{e=1}^{e^{(k)}_n} \nabla F^{(k)}_n(\mathbf{w}_n^{(k),e-1})
    \Bigg\Vert^2\Bigg]
    \\&    \overset{(i)}{\leq} {\beta \eta^2_{k}}~ \mathbb{E}_k\Bigg[\Bigg\Vert \sum_{n'\in \mathcal{N}}\frac{\widehat{{D}}^{(k)}_{n'}e^{(k)}_{n'}}{{D}^{(k)} }
    \sum_{n\in \mathcal{N}}\frac{\widehat{{D}}^{(k)}_n}{{D}^{(k)} e^{(k)}_n} \Bigg({\frac{1}{\widehat{{D}}^{(k)}_n}} \sum_{e=1}^{e^{(k)}_n} \sum_{j=1}^{{S}^{(k)}_{n}}\sum_{d\in \mathcal{B}^{(k),e}_{n,j}} \hspace{-1mm} {\frac{{S}^{(k)}_{n,j}\nabla  f(\mathbf{w}^{(k),e-1}_{n},d)}{{B}^{(k)}_{n,j}}}\Bigg)
    \\&-\sum_{n'\in \mathcal{N}}\frac{\widehat{{D}}^{(k)}_{n'}e^{(k)}_{n'}}{{D}^{(k)} }
    \sum_{n\in \mathcal{N}}\frac{\widehat{{D}}^{(k)}_n}{{D}^{(k)} e^{(k)}_n}\sum_{e=1}^{e^{(k)}_n} \nabla F^{(k)}_n(\mathbf{w}_n^{(k),e-1})\Bigg\Vert^2\Bigg]+{\beta \eta^2_{k}}~\mathbb{E}_k\Bigg[\Bigg\Vert\sum_{n'\in \mathcal{N}}\frac{\widehat{{D}}^{(k)}_{n'}e^{(k)}_{n'}}{{D}^{(k)} }
    \sum_{n\in \mathcal{N}}\frac{\widehat{{D}}^{(k)}_n}{{D}^{(k)} e^{(k)}_n}\sum_{e=1}^{e^{(k)}_n} \nabla F^{(k)}_n(\mathbf{w}_n^{(k),e-1})\Bigg\Vert^2\Bigg]
%%%%%%%%%%%%%%%%
    \\&={\beta \eta^2_{k}}\left(\sum_{n'\in \mathcal{N}}\frac{\widehat{{D}}^{(k)}_{n'}e^{(k)}_{n'}}{{D}^{(k)} } \right)^2~ \underbrace{\mathbb{E}_k\Bigg[\Bigg\Vert 
    \sum_{n\in \mathcal{N}}\frac{\widehat{{D}}^{(k)}_n}{{D}^{(k)} e^{(k)}_n} \Bigg({\frac{1}{\widehat{{D}}^{(k)}_n}} \sum_{e=1}^{e^{(k)}_n} \sum_{j=1}^{{S}^{(k)}_{n}}\sum_{d\in \mathcal{B}^{(k),e}_{n,j}} \hspace{-1mm} {\frac{{S}^{(k)}_{n,j}\nabla  f(\mathbf{w}^{(k),e-1}_{n},d)}{{B}^{(k)}_{n,j}}}
   -
    \sum_{e=1}^{e^{(k)}_n} \nabla F^{(k)}_n(\mathbf{w}_n^{(k),e-1})\Bigg)\Bigg\Vert^2\Bigg]}_{(a)}
    \\&+{\beta \eta^2_{k}}\left(\sum_{n'\in \mathcal{N}}\frac{\widehat{{D}}^{(k)}_{n'}e^{(k)}_{n'}}{{D}^{(k)} } \right)^2~\mathbb{E}_k\Bigg[\Bigg\Vert
    \sum_{n\in \mathcal{N}}\frac{\widehat{{D}}^{(k)}_n}{{D}^{(k)} e^{(k)}_n}\sum_{e=1}^{e^{(k)}_n} \nabla F^{(k)}_n(\mathbf{w}_n^{(k),e-1})\Bigg\Vert^2\Bigg],
    \end{aligned}\hspace{-14mm}
\end{equation}
where in inequality $(i)$ we have used the Cauchy-Schwarz inequality $\Vert \mathbf{a}+\mathbf{b} \Vert^2\leq 2 \Vert \mathbf{a} \Vert^2+2\Vert \mathbf{b} \Vert^2$.

\textbf{Bounding term $(a)$ in~\eqref{eq:res2}:}  Using~\eqref{eq:avrgNablas}, i.e., each local gradient estimation is unbiased, combined with the fact that  the noise of gradient estimation is assumed to be independent across the nodes we get
\begin{equation}\label{eq:firstVar}
\begin{aligned}
    (a) \overset{(i)}{=}&\sum_{n\in \mathcal{N}}\left(\frac{\widehat{{D}}^{(k)}_n}{{D}^{(k)} e^{(k)}_n} \right)^2\mathbb{E}_k\Bigg[\Bigg\Vert 
     \Bigg({\frac{1}{\widehat{{D}}^{(k)}_n}} \sum_{e=1}^{e^{(k)}_n} \sum_{j=1}^{{S}^{(k)}_{n}}\sum_{d\in \mathcal{B}^{(k),e}_{n,j}} \hspace{-1mm} {\frac{{S}^{(k)}_{n,j}\nabla  f(\mathbf{w}^{(k),e-1}_{n},d)}{{B}^{(k)}_{n,j}}}
   -
    \sum_{e=1}^{e^{(k)}_n} \nabla F^{(k)}_n(\mathbf{w}_n^{(k),e-1})\Bigg)\Bigg\Vert^2\Bigg]
    \\&=
    \sum_{n\in \mathcal{N}}\left(\frac{\widehat{{D}}^{(k)}_n}{{D}^{(k)} e^{(k)}_n} \right)^2\mathbb{E}_k\Bigg[\Bigg\Vert 
    \sum_{e=1}^{e^{(k)}_n}  \underbrace{\Bigg({\frac{1}{\widehat{{D}}^{(k)}_n}}  \sum_{j=1}^{{S}^{(k)}_{n}}\sum_{d\in \mathcal{B}^{(k),e}_{n,j}} \hspace{-1mm} {\frac{{S}^{(k)}_{n,j}\nabla  f(\mathbf{w}^{(k),e-1}_{n},d)}{{B}^{(k)}_{n,j}}}
   -
    \nabla F^{(k)}_n(\mathbf{w}_n^{(k),e-1})\Bigg)}_{(b)}\Bigg\Vert^2\Bigg]
    \\& \overset{(ii)}{=}\sum_{n\in \mathcal{N}} \left(\frac{\widehat{{D}}^{(k)}_n}{{D}^{(k)} e^{(k)}_n} \right)^2~\sum_{e=1}^{e^{(k)}_n}\mathbb{E}_k\Bigg[\Bigg\Vert 
    \Bigg({\frac{1}{\widehat{{D}}^{(k)}_n}}  \sum_{j=1}^{{S}^{(k)}_{n}}\sum_{d\in \mathcal{B}^{(k),e}_{n,j}} \hspace{-1mm} {\frac{{S}^{(k)}_{n,j}\nabla  f(\mathbf{w}^{(k),e-1}_{n},d)}{{B}^{(k)}_{n,j}}}
   -
    \nabla F^{(k)}_n(\mathbf{w}_n^{(k),e-1})\Bigg)\Bigg\Vert^2\Bigg] .
    \end{aligned}
\end{equation}
where to obtain equality $(i)$ we expanded $(a)$ in~\eqref{eq:res2} and used the fact that the noise of gradient estimation across the mini-batches are independent and zero mean, and to obtain (ii) we used the fact that each individual term in $(b)$ is zero mean.
To further bound~\eqref{eq:firstVar}, we consider the result of~\cite{lohr2019sampling} (Chapter 3, Eq. (3.5)), based on which the above equality can be further simplified as follows:
\begin{equation}\label{eq:firstA}
 \begin{aligned}
     (a)&=\sum_{n\in \mathcal{N}} \left(\frac{\widehat{{D}}^{(k)}_n}{{D}^{(k)} e^{(k)}_n} \right)^2~\sum_{e=1}^{e^{(k)}_n}\left[\left\Vert \frac{1}{\widehat{{D}}^{(k)}_n} \sum_{j=1}^{{S}^{(k)}_{n}}\sum_{d\in \mathcal{B}^{(k),e}_{n,j}} \hspace{-1mm} {\frac{{S}^{(k)}_{n,j}\nabla  f(\mathbf{w}^{(k),e-1}_{n},d)}{{B}^{(k)}_{n,j}}}
     -\frac{1}{\widehat{{D}}^{(k)}_n} \sum_{d\in \widehat{\mathcal{D}}^{(k)}_{n}}\nabla  f(\mathbf{w}^{(k),e-1}_{n},d)\right\Vert^2\right] \\
     &=\sum_{n\in \mathcal{N}} \left(\frac{\widehat{{D}}^{(k)}_n}{{D}^{(k)} e^{(k)}_n} \right)^2~\sum_{e=1}^{e^{(k)}_n}  \sum_{j=1}^{{S}^{(k)}_{n}} \left(1-\frac{{B}^{(k)}_{n,j}}{{S}^{(k)}_{n,j}} \right) \frac{\left({S}^{(k)}_{n,j}\right)^2}{\left(\widehat{{D}}^{(k)}_{n}\right)^2} \frac{\left(\sigma_{n,j}^{(k),e-1}\right)^2}{{B}^{(k)}_{n,j}},
     \end{aligned}
 \end{equation}
 where $\sigma_{n,j}^{(k),e-1}$ denotes the \textit{variance of the gradients} evaluated at the particular local descent iteration $e-1$ for the parameter realization $\mathbf{w}^{(k),e-1}_{n}$. We further bound this quantity as follows:
 \begin{align}\label{eq:stratVar}
     \left(\sigma_{n,j}^{(k),e-1}\right)^2&= \frac{\sum_{d\in\mathcal{S}^{(k)}_{n,j} } \Big\Vert \nabla  f(\mathbf{w}^{(k),e-1}_{n},d)-{\sum_{d'\in\mathcal{S}^{(k)}_{n,j}} }\frac{\nabla  f(\mathbf{w}^{(k),e-1}_{n},d')}{{S}^{(k)}_{n,j}}\Big\Vert^2}{{S}^{(k)}_{n,j}-1}
     \nonumber\\
     &=\frac{\sum_{d\in\mathcal{S}^{(k)}_{n,j} }\frac{1}{\left({S}^{(k)}_{n,j}\right)^2} \Big\Vert {S}^{(k)}_{n,j}\nabla  f(\mathbf{w}^{(k),e-1}_{n},d)-{\sum_{d'\in\mathcal{S}^{(k)}_{n,j}} }{\nabla  f(\mathbf{w}^{(k),e-1}_{n},d')}\Big\Vert^2}{{S}^{(k)}_{n,j}-1}
     \nonumber\\
     &\overset{(i)}{\leq} \frac{\sum_{d\in\mathcal{S}^{(k)}_{n,j} }\frac{{S}^{(k)}_{n,j}-1}{\left({S}^{(k)}_{n,j}\right)^2} \sum_{d'\in\mathcal{S}^{(k)}_{n,j}} \Big\Vert \nabla  f(\mathbf{w}^{(k),e-1}_{n},d)-{\nabla  f(\mathbf{w}^{(k),e-1}_{n},d')}\Big\Vert^2}{{S}^{(k)}_{n,j}-1}
     \nonumber\\&
     \leq \frac{\sum_{d\in\mathcal{S}^{(k)}_{n,j} }\frac{({S}^{(k)}_{n,j}-1)\Theta^2}{\left({S}^{(k)}_{n,j}\right)^2} \sum_{d'\in\mathcal{S}^{(k)}_{n,j}} \Big\Vert \bm{d}-\bm{d}'\Big\Vert^2}{{S}^{(k)}_{n,j}-1}
     \leq \frac{({S}^{(k)}_{n,j}-1)\Theta^2}{\left({S}^{(k)}_{n,j}\right)^2}\frac{\sum_{d\in\mathcal{S}^{(k)}_{n,j} } \sum_{d'\in\mathcal{S}^{(k)}_{n,j}} \Big\Vert \bm{d}-\bm{d}'+\widetilde{\bm{\lambda}}^{(k)}_{n,j}-\widetilde{\bm{\lambda}}^{(k)}_{n,j}\Big\Vert^2}{{S}^{(k)}_{n,j}-1}\nonumber\\
     &=\frac{({S}^{(k)}_{n,j}-1)\Theta^2}{\left({S}^{(k)}_{n,j}\right)^2}\frac{\sum_{d\in\mathcal{S}^{(k)}_{n,j} } \sum_{d'\in\mathcal{S}^{(k)}_{n,j}} \left[\Big\Vert \bm{d}-  \widetilde{\bm{\lambda}}^{(k)}_{n,j} \Big\Vert^2 + \Big\Vert \bm{d}'-  \widetilde{\bm{\lambda}}^{(k)}_{n,j}\Big\Vert^2 - 2(\bm{d}-  \widetilde{\bm{\lambda}}^{(k)}_{n,j})^\top(\bm{d}'-  \widetilde{\bm{\lambda}}^{(k)}_{n,j}) \right]}{{S}^{(k)}_{n,j}-1}
     \nonumber\\&\overset{(ii)}{=}\frac{({S}^{(k)}_{n,j}-1)\Theta^2}{\left({S}^{(k)}_{n,j}\right)^2}\frac{ {S}^{(k)}_{n,j} \sum_{d\in\mathcal{S}^{(k)}_{n,j} } \Big\Vert \bm{d}-  \widetilde{\bm{\lambda}}^{(k)}_{n,j} \Big\Vert^2 +  {S}^{(k)}_{n,j} \sum_{d'\in\mathcal{S}^{(k)}_{n,j}} \Big\Vert \bm{d}'-  \widetilde{\bm{\lambda}}^{(k)}_{n,j}\Big\Vert^2}{{S}^{(k)}_{n,j}-1}\nonumber \\
     &
     =\frac{2({S}^{(k)}_{n,j}-1)\Theta^2}{{S}^{(k)}_{n,j}}
     \left(\widetilde{\sigma}_{n,j}^{(k)}\right)^2, 
     \end{align}
 where $\bm{d}$ denotes the feature vector of data point $d$, and  $\widetilde{\bm{\lambda}}^{(k)}_{n,j}$ and $\widetilde{\sigma}_{n,j}^{(k)}$ denote the mean and sample variance of data points in stratum $\mathcal{S}^{(k)}_{n,j}$, which are gradient independent. Also, in inequality $(i)$ of~\eqref{eq:stratVar}, we used the Cauchy-Schwarz inequality, and in $(ii)$ we used the fact that $\sum_{d\in\mathcal{S}^{(k)}_{n,j} } (\bm{d}-  \widetilde{\bm{\lambda}}^{(k)}_{n,j}) =\bm{0}$.
Replacing the above result in~\eqref{eq:firstA}, gives us
\begin{equation}\label{eq:111}
   (a)\leq 2 \Theta^2\sum_{n\in \mathcal{N}} \left(\frac{\widehat{{D}}^{(k)}_n}{{D}^{(k)} e^{(k)}_n} \right)^2~\sum_{e=1}^{e^{(k)}_n}  \sum_{j=1}^{{S}^{(k)}_{n}} \left(1-\frac{{B}^{(k)}_{n,j}}{{S}^{(k)}_{n,j}} \right) \frac{{S}^{(k)}_{n,j}}{\left(\widehat{{D}}^{(k)}_{n}\right)^2} \frac{{({S}^{(k)}_{n,j}-1)}
     \left(\widetilde{\sigma}_{n,j}^{(k)}\right)^2}{{B}^{(k)}_{n,j}}.
\end{equation}
Replacing the above result back in~\eqref{eq:res2}, inequality~\eqref{eq:firstlossIneq} can be written as follows:
 \begin{align}\label{eq:secondlossIneq}
 &\mathbb E_k \left[F^{(k)}(\mathbf{w}^{(k+1)})\right] \leq F^{(k)}(\mathbf{w}^{(k)}) - \frac{\eta_{k}}{2}\sum_{n'\in \mathcal{N}}\frac{\widehat{{D}}^{(k)}_{n'}e^{(k)}_{n'}}{{D}^{(k)} } \mathbb{E}_k\vast[
 \left\Vert\nabla{F^{(k)}(\mathbf{w}^{(k)})}\right\Vert^2 +\left\Vert 
    \sum_{n\in \mathcal{N}}\frac{\widehat{{D}}^{(k)}_n}{{D}^{(k)} e^{(k)}_n} \sum_{e=1}^{e^{(k)}_n}  {\nabla  F^{(k)}_n(\mathbf{w}^{(k),e-1}_{n})}\right\Vert^2
    \nonumber\\&-  \left\Vert
    \nabla{F^{(k)}(\mathbf{w}^{(k)})}-
    \sum_{n\in \mathcal{N}}\frac{\widehat{{D}}^{(k)}_n}{{D}^{(k)} e^{(k)}_n} \sum_{e=1}^{e^{(k)}_n}  {\nabla  F^{(k)}_n(\mathbf{w}^{(k),e-1}_{n})}
    \right\Vert^2
    \vast]
     \nonumber\\&+2 \Theta^2 {\beta \eta^2_{k}}\left(\sum_{n'\in \mathcal{N}}\frac{\widehat{{D}}^{(k)}_{n'}e^{(k)}_{n'}}{{D}^{(k)} } \right)^2\sum_{n\in \mathcal{N}} \left(\frac{\widehat{{D}}^{(k)}_n}{{D}^{(k)} e^{(k)}_n} \right)^2~\sum_{e=1}^{e^{(k)}_n}  \sum_{j=1}^{{S}^{(k)}_{n}} \left(1-\frac{{B}^{(k)}_{n,j}}{{S}^{(k)}_{n,j}} \right) \frac{{S}^{(k)}_{n,j}}{\left(\widehat{{D}}^{(k)}_{n}\right)^2} \frac{{({S}^{(k)}_{n,j}-1)}
     \left(\widetilde{\sigma}_{n,j}^{(k)}\right)^2}{{B}^{(k)}_{n,j}}\hspace{-14mm}
     \nonumber\\&+{\beta \eta^2_{k}}\left(\sum_{n'\in \mathcal{N}}\frac{\widehat{{D}}^{(k)}_{n'}e^{(k)}_{n'}}{{D}^{(k)} } \right)^2~\mathbb{E}_k\left[\Bigg\Vert
    \sum_{n\in \mathcal{N}}\frac{\widehat{{D}}^{(k)}_n}{{D}^{(k)} e^{(k)}_n}\sum_{e=1}^{e^{(k)}_n} \nabla F^{(k)}_n(\mathbf{w}_n^{(k),e-1})\Bigg\Vert^2\right].\hspace{-4mm}
    \end{align}
Performing some algebraic manipulations on the above inequality and gathering the terms yields
\begin{align}\label{eq:secondlossIneq}
\hspace{-10mm}
 \mathbb E_k &\left[F^{(k)} (\mathbf{w}^{(k+1)})\right] \leq F^{(k)}(\mathbf{w}^{(k)}) - \frac{\eta_{k}}{2}\sum_{n'\in \mathcal{N}}\frac{\widehat{{D}}^{(k)}_{n'}e^{(k)}_{n'}}{{D}^{(k)} } 
 \left\Vert\nabla{F^{(k)}(\mathbf{w}^{(k)})}\right\Vert^2 \nonumber\\&\hspace{-6mm}
 +\underbrace{\left({\beta \eta^2_{k}}\left(\sum_{n'\in \mathcal{N}}\frac{\widehat{{D}}^{(k)}_{n'}e^{(k)}_{n'}}{{D}^{(k)} } \right)^2 - \frac{\eta_{k}}{2}\sum_{n'\in \mathcal{N}}\frac{\widehat{{D}}^{(k)}_{n'}e^{(k)}_{n'}}{{D}^{(k)} }\right) \mathbb{E}_k\left[\left\Vert 
    \sum_{n\in \mathcal{N}}\frac{\widehat{{D}}^{(k)}_n}{{D}^{(k)} e^{(k)}_n} \sum_{e=1}^{e^{(k)}_n}  {\nabla  F^{(k)}_n(\mathbf{w}^{(k),e-1}_{n})}\right\Vert^2\right]}_{(c)}
    \nonumber\\&\hspace{-6mm}
    + \frac{\eta_{k}}{2}\sum_{n'\in \mathcal{N}}\frac{\widehat{{D}}^{(k)}_{n'}e^{(k)}_{n'}}{{D}^{(k)} }   \underbrace{\mathbb{E}_k\left[\left\Vert
    \nabla{F^{(k)}(\mathbf{w}^{(k)})}-
    \sum_{n\in \mathcal{N}}\frac{\widehat{{D}}^{(k)}_n}{{D}^{(k)} e^{(k)}_n} \sum_{e=1}^{e^{(k)}_n}  {\nabla  F^{(k)}_n(\mathbf{w}^{(k),e-1}_{n})}
    \right\Vert^2\right]}_{(d)}
     \nonumber\\&\hspace{-6mm}
     +2 \Theta^2 {\beta \eta^2_{k}}\left(\sum_{n'\in \mathcal{N}}\frac{\widehat{{D}}^{(k)}_{n'}e^{(k)}_{n'}}{{D}^{(k)} } \right)^2\sum_{n\in \mathcal{N}} \left(\frac{\widehat{{D}}^{(k)}_n}{{D}^{(k)} e^{(k)}_n} \right)^2~\sum_{e=1}^{e^{(k)}_n}  \sum_{j=1}^{{S}^{(k)}_{n}} \left(1-\frac{{B}^{(k)}_{n,j}}{{S}^{(k)}_{n,j}} \right) \frac{{S}^{(k)}_{n,j}}{\left(\widehat{{D}}^{(k)}_{n}\right)^2} \frac{{({S}^{(k)}_{n,j}-1)}
     \left(\widetilde{\sigma}_{n,j}^{(k)}\right)^2}{{B}^{(k)}_{n,j}}.\hspace{-14mm}
    \end{align}
    Assuming $\eta_k\leq \left(2\beta\sum_{n'\in \mathcal{N}}\frac{\widehat{{D}}^{(k)}_{n'}e^{(k)}_{n'}}{{D}^{(k)} } \right)^{-1}$ makes the coefficient in term $(c)$ negative and thus it can be removed from the upper bound. In the following, we aim to bound term $(d)$:
    \begin{align}\label{eq:res3}
        (d) & \overset{(i)}{\leq}   \sum_{n\in \mathcal{N}}\frac{\widehat{{D}}^{(k)}_n}{{D}^{(k)} }\mathbb{E}_k\left[\left\Vert
    \nabla{F^{(k)}_n(\mathbf{w}^{(k)})}-
    \frac{1}{e^{(k)}_n} \sum_{e=1}^{e^{(k)}_n}  {\nabla  F^{(k)}_n(\mathbf{w}^{(k),e-1}_{n})}
    \right\Vert^2\right]
    \nonumber 
    \\& 
    \overset{(ii)}{\leq} \sum_{n\in \mathcal{N}}\frac{\widehat{{D}}^{(k)}_n}{D^{(k)} e_n^{(k)}}\sum_{e=1}^{e^{(k)}_n} \mathbb{E}_k\left[\left\Vert
   \nabla{F^{(k)}_n(\mathbf{w}^{(k)})}-
     {\nabla  F^{(k)}_n(\mathbf{w}^{(k),e-1}_{n})}
    \right\Vert^2\right]
    \leq \beta^2\sum_{n\in \mathcal{N}}\frac{\widehat{{D}}^{(k)}_n}{D^{(k)} e_n^{(k)}}\sum_{e=1}^{e^{(k)}_n} \mathbb{E}_k\left[\left\Vert
   \mathbf{w}^{(k)}-
     \mathbf{w}^{(k),e-1}_{n}
    \right\Vert^2\right].
    \end{align}
    To bound $\mathbb{E}_k\left[\left\Vert
   \mathbf{w}^{(k)}-
     \mathbf{w}^{(k),e-1}_{n}
    \right\Vert^2\right]$, we take the following steps:
  \begin{align}\label{eq:res4}
    &
    \mathbb{E}_k\left[\left\Vert
   \mathbf{w}^{(k)}-
     \mathbf{w}^{(k),e-1}_{n}
    \right\Vert^2\right]\overset{(iii)}{=}\eta_k^2  \mathbb{E}_k\left[\left\Vert
   {\frac{1}{\widehat{{D}}^{(k)}_n}} \sum_{e'=1}^{e-1} \sum_{j=1}^{{S}^{(k)}_{n}}\sum_{d\in \mathcal{B}^{(k),e'}_{n,j}} \hspace{-1mm} {\frac{{S}^{(k)}_{n,j}\nabla  f(\mathbf{w}^{(k),e'-1}_{n},d)}{{B}^{(k)}_{n,j}}}
    \right\Vert^2 \right]
    \nonumber\\&= \eta_k^2 \mathbb{E}_k\vast[\Bigg\Vert
   {\frac{1}{\widehat{{D}}^{(k)}_n}} \sum_{e'=1}^{e-1} \sum_{j=1}^{{S}^{(k)}_{n}}\sum_{d\in \mathcal{B}^{(k),e'}_{n,j}} \hspace{-1mm} {\frac{{S}^{(k)}_{n,j}\nabla  f(\mathbf{w}^{(k),e'-1}_{n},d)}{{B}^{(k)}_{n,j}}}
   \nonumber\\&~~~~~~~~~~~~~~~~~~~~~~~~~~~~~~-\frac{1}{\widehat{{D}}^{(k)}_n}\sum_{e'=1}^{e-1} \sum_{d\in \widehat{\mathcal{D}}^{(k)}_{n}}\nabla  f(\mathbf{w}^{(k),e'-1}_{n},d)+\frac{1}{\widehat{{D}}^{(k)}_n} \sum_{e'=1}^{e-1}\sum_{d\in \widehat{\mathcal{D}}^{(k)}_{n}}\nabla  f(\mathbf{w}^{(k),e'-1}_{n},d)
    \Bigg\Vert^2 \vast]
    \nonumber \\& \overset{(iv)}{\leq }2 \eta_k^2  \mathbb{E}_k\left[\Bigg\Vert
   {\frac{1}{\widehat{{D}}^{(k)}_n}} \sum_{e'=1}^{e-1} \sum_{j=1}^{{S}^{(k)}_{n}}\sum_{d\in \mathcal{B}^{(k),e'}_{n,j}} \hspace{-1mm} {\frac{{S}^{(k)}_{n,j}\nabla  f(\mathbf{w}^{(k),e'-1}_{n},d)}{{B}^{(k)}_{n,j}}}
   -\frac{1}{\widehat{{D}}^{(k)}_n} \sum_{e'=1}^{e-1}\sum_{d\in \widehat{\mathcal{D}}^{(k)}_{n}}\nabla  f(\mathbf{w}^{(k),e'-1}_{n},d)\Bigg\Vert^2\right] \hspace{-10mm}\nonumber\\&~~~~~~~~~~~~~~~~~~~~~~~~~~~~~~+2 \eta_k^2 \mathbb{E}_k\left[\Bigg\Vert \frac{1}{\widehat{{D}}^{(k)}_n} \sum_{e'=1}^{e-1}\sum_{d\in \widehat{\mathcal{D}}^{(k)}_{n}}\nabla  f(\mathbf{w}^{(k),e'-1}_{n},d)
    \Bigg\Vert^2 \right]
   \nonumber \\&
    \overset{(v)}{=} \underbrace{2 \eta_k^2 \sum_{e'=1}^{e-1}\mathbb{E}_k\left[\Bigg\Vert
   {\frac{1}{\widehat{{D}}^{(k)}_n}}  \sum_{j=1}^{{S}^{(k)}_{n}}\sum_{d\in \mathcal{B}^{(k),e'}_{n,j}} \hspace{-1mm} {\frac{{S}^{(k)}_{n,j}\nabla  f(\mathbf{w}^{(k),e'-1}_{n},d)}{{B}^{(k)}_{n,j}}}
   -\frac{1}{\widehat{{D}}^{(k)}_n} \sum_{d\in \widehat{\mathcal{D}}^{(k)}_{n}}\nabla  f(\mathbf{w}^{(k),e'-1}_{n},d)\Bigg\Vert^2\right]}_{(e)} \hspace{-10mm}\nonumber\\&~~~~~~~~~~~~~~~~~~~~~~~~~~~~~~+\underbrace{2 \eta_k^2 \mathbb{E}_k\left[\Bigg\Vert \frac{1}{\widehat{{D}}^{(k)}_n} \sum_{e'=1}^{e-1}\sum_{d\in \widehat{\mathcal{D}}^{(k)}_{n}}\nabla  f(\mathbf{w}^{(k),e'-1}_{n},d)
    \Bigg\Vert^2\right]}_{(f)} ,
    \end{align}
where in inequalities $(i)$ and $(ii)$ in~\eqref{eq:res3} we used Jenson's inequality; and to obtain~\eqref{eq:res4}, in equality $(iii)$ we used~\eqref{eq:WeightupdateStrat}, in inequality $(iv)$ we used  Cauchy–Schwarz inequality, and $(v)$
uses the fact that each local gradient estimation is unbiased (i.e., zero mean) conditioned on its own local parameter and the law of total expectation (across the mini-batches $e'$). 

Similar to~\eqref{eq:firstA}, using~\eqref{eq:stratVar} we upper bound term $(e)$ in~\eqref{eq:res4} as follows:
\begin{equation}\label{eq:A2}
    (e) \leq 4 \Theta^2 \eta_k^2 \sum_{e'=1}^{e-1}\sum_{j=1}^{{S}^{(k)}_{n}} \left(1-\frac{{B}^{(k)}_{n,j}}{{S}^{(k)}_{n,j}} \right) \frac{{S}^{(k)}_{n,j}}{\left(\widehat{{D}}^{(k)}_{n}\right)^2} \frac{{({S}^{(k)}_{n,j}-1)}
     \left(\widetilde{\sigma}_{n,j}^{(k)}\right)^2}{{B}^{(k)}_{n,j}}.
\end{equation}
Also, for term $(f)$, we have
\begin{align}\label{eq:B2}
   (f)&
   \overset{(i)}{\leq} 2 \eta_k^2 (e-1)\sum_{e'=1}^{e-1}\mathbb{E}_k\left[\Bigg\Vert \frac{1}{\widehat{{D}}^{(k)}_n} \sum_{d\in \widehat{\mathcal{D}}^{(k)}_{n}}\nabla  f(\mathbf{w}^{(k),e'-1}_{n},d) - \nabla F^{(k)}_n(\mathbf{w}^{(k)})+ \nabla F^{(k)}_n(\mathbf{w}^{(k)})
    \Bigg\Vert^2 \right]
   \nonumber \\
    &\overset{(ii)}{\leq} 4 \eta_k^2 (e-1)\sum_{e'=1}^{e-1}\mathbb{E}_k\left[\Bigg\Vert \nabla  F^{(k)}_n(\mathbf{w}^{(k),e'-1}_{n}) - \nabla F^{(k)}_n(\mathbf{w}^{(k)})\Bigg\Vert^2\right] + 4 \eta_k^2 (e-1)\sum_{e'=1}^{e-1} \Bigg\Vert\nabla F^{(k)}_n(\mathbf{w}^{(k)})
    \Bigg\Vert^2 
    \nonumber \\&
    \leq 4 \eta_k^2 (e-1)\sum_{e'=1}^{e-1}\mathbb{E}_k\left[\Bigg\Vert \nabla  F^{(k)}_n(\mathbf{w}^{(k),e'-1}_{n}) - \nabla F^{(k)}_n(\mathbf{w}^{(k)})\Bigg\Vert^2\right] + 4 \eta_k^2 (e-1)\sum_{e'=1}^{e-1} \Bigg\Vert\nabla F^{(k)}_n(\mathbf{w}^{(k)})
    \Bigg\Vert^2 
    \nonumber \\&
    \leq 4\eta_k^2\beta^2 (e-1) \sum_{e'=1}^{e-1}\mathbb{E}_k\left[\Bigg\Vert \mathbf{w}^{(k),e'-1}_{n} - \mathbf{w}^{(k)}\Bigg\Vert^2\right]+ 4 \eta_k^2 (e-1)\sum_{e'=1}^{e-1} \Bigg\Vert\nabla F^{(k)}_n(\mathbf{w}^{(k)})
    \Bigg\Vert^2 ,
\end{align}
where inequalities $(i)$ and $(ii)$ are obtained via Cauchy-Schwarz inequality. Replacing the result of~\eqref{eq:A2} and~\eqref{eq:B2} back in~\eqref{eq:res4} we have
\begin{align}
  \mathbb{E}_k\left[  \left\Vert
  \mathbf{w}^{(k)}-
     \mathbf{w}^{(k),e-1}_{n}
    \right\Vert^2\right] \leq& 4 \Theta^2 \eta_k^2 \sum_{e'=1}^{e-1}\sum_{j=1}^{{S}^{(k)}_{n}} \left(1-\frac{{B}^{(k)}_{n,j}}{{S}^{(k)}_{n,j}} \right) \frac{{S}^{(k)}_{n,j}}{\left(\widehat{{D}}^{(k)}_{n}\right)^2} \frac{{({S}^{(k)}_{n,j}-1)}
     \left(\widetilde{\sigma}_{n,j}^{(k)}\right)^2}{{B}^{(k)}_{n,j}} \nonumber \\
     & + 4\eta_k^2\beta^2 (e-1) \sum_{e'=1}^{e-1}\mathbb{E}_k\left[\Bigg\Vert \mathbf{w}^{(k),e'-1}_{n} - \mathbf{w}^{(k)}\Bigg\Vert^2\right]+ 4 \eta_k^2 (e-1)\sum_{e'=1}^{e-1} \Bigg\Vert\nabla F^{(k)}_n(\mathbf{w}^{(k)})
    \Bigg\Vert^2, 
\end{align}
which implies
\begin{align}
   &\sum_{e=1}^{e_n^{(k)}} \mathbb{E}_k\left[\left\Vert
   \mathbf{w}^{(k)}-
     \mathbf{w}^{(k),e-1}_{n}
    \right\Vert^2\right] \leq 4 \Theta^2 \eta_k^2\sum_{e=1}^{e_n^{(k)}} \sum_{e'=1}^{e-1}\sum_{j=1}^{{S}^{(k)}_{n}} \left(1-\frac{{B}^{(k)}_{n,j}}{{S}^{(k)}_{n,j}} \right) \frac{{S}^{(k)}_{n,j}}{\left(\widehat{{D}}^{(k)}_{n}\right)^2} \frac{{({S}^{(k)}_{n,j}-1)}
     \left(\widetilde{\sigma}_{n,j}^{(k)}\right)^2}{{B}^{(k)}_{n,j}} \nonumber \\
     & + 4\eta_k^2\beta^2 \sum_{e=1}^{e_n^{(k)}}(e-1) \sum_{e'=1}^{e-1}\mathbb{E}_k\left[\Bigg\Vert \mathbf{w}^{(k),e'-1}_{n} - \mathbf{w}^{(k)}\Bigg\Vert^2\right]+ 4 \eta_k^2 \sum_{e=1}^{e_n^{(k)}}(e-1)\sum_{e'=1}^{e-1} \Bigg\Vert\nabla F^{(k)}_n(\mathbf{w}^{(k)})
    \Bigg\Vert^2
   \nonumber \\&\leq
    4 \Theta^2 \eta_k^2 \left(e_n^{(k)}\right)\left(e_n^{(k)}-1\right)\sum_{j=1}^{{S}^{(k)}_{n}} \left(1-\frac{{B}^{(k)}_{n,j}}{{S}^{(k)}_{n,j}} \right) \frac{{S}^{(k)}_{n,j}}{\left(\widehat{{D}}^{(k)}_{n}\right)^2} \frac{{({S}^{(k)}_{n,j}-1)}
     \left(\widetilde{\sigma}_{n,j}^{(k)}\right)^2}{{B}^{(k)}_{n,j}} \nonumber \\
     & + 4\eta_k^2\beta^2 \left(e_n^{(k)}\right)\left(e_n^{(k)}-1\right) \sum_{e=1}^{e_n^{(k)}}\mathbb{E}_k\left[\Bigg\Vert \mathbf{w}^{(k),e-1}_{n} - \mathbf{w}^{(k)}\Bigg\Vert^2\right]+ 4 \eta_k^2 \left(e_n^{(k)}\right)\left(e_n^{(k)}-1\right) \sum_{e=1}^{e_n^{(k)}} \Bigg\Vert\nabla F^{(k)}_n(\mathbf{w}^{(k)})
    \Bigg\Vert^2,
\end{align}
% and in turn
% \begin{align}
%   \sum_{e=1}^{e_n^{(k)}} \left\Vert
%   \mathbf{w}^{(k)}-
%      \mathbf{w}^{(k),e-1}_{n}
%     \right\Vert^2 \leq&
%     8 \Theta^2 \eta_k^2 \left(e_n^{(k)}\right)\left(e_n^{(k)}-1\right)\sum_{j=1}^{{S}^{(k)}_{n}} \left(1-\frac{{B}^{(k)}_{n,j}}{{S}^{(k)}_{n,j}} \right) \frac{{S}^{(k)}_{n,j}}{\left(\widehat{{D}}^{(k)}_{n}\right)^2} \frac{{({S}^{(k)}_{n,j}-1)}
%      \left(\widetilde{\sigma}_{n,j}^{(k)}\right)^2}{{B}^{(k)}_{n,j}} \nonumber \\
%      & + 4\eta_k^2\beta^2 \left(e_n^{(k)}\right)\left(e_n^{(k)}-1\right) \sum_{e=1}^{e_n^{(k)}}\Bigg\Vert \mathbf{w}^{(k),e-1}_{n} - \mathbf{w}^{(k)}\Bigg\Vert^2+ 4 \eta_k^2 \left(e_n^{(k)}\right)\left(e_n^{(k)}-1\right) \sum_{e=1}^{e_n^{(k)}} \Bigg\Vert\nabla F^{(k)}_n(\mathbf{w}^{(k)})
%     \Bigg\Vert^2.
% \end{align}
Assuming $\eta_k \leq \left(2\beta\sqrt{e_n^{(k)}(e_n^{(k)}-1)}\right)^{-1},\forall n$, the above inequality implies
\begin{align}
    \sum_{e=1}^{e_n^{(k)}} \mathbb{E}_k\left[\left\Vert
   \mathbf{w}^{(k)}-
     \mathbf{w}^{(k),e-1}_{n}
    \right\Vert^2\right] \leq&
    \frac{4 \Theta^2 \eta_k^2 e_n^{(k)}\left(e_n^{(k)}-1\right)}{1- 4\eta_k^2\beta^2 e_n^{(k)}\left(e_n^{(k)}-1\right)}\sum_{j=1}^{{S}^{(k)}_{n}} \left(1-\frac{{B}^{(k)}_{n,j}}{{S}^{(k)}_{n,j}} \right) \frac{{S}^{(k)}_{n,j}}{\left(\widehat{{D}}^{(k)}_{n}\right)^2} \frac{{({S}^{(k)}_{n,j}-1)}
     \left(\widetilde{\sigma}_{n,j}^{(k)}\right)^2}{{B}^{(k)}_{n,j}} 
     \nonumber \\
     & + \frac{4 \eta_k^2 \left(e_n^{(k)}\right)^2\left(e_n^{(k)}-1\right)}{1- 4\eta_k^2\beta^2 e_n^{(k)}\left(e_n^{(k)}-1\right)}  \Bigg\Vert\nabla F^{(k)}_n(\mathbf{w}^{(k)})
    \Bigg\Vert^2.
\end{align}
Replacing this result back in~\eqref{eq:res3} we have
 \begin{align}\label{eq:res5}
        (d) 
    \leq &{4 \beta^2\Theta^2 \eta_k^2}\sum_{n\in \mathcal{N}}\frac{\left(e_n^{(k)}\right)\left(e_n^{(k)}-1\right)}{1- 4\eta_k^2\beta^2 e_n^{(k)}\left(e_n^{(k)}-1\right)}\frac{\widehat{{D}}^{(k)}_n}{D^{(k)} e_n^{(k)}}
   %%%%%%%%%%%%%
    \sum_{j=1}^{{S}^{(k)}_{n}} \left(1-\frac{{B}^{(k)}_{n,j}}{{S}^{(k)}_{n,j}} \right) \frac{{S}^{(k)}_{n,j}}{\left(\widehat{{D}}^{(k)}_{n}\right)^2} \frac{{({S}^{(k)}_{n,j}-1)}
     \left(\widetilde{\sigma}_{n,j}^{(k)}\right)^2}{{B}^{(k)}_{n,j}}
     %%%%%%%%%
   \nonumber  \\&
     + {4 \eta_k^2\beta^2 }\sum_{n\in \mathcal{N}}\frac{\widehat{{D}}^{(k)}_n}{D^{(k)} e_n^{(k)}}\frac{\left(e_n^{(k)}\right)^2\left(e_n^{(k)}-1\right)}{{1- 4\eta_k^2\beta^2 e_n^{(k)}\left(e_n^{(k)}-1\right)}}  \Bigg\Vert\nabla F^{(k)}_n(\mathbf{w}^{(k)})
    \Bigg\Vert^2
    \nonumber \\&
    ={4 \beta^2\Theta^2 \eta_k^2}\sum_{n\in \mathcal{N}}\frac{\left(e_n^{(k)}\right)\left(e_n^{(k)}-1\right)}{1- 4\eta_k^2\beta^2 e_n^{(k)}\left(e_n^{(k)}-1\right)}\frac{\widehat{{D}}^{(k)}_n}{D^{(k)} e_n^{(k)}}
   %%%%%%%%%%%%%
    \sum_{j=1}^{{S}^{(k)}_{n}} \left(1-\frac{{B}^{(k)}_{n,j}}{{S}^{(k)}_{n,j}} \right) \frac{{S}^{(k)}_{n,j}}{\left(\widehat{{D}}^{(k)}_{n}\right)^2} \frac{{({S}^{(k)}_{n,j}-1)}
     \left(\widetilde{\sigma}_{n,j}^{(k)}\right)^2}{{B}^{(k)}_{n,j}}
     %%%%%%%%%
   \nonumber  \\&
     + {4 \eta_k^2\beta^2 }\sum_{n\in \mathcal{N}}\frac{\widehat{{D}}^{(k)}_n}{{D}^{(k)} }\frac{e_n^{(k)}\left(e_n^{(k)}-1\right)}{{1- 4\eta_k^2\beta^2 e_n^{(k)}\left(e_n^{(k)}-1\right)}}  \Bigg\Vert\nabla F^{(k)}_n(\mathbf{w}^{(k)})
    \Bigg\Vert^2
    \end{align}
    Assuming $e^{(k)}_{\mathsf{max}}=\max_{n\in\mathcal{N}}\{e^{(k)}_n\}$, and using the bounded dissimilarity assumption among the local gradients (Assumption~\ref{Assup:Dissimilarity}), we get
    \begin{align}\label{eq:res4444}
        (d) 
    \leq &{4 \beta^2\Theta^2 \eta_k^2}\sum_{n\in \mathcal{N}}\frac{\left(e_n^{(k)}\right)\left(e_n^{(k)}-1\right)}{1- 4\eta_k^2\beta^2 e_n^{(k)}\left(e_n^{(k)}-1\right)}\frac{\widehat{{D}}^{(k)}_n}{D^{(k)} e_n^{(k)}}
   %%%%%%%%%%%%%
    \sum_{j=1}^{{S}^{(k)}_{n}} \left(1-\frac{{B}^{(k)}_{n,j}}{{S}^{(k)}_{n,j}} \right) \frac{{S}^{(k)}_{n,j}}{\left(\widehat{{D}}^{(k)}_{n}\right)^2} \frac{{({S}^{(k)}_{n,j}-1)}
     \left(\widetilde{\sigma}_{n,j}^{(k)}\right)^2}{{B}^{(k)}_{n,j}}
     %%%%%%%%%
   \nonumber  \\&
     + \frac{4 \eta_k^2\beta^2 \left(e_{\mathsf{max}}^{(k)}\right)\left(e_{\mathsf{max}}^{(k)}-1\right)}{1- 4\eta_k^2\beta^2 e_{\mathsf{max}}^{(k)}\left(e_{\mathsf{max}}^{(k)}-1\right)}\left(\zeta_1  \Bigg\Vert \sum_{n\in \mathcal{N}}\frac{\widehat{{D}}^{(k)}_n}{{D}^{(k)} } \nabla F^{(k)}_n(\mathbf{w}^{(k)})
    \Bigg\Vert^2 + \zeta_2 \right)
    \nonumber  \\&
    ={4 \beta^2\Theta^2 \eta_k^2}\sum_{n\in \mathcal{N}}\frac{\left(e_n^{(k)}\right)\left(e_n^{(k)}-1\right)}{1- 4\eta_k^2\beta^2 e_n^{(k)}\left(e_n^{(k)}-1\right)}\frac{\widehat{{D}}^{(k)}_n}{D^{(k)} e_n^{(k)}}
   %%%%%%%%%%%%%
    \sum_{j=1}^{{S}^{(k)}_{n}} \left(1-\frac{{B}^{(k)}_{n,j}}{{S}^{(k)}_{n,j}} \right) \frac{{S}^{(k)}_{n,j}}{\left(\widehat{{D}}^{(k)}_{n}\right)^2} \frac{{({S}^{(k)}_{n,j}-1)}
     \left(\widetilde{\sigma}_{n,j}^{(k)}\right)^2}{{B}^{(k)}_{n,j}}
     %%%%%%%%%
   \nonumber  \\&
     + \frac{4 \eta_k^2\beta^2 \left(e_{\mathsf{max}}^{(k)}\right)\left(e_{\mathsf{max}}^{(k)}-1\right)}{1- 4\eta_k^2\beta^2 e_{\mathsf{max}}^{(k)}\left(e_{\mathsf{max}}^{(k)}-1\right)}\left(\zeta_1  \left\Vert \nabla F^{(k)}(\mathbf{w}^{(k)})
    \right\Vert^2 + \zeta_2 \right).
    \end{align}
 Replacing this result back in~\eqref{eq:secondlossIneq} and gathering the terms leads to
 \begin{align}\label{eq:thirdlossIneq}
\hspace{-10mm}
 \mathbb E_k &\left[F^{(k)} (\mathbf{w}^{(k+1)})\right] \leq F^{(k)}(\mathbf{w}^{(k)}) + \frac{\eta_{k}}{2}\sum_{n\in \mathcal{N}}\frac{\widehat{{D}}^{(k)}_{n}e^{(k)}_{n}}{{D}^{(k)} } \underbrace{\left(\frac{4 \eta_k^2\beta^2 \left(e_{\mathsf{max}}^{(k)}\right)\left(e_{\mathsf{max}}^{(k)}-1\right)}{1- 4\eta_k^2\beta^2 e_{\mathsf{max}}^{(k)}\left(e_{\mathsf{max}}^{(k)}-1\right)}\zeta_1-1 \right)}_{(g)}
 \left\Vert\nabla{F^{(k)}(\mathbf{w}^{(k)})}\right\Vert^2 \nonumber\\&\hspace{-6mm}
    + \frac{\eta_{k}}{2}\sum_{n'\in \mathcal{N}}\frac{\widehat{{D}}^{(k)}_{n'}e^{(k)}_{n'}}{{D}^{(k)} }  \vast({4 \beta^2\Theta^2 \eta_k^2}\sum_{n\in \mathcal{N}}\frac{\left(e_n^{(k)}\right)\left(e_n^{(k)}-1\right)}{{1- 4\eta_k^2\beta^2 e_n^{(k)}\left(e_n^{(k)}-1\right)}}\frac{\widehat{{D}}^{(k)}_n}{D^{(k)} e_n^{(k)}}
   %%%%%%%%%%%%%
    \sum_{j=1}^{{S}^{(k)}_{n}} \left(1-\frac{{B}^{(k)}_{n,j}}{{S}^{(k)}_{n,j}} \right) \frac{{S}^{(k)}_{n,j}}{\left(\widehat{{D}}^{(k)}_{n}\right)^2} \frac{{({S}^{(k)}_{n,j}-1)}
     \left(\widetilde{\sigma}_{n,j}^{(k)}\right)^2}{{B}^{(k)}_{n,j}}\hspace{-10mm}\nonumber\\
     &
     + \frac{4\zeta_2 \eta_k^2\beta^2 \left(e_{\mathsf{max}}^{(k)}\right)\left(e_{\mathsf{max}}^{(k)}-1\right)}{1- 4\eta_k^2\beta^2 e_{\mathsf{max}}^{(k)}\left(e_{\mathsf{max}}^{(k)}-1\right)}
     \vast)
     \nonumber\\&\hspace{-6mm}
     +2 \Theta^2 {\beta \eta^2_{k}}\left(\sum_{n'\in \mathcal{N}}\frac{\widehat{{D}}^{(k)}_{n'}e^{(k)}_{n'}}{{D}^{(k)} } \right)^2\sum_{n\in \mathcal{N}} \left(\frac{\widehat{{D}}^{(k)}_n}{{D}^{(k)} e^{(k)}_n} \right)^2 e^{(k)}_n~ \sum_{j=1}^{{S}^{(k)}_{n}} \left(1-\frac{{B}^{(k)}_{n,j}}{{S}^{(k)}_{n,j}} \right) \frac{{S}^{(k)}_{n,j}}{\left(\widehat{{D}}^{(k)}_{n}\right)^2} \frac{{({S}^{(k)}_{n,j}-1)}
     \left(\widetilde{\sigma}_{n,j}^{(k)}\right)^2}{{B}^{(k)}_{n,j}}.\hspace{-14mm}
    \end{align}
 With a proper choice of step size, we aim to make term $(g)$ negative, for which we first impose the following condition on the denominator of the fraction in $(g)$ (the following condition is equivalent to our prior condition: $\eta_k \leq \left(2\beta\sqrt{e_n^{(k)}(e_n^{(k)}-1)}\right)^{-1},\forall n$):
 \begin{align}\label{eq:firstConStep}
     1- 4\eta_k^2\beta^2 e_{\mathsf{max}}^{(k)}\left(e_{\mathsf{max}}^{(k)}-1\right) \geq 0\Rightarrow \eta_k \leq \frac{1}{2\beta} \sqrt{\frac{1}{ e_{\mathsf{max}}^{(k)}\left(e_{\mathsf{max}}^{(k)}-1\right)}}.
 \end{align}
 We also assume that there exist a set of constants $\Lambda^{(k)}$, $\forall k$, where
 \begin{align}
 \frac{4 \eta_k^2\beta^2 \left(e_{\mathsf{max}}^{(k)}\right)\left(e_{\mathsf{max}}^{(k)}-1\right)}{1- 4\eta_k^2\beta^2 e_{\mathsf{max}}^{(k)}\left(e_{\mathsf{max}}^{(k)}-1\right)}\zeta_1 \leq  \Lambda^{(k)} < 1\Rightarrow \frac{1}{1- 4\eta_k^2\beta^2 e_{\mathsf{max}}^{(k)}\left(e_{\mathsf{max}}^{(k)}-1\right)}\leq   \frac{\zeta_1+\Lambda^{(k)}}{\zeta_1},
 \end{align}
which can be obtained under the following condition on the step size:
\begin{align}
    \eta_k \leq \frac{1}{2\beta} \sqrt{ \frac{\Lambda^{(k)}}{(\zeta_1+\Lambda^{(k)})\left( e_{\mathsf{max}}^{(k)}\left(e_{\mathsf{max}}^{(k)}-1\right)\right)}},
\end{align}
which is a tighter condition on the step size as compared to~\eqref{eq:firstConStep}. Let us further define $\Gamma^{(k)}=\frac{\eta_{k}}{2}\sum_{n\in \mathcal{N}}\frac{\widehat{{D}}^{(k)}_{n}e^{(k)}_{n}}{{D}^{(k)} }$. Under these conditions/assumptions, we simplify~\eqref{eq:thirdlossIneq} as follows:
 \begin{align}\label{eq:final_one_step}
 &\left\Vert\nabla{F^{(k)}(\mathbf{w}^{(k)})}\right\Vert^2  \leq 
      \frac{F^{(k)}(\mathbf{w}^{(k)}) - \mathbb E_k \left[F^{(k)} (\mathbf{w}^{(k+1)})\right]}{\Gamma^{(k)}(1-\Lambda^{(k)})}
  \nonumber\\&\hspace{-6mm}
    + \frac{1}{(1-\Lambda^{(k)})}\Bigg({4 \beta^2\Theta^2 \eta_k^2 }\sum_{n\in \mathcal{N}}\frac{\widehat{{D}}^{(k)}_n}{{D}^{(k)} }( e_n^{(k)}-1) \left(
    \frac{\zeta_1+\Lambda^{(k)}}{\zeta_1}
    \right)
   %%%%%%%%%%%%%
    \sum_{j=1}^{{S}^{(k)}_{n}} \left(1-\frac{{B}^{(k)}_{n,j}}{{S}^{(k)}_{n,j}} \right) \frac{{S}^{(k)}_{n,j}}{\left(\widehat{{D}}^{(k)}_{n}\right)^2} \frac{{({S}^{(k)}_{n,j}-1)}
     \left(\widetilde{\sigma}_{n,j}^{(k)}\right)^2}{{B}^{(k)}_{n,j}}\hspace{-10mm}\nonumber \\&
     +{4\zeta_2 \eta_k^2\beta^2 \left(e_{\mathsf{max}}^{(k)}\right)\left(e_{\mathsf{max}}^{(k)}-1\right)}
     \frac{\zeta_1+\Lambda^{(k)}}{\zeta_1}
     \Bigg)
     \nonumber\\&\hspace{-6mm}
     +\frac{4 \Theta^2 {\beta \eta_{k}}}{(1-\Lambda^{(k)})}\left(\sum_{n'\in \mathcal{N}}\frac{\widehat{{D}}^{(k)}_{n'}e^{(k)}_{n'}}{{D}^{(k)} } \right)\sum_{n\in \mathcal{N}} \left(\frac{\widehat{{D}}^{(k)}_n}{{D}^{(k)} } \right)^2\frac{1}{e^{(k)}_n} \sum_{j=1}^{{S}^{(k)}_{n}} \left(1-\frac{{B}^{(k)}_{n,j}}{{S}^{(k)}_{n,j}} \right) \frac{{S}^{(k)}_{n,j}}{\left(\widehat{{D}}^{(k)}_{n}\right)^2} \frac{{({S}^{(k)}_{n,j}-1)}
     \left(\widetilde{\sigma}_{n,j}^{(k)}\right)^2}{{B}^{(k)}_{n,j}}.\hspace{-14mm}
 \end{align}
 
%  where $\Gamma^{(k)}=\frac{\eta_{k}}{2}\sum_{n\in \mathcal{N}}\frac{\widehat{{D}}^{(k)}_{n}e^{(k)}_{n}}{{D}^{(k)} }$.

%   The devices use $\mathbf{w}^{(0)}$ for model inference throughout the waiting period for which the first global aggregation starts and during the first global aggregation, which accounts for $T^{\mathsf{Tot},(1)}+\Omega^{(1)}$ amount of time. Based on a similar argument, the global parameter $\mathbf{w}^{(k)}$ will be used during  $T^{\mathsf{Tot},(k+1)}+\Omega^{(k+1)}$ time instances. Thus, the cumulative average of the global model performance at the devices can be written as follows:
%   \begin{align}\label{eq:timedomain}
%      \frac{1}{{K}} \sum_{k=0}^{K-1} \Vert \nabla F(\mathbf{w}({t})) \Vert^2=&\frac{1}{{K}} \Bigg[ \left(T^{\mathsf{Tot},(1)}+\Omega^{(1)} \right)\Vert \nabla F(\mathbf{w}^{(0)}) \Vert^2
%      \nonumber \\&+
%      \left(T^{\mathsf{Tot},(2)}+\Omega^{(2)} \right)\Vert \nabla F(\mathbf{w}^{(1)}) \Vert^2+
%       \left(T^{\mathsf{Tot},(3)}+\Omega^{(3)} \right)\Vert \nabla F(\mathbf{w}^{(2)}) \Vert^2
%      \nonumber  \\&+ \cdots+ \left(T^{\mathsf{Tot},(K)}+\Omega^{(K)} \right)\Vert \nabla F(\mathbf{w}^{(K-1)}) \Vert^2\Bigg]\\
%      &=\frac{1}{{K}} \sum_{k=0}^{K-1} \Vert \nabla F^{(k)}(\mathbf{w}^{(k)}) \Vert^2 .
%  \end{align}
%  FIX THIS UP! This SUM INDEX is NOT Consistent!!!!

 Taking total expectation and averaging across global aggregations, we have:
 \begin{align}\label{eq:final_one_stepC}
 &\frac{1}{K} \sum_{k=0}^{K-1}\mathbb E\left\Vert\nabla{F^{(k)}(\mathbf{w}^{(k)})}\right\Vert^2  \leq \underbrace{\frac{1}{K} \sum_{k=0}^{K-1} \left[
      \frac{ \mathbb E\left[ F^{(k)}(\mathbf{w}^{(k)})\right] - \mathbb E \left[F^{(k)} (\mathbf{w}^{(k+1)})\right]}{\Gamma^{(k)}(1-\Lambda^{(k)})}\right]}_{(a)}
  \nonumber\\&\hspace{-6mm}
    + \frac{1}{K} \sum_{k=0}^{K-1} \vast[ \frac{1}{(1-\Lambda^{(k)})}\Bigg({4 \beta^2\Theta^2 \eta_k^2 }\sum_{n\in \mathcal{N}}\frac{\widehat{{D}}^{(k)}_n}{{D}^{(k)} }( e_n^{(k)}-1) \left(
    \frac{\zeta_1+\Lambda^{(k)}}{\zeta_1}
    \right)
   %%%%%%%%%%%%%
    \sum_{j=1}^{{S}^{(k)}_{n}} \left(1-\frac{{B}^{(k)}_{n,j}}{{S}^{(k)}_{n,j}} \right) \frac{{S}^{(k)}_{n,j}}{\left(\widehat{{D}}^{(k)}_{n}\right)^2} \frac{{({S}^{(k)}_{n,j}-1)}
     \left(\widetilde{\sigma}_{n,j}^{(k)}\right)^2}{{B}^{(k)}_{n,j}}\hspace{-10mm}\nonumber \\&
     +{4\zeta_2 \eta_k^2\beta^2 \left(e_{\mathsf{max}}^{(k)}\right)\left(e_{\mathsf{max}}^{(k)}-1\right)}
     \frac{\zeta_1+\Lambda^{(k)}}{\zeta_1}
     \Bigg)
     \nonumber\\&\hspace{-6mm}
     +\frac{4 \Theta^2 {\beta \eta_{k}}}{(1-\Lambda^{(k)})}\left(\sum_{n'\in \mathcal{N}}\frac{\widehat{{D}}^{(k)}_{n'}e^{(k)}_{n'}}{{D}^{(k)} } \right)\sum_{n\in \mathcal{N}} \left(\frac{\widehat{{D}}^{(k)}_n}{{D}^{(k)} } \right)^2\frac{1}{e^{(k)}_n} \sum_{j=1}^{{S}^{(k)}_{n}} \left(1-\frac{{B}^{(k)}_{n,j}}{{S}^{(k)}_{n,j}} \right) \frac{{S}^{(k)}_{n,j}}{\left(\widehat{{D}}^{(k)}_{n}\right)^2} \frac{{({S}^{(k)}_{n,j}-1)}
     \left(\widetilde{\sigma}_{n,j}^{(k)}\right)^2}{{B}^{(k)}_{n,j}}\vast].\hspace{-14mm}
 \end{align}

 We next focus on the first term on the right hand side of~\eqref{eq:final_one_step} and aim to upper bound it. Let $t_k$ denotes the wall-clock time index (in seconds) in which global aggregation $k$ begins and $t'_k=t_k+T^{\mathsf{Tot},(k)}$ denotes the time in which global aggregation $k$ concludes. In particular, $t\in[t'_k,t_{k+1}]$ is the idle time in between global model training rounds of $k$ and $k+1$ and the length of the idle time is given by $\Omega^{(k+1)}=t_{k+1}-t'_k$. To avoid confusion between wall-clock time index and global aggregation index, we let $\widehat{\mathcal{D}}_n(t)$  denotes the dataset of device $n$ at wall-clock time $t$ after all the data arrivals and departures occur (and for global dataset $\widehat{\mathcal{D}}(t)$), while $\widehat{\mathcal{D}}_n^{(k)}$ denotes the respective dataset possessed by node $n$  used to obtain the model for global aggregation  $k+1$ (and for the global dataset $\widehat{\mathcal{D}}^{(k)}$). Note that we assume that the data arrival/departures only happen during the idle time in between the global aggregations, so during the period of each device acquisition time, the loss functions are stationary. Noting that $F^{(k)}(\mathbf{w}^{(k)})$ is the loss function under which global aggregation $k+1$ starts from (at $t_{k+1}$)\footnote{Note that we have assumed that to obtain the local model parameters to conduct the $k+1$-th global aggregation, each device $n$ uses $\widehat{\mathcal D}_n^{(k)}$.}, we bound it in terms of $F^{(k-1)}(\mathbf{w}^{(k)})$, i.e., the loss under which global aggregation $k$ concludes as follows:
\begin{align}\label{eq:driftLoss}
F^{(k)}(\mathbf{w}^{(k)})=    F(\mathbf{w}^{(k)} | \widehat{\mathcal{D}}^{(k)})&= F(\mathbf{w}^{(k)} | \widehat{\mathcal{D}}(t_{k+1}))= F(\mathbf{w}^{(k)} | \widehat{\mathcal{D}}(t_{k+1}))-F(\mathbf{w}^{(k)} | \widehat{\mathcal{D}}(t_{k+1}-1))+F(\mathbf{w}^{(k)} | \widehat{\mathcal{D}}(t_{k+1}-1))
    \nonumber \\&~~-F(\mathbf{w}^{(k)} | 
     \widehat{\mathcal{D}}(t_{k+1}-2))+F(\mathbf{w}^{(k)} | \widehat{\mathcal{D}}(t_{k+1}-2))-\cdots +F(\mathbf{w}^{(k)} | \widehat{\mathcal{D}}(t'_{k})) \nonumber\\&  
    %  \frac{1}{\widehat{D}^{(k+1)}} \sum_{n \in \mathcal{N}} \widehat{D}^{(k+1)}_n F_{n}(\mathbf{w}^{(k)}|\widehat{\mathcal{D}}^{(k+1)}_n)= \frac{1}{|\widehat{\mathcal{D}}(t_{k+1})|} \sum_{n \in \mathcal{N}} |\widehat{\mathcal{D}}_n(t_{k+1})|~ F_{n}(\mathbf{w}^{(k)}|\widehat{\mathcal{D}}_n{(t_{k+1})})
    = \sum_{n \in \mathcal{N}} \Bigg[\frac{|\widehat{\mathcal{D}}_n(t_{k+1})|}{|\widehat{\mathcal{D}}(t_{k+1})|}F_{n}(\mathbf{w}^{(k)}|\widehat{\mathcal{D}}_n{(t_{k+1})})
   -
   \frac{|\widehat{\mathcal{D}}_n(t_{k+1}-1)|}{|\widehat{\mathcal{D}}(t_{k+1}-1)|}F_{n}(\mathbf{w}^{(k)}|\widehat{\mathcal{D}}_n{(t_{k+1}-1)})
 \nonumber  \\& ~~+
   \frac{|\widehat{\mathcal{D}}_n(t_{k+1}-1)|}{|\widehat{\mathcal{D}}(t_{k+1}-1)|}F_{n}(\mathbf{w}^{(k)}|\widehat{\mathcal{D}}_n{(t_{k+1}-1)})-\cdots + \frac{|\widehat{\mathcal{D}}_n(t'_{k})|}{|\widehat{\mathcal{D}}(t'_{k})|}F_{n}(\mathbf{w}^{(k)}|\widehat{\mathcal{D}}_n{(t'_{k})})\Bigg]
   \nonumber 
   \\&\leq (t_{k+1}-t'_{k}) \sum_{n\in \mathcal{N}} \Delta_n^{(k+1)} + \sum_{n\in \mathcal{N}} \frac{|\widehat{\mathcal{D}}_n(t'_{k})|}{|\widehat{\mathcal{D}}(t'_{k})|}F_{n}(\mathbf{w}^{(k)}|\widehat{\mathcal{D}}_n{(t'_{k})})\nonumber \\
   &= \Omega^{(k+1)} \sum_{n\in \mathcal{N}} \Delta_n^{(k+1)} + \sum_{n\in \mathcal{N}} \frac{|\widehat{\mathcal{D}}_n^{(k-1)}|}{|\widehat{\mathcal{D}}^{(k-1)}|}F_{n}(\mathbf{w}^{(k)}|\widehat{\mathcal{D}}_n^{(k-1)})
 \nonumber   \\&=\Omega^{(k+1)} \sum_{n\in \mathcal{N}} \Delta_n^{(k+1)} + F(\mathbf{w}^{(k)}|\widehat{\mathcal{D}}^{(k-1)})=\Omega^{(k+1)}  \Delta^{(k+1)} + F^{(k-1)}(\mathbf{w}^{(k)}),
   \end{align}
where $\Delta_n^{(k+1)}=\max_{t\in T^{\mathsf{Idle},(k+1)}} \Delta_n(t)$, $\forall n$, $T^{\mathsf{Idle},(k+1)}$ denotes the set of idle time instances between global aggregations $k+1$ and $k$), $\Delta^{(k+1)}=\sum_{n\in \mathcal{N}}\Delta_n^{(k+1)}$. Let $F^{(-1)}(\mathbf{w}^{(0)})$ denote the initial loss of the algorithm before model training starts, using the above upper bound we have $F^{(0)}(\mathbf{w}^{(0)}) \leq F^{(-1)}(\mathbf{w}^{(0)})+\Omega^{(1)}  \Delta^{(1)}$, where $F^{(0)}(\mathbf{w}^{(0)})$ is the initial loss that the first global model training round starts from after having the idle period of $\Omega^{(1)}$. 

% Also using the above bound recursively in the term $(a)$ of~\eqref{eq:final_one_stepC}, we get

% In the following, we use $F^{(k)}(\mathbf{w}^{(k)})=F(\mathbf{w}^{(k)}|\widehat{\mathcal{D}}^{(k)})$ to denote the loss experienced in the conclusion of the $k$th global iteration and thus eliminate the discrepancy in the notations that we initially used for global loss as explain in Remark~\ref{remark:modelDisc}.

Using~\eqref{eq:final_one_step} and~\eqref{eq:driftLoss}, applying $\frac{\zeta_1+\Lambda^{(k)}}{\zeta_1} < 2$ yields
\vspace{-4mm}

% {\small
 \begin{align}\label{eq:th1proof}
 &\hspace{-16mm}\frac{1}{K} \sum_{k=0}^{K-1}\mathbb E \Vert \nabla F^{(k)}(\mathbf{w}^{({k})}) \Vert^2 \leq \frac{1}{K} \vast[
 \sum_{k=0}^{K-1}\frac{\mathbb E\left[{F}^{(k-1)}(\mathbf{w}^{(k)})\right] - \mathbb E\left[F^{(k)}(\mathbf{w}^{(k+1)})\right]}{\Gamma^{(k)}(1-\Lambda^{(k)})}
 +\sum_{k=0}^{K-1}\frac{\Omega^{(k+1)} \Delta^{(k+1)} }{\Gamma^{(k)}(1-\Lambda^{(k)})}
  \hspace{-16mm} \nonumber 
   \\&\hspace{-16mm}+\sum_{k=0}^{K-1} \frac{1}{(1-\Lambda^{(k)})}\Bigg({8 \beta^2\Theta^2 \eta_k^2 }\sum_{n\in \mathcal{N}}\frac{\widehat{{D}}^{(k)}_n}{{D}^{(k)} }( e_n^{(k)}-1)   %%%%%%%%%%%%%
    \sum_{j=1}^{{S}^{(k)}_{n}} \left(1-\frac{{B}^{(k)}_{n,j}}{{S}^{(k)}_{n,j}} \right) \frac{{S}^{(k)}_{n,j}}{\left(\widehat{{D}}^{(k)}_{n}\right)^2} \frac{{({S}^{(k)}_{n,j}-1)}
     \left(\widetilde{\sigma}_{n,j}^{(k)}\right)^2}{{B}^{(k)}_{n,j}}\hspace{-14mm}\nonumber \hspace{-16mm}
     \\&
     \hspace{-16mm}+{8\zeta_2 \eta_k^2\beta^2 \left(e_{\mathsf{max}}^{(k)}\right)\left(e_{\mathsf{max}}^{(k)}-1\right)}
     \Bigg)
     \hspace{-16mm} \nonumber \\&
     \hspace{-16mm}+\sum_{k=0}^{K-1}\frac{8 \Theta^2 {\beta \Gamma^{(k)} }}{(1-\Lambda^{(k)})}\sum_{n\in \mathcal{N}} \left(\frac{\widehat{{D}}^{(k)}_n}{{D}^{(k)} } \right)^2\frac{1}{e^{(k)}_n} \sum_{j=1}^{{S}^{(k)}_{n}} \left(1-\frac{{B}^{(k)}_{n,j}}{{S}^{(k)}_{n,j}} \right) \frac{{S}^{(k)}_{n,j}}{\left(\widehat{{D}}^{(k)}_{n}\right)^2} \frac{{({S}^{(k)}_{n,j}-1)}
     \left(\widetilde{\sigma}_{n,j}^{(k)}\right)^2}{{B}^{(k)}_{n,j}}
 \vast],\hspace{-6mm}
 \end{align}
%  }
%  \vspace{-4mm}
  which concludes the proof.
 
 \newpage

 \section{Proof of Corollary~\ref{cor:1}}\label{app:cor:1}
 Considering~\eqref{eq:th1proof}, we have
 \begin{align}\label{eq:th1proofC}
 &\hspace{-16mm}\frac{1}{K} \sum_{k=0}^{K-1}\mathbb E \Vert \nabla F^{(k)}(\mathbf{w}^{({k})}) \Vert^2 \leq \frac{1}{K} \vast[
 \sum_{k=0}^{K-1}\frac{\mathbb E\left[{F}^{(k-1)}(\mathbf{w}^{(k)})\right] - \mathbb E\left[F^{(k)}(\mathbf{w}^{(k+1)})\right]}{\left(\eta_k e_{\mathsf{avg}}^{(k)}/2\right)(1-\Lambda^{(k)})}
 +\sum_{k=0}^{K-1}\frac{\Omega^{(k+1)} \Delta^{(k+1)} }{\left(\eta_k e_{\mathsf{avg}}^{(k)}/2\right)(1-\Lambda^{(k)})}
  \hspace{-16mm} \nonumber 
   \\&\hspace{-16mm}+\sum_{k=0}^{K-1} \frac{1}{(1-\Lambda^{(k)})}\Bigg({8 \beta^2\Theta^2 \eta_k^2 }\sum_{n\in \mathcal{N}}\frac{\widehat{{D}}^{(k)}_n}{{D}^{(k)} }( e_n^{(k)}-1)   %%%%%%%%%%%%%
    \sum_{j=1}^{{S}^{(k)}_{n}} \left(1-\frac{{B}^{(k)}_{n,j}}{{S}^{(k)}_{n,j}} \right) \frac{{S}^{(k)}_{n,j}}{\left(\widehat{{D}}^{(k)}_{n}\right)^2} \frac{{({S}^{(k)}_{n,j}-1)}
     \left(\widetilde{\sigma}_{n,j}^{(k)}\right)^2}{{B}^{(k)}_{n,j}}\hspace{-14mm}\nonumber \hspace{-16mm}
     \\&
     \hspace{-16mm}+{8\zeta_2 \eta_k^2\beta^2 \left(e_{\mathsf{max}}^{(k)}\right)\left(e_{\mathsf{max}}^{(k)}-1\right)}
     \Bigg)
     \hspace{-16mm} \nonumber \\&
     \hspace{-16mm}+\sum_{k=0}^{K-1}\frac{8 \Theta^2 {\beta \left(\eta_k e_{\mathsf{avg}}^{(k)}/2\right) }}{(1-\Lambda^{(k)})}\sum_{n\in \mathcal{N}} \left(\frac{\widehat{{D}}^{(k)}_n}{{D}^{(k)} } \right)^2\frac{1}{e^{(k)}_n} \sum_{j=1}^{{S}^{(k)}_{n}} \left(1-\frac{{B}^{(k)}_{n,j}}{{S}^{(k)}_{n,j}} \right) \frac{{S}^{(k)}_{n,j}}{\left(\widehat{{D}}^{(k)}_{n}\right)^2} \frac{{({S}^{(k)}_{n,j}-1)}
     \left(\widetilde{\sigma}_{n,j}^{(k)}\right)^2}{{B}^{(k)}_{n,j}}
 \vast].\hspace{-6mm}
 \end{align}
 Assuming $\eta_k = \alpha \frac{1}{\sqrt{{e}^{(k)}_{\mathsf{sum}} K/N}}$ with a finite positive constant $\alpha$, $\max_k \left\{\Lambda^{(k)}\right\} \leq \Lambda_{\mathsf{max}}< 1$, 
 $( \overline{e}_{\mathsf{max}})^{-1} \leq \left(e^{(k)}_{\mathsf{avg}}\right)^{-1}\leq ( \overline{e}_{\mathsf{min}})^{-1}$, where $e^{(k)}_{\mathsf{avg}}=\sum_{n\in \mathcal{N}}\frac{\widehat{{D}}^{(k)}_{n}e^{(k)}_{n}}{{D}^{(k)} }$, and $ \widehat{e}_{\mathsf{min}} \leq e^{(k)}_{\mathsf{sum}}\leq  \widehat{e}_{\mathsf{max}}$,
 we get
%  \begin{align}
%  &\frac{1}{\Upsilon T^{\mathsf{ML}}} \sum_{k=0}^{K-1} \Vert \nabla F(\mathbf{w}({t})) \Vert^2 \leq \frac{1}{{K}} \vast[
%  \sum_{k=0}^{K-1}\frac{\left({F}(\mathbf{w}^{(k)}) - F^{(k)}(\mathbf{w}^{(k+1)})\right)}{\Gamma^{(k)}(1-\Lambda^{(k)})}
%  +\sum_{k=0}^{K-1}\frac{\Omega^{(k+1)} \Delta^{(k+1)} }{\Gamma^{(k)}(1-\Lambda^{(k)})}
%   \nonumber 
%   \hspace{-6mm} \\&+\sum_{k=0}^{K-1} \frac{1}{(1-\Lambda^{(k)})}\Bigg({8 \beta^2\Theta^2 \eta_k^2 }\sum_{n\in \mathcal{N}}\frac{\widehat{{D}}^{(k)}_n}{{D}^{(k)} }( e_n^{(k)}-1) \left(
%     \frac{\zeta_1+\Lambda^{(k)}}{\zeta_1}
%     \right)
%   %%%%%%%%%%%%%
%     \sum_{j=1}^{{S}^{(k)}_{n}} \left(1-\frac{{B}^{(k)}_{n,j}}{{S}^{(k)}_{n,j}} \right) \frac{{S}^{(k)}_{n,j}}{\left(\widehat{{D}}^{(k)}_{n}\right)^2} \frac{{({S}^{(k)}_{n,j}-1)}
%      \left(\widetilde{\sigma}_{n,j}^{(k)}\right)^2}{{B}^{(k)}_{n,j}}\nonumber \hspace{-6mm}\\&
%      +{4\zeta_2 \eta_k^2\beta^2 \left(e_{\mathsf{max}}^{(k)}\right)\left(e_{\mathsf{max}}^{(k)}-1\right)}
%      \frac{\zeta_1+\Lambda^{(k)}}{\zeta_1}
%      \Bigg)
%      \nonumber\\&
%      +\sum_{k=0}^{K-1}\frac{4 \Theta^2 {\beta \Gamma^{(k)} }}{(1-\Lambda^{(k)})}\sum_{n\in \mathcal{N}} \left(\frac{\widehat{{D}}^{(k)}_n}{{D}^{(k)} } \right)^2\frac{1}{e^{(k)}_n} \sum_{j=1}^{{S}^{(k)}_{n}} \left(1-\frac{{B}^{(k)}_{n,j}}{{S}^{(k)}_{n,j}} \right) \frac{{S}^{(k)}_{n,j}}{\left(\widehat{{D}}^{(k)}_{n}\right)^2} \frac{{({S}^{(k)}_{n,j}-1)}
%      \left(\widetilde{\sigma}_{n,j}^{(k)}\right)^2}{{B}^{(k)}_{n,j}}
%  \vast].\hspace{-6mm}
%  \end{align}
 
%  $\max{\Gamma^{(k)}}\leq {\Gamma^{\mathsf{max}}}$
 \begin{align}\label{eq:cor1proof}
 &\frac{1}{K} \sum_{k=0}^{K-1} \mathbb E\left[\Vert \nabla F^{(k)}(\mathbf{w}^{({k})}) \Vert^2\right] \leq 
      2 \sqrt{\widehat{{e}}_{\mathsf{max}}} \frac{F^{(-1)}(\mathbf{w}^{(0)}) - F^{(K)^\star}}{\overline{e}_{\mathsf{min}}\alpha \sqrt{NK}(1-\Lambda_{\mathsf{max}})}
      +\frac{ 2 \sqrt{\widehat{{e}}_{\mathsf{max}}}}{\overline{e}_{\mathsf{min}}\alpha \sqrt{NK}}\sum_{k=0}^{K-1}\frac{\Omega^{(k+1)} \Delta^{(k+1)} }{1-\Lambda_{\mathsf{max}}}
   \nonumber \\&+\frac{1}{K}\sum_{k=0}^{K-1} \frac{1}{(1-\Lambda_{\mathsf{max}})}\Bigg({8 \beta^2\Theta^2 \frac{\alpha^2 N}{K e^{(k)}_{\mathsf{sum}}} }\sum_{n\in \mathcal{N}}\frac{\widehat{{D}}^{(k)}_n}{{D}^{(k)} }( e_n^{(k)}-1) 
   %%%%%%%%%%%%%
    \sum_{j=1}^{{S}^{(k)}_{n}} \left(1-\frac{{B}^{(k)}_{n,j}}{{S}^{(k)}_{n,j}} \right) \frac{{S}^{(k)}_{n,j}}{\left(\widehat{{D}}^{(k)}_{n}\right)^2} \frac{{({S}^{(k)}_{n,j}-1)}
     \left(\widetilde{\sigma}_{n,j}^{(k)}\right)^2}{{B}^{(k)}_{n,j}} \hspace{-10mm}\nonumber \\&
     +{8\zeta_2 \frac{\alpha^2 N}{K e^{(k)}_{\mathsf{sum}}}\beta^2 \left(e_{\mathsf{max}}^{(k)}\right)\left(e_{\mathsf{max}}^{(k)}-1\right)}
     \Bigg)
     \nonumber\\&\hspace{-6mm}
     +\frac{1}{K}\sum_{k=0}^{K-1}\frac{4\overline{e}_{\mathsf{max}}\alpha \Theta^2 {\beta \sqrt{N}}}{(1-\Lambda_{\mathsf{max}})\sqrt{{e}^{(k)}_{\mathsf{sum}}}\sqrt{K}}\sum_{n\in \mathcal{N}} \left(\frac{\widehat{{D}}^{(k)}_n}{{D}^{(k)} } \right)^2\frac{1}{e^{(k)}_n} \sum_{j=1}^{{S}^{(k)}_{n}} \left(1-\frac{{B}^{(k)}_{n,j}}{{S}^{(k)}_{n,j}} \right) \frac{{S}^{(k)}_{n,j}}{\left(\widehat{{D}}^{(k)}_{n}\right)^2} \frac{{({S}^{(k)}_{n,j}-1)}
     \left(\widetilde{\sigma}_{n,j}^{(k)}\right)^2}{{B}^{(k)}_{n,j}},\hspace{-14mm}
 \end{align}
 which concludes the proof.
 
\newpage
 \section{Proof of Corollary~\ref{cor:2}}\label{app:cor:2}
 
Considering~\eqref{eq:cor1proof}, assuming a bounded stratified sampling noise: $\max_{k,n}\left\{\sum_{j=1}^{{S}^{(k)}_{n}} \left(1-\frac{{B}^{(k)}_{n,j}}{{S}^{(k)}_{n,j}} \right) \frac{{S}^{(k)}_{n,j}}{\left(\widehat{{D}}^{(k)}_{n}\right)^2} \frac{{({S}^{(k)}_{n,j}-1)}
     \left(\widetilde{\sigma}_{n,j}^{(k)}\right)^2}{{B}^{(k)}_{n,j}}\right\}\leq \sigma_{\mathsf{max}}$, $\forall n,k$, bounded local iterations, $\max_{k}\{ e^{(k)}_{\mathsf{max}}\}\leq e_{\mathsf{max}} $,  $ \widehat{e}_{\mathsf{min}} \leq e^{(k)}_{\mathsf{sum}}\leq  \widehat{e}_{\mathsf{max}}$, and $\Delta^{(k)}\leq \left[\frac{\gamma}{K\Omega^{(k)}}\right]^{+}$, $\forall k$,  for a finite non-negative constant $\gamma$, we have
     \begin{align}\label{eq:cor2proof}
 &\frac{1}{K} \sum_{k=0}^{K-1} \mathbb E\left[\Vert \nabla F^{(k)}(\mathbf{w}^{({k})}) \Vert^2\right] \leq 
      2  \sqrt{\widehat{e}_{\mathsf{max}}} \frac{F^{(-1)}(\mathbf{w}^{(0)}) - F^{(K)^\star}}{\overline{e}_{\mathsf{min}}\alpha \sqrt{NK}(1-\Lambda_{\mathsf{max}})}
      +\frac{ 2   \sqrt{\widehat{e}_{\mathsf{max}}}}{\overline{e}_{\mathsf{min}}\alpha \sqrt{NK}}\sum_{k=0}^{K-1}\frac{\gamma}{K(1-\Lambda_{\mathsf{max}})}
   \nonumber \\&+\frac{1}{{K}}\sum_{k=0}^{K-1} \frac{1}{(1-\Lambda_{\mathsf{max}})}\Bigg({8 \beta^2\Theta^2 \frac{\alpha^2N}{\widehat{e}_{\mathsf{min}}K} }( e_{\mathsf{max}}-1) \sigma_{\mathsf{max}}+{8\zeta_2 \frac{\alpha^2N}{\widehat{e}_{\mathsf{min}}K}\beta^2 \left(e_{\mathsf{max}}\right)\left(e_{\mathsf{max}}-1\right)}
     \Bigg)
   %%%%%%%%%%%%%
     \nonumber\\&
     +\frac{1}{{K}}\sum_{k=0}^{K-1}\frac{4\overline{e}_{\mathsf{max}}\alpha \Theta^2 {\beta \sqrt{N}}}{(1-\Lambda_{\mathsf{max}})\sqrt{\widehat{e}_{\mathsf{min}}}\sqrt{K}} \sigma_{\mathsf{max}} 
     %%%%%%%%%%%%%%%%%%%%%%%%%%%%%%%%%%%%%%%%%%%%%%%%%%%%
     %%%%%%%%%%%%%%%%%%%%%%%%%%%%%%%%%%%%%%%%%%%%%%%%%%%%
    \nonumber \\ &\leq 
      2 \sqrt{\widehat{e}_{\mathsf{max}}} \frac{F^{(-1)}(\mathbf{w}^{(0)}) - F^{(K)^\star}}{\overline{e}_{\mathsf{min}}\alpha \sqrt{N K}(1-\Lambda_{\mathsf{max}})}
      +\frac{ 2 \sqrt{\widehat{e}_{\mathsf{max}}}\gamma  }{\overline{e}_{\mathsf{min}}\alpha \sqrt{N K}(1-\Lambda_{\mathsf{max}})}
   \nonumber \\&+\frac{1}{{K}}\sum_{k=0}^{K-1} \frac{1}{(1-\Lambda_{\mathsf{max}})}\Bigg({8 \beta^2\Theta^2 \frac{\alpha^2N}{\widehat{e}_{\mathsf{min}} K} }( e_{\mathsf{max}}-1) \sigma_{\mathsf{max}}+{8\zeta_2 \frac{\alpha^2N}{\widehat{e}_{\mathsf{min}} K}\beta^2 \left(e_{\mathsf{max}}\right)\left(e_{\mathsf{max}}-1\right)}
    \Bigg)
   %%%%%%%%%%%%%
     \nonumber\\&
     +\frac{1}{{K}}\sum_{k=0}^{K-1}\frac{4\overline{e}_{\mathsf{max}}\alpha \Theta^2 {\beta \sqrt{N} }}{(1-\Lambda_{\mathsf{max}})\sqrt{\widehat{e}_{\mathsf{min}}}\sqrt{K}} \sigma_{\mathsf{max}}  
    \nonumber 
    \\ &=
    2 \sqrt{\widehat{e}_{\mathsf{max}}}  \frac{F^{(-1)}(\mathbf{w}^{(0)}) - F^{(K)^\star}}{\overline{e}_{\mathsf{min}}\alpha \sqrt{N K}(1-\Lambda_{\mathsf{max}})}
      +\frac{ 2\sqrt{\widehat{e}_{\mathsf{max}}} \gamma   }{\overline{e}_{\mathsf{min}}\alpha \sqrt{N K}(1-\Lambda_{\mathsf{max}})}
   \nonumber \\&+\frac{1}{{K}}\frac{1}{(1-\Lambda_{\mathsf{max}})}\Bigg({8 \beta^2\Theta^2 {\alpha^2} N}( e_{\mathsf{max}}-1) \sigma_{\mathsf{max}}/\widehat{e}_{\mathsf{min}}+{8\zeta_2 {\alpha^2}\beta^2 N \left(e_{\mathsf{max}}\right)\left(e_{\mathsf{max}}-1\right)}/\widehat{e}_{\mathsf{min}}
     \Bigg)
   %%%%%%%%%%%%%
     \nonumber\\&
     +\frac{4\overline{e}_{\mathsf{max}}\alpha \Theta^2 {\beta  \sqrt{N} }}{(1-\Lambda_{\mathsf{max}})\sqrt{\widehat{e}_{\mathsf{min}}}\sqrt{K}}  \sigma_{\mathsf{max}} ,
  \end{align}
  which concludes the proof.
%   where in $(i)$ we have used the fact that $K \leq T^{\mathsf{ML}}$: the number of global aggregations is always less than or equal to the ML model training time.
  
  \newpage
%      \begin{align}
%      &\frac{1}{K}\sum_{k=0}^{K-1}\left\Vert\nabla{F^{(k)}(\mathbf{w}^{(k)})}\right\Vert^2  \leq 
%       2\overline{e}_{\mathsf{max}}\frac{F^{(-1)}(\mathbf{w}^{(0)}) - F^{(K)^\star}}{\Omega \sqrt{K}(1-\Lambda_{\mathsf{max}})}
%   \nonumber \\&+\frac{1}{K} \frac{1}{(1-\Lambda_{\mathsf{max}})}\Bigg({8 \beta^2\Theta^2 {\Omega^2} }
%   |\mathcal{N}| ( e_{\mathsf{max}}-1) \left(
%     {1+\Lambda_{\mathsf{max}}}
%     \right)
%   %%%%%%%%%%%%%
%     \sigma_{\mathsf{max}}
%      +{4\zeta_2 {\Omega^2}\beta^2 \left(e_{\mathsf{max}}\right)\left(e_{\mathsf{max}}-1\right)}
%     \left( 1+{\Lambda_{\mathsf{max}}}\right)
%      \Bigg)
%      \nonumber\\&\hspace{-6mm}
%      +\frac{1}{\sqrt{K}}\frac{4 \Omega \Theta^2 {\beta }}{(1-\Lambda_{\mathsf{max}})}|\mathcal{N}| \sigma_{\mathsf{max}}.\hspace{-14mm}
%      \end{align}
    %  which implies the convergence of the normalized cumulative average gradient value with rate of $\mathcal{O}(1/\sqrt{K})$. 
%   \pagebreak
 
 \section{Proof of Proposition~\ref{prop:neyman}}\label{app:prop:neyman}
 We focus on the term described in~\eqref{eq:A2}, and bound the following summation:
 \begin{align}
     \sum_{j=1}^{{S}^{(k)}_{n}} \left(1-\frac{{B}^{(k)}_{n,j}}{{S}^{(k)}_{n,j}} \right) \frac{{S}^{(k)}_{n,j}}{\left(\widehat{{D}}^{(k)}_{n}\right)^2} \frac{{({S}^{(k)}_{n,j}-1)}
     \left(\widetilde{\sigma}_{n,j}^{(k)}\right)^2}{{B}^{(k)}_{n,j}}.
 \end{align}
 Since the dataset is fixed during each round of local model training and the sampling of the mini-batches from different strata is conducted with respect to the variance of the data at each minibatch, we can omit the index of local SGD iteration $e'$ from the above expression to get
 \begin{align}\label{eq:neymanpre}
    \sum_{j=1}^{{S}^{(k)}_{n}} \left(1-\frac{{B}^{(k)}_{n,j}}{{S}^{(k)}_{n,j}} \right) \frac{{S}^{(k)}_{n,j}}{\left(\widehat{{D}}^{(k)}_{n}\right)^2} \frac{{({S}^{(k)}_{n,j}-1)}
     \left(\widetilde{\sigma}_{n,j}^{(k)}\right)^2}{{B}^{(k)}_{n,j}}\leq    \sum_{j=1}^{{S}^{(k)}_{n}} \left(1-\frac{{B}^{(k)}_{n,j}}{{S}^{(k)}_{n,j}} \right) \frac{\left({S}^{(k)}_{n,j}\right)^2}{\left(\widehat{{D}}^{(k)}_{n}\right)^2} \frac{
     \left(\widetilde{\sigma}_{n,j}^{(k)}\right)^2}{{B}^{(k)}_{n,j}}.
 \end{align}
 Exploiting a similar technique used in~\cite{neyman1992two} (see Eq. (39) of~\cite{neyman1992two}), it can be shown that the the right hand side of~\eqref{eq:neymanpre} can be written as
 \begin{align}
     \sum_{j=1}^{{S}^{(k)}_{n}} \left(1-\frac{{B}^{(k)}_{n,j}}{{S}^{(k)}_{n,j}} \right) \frac{\left({S}^{(k)}_{n,j}\right)^2}{\left(\widehat{{D}}^{(k)}_{n}\right)^2} \frac{
     \left(\widetilde{\sigma}_{n,j}^{(k)}\right)^2}{{B}^{(k)}_{n,j}}=&\frac{\widehat{{D}}^{(k)}_{n}-{B}_n^{(k)}}{{B}_n^{(k)}\left(\widehat{{D}}^{(k)}_{n}\right)^2}\sum_{j=1}^{{S}^{(k)}_{n}} {S}^{(k)}_{n,j} \left(\widetilde{\sigma}_{n,j}^{(k)}\right)^2  \nonumber \\&+\sum_{j=1}^{{S}^{(k)}_{n}} {B}^{(k)}_{n,j}\left(\frac{{S}^{(k)}_{n,j} \widetilde{\sigma}_{n,j}^{(k)}}{{B}^{(k)}_{n,j}\left(\widehat{{D}}^{(k)}_{n}\right)^2}
    -\frac{\sum_{\mathcal{S}^{(k)}_{n,j'}\in \mathcal{S}^{(k)}_{n}} {S}^{(k)}_{n,j'}\widetilde{\sigma}_{n,j'}^{(k)}}{{B}_n^{(k)}\left(\widehat{{D}}^{(k)}_{n}\right)^2} \right)^2 
   \nonumber  \\&-\frac{1}{{B}_n^{(k)}\widehat{{D}}^{(k)}_{n}}\sum_{j=1}^{{S}^{(k)}_{n}} {S}^{(k)}_{n,j}\left(\widetilde{\sigma}_{n,j}^{(k)}-\frac{\sum_{\mathcal{S}^{(k)}_{n,j'}\in \mathcal{S}^{(k)}_{n}}{S}^{(k)}_{n,j'} \widetilde{\sigma}_{n,j'}^{(k)}}{\widehat{{D}}^{(k)}_{n}} \right)^2.
 \end{align}
 Considering the fact that only the term on the second line is dependent on the mini-batch size (${B}^{(k)}_{n,j}$),
 minimization of this expression can be achieved with the choice of mini-batch size $ {B}^{(k)}_{n,j} = \frac{B^{(k)}_n\widetilde{\sigma}^{(k)}_{n,j}{S}^{(k)}_{n,j} }{\sum_{\mathcal{S}^{(k)}_{n,j'}\in \mathcal{S}^{(k)}_{n}} \widetilde{\sigma}^{(k)}_{n,j'}\vert \mathcal{S}^{(k)}_{n,j'}\vert }$, $\forall n,j$. 
 We subsequently simplify the above expression as follows:
 \begin{align}\label{eq:109}
 \begin{aligned}
      &\sum_{j=1}^{{S}^{(k)}_{n}} \left(1-\frac{{B}^{(k)}_{n,j}}{{S}^{(k)}_{n,j}} \right) \frac{\left({S}^{(k)}_{n,j}\right)^2}{\left(\widehat{{D}}^{(k)}_{n}\right)^2} \frac{
     \left(\widetilde{\sigma}_{n,j}^{(k)}\right)^2}{{B}^{(k)}_{n,j}}=\sum_{j=1}^{{S}^{(k)}_{n}} \left(\frac{1}{{B}^{(k)}_{n,j}}-\frac{1}{{S}^{(k)}_{n,j}} \right)\frac{\left({S}^{(k)}_{n,j}\right)^2}{\left(\widehat{{D}}^{(k)}_{n}\right)^2} {\left(\widetilde{\sigma}_{n,j}^{(k)}\right)^2}
     \\
     &\overset{(i)}{=} \sum_{j=1}^{{S}^{(k)}_{n}} \frac{1}{\left(\widehat{D}_n^{(k)}\right)^2}\left[ \left(\frac{\sum_{\mathcal{S}^{(k)}_{n,j'}\in \mathcal{S}^{(k)}_{n}} ~\vert \mathcal{S}^{(k)}_{n,j'}\vert \widetilde{\sigma}^{(k)}_{n,j'}}{ {B}_n^{(k)}~{S}^{(k)}_{n,j} ~\widetilde{\sigma}^{(k)}_{n,j}}\right)\left({S}^{(k)}_{n,j}\right)^2 {\left(\widetilde{\sigma}^{(k)}_{n,j}\right)^2}-{S}^{(k)}_{n,j} {\left(\widetilde{\sigma}^{(k)}_{n,j}\right)^2}\right]
     \\&
    = \sum_{j=1}^{{S}^{(k)}_{n}} \frac{1}{\left(\widehat{D}_n^{(k)}\right)^2}\left[ \left(\frac{\sum_{\mathcal{S}^{(k)}_{n,j'}\in \mathcal{S}^{(k)}_{n}} \widetilde{\sigma}^{(k)}_{n,j'}\vert \mathcal{S}^{(k)}_{n,j'}\vert }{ {B}_n^{(k)} }\right){S}^{(k)}_{n,j} {\widetilde{\sigma}^{(k)}_{n,j}}-{S}^{(k)}_{n,j} {\left(\widetilde{\sigma}^{(k)}_{n,j}\right)^2}\right]
     \\&
    = \frac{1}{\left(\widehat{D}_n^{(k)}\right)^2}\left[ \frac{1}{{B}_n^{(k)}} \left(\sum_{j=1}^{{S}^{(k)}_{n}} \widetilde{\sigma}^{(k)}_{n,j}{S}^{(k)}_{n,j} \right)^2-\sum_{j=1}^{{S}^{(k)}_{n}}{S}^{(k)}_{n,j} {\left(\widetilde{\sigma}^{(k)}_{n,j}\right)^2}\right]
     \end{aligned}
     \hspace{-7mm}
 \end{align}
where in equality~$(i)$ we have used the sampling rule: $ {B}^{(k)}_{n,j} = \frac{B^{(k)}_n\widetilde{\sigma}^{(k)}_{n,j}{S}^{(k)}_{n,j} }{\sum_{\mathcal{S}^{(k)}_{n,j'}\in \mathcal{S}^{(k)}_{n}} \widetilde{\sigma}^{(k)}_{n,j'}\vert \mathcal{S}^{(k)}_{n,j'}\vert }$. It is worth noting that the only control parameter in the above inequality is the mini-batch size ${B}_n^{(k)}$ appearing in the first term inside the bracket, while the rest of the parameters are fixed.
 
 Replacing the above result in~\eqref{eq:th1proof}, we get
 
 {\small
\begin{align*}
 &\frac{1}{{K}} \sum_{k=0}^{K-1} \mathbb E\left[\Vert \nabla F^{(k)}(\mathbf{w}^{({k})}) \Vert^2\right]\leq \frac{1}{{K}} \Vast[
 \sum_{k=0}^{K-1}\frac{\mathbb E\left[{F}^{(k-1)}(\mathbf{w}^{(k)})\right] - \mathbb E\left[F^{(k)}(\mathbf{w}^{(k+1)})\right]}{\Gamma^{(k)}(1-\Lambda^{(k)})}
 +\sum_{k=0}^{K-1}\frac{\Omega^{(k+1)} \Delta^{(k+1)} }{\Gamma^{(k)}(1-\Lambda^{(k)})}
  \hspace{-16mm} \nonumber 
  \hspace{-16mm} 
  \\&+\sum_{k=0}^{K-1} \frac{1}{(1-\Lambda^{(k)})}\vast({8 \beta^2\Theta^2 \eta_k^2 }\sum_{n\in \mathcal{N}}\frac{\widehat{{D}}^{(k)}_n}{{D}^{(k)} }( e_n^{(k)}-1)
   %%%%%%%%%%%%%
    \frac{1}{\left(\widehat{D}_n^{(k)}\right)^2}\vast[ \frac{1}{{B}_n^{(k)}} \left(\sum_{j=1}^{{S}^{(k)}_{n}} \widetilde{\sigma}^{(k)}_{n,j}{S}^{(k)}_{n,j} \right)^2\nonumber \\&-\sum_{j=1}^{{S}^{(k)}_{n}}{S}^{(k)}_{n,j} {\left(\widetilde{\sigma}^{(k)}_{n,j}\right)^2}\vast]
     +{8\zeta_2 \eta_k^2\beta^2 \left(e_{\mathsf{max}}^{(k)}\right)\left(e_{\mathsf{max}}^{(k)}-1\right)}
     \vast)
     \hspace{-16mm} \nonumber  \hspace{-16mm}\\&
     +\sum_{k=0}^{K-1}\frac{8 \Theta^2 {\beta \Gamma^{(k)} }}{(1-\Lambda^{(k)})}\sum_{n\in \mathcal{N}} \left(\frac{\widehat{{D}}^{(k)}_n}{{D}^{(k)} } \right)^2\frac{1}{e^{(k)}_n}\frac{1}{\left(\widehat{D}_n^{(k)}\right)^2}\left[ \frac{1}{{B}_n^{(k)}} \left(\sum_{j=1}^{{S}^{(k)}_{n}} \widetilde{\sigma}^{(k)}_{n,j}{S}^{(k)}_{n,j} \right)^2-\sum_{j=1}^{{S}^{(k)}_{n}}{S}^{(k)}_{n,j} {\left(\widetilde{\sigma}^{(k)}_{n,j}\right)^2}\right]
 \Vast].\hspace{-6mm}
 \end{align*} 
 }
 
 Considering $\eta_k = \alpha \frac{1}{\sqrt{e^{(k)}_{\mathsf{sum}} K/N}}$, defining $e^{(k)}_{\mathsf{avg}}=\sum_{n\in \mathcal{N}}\frac{\widehat{{D}}^{(k)}_{n}e^{(k)}_{n}}{{D}^{(k)} }$ since $\Gamma^{(k)}=\frac{\eta_{k}}{2}\sum_{n\in \mathcal{N}}\frac{\widehat{{D}}^{(k)}_{n}e^{(k)}_{n}}{{D}^{(k)} }=\frac{\eta_k}{2}e^{(k)}_{\mathsf{avg}}$, we get
 \begin{align}
 &\frac{1}{{K}} \sum_{k=0}^{K-1} \mathbb E\left[\Vert \nabla F^{(k)}(\mathbf{w}^{({k})}) \Vert^2\right]\leq
 \sum_{k=0}^{K-1} \frac{2\sqrt{e^{(k)}_{\mathsf{sum}}}\left(\mathbb E\left[{F}^{(k-1)}(\mathbf{w}^{(k)})\right] - \mathbb E\left[F^{(k)}(\mathbf{w}^{(k+1)})\right]\right)}{\alpha e^{(k)}_{\mathsf{avg}}\sqrt{N K}(1-\Lambda^{(k)})}
 \nonumber\\&
 +\sum_{k=0}^{K-1}\frac{2\sqrt{e^{(k)}_{\mathsf{sum}}}\Omega^{(k+1)} \Delta^{(k+1)} }{\alpha e^{(k)}_{\mathsf{avg}}\sqrt{N K}(1-\Lambda^{(k)})}
  \hspace{-16mm} \nonumber 
  \hspace{-16mm} 
  \\&+\sum_{k=0}^{K-1} \frac{1}{(1-\Lambda^{(k)})}\vast({8 \beta^2\Theta^2 \frac{\alpha^2N}{{e^{(k)}_{\mathsf{sum}} K^2}} }\sum_{n\in \mathcal{N}}\frac{\widehat{{D}}^{(k)}_n}{{D}^{(k)} }( e_n^{(k)}-1) 
   %%%%%%%%%%%%%
    \frac{1}{\left(\widehat{D}_n^{(k)}\right)^2}\vast[ \frac{1}{{B}_n^{(k)}} \left(\sum_{j=1}^{{S}^{(k)}_{n}} \widetilde{\sigma}^{(k)}_{n,j}{S}^{(k)}_{n,j} \right)^2 \hspace{-10mm}\nonumber \hspace{-10mm} \\&-\sum_{j=1}^{{S}^{(k)}_{n}}{S}^{(k)}_{n,j} {\left(\widetilde{\sigma}^{(k)}_{n,j}\right)^2}\vast]
     +{8\zeta_2 \frac{\alpha^2N}{{e^{(k)}_{\mathsf{sum}} K^2}}\beta^2 \left(e_{\mathsf{max}}^{(k)}\right)\left(e_{\mathsf{max}}^{(k)}-1\right)}
     \vast)
     \hspace{-16mm} \nonumber  \hspace{-16mm}
     \\&
     +\sum_{k=0}^{K-1}\frac{4 e^{(k)}_{\mathsf{avg}} \alpha \Theta^2 {\beta \sqrt{N} }}{\sqrt{e^{(k)}_{\mathsf{sum}}}K\sqrt{K}(1-\Lambda^{(k)})}\sum_{n\in \mathcal{N}} \left(\frac{\widehat{{D}}^{(k)}_n}{{D}^{(k)} } \right)^2\frac{1}{e^{(k)}_n}\frac{1}{\left(\widehat{D}_n^{(k)}\right)^2} 
    \nonumber \\&  \hspace{65mm}\times\left[ \frac{1}{{B}_n^{(k)}} \left(\sum_{j=1}^{{S}^{(k)}_{n}} \widetilde{\sigma}^{(k)}_{n,j}{S}^{(k)}_{n,j} \right)^2-\sum_{j=1}^{{S}^{(k)}_{n}}{S}^{(k)}_{n,j} {\left(\widetilde{\sigma}^{(k)}_{n,j}\right)^2}\right].\hspace{-6mm}
 \end{align} 
  To remove the dependency of the bound on $F^{(k)}(\mathbf{w}^{(k+1)})$, $1\leq k\leq K$, which requires non-causal information to optimize, we further upper bound the first term on the right hand side of the above inequality and obtain the following upper bound
  \begin{align}\label{eq:prop1FinalProofBound}
 &\frac{1}{{K}} \sum_{k=0}^{K-1} \mathbb E\left[\Vert \nabla F^{(k)}(\mathbf{w}^{({k})}) \Vert^2\right]\leq
 \frac{2\sqrt{\widehat{e}_{\mathsf{max}}}\left({F}^{(-1)}(\mathbf{w}^{(0)}) - F^{(K)^\star}\right)}{\alpha \overline{e}_{\mathsf{min}}\sqrt{N K}(1-\Lambda_{\mathsf{max}})}
 +\sum_{k=0}^{K-1}\frac{2\sqrt{e^{(k)}_{\mathsf{sum}}}\Omega^{(k+1)} \Delta^{(k+1)} }{\alpha e^{(k)}_{\mathsf{avg}}\sqrt{N K}(1-\Lambda^{(k)})}
  \hspace{-16mm} \nonumber 
  \hspace{-16mm} 
  \\&+\sum_{k=0}^{K-1} \frac{1}{(1-\Lambda^{(k)})}\vast({8 \beta^2\Theta^2 \frac{\alpha^2N}{{e^{(k)}_{\mathsf{sum}} K^2}} }\sum_{n\in \mathcal{N}}\frac{\widehat{{D}}^{(k)}_n}{{D}^{(k)} }( e_n^{(k)}-1)
   %%%%%%%%%%%%%
    \frac{1}{\left(\widehat{D}_n^{(k)}\right)^2}\vast[ \frac{1}{{B}_n^{(k)}} \left(\sum_{j=1}^{{S}^{(k)}_{n}} \widetilde{\sigma}^{(k)}_{n,j}{S}^{(k)}_{n,j} \right)^2 \hspace{-10mm}\nonumber \hspace{-10mm} \\&-\sum_{j=1}^{{S}^{(k)}_{n}}{S}^{(k)}_{n,j} {\left(\widetilde{\sigma}^{(k)}_{n,j}\right)^2}\vast]
     +{8\zeta_2 \frac{\alpha^2 N}{{e^{(k)}_{\mathsf{sum}} K^2}}\beta^2 \left(e_{\mathsf{max}}^{(k)}\right)\left(e_{\mathsf{max}}^{(k)}-1\right)}
     \vast)
     \hspace{-16mm} \nonumber  \hspace{-16mm}
     \\&
     +\sum_{k=0}^{K-1}\frac{4 e^{(k)}_{\mathsf{avg}} \alpha \Theta^2 {\beta \sqrt{N} }}{\sqrt{e^{(k)}_{\mathsf{sum}}}K\sqrt{K}(1-\Lambda^{(k)})}\sum_{n\in \mathcal{N}} \left(\frac{\widehat{{D}}^{(k)}_n}{{D}^{(k)} } \right)^2\frac{1}{e^{(k)}_n}\frac{1}{\left(\widehat{D}_n^{(k)}\right)^2} 
    \nonumber \\&  \hspace{65mm}\times\left[ \frac{1}{{B}_n^{(k)}} \left(\sum_{j=1}^{{S}^{(k)}_{n}} \widetilde{\sigma}^{(k)}_{n,j}{S}^{(k)}_{n,j} \right)^2-\sum_{j=1}^{{S}^{(k)}_{n}}{S}^{(k)}_{n,j} {\left(\widetilde{\sigma}^{(k)}_{n,j}\right)^2}\right],\hspace{-6mm}
 \end{align} 
  which concludes the proof.
  \newpage
  \section{Further Convergence Results Under Optimal Sampling}\label{app:furtherOptSample}
  \begin{corollary}[Convergence of PSL with Optimal Local Sampling under Bounded Local Iterations]\label{cor:neyman1} In addition to the conditions described in Proposition~\ref{prop:neyman}, further assume that  $\max_k \left\{\Lambda^{(k)}\right\} \leq \Lambda_{\mathsf{max}}< 1$, $ \overline{e}_{\mathsf{min}} \leq e^{(k)}_{\mathsf{avg}}\leq  \overline{e}_{\mathsf{max}}$, $\forall k$, and $ \widehat{e}_{\mathsf{min}} \leq e^{(k)}_{\mathsf{sum}}\leq  \widehat{e}_{\mathsf{max}}$, $\forall k$. Then, the convergence behavior of the cumulative average of gradient of the global loss functions for PSL is described by
  \begin{equation}\label{cor3Res}
      \begin{aligned}
       &\frac{1}{K} \sum_{k=0}^{K-1} \mathbb E\left[\Vert \nabla F^{(k)}(\mathbf{w}^{({k})}) \Vert^2\right] \leq
  \frac{2\sqrt{\widehat{e}_{\mathsf{max}}}\left({F}^{(-1)}(\mathbf{w}^{(0)}) - F^{(K)^\star}\right)}{\alpha \overline{e}_{\mathsf{min}}\sqrt{N K}(1-\Lambda_{\mathsf{max}})}
 + \frac{2\sqrt{\widehat{e}_{\mathsf{max}}}}{\alpha \overline{e}_{\mathsf{min}}\sqrt{N K}}\sum_{k=0}^{K-1}\frac{\Omega^{(k+1)} \Delta^{(k+1)} }{(1-\Lambda_{\mathsf{max}})}
  \hspace{-16mm} 
  \\&+\sum_{k=0}^{K-1} \frac{1}{(1-\Lambda_{\mathsf{max}})}\vast({8 \beta^2\Theta^2 \frac{\alpha^2 N}{{\widehat{e}_{\mathsf{min}} K^2}} }\sum_{n\in \mathcal{N}}\frac{\widehat{{D}}^{(k)}_n}{{D}^{(k)} }( e_n^{(k)}-1)
   %%%%%%%%%%%%%
    \frac{1}{\left(\widehat{D}_n^{(k)}\right)^2}\vast[ \frac{1}{{B}_n^{(k)}} \left(\sum_{j=1}^{{S}^{(k)}_{n}} \widetilde{\sigma}^{(k)}_{n,j}{S}^{(k)}_{n,j} \right)^2 \hspace{-10mm} \\&-\sum_{j=1}^{{S}^{(k)}_{n}}{S}^{(k)}_{n,j} {\left(\widetilde{\sigma}^{(k)}_{n,j}\right)^2}\vast]
     +{8\zeta_2 \frac{\alpha^2 N}{{\widehat{e}_{\mathsf{min}} K^2}}\beta^2 \left(e_{\mathsf{max}}^{(k)}\right)\left(e_{\mathsf{max}}^{(k)}-1\right)}
     \vast)
     \hspace{-16mm}
     \\&
     +\sum_{k=0}^{K-1}\frac{4 \overline{e}_{\mathsf{max}} \alpha \Theta^2 {\beta \sqrt{N} }}{\sqrt{\widehat{e}_{\mathsf{min}}}K\sqrt{K}(1-\Lambda_{\mathsf{max}})}\sum_{n\in \mathcal{N}} \left(\frac{1}{{D}^{(k)}\sqrt{e^{(k)}_n} } \right)^2\left[ \frac{1}{{B}_n^{(k)}} \left(\sum_{j=1}^{{S}^{(k)}_{n}} \widetilde{\sigma}^{(k)}_{n,j}{S}^{(k)}_{n,j} \right)^2-\sum_{j=1}^{{S}^{(k)}_{n}}{S}^{(k)}_{n,j} {\left(\widetilde{\sigma}^{(k)}_{n,j}\right)^2}\right].
      \end{aligned}
  \end{equation}
\end{corollary} 
%  $\max{\Gamma^{(k)}}\leq {\Gamma^{\mathsf{max}}}$
\begin{proof} Considering~\eqref{eq:prop1FinalProofBound}, assuming  $  \max_{k} \left\{\Lambda^{(k)}\right\}  \leq \Lambda_{\mathsf{max}}< 1$, $( \overline{e}_{\mathsf{max}})^{-1} \leq \left(e^{(k)}_{\mathsf{avg}}\right)^{-1}\leq ( \overline{e}_{\mathsf{min}})^{-1}$, $\forall k$ and $ \widehat{e}_{\mathsf{min}} \leq e^{(k)}_{\mathsf{sum}}\leq  \widehat{e}_{\mathsf{max}}$, $\forall k$, we get
 \begin{align}\label{eq:prop2FinalProofBound}
 &\frac{1}{{K}} \sum_{k=0}^{K-1} \mathbb E\left[\Vert \nabla F^{(k)}(\mathbf{w}^{({k})}) \Vert^2\right]\leq
 \frac{2\sqrt{\widehat{e}_{\mathsf{max}}}\left({F}^{(0)}(\mathbf{w}^{(0)}) - F^{(K)^\star}\right)}{\alpha \overline{e}_{\mathsf{min}}\sqrt{N K}(1-\Lambda_{\mathsf{max}})}
 \nonumber\\&
 +\sum_{k=0}^{K-1}\frac{2\sqrt{\widehat{e}_{\mathsf{max}}}\Omega^{(k+1)} \Delta^{(k+1)} }{\alpha \overline{e}_{\mathsf{min}}\sqrt{N K}(1-\Lambda_{\mathsf{max}})}
  \hspace{-16mm} \nonumber 
  \hspace{-16mm} 
  \\&+\sum_{k=0}^{K-1} \frac{1}{(1-\Lambda_{\mathsf{max}})}\vast({8 \beta^2\Theta^2 \frac{\alpha^2N}{{\widehat{e}_{\mathsf{min}} K^2}} }\sum_{n\in \mathcal{N}}\frac{\widehat{{D}}^{(k)}_n}{{D}^{(k)} }( e_n^{(k)}-1)
   %%%%%%%%%%%%%
    \frac{1}{\left(\widehat{D}_n^{(k)}\right)^2}\vast[ \frac{1}{{B}_n^{(k)}} \left(\sum_{j=1}^{{S}^{(k)}_{n}} \widetilde{\sigma}^{(k)}_{n,j}{S}^{(k)}_{n,j} \right)^2 \hspace{-10mm}\nonumber \hspace{-10mm} \\&-\sum_{j=1}^{{S}^{(k)}_{n}}{S}^{(k)}_{n,j} {\left(\widetilde{\sigma}^{(k)}_{n,j}\right)^2}\vast]
     +{8\zeta_2 \frac{\alpha^2 N}{{\widehat{e}_{\mathsf{min}} K^2}}\beta^2 \left(e_{\mathsf{max}}^{(k)}\right)\left(e_{\mathsf{max}}^{(k)}-1\right)}
     \vast)
     \hspace{-16mm} \nonumber  \hspace{-16mm}
     \\&
     +\sum_{k=0}^{K-1}\frac{4 \overline{e}_{\mathsf{max}} \alpha \Theta^2 {\beta \sqrt{N} }}{\sqrt{\widehat{e}_{\mathsf{min}}}K\sqrt{K}(1-\Lambda_{\mathsf{max}})}\sum_{n\in \mathcal{N}} \left(\frac{\widehat{{D}}^{(k)}_n}{{D}^{(k)} } \right)^2\frac{1}{e^{(k)}_n}\frac{1}{\left(\widehat{D}_n^{(k)}\right)^2} 
    \nonumber \\&  \hspace{65mm}\times\left[ \frac{1}{{B}_n^{(k)}} \left(\sum_{j=1}^{{S}^{(k)}_{n}} \widetilde{\sigma}^{(k)}_{n,j}{S}^{(k)}_{n,j} \right)^2-\sum_{j=1}^{{S}^{(k)}_{n}}{S}^{(k)}_{n,j} {\left(\widetilde{\sigma}^{(k)}_{n,j}\right)^2}\right].\hspace{-6mm}
 \end{align} 

 Rearranging the terms concludes the proof.
\end{proof}
\begin{corollary}[Convergence of PSL with Optimal Local Sampling under Unified Upperbounds on the Sampling Noise]\label{cor:furtherOptSam2}
Under the conditions specified in Corollary~\ref{cor:neyman1}, further assume a bounded stratified sampling noise $\max_{k,n}\left\{\frac{1}{\left(\widehat{D}_n^{(k)}\right)^2} \left[ \frac{1}{{B}_n^{(k)}} \left(\sum_{j=1}^{{S}^{(k)}_{n}} \widetilde{\sigma}^{(k)}_{n,j}{S}^{(k)}_{n,j} \right)^2-\sum_{j=1}^{{S}^{(k)}_{n}}{S}^{(k)}_{n,j} {\left(\widetilde{\sigma}^{(k)}_{n,j}\right)^2}\right]\right\}\leq \sigma_{\mathsf{max}}$, $\forall n,k$, bounded local iterations $\max_{k}\{ e^{(k)}_{\mathsf{max}}\}\leq e_{\mathsf{max}} $, and bounded idle period as $\Omega^{(k)}\leq \left[\frac{\gamma}{K\Delta^{(k)}}\right]^+$, $\forall k$, for a finite non-negative constant $\gamma$. Then, the convergence behavior of the  cumulative average of gradient of the global loss functions across the global aggregation instances of PSL is described by the following upper bound, which implies $\frac{1}{{K}} \sum_{k=0}^{K-1} \mathbb E\left[\Vert \nabla F^{(k)}(\mathbf{w}^{({k})}) \Vert^2\right] \leq \mathcal{O}(1/\sqrt{K})$:
\begin{equation}
\footnotesize
    \begin{aligned}
    &\frac{1}{K} \sum_{k=0}^{K-1} \Vert \nabla F^{(k)}(\mathbf{w}^{({k})}) \Vert^2 \leq
    \frac{2\sqrt{\widehat{e}_{\mathsf{max}}}\left({F}^{(0)}(\mathbf{w}^{(0)}) - F^{(K)^\star}\right)}{\alpha \overline{e}_{\mathsf{min}}\sqrt{N K}(1-\Lambda_{\mathsf{max}})}
 + \frac{2\sqrt{\widehat{e}_{\mathsf{max}}}}{\alpha \overline{e}_{\mathsf{min}}\sqrt{N K}}\frac{\gamma }{(1-\Lambda_{\mathsf{max}})}  
     +\frac{4 \overline{e}_{\mathsf{max}} \alpha \Theta^2 {\beta \sqrt{N}  }}{\sqrt{\widehat{e}_{\mathsf{min}}}\sqrt{K}(1-\Lambda_{\mathsf{max}})} \sigma_{\mathsf{max}}\hspace{-6mm}
  \\&+ \frac{1}{K(1-\Lambda_{\mathsf{max}})}\vast({ \frac{8 \beta^2\Theta^2\alpha^2 N}{{\widehat{e}_{\mathsf{min}} }} }( {e}_{\mathsf{max}}-1)
   %%%%%%%%%%%%%
   \sigma_{\mathsf{max}}
     +{8\zeta_2 \frac{\alpha^2N}{{\widehat{e}_{\mathsf{min}} }}\beta^2 \left({e}_{\mathsf{max}}\right)\left({e}_{\mathsf{max}}-1\right)}
     \vast).
    \end{aligned}
\end{equation}
\end{corollary}
\begin{proof}
Considering~\eqref{cor3Res}, assuming a bounded stratified sampling noise:\\ $\max_{k,n}\left\{\frac{1}{\left(\widehat{D}_n^{(k)}\right)^2} \left[ \frac{1}{{B}_n^{(k)}} \left(\sum_{j=1}^{{S}^{(k)}_{n}} \widetilde{\sigma}^{(k)}_{n,j}{S}^{(k)}_{n,j} \right)^2-\sum_{j=1}^{{S}^{(k)}_{n}}{S}^{(k)}_{n,j} {\left(\widetilde{\sigma}^{(k)}_{n,j}\right)^2}\right]\right\}\leq \sigma_{\mathsf{max}}$, $\forall n,k$, bounded local iterations, $\max_{k}\{ e^{(k)}_{\mathsf{max}}\}\leq e_{\mathsf{max}} $, and $\Delta^{(k)}\leq \left[ \frac{\gamma}{K\Omega^{(k)}}\right]^+$, $\forall k$, for a finite non-negative constant $\gamma$, we have
\begin{align}
 &\frac{1}{K} \sum_{k=0}^{K-1} \mathbb{E} \Vert \nabla F^{(k)}(\mathbf{w}^{({k})}) \Vert^2 \leq
  \frac{2\sqrt{\widehat{e}_{\mathsf{max}}}\left({F}^{(-1)}(\mathbf{w}^{(0)}) - F^{(K)^\star}\right)}{\alpha \overline{e}_{\mathsf{min}}\sqrt{N K}(1-\Lambda_{\mathsf{max}})}
 + \frac{2\sqrt{\widehat{e}_{\mathsf{max}}}}{\alpha \overline{e}_{\mathsf{min}}\sqrt{N K}}\sum_{k=0}^{K-1}\frac{\gamma }{K(1-\Lambda_{\mathsf{max}})}
  \hspace{-16mm} \nonumber 
  \\&+\sum_{k=0}^{K-1} \frac{1}{(1-\Lambda_{\mathsf{max}})}\vast({ \frac{8 \beta^2\Theta^2\alpha^2 N}{{\widehat{e}_{\mathsf{min}} K^2}} }\sum_{n\in \mathcal{N}}\frac{\widehat{{D}}^{(k)}_n}{{D}^{(k)} }( {e}_{\mathsf{max}}-1) 
   %%%%%%%%%%%%%
   \sigma_{\mathsf{max}}
     +{8\zeta_2 \frac{\alpha^2 N}{{\widehat{e}_{\mathsf{min}} K^2}}\beta^2 \left({e}_{\mathsf{max}}\right)\left({e}_{\mathsf{max}}-1\right)}
     \vast)
     \hspace{-16mm} \nonumber 
     \\&
     +\sum_{k=0}^{K-1}\frac{4 \overline{e}_{\mathsf{max}} \alpha \Theta^2 {\beta \sqrt{N}  }}{\sqrt{\widehat{e}_{\mathsf{min}}}K\sqrt{K}(1-\Lambda_{\mathsf{max}})} \sigma_{\mathsf{max}},\hspace{-6mm}
  \nonumber   \\ &\leq 
     \frac{2\sqrt{\widehat{e}_{\mathsf{max}}}\left({F}^{(0)}(\mathbf{w}^{(0)}) - F^{(K)^\star}\right)}{\alpha \overline{e}_{\mathsf{min}}\sqrt{N K}(1-\Lambda_{\mathsf{max}})}
 + \frac{2\sqrt{\widehat{e}_{\mathsf{max}}}}{\alpha \overline{e}_{\mathsf{min}}\sqrt{N K}}\frac{\gamma }{(1-\Lambda_{\mathsf{max}})}
  \hspace{-16mm} \nonumber 
  \\&+\sum_{k=0}^{K-1} \frac{1}{(1-\Lambda_{\mathsf{max}})}\vast({ \frac{8 \beta^2\Theta^2\alpha^2 N}{{\widehat{e}_{\mathsf{min}} K^2}} }\sum_{n\in \mathcal{N}}\frac{\widehat{{D}}^{(k)}_n}{{D}^{(k)} }( {e}_{\mathsf{max}}-1)
   %%%%%%%%%%%%%
   \sigma_{\mathsf{max}}
     +{8\zeta_2 \frac{\alpha^2 N}{{\widehat{e}_{\mathsf{min}} K^2}}\beta^2 \left({e}_{\mathsf{max}}\right)\left({e}_{\mathsf{max}}-1\right)}
     \vast)
     \hspace{-16mm} \nonumber  \hspace{-16mm}
     \\&
     +\sum_{k=0}^{K-1}\frac{4 \overline{e}_{\mathsf{max}} \alpha \Theta^2 {\beta \sqrt{N} }}{\sqrt{\widehat{e}_{\mathsf{min}}}K\sqrt{K}(1-\Lambda_{\mathsf{max}})}\sigma_{\mathsf{max}},\hspace{-6mm}
   \nonumber   \\&=
     \frac{2\sqrt{\widehat{e}_{\mathsf{max}}}\left({F}^{(0)}(\mathbf{w}^{(0)}) - F^{(K)^\star}\right)}{\alpha \overline{e}_{\mathsf{min}}\sqrt{N K}(1-\Lambda_{\mathsf{max}})}
 + \frac{2\sqrt{\widehat{e}_{\mathsf{max}}}}{\alpha \overline{e}_{\mathsf{min}}\sqrt{N K}}\frac{\gamma }{(1-\Lambda_{\mathsf{max}})}
  + \frac{1}{(1-\Lambda_{\mathsf{max}})}\vast({ \frac{8 \beta^2\Theta^2\alpha^2 N}{{\widehat{e}_{\mathsf{min}} K}} }( {e}_{\mathsf{max}}-1) \\&
   +{8\zeta_2 \frac{\alpha^2 N}{{\widehat{e}_{\mathsf{min}} K}}\beta^2 \left({e}_{\mathsf{max}}\right)\left({e}_{\mathsf{max}}-1\right)}
     \vast)
          +\sum_{k=0}^{K-1}\frac{4 \overline{e}_{\mathsf{max}} \alpha \Theta^2 {\beta \sqrt{N} }}{\sqrt{\widehat{e}_{\mathsf{min}}}K\sqrt{K}(1-\Lambda_{\mathsf{max}})} \sigma_{\mathsf{max}},\hspace{-6mm}
   \nonumber  \\&=
    \frac{2\sqrt{\widehat{e}_{\mathsf{max}}}\left({F}^{(0)}(\mathbf{w}^{(0)}) - F^{(K)^\star}\right)}{\alpha \overline{e}_{\mathsf{min}}\sqrt{N K}(1-\Lambda_{\mathsf{max}})}
 + \frac{2\sqrt{\widehat{e}_{\mathsf{max}}}}{\alpha \overline{e}_{\mathsf{min}}\sqrt{N K}}\frac{\gamma }{(1-\Lambda_{\mathsf{max}})}
  \hspace{-16mm} \nonumber 
  \\&+ \frac{1}{K(1-\Lambda_{\mathsf{max}})}\vast({ \frac{8 \beta^2\Theta^2\alpha^2 N}{{\widehat{e}_{\mathsf{min}} }} }( {e}_{\mathsf{max}}-1)
   %%%%%%%%%%%%%
   \sigma_{\mathsf{max}}
     +{8\zeta_2 \frac{\alpha^2N}{{\widehat{e}_{\mathsf{min}} }}\beta^2 \left({e}_{\mathsf{max}}\right)\left({e}_{\mathsf{max}}-1\right)}
     \vast)
     \hspace{-16mm} \nonumber  \hspace{-16mm}
     \\&
     +\frac{4 \overline{e}_{\mathsf{max}} \alpha \Theta^2 {\beta \sqrt{N} }}{\sqrt{\widehat{e}_{\mathsf{min}}}\sqrt{K}(1-\Lambda_{\mathsf{max}})} \sigma_{\mathsf{max}},\hspace{-6mm}
 \end{align} 
  where concludes the proof. \end{proof}
  \newpage

 \section{Transforming the Network-Aware Optimization of  PSL Problem}\label{app:optTransform}
   \subsection{Geometric Programming} \label{sec:GPtransConv}
   A basic knowledge of \textit{monomials} and \textit{posynomials}, which is given below, is a prerequisite to understand the GP. 
         \begin{definition}
         A \textbf{{monomial}} is a function $f: \mathbb{R}^n_{++}\rightarrow \mathbb{R}$:\footnote{$\mathbb{R}^n_{++}$ denotes the strictly positive quadrant of $n$-dimensional Euclidean space.} $f(\bm{y})=d y_1^{\alpha_1} y_2^{\alpha_2} \cdots y_n ^{\alpha_n}$, where $d\geq 0$, $\bm{y}=[y_1,\cdots,y_n]$, and $\alpha_j\in \mathbb{R}$, $\forall j$. Based on the definition of monomials, a \textbf{posynomial} $g$ is defined as a sum of monomials: $g(\bm{y})= \sum_{m=1}^{M} d_m y_1^{\alpha^{(1)}_m} y_2^{\alpha^{(2)}_m} \cdots y_n ^{\alpha^{(n)}_m}$.
\end{definition}
A standard GP is a non-convex optimization
problem defined as minimizing a posynomial subject to inequality constraints on posynomials (and monomials) and equality constraints on monomials~\cite{chiang2005geometric,boyd2007tutorial}:
        \begin{equation}\label{eq:GPformat}
            \begin{aligned}
            &\min_{\bm{y}} f_0 (\bm{y})\\
            &\textrm{s.t.} ~~~ f_i(\bm{y})\leq 1, \;\; i=1,\cdots,I,\\
           &~~~~~~ h_l(\bm{y})=1, \;\; l=1,\cdots,L,
            \end{aligned}
        \end{equation}
        where  $f_i(\bm{y})=\sum_{m=1}^{M_i} d_{i,m} y_1^{\alpha^{(1)}_{i,m}} y_2^{\alpha^{(2)}_{i,m}} \cdots y_n ^{\alpha^{(n)}_{i,m}}$, $\forall i$, and $h_l(\bm{y})= d_l y_1^{\alpha^{(1)}_l} y_2^{\alpha^{(2)}_l} \cdots y_n ^{\alpha^{(n)}_l}$, $\forall l$. Due to the fact that the log-sum-exp function $f(\bm{y}) = \log \sum_{j=1}^n e^{y_j}$ is convex ($\log$ denotes the natural logarithm) with the following change of variables $z_i=\log(y_i)$, $b_{i,k}=\log(d_{i,k})$, $b_l=\log (d_l)$  the GP can be converted into the following convex format
          \begin{equation}~\label{GPtoConvex}
            \begin{aligned}
            &\min_{\bm{z}} \;\log \sum_{m=1}^{M_0} e^{\left(\bm{\alpha}^{\top}_{0,m}\bm{z}+ b_{0,m}\right)}\\
            &\textrm{s.t.} ~~~ \log \sum_{m=1}^{M_i} e^{\left(\bm{\alpha}^{\top}_{i,m}\bm{z}+ b_{i,m}\right)}\leq 0 \;\; i=1,\cdots,I,\\
           &~~~~~~~ \bm{\alpha}_l^\top \bm{z}+b_l =0\;\; l=1,\cdots,L,
            \end{aligned}
        \end{equation}
        where $\bm{z}=[z_1,\cdots,z_n]^\top$, $\bm{\alpha}_{i,k}=\left[\alpha_{i,k}^{(1)},\alpha_{i,k}^{(2)}\cdots, \alpha_{i,k}^{(n)}\right]^\top$, $\forall i,k$, and $\bm{\alpha}_{l}=\left[\alpha_{l}^{(1)},a_{l}^{(2)}\cdots, \alpha_{l}^{(n)}\right]^\top$\hspace{-2mm}, $\forall l$.

%   \begin{lemma}[\textbf{Arithmetic-geometric mean inequality}~\cite{duffin1972reversed,chiang2005geometric}]\label{Lemma:ArethmaticGeometric}
%          Consider a posynomial function $g(\bm{y})=\sum_{i=1}^{i'} u_i(\bm{y})$, where $u_i(\bm{y})$ is a monomial, $\forall i$. The following inequality holds:
%          \begin{equation}\label{eq:approxPosMon}
%              g(\bm{y})\geq \hat{g}(\bm{y})\triangleq \prod_{i=1}^{i'}\left( \frac{u_i(\bm{y})}{\alpha_i(\bm{z})}\right)^{\alpha_i(\bm{z})},
%          \end{equation}
%          where $\alpha_i(\bm{z})=u_i(\bm{z})/g(\bm{z})$, $\forall i$, and $\bm{z}>0$ is a fixed point.
%          \end{lemma}
        
        \subsection{Optimization Problem Transformation}
 Let us revisit Problem~$\bm{\mathcal{P}}$ in the following, where we use the explicit expressions for the constraints w.r.t the optimization variables:
 \begin{align}
     &(\bm{\mathcal{P}}): ~~\min \frac{1}{K}\left[\sum_{k=0}^{K-1} c_1 E^{\mathsf{Tot},(k)}+ c_2 T^{\mathsf{Tot},(k)}\right]+ c_3 \underbrace{\frac{1}{K}\sum_{k=0}^{K-1}\mathbb E \left\Vert \nabla F^{(k)}(\mathbf{w}^{({k})}) \right\Vert^2}_{= \Xi\left(\widehat{\bm{D}}^{(k)},\bm{B}^{(k)},\Omega^{(k)},\Delta^{(k)} \right) ~\textrm{given by}~\eqref{eq:gen_conv_neyman_main}}\\
     & \textrm{s.t.}\nonumber\\ 
     &  T^{\mathsf{Tot},(k)} = T^{\mathsf{D},(k)}+T^{\mathsf{L},(k)}+T^{\mathsf{M},(k)}+T^{\mathsf{U},(k)},\label{prob:TtotApp}\\
     & E^{\mathsf{Tot},(k)}= \sum_{n\in \mathcal{N}}E^{(k)}_n,\label{prob:EtotApp}\\
     & \sum_{k=1}^{K} T^{\mathsf{Tot},(k)}+\Omega^{(k)} = T^{\mathsf{ML}},\label{prob:TmlApp}\\
    %  & \widehat{D}_n^{(k)} = \sum_{m\in\mathcal{N}} \varrho^{(k)}_{m,n}D^{(k)}_m, ~~n\in\mathcal{N}, \\
     & \max_{n\in\mathcal{N}} \left\{  \max_{m\in\mathcal{N}} \left\{\varrho^{(k)}_{m,n} {D}^{(k)}_m b^{\mathsf{D}} /r_{m,n}\right\} \right\} \leq T^{\mathsf{D},(k)},\label{prob:TdApp}\\
     & \max_{n\in\mathcal{N}} \left\{e^{(k)}_n\frac{a_n B^{(k)}_n}{f^{(k)}_n}\right\} \leq T^{\mathsf{L},(k)},\label{prob:TlApp}\\
     & \max_{n\in\mathcal{N}} \left\{  \max_{m\in\mathcal{N}} \left\{ {\varphi_{m,n} {M}b^{\mathsf{G}} }/{r_{m,n}} \right\} \right\} \leq T^{\mathsf{M},(k)},\label{prob:TmApp}\\
     & \max_{n\in\mathcal{N}} \left\{   \frac{Mb^{\mathsf{G}}\varphi_{n,n}}{r_n}\right\} \leq T^{\mathsf{U},(k)},\label{prob:TuApp}\\
    %  &D^{\mathsf{F},(k)}_n=\sum_{m\in \mathcal{N}} \varrho^{(k)}_{m,n}{D}^{(k)}_m ,~~ n\in\mathcal{N},\\
      & \sum_{m\in \mathcal{N}}\varrho^{(k)}_{n,m} = 1,~ n\in\mathcal{N},\label{prob:varrhoApp}\\
      & \sum_{m\in \mathcal{N}} \varphi^{(k)}_{n,m}=1,~~n\in\mathcal{N}, \label{prob:varphiApp}\\
        &\varphi_{n,n}^{(k)}\sum_{m \in \mathcal{N} \setminus \{n\}} \varphi_{n,m}^{(k)} \leq 0, ~n\in\mathcal{N}, \label{eq:varphi1App}\\
        &(1-\varphi_{n,n}^{(k)})\sum_{m \in \mathcal{N} \setminus \{n\}} \varphi_{m,n}^{(k)} \leq 0, ~n\in\mathcal{N},\label{eq:varphi2App}\\
           & f^{\mathsf{min}}_n\leq f^{(k)}_n\leq f^{\mathsf{max}}_n,~1\leq B_n^{(k)}\leq \widehat{D}_n^{(k)},~~n\in\mathcal{N},\label{prob:freqApp}\\
     &  \varrho^{(k)}_{n,m},\varphi^{(k)}_{n,m} \geq 0, ~~n,m\in \mathcal{N},\label{prob:feasApp}\\
     &\hspace{-2mm}\textrm{Variables:}\nonumber\\
     &\hspace{-2mm}\small K,\big\{\mathbf{f}^{(k)},\mathbf{B}^{(k)},\bm{\varrho}^{(k)}, \bm{\varphi}^{(k)},T^{\mathsf{D},(k)},T^{\mathsf{L},(k)},T^{\mathsf{M},(k)},T^{\mathsf{U},(k)},\Omega^{(k)}\big\}_{k=1}^{K} \nonumber \hspace{-10mm} 
 \end{align}

In the following, we aim to transform the problem into GP format:

\textbf{Constraint~\eqref{prob:TdApp}:} We transform this constraint via expressing it as an normalized inequality on monomials as follows:
 \begin{tcolorbox}[ams align]
    & \left(T^{\mathsf{D},(k)}\right)^{-1}\varrho^{(k)}_{m,n} {D}^{(k)}_m b^{\mathsf{D}} /r_{m,n} \leq 1,~ m,n\in\mathcal{N}.
 \end{tcolorbox}

% via introducing a set of auxiliary variables as follows:
% \begin{equation}
%     \max_{n\in\mathcal{N}} \left\{  \max_{m\in\mathcal{N}} \left\{\left(T^{\mathsf{D},(k)}\right)^{-1}\varrho^{(k)}_{m,n} {D}^{(k)}_m b^{\mathsf{D}} /r_{m,n}\right\} \right\} \leq 1,
% \end{equation}
% which can be expressed as
% \begin{align}
%     & \max_{n\in\mathcal{N}} \left\{ A^{\mathsf{D},{(k)}}_{n}\right\} \leq 1,\\
%     & \left(T^{\mathsf{D},(k)}\right)^{-1}\varrho^{(k)}_{m,n} {D}^{(k)}_m b^{\mathsf{D}} /r_{m,n} \leq A^{\mathsf{D},(k)}_{n}, \forall m,
% \end{align}
% which further can be expressed as:
% \begin{align}
%     & A^{\mathsf{D},{(k)}} \leq 1,\\
%     & \left(T^{\mathsf{D},(k)}\right)^{-1}\varrho^{(k)}_{m,n} {D}^{(k)}_m b^{\mathsf{D}} /r_{m,n} \leq A^{\mathsf{D},(k)}_{n},~ \forall m, \\
%     & A^{\mathsf{D},(k)}_{n} \leq A^{\mathsf{D},{(k)}},~ \forall n
% \end{align}
% which can be expressed as an normalized inequality on monomials as follows:
% \begin{tcolorbox}[ams align]
%     & A^{\mathsf{D},{(k)}} \leq 1,\\
%     & \left(T^{\mathsf{D},(k)}A^{\mathsf{D},(k)}_{n}\right)^{-1}\varrho^{(k)}_{m,n} {D}^{(k)}_m b^{\mathsf{D}} /r_{m,n} \leq 1,~ \forall m, \\
%     & A^{\mathsf{D},(k)}_{n} \left(A^{\mathsf{D},{(k)}}\right)^{-1} \leq 1,~ \forall n
% \end{tcolorbox}

\textbf{Constraint~\eqref{prob:TlApp}:} We transform this constraint via expressing it as an normalized inequality on posynomials as follows:
\begin{tcolorbox}[ams align]
    &\left(T^{\mathsf{L},(k)}\right)^{-1}e^{(k)}_n\frac{a_n B^{(k)}_n}{f^{(k)}_n} \leq 1, ~ n\in\mathcal{N}.
\end{tcolorbox}
% \begin{equation}
%     \max_{n\in\mathcal{N}} \left\{\left(T^{\mathsf{L},(k)}\right)^{-1}e^{(k)}_n\frac{a_n b^{(k)}_n\sum_{m\in \mathcal{N}}\varrho^{(k)}_{m,n}{D}^{(k)}_m}{f^{(k)}_n}\right\} \leq 1
% \end{equation}
% which can be written as
% \begin{align}
%   & A^{\mathsf{L},(k)} \leq 1
%     \\
%     &\left(T^{\mathsf{L},(k)}\right)^{-1}e^{(k)}_n\frac{a_n b^{(k)}_n\sum_{m\in \mathcal{N}}\varrho^{(k)}_{m,n}{D}^{(k)}_m}{f^{(k)}_n} \leq A^{\mathsf{L},(k)}, \forall n
% \end{align}
% which can be expressed as an normalized inequality on posynomials as follows:
% \begin{tcolorbox}[ams align]
%     & A^{\mathsf{L},(k)} \leq 1
%     \\
%     &\left(T^{\mathsf{L},(k)}A^{\mathsf{L},(k)}\right)^{-1}e^{(k)}_n\frac{a_n b^{(k)}_n\sum_{m\in \mathcal{N}}\varrho^{(k)}_{m,n}{D}^{(k)}_m}{f^{(k)}_n} \leq 1, \forall n
% \end{tcolorbox}

\textbf{Constraint~\eqref{prob:TmApp}:} We transform this constraint via expressing it as an normalized inequality on monomials as follows:
\begin{tcolorbox}[ams align]
    & {\left(T^{\mathsf{M},(k)} \right)^{-1}\varphi_{m,n} {M}b^{\mathsf{G}} }/{r_{m,n}} \leq 1, ~ m,n\in \mathcal{N}.
\end{tcolorbox}
% \begin{equation}
%     \max_{n\in\mathcal{N}} \left\{  \max_{m\in\mathcal{N}} \left\{ {\left(T^{\mathsf{M},(k)}\right)^{-1}\varphi_{m,n} {M}b^{\mathsf{G}} }/{r_{m,n}} \right\} \right\} \leq 1,
% \end{equation}
% which can be expressed as
% \begin{align}
%   &\max_{n\in\mathcal{N}} \left\{ A^{\mathsf{M},(k)}_n \right\}\leq 1
%     \\
%     & {\left(T^{\mathsf{M},(k)}\right)^{-1}\varphi_{m,n} {M}b^{\mathsf{G}} }/{r_{m,n}} \leq A^{\mathsf{M},(k)}_n, \forall m
% \end{align}
% which further can be expressed as:
% \begin{align}
%   & A^{\mathsf{M},(k)}\leq 1
%     \\
%     & {\left(T^{\mathsf{M},(k)}\right)^{-1}\varphi_{m,n} {M}b^{\mathsf{G}} }/{r_{m,n}} \leq A^{\mathsf{M},(k)}_n, \forall m\\
%     &A^{\mathsf{M},(k)}_n \leq A^{\mathsf{M},(k)}, \forall n
% \end{align}
% which can be expressed as an normalized inequality on monomials as follows:
% \begin{tcolorbox}[ams align]
%   & A^{\mathsf{M},(k)}\leq 1
%     \\
%     & {\left(T^{\mathsf{M},(k)} A^{\mathsf{M},(k)}_n\right)^{-1}\varphi_{m,n} {M}b^{\mathsf{G}} }/{r_{m,n}} \leq 1, \forall m\\
%     &A^{\mathsf{M},(k)}_n \left(A^{\mathsf{M},(k)} \right)^{-1} \leq 1, \forall n
% \end{tcolorbox}

\textbf{Constraint~\eqref{prob:TuApp}}: We transform this constraint via expressing it as an normalized inequality on monomials as follows:
\begin{tcolorbox}[ams align]
  &  \left( T^{\mathsf{U},(k)} \right)^{-1}   \frac{Mb^{\mathsf{G}}\varphi_{n,n}}{r_n} \leq 1,  n\in\mathcal{N}.
\end{tcolorbox}
% \begin{equation}
%     \max_{n\in\mathcal{N}} \left\{ \left( T^{\mathsf{U},(k)}\right)^{-1}   \frac{Mb^{\mathsf{G}}\varphi_{n,n}}{r_n}\right\} \leq 1
% \end{equation}
% which can be expressed as follows:
% \begin{align}
% & A^{\mathsf{U},(k)} \leq 1
% \\
%   &  \left( T^{\mathsf{U},(k)}\right)^{-1}   \frac{Mb^{\mathsf{G}}\varphi_{n,n}}{r_n} \leq A^{\mathsf{U},(k)}, \forall n
% \end{align}
% which can be expressed as an normalized inequality on monomials as follows:
% \begin{tcolorbox}[ams align]
%   & A^{\mathsf{U},(k)} \leq 1
% \\
%   &  \left( T^{\mathsf{U},(k)} A^{\mathsf{U},(k)}\right)^{-1}   \frac{Mb^{\mathsf{G}}\varphi_{n,n}}{r_n} \leq 1, \forall n
% \end{tcolorbox}

Note that~\eqref{prob:TtotApp},~\eqref{prob:EtotApp} are solely definition of the terms in the objective function which are posynomials.
We next focus on the violating constraints~\eqref{prob:TmlApp},~\eqref{prob:varrhoApp},~\eqref{prob:varphiApp} since they impose equality constraints on posynomials which are unacceptable in GP. We transform these violating constraints in the following:

\textbf{Constraint~\eqref{prob:TmlApp}:} We first revisit this constraints:
\begin{equation}
   \sum_{k=1}^{K} [ T^{\mathsf{Tot},(k)}+\Omega^{(k)}] = T^{\mathsf{ML}}.
\end{equation}
We write this constraint via introducing an auxiliary variable and then write the constraint as two inequalities:
\begin{align}
 & (T^{\mathsf{ML}})^{-1} \left(\sum_{k=1}^{K} T^{\mathsf{D},(k)}+T^{\mathsf{L},(k)}+T^{\mathsf{M},(k)}+T^{\mathsf{U},(k)}+\Omega^{(k)}\right) \leq 1,\\
 & \frac{(A^{\mathsf{ML}})^{-1}}{(T^{\mathsf{ML}})^{-1} \left(\sum_{k=1}^{K} T^{\mathsf{D},(k)}+T^{\mathsf{L},(k)}+T^{\mathsf{M},(k)}+T^{\mathsf{U},(k)}+\Omega^{(k)}\right)} \leq 1,\label{frac1App}\\
 & A^{\mathsf{ML}}\geq 1,
\end{align}
where $ A^{\mathsf{ML}}$ is added with a large penalty term to the objective function to force $A^{\mathsf{ML}}\downarrow 1$ at the optimal point. It can be seen that the fraction in~\eqref{frac1App} is not in the format of GP since it is an inequality with a posynomial in the denominator, which is not a posynomial. We thus exploit arithmetic-geometric mean inequality (Lemma~\ref{Lemma:ArethmaticGeometric}) to approximate the denominator with a monomial:
\begin{align}\label{eq:testerappr}
    H(\bm{x})=&\sum_{k=1}^{K} T^{\mathsf{Tot},(k)}+\Omega^{(k)} \geq \widehat{H}(\bm{x};\ell) \triangleq \prod_{k=1}^{K} \left(\frac{T^{\mathsf{D},(k)}  H([\bm{x}]^{\ell-1})}{\left[T^{\mathsf{D},(k)} \right]^{\ell-1}}\right)^{\frac{\left[T^{\mathsf{D},(k)} \right]^{\ell-1} }{H([\bm{x}]^{\ell-1})}} 
    \left(\frac{T^{\mathsf{L},(k)}  H([\bm{x}]^{\ell-1})}{\left[T^{\mathsf{L},(k)} \right]^{\ell-1}}\right)^{\frac{\left[T^{\mathsf{L},(k)} \right]^{\ell-1} }{H([\bm{x}]^{\ell-1})}} \nonumber \\&
    \left(\frac{T^{\mathsf{M},(k)}  H([\bm{x}]^{\ell-1})}{\left[T^{\mathsf{M},(k)} \right]^{\ell-1}}\right)^{\frac{\left[T^{\mathsf{M},(k)} \right]^{\ell-1} }{H([\bm{x}]^{\ell-1})}} 
    \left(\frac{T^{\mathsf{U},(k)}  H([\bm{x}]^{\ell-1})}{\left[T^{\mathsf{U},(k)} \right]^{\ell-1}}\right)^{\frac{\left[T^{\mathsf{U},(k)} \right]^{\ell-1} }{H([\bm{x}]^{\ell-1})}} 
    \left(\frac{\Omega^{(k)} H([\bm{x}]^{\ell-1})}{\left[\Omega^{(k)} \right]^{\ell-1}}\right)^{\frac{\left[\Omega^{(k)} \right]^{\ell-1} }{H([\bm{x}]^{\ell-1})}}.
\end{align}

We finally approximate the constraint  as follows:
\begin{tcolorbox}[ams align]
 & (T^{\mathsf{ML}})^{-1} \left(\sum_{k=1}^{K} T^{\mathsf{Tot},(k)}+\Omega^{(k)}\right) \leq 1,\\
 & \frac{(A^{\mathsf{ML}})^{-1}}{(T^{\mathsf{ML}})^{-1} 
 \widehat{H}(\bm{x};\ell)} \leq 1,\\
 & A^{\mathsf{ML}}\geq 1.
\end{tcolorbox}

\textbf{Constraint~\eqref{prob:varrhoApp}}:  Considering this constraint,
    $\sum_{m\in \mathcal{N}}\varrho^{(k)}_{n,m} = 1$, 
we write it via introducing an auxiliary variable (for data offloading) as two inequalities:
\begin{align}
 & \sum_{m\in \mathcal{N}}\varrho^{(k)}_{n,m} \leq 1,\\
 & \frac{(A^{\mathsf{DO}})^{-1}}{\sum_{m\in \mathcal{N}}\varrho^{(k)}_{n,m}} \leq 1,\label{frac2App}\\
 & A^{\mathsf{DO}}\geq 1,
\end{align}
where $ A^{\mathsf{DO}}$ is added with a large penalty term to the objective function to force $ A^{\mathsf{DO}}\downarrow 1$ at the optimal point.
We further exploit arithmetic-geometric mean inequality (Lemma~\ref{Lemma:ArethmaticGeometric}) to approximate the denominator of~\eqref{frac2App}:
\begin{equation}
    G(\bm{x})=\sum_{m\in \mathcal{N}}\varrho^{(k)}_{n,m} \geq \widehat{G}(\bm{x};\ell)\triangleq
    \prod_{m\in \mathcal{N}} \left(\frac{\varrho^{(k)}_{n,m} G([\bm{x}]^{\ell-1})}{\left[\varrho^{(k)}_{n,m} \right]^{\ell-1}}\right)^{\frac{\left[\varrho^{(k)}_{n,m} \right]^{\ell-1} }{G([\bm{x}]^{\ell-1})}} .
\end{equation}
We finally approximate this constraint as follows:
\begin{tcolorbox}[ams align]
 & \sum_{m\in \mathcal{N}}\varrho^{(k)}_{n,m} \leq 1,\\
 & \frac{(A^{\mathsf{DO}})^{-1}}{\widehat{G}(\bm{x};\ell)} \leq 1,\\
 & A^{\mathsf{DO}}\geq 1.
\end{tcolorbox}

\textbf{Constraint~\eqref{prob:varphiApp}}: Considering this constraint, 
$ \sum_{m\in \mathcal{N}} \varphi^{(k)}_{n,m}=1,~~n\in\mathcal{N}$,
we write it via introducing an auxiliary variable (for model offloading) as two inequalities:
\begin{align}
 & \sum_{m\in \mathcal{N}} \varphi^{(k)}_{n,m}\leq 1,\\
 & \frac{(A^{\mathsf{MO},(k)})^{-1}}{\sum_{m\in \mathcal{N}} \varphi^{(k)}_{n,m}} \leq 1,\label{fracApp3}\\
 & A^{\mathsf{MO},(k)}\geq 1,
\end{align}
where $ A^{\mathsf{MO},(k)}$ is added with a large penalty term to the objective function to force $A^{\mathsf{MO},(k)}\downarrow 1$ at the optimal point.
We further exploit arithmetic-geometric mean inequality (Lemma~\ref{Lemma:ArethmaticGeometric}) to approximate the denominator of~\eqref{fracApp3}:
\begin{equation}
    J(\bm{x})=\sum_{m\in \mathcal{N}}\varphi^{(k)}_{n,m} \geq \widehat{J}(\bm{x};\ell)\triangleq
    \prod_{m\in \mathcal{N}} \left(\frac{\varphi^{(k)}_{n,m} J([\bm{x}]^{\ell-1})}{\left[\varphi^{(k)}_{n,m} \right]^{\ell-1}}\right)^{\frac{\left[\varphi^{(k)}_{n,m} \right]^{\ell-1} }{J([\bm{x}]^{\ell-1})}} .
\end{equation}

Finally, we transform this constraint as follows:
\begin{tcolorbox}[ams align]
 & \sum_{m\in \mathcal{N}}\varphi^{(k)}_{n,m} \leq 1,\\
 & \frac{(A^{\mathsf{MO},(k)})^{-1}}{\widehat{J}(\bm{x};\ell)} \leq 1,\\
 & A^{\mathsf{MO},(k)}\geq 1.
\end{tcolorbox}

    \textbf{Constraints}~\eqref{eq:varphi1App},~\eqref{eq:varphi2App}: These two constraints are not in the standard format of GP, since the right hand sides are $0$. We first focus on~\eqref{eq:varphi1App} and write it as follows:
    \begin{align}
        &\varphi_{n,n}^{(k)}\sum_{m \in \mathcal{N} \setminus \{n\}} \varphi_{n,m}^{(k)}+1 \leq A^{\mathbf{B_1},(k)} \label{frac4App},\\
        & A^{\mathbf{B_1},(k)}\geq 1,
    \end{align}
    where $A^{\mathbf{B_1},(k)}$ is added with a large penalty term to the objective function to force $A^{\mathbf{B_1},(k)}\downarrow 1$. Note that~\eqref{frac4App} can be readily written as an inequality on a posynomial, and thus the constraint can be written as follows:
     \begin{tcolorbox}[ams align]
&\left(A^{\mathbf{B_1},(k)}\right)^{-1}\left(\varphi_{n,n}^{(k)}\sum_{m \in \mathcal{N} \setminus \{n\}} \varphi_{n,m}^{(k)}+1\right) \leq 1\\
        & A^{\mathbf{B_1},(k)}\geq 1.
\end{tcolorbox}
    We next focus on~\eqref{eq:varphi2App} and perform the following algebraic steps:
     \begin{align}
         &(1-\varphi_{n,n}^{(k)})\sum_{m \in \mathcal{N} \setminus \{n\}} \varphi_{m,n}^{(k)} +1 \leq A^{\mathbf{B_2},(k)}\nonumber \\
         &
        \Rightarrow \sum_{m \in \mathcal{N} \setminus \{n\}} \varphi_{m,n}^{(k)} -\varphi_{n,n}^{(k)}\sum_{m \in \mathcal{N} \setminus \{n\}} \varphi_{m,n}^{(k)} +1 \leq A^{\mathbf{B_2},(k)}\\
         &\Rightarrow  \sum_{m \in \mathcal{N} \setminus \{n\}} \varphi_{m,n}^{(k)}  +1 \leq A^{\mathbf{B_2},(k)} +\varphi_{n,n}^{(k)}\sum_{m \in \mathcal{N} \setminus \{n\}} \varphi_{m,n}^{(k)}
         \\
         &\Rightarrow 
         \frac{\sum_{m \in \mathcal{N} \setminus \{n\}} \varphi_{m,n}^{(k)}  +1}{A^{\mathbf{B_2},(k)} +\varphi_{n,n}^{(k)}\sum_{m \in \mathcal{N} \setminus \{n\}} \varphi_{m,n}^{(k)}} \leq 1, \label{frac5App}
     \end{align}
    where $A^{\mathbf{B_2},(k)}$ is added with a penalty term to the objective function to ensure $A^{\mathbf{B_2},(k)} \downarrow 1$ at the optimal point.
     We next use Arithmetic-geometric mean inequality to approximate the denominator of~\eqref{frac5App} with a monomial as follows:
     \begin{equation}
      L(\bm{x})=   1 +\varphi_{n,n}^{(k)}\sum_{m \in \mathcal{N} \setminus \{n\}} \varphi_{m,n}^{(k)} \geq \widehat{L}(\bm{x};\ell) \triangleq \left(\frac{ L([\bm{x}]^{\ell-1})}{1}\right)^{\frac{1}{L([\bm{x}]^{\ell-1})}} 
      \prod_{m\in \mathcal{N}\setminus\{n\}} \left(\frac{\varphi^{(k)}_{n,n}\varphi^{(k)}_{m,n} L([\bm{x}]^{\ell-1})}{\left[\varphi^{(k)}_{n,n}\varphi^{(k)}_{m,n} \right]^{\ell-1}}\right)^{\frac{\left[\varphi^{(k)}_{n,n}\varphi^{(k)}_{m,n} \right]^{\ell-1} }{L([\bm{x}]^{\ell-1})}} .
     \end{equation}
     Finally, we write this constraint as follows:
     \begin{tcolorbox}[ams align]
 & \frac{\sum_{m \in \mathcal{N} \setminus \{n\}} \varphi_{m,n}^{(k)}  +1}{\widehat{L}(\bm{x};\ell)} \leq 1,\\
 &A^{\mathbf{B_2},(k)}\geq 1.
\end{tcolorbox}
     
\textbf{The ML bound in the objective function:} We now turn our attention to the term associated with the ML performance in the objective function:
\begin{equation}\label{eq:gen_conv_neymanApp1}
 \begin{aligned}
 & \hspace{-7mm}\frac{1}{K} \hspace{-1mm} \sum_{k=0}^{K-1} \mathbb E\left[\Vert \nabla F^{(k)}(\mathbf{w}^{({k})}) \Vert^2\right]\leq
 \frac{2\sqrt{\widehat{e}_{\mathsf{max}}}\left({F}^{(0)}(\mathbf{w}^{(0)})\right)}{\alpha \overline{e}_{\mathsf{min}}\sqrt{N K}(1-\Lambda_{\mathsf{max}})}
 +\sum_{k=0}^{K-1}\frac{2\sqrt{e^{(k)}_{\mathsf{sum}}}\Omega^{(k+1)} \Delta^{(k+1)} }{\alpha e^{(k)}_{\mathsf{avg}}\sqrt{N K}(1-\Lambda^{(k)})}\hspace{-16mm}
  \\&+\sum_{k=0}^{K-1} \frac{1}{(1-\Lambda^{(k)})}\vast({8 \beta^2\Theta^2 \frac{\alpha^2 N}{{e^{(k)}_{\mathsf{sum}} K^2}} }\sum_{n\in \mathcal{N}}\frac{\widehat{{D}}^{(k)}_n}{{D}^{(k)} }( e_n^{(k)})
   %%%%%%%%%%%%%
    \frac{1}{\left(\widehat{D}_n^{(k)}\right)^2}\vast[ \frac{1}{{B}_n^{(k)}} \left(\sum_{j=1}^{{S}^{(k)}_{n}} \widetilde{\sigma}^{(k)}_{n,j}{S}^{(k)}_{n,j} \right)^2  
    \\&-\sum_{j=1}^{{S}^{(k)}_{n}}{S}^{(k)}_{n,j} {\left(\widetilde{\sigma}^{(k)}_{n,j}\right)^2}\vast]
     +{8\zeta_2 \frac{\alpha^2N}{{e^{(k)}_{\mathsf{sum}} K^2}}\beta^2 \left(e_{\mathsf{max}}^{(k)}\right)^2}
     \vast)
\\&
     +\sum_{k=0}^{K-1}\frac{4 e^{(k)}_{\mathsf{avg}} \alpha \Theta^2 {\beta \sqrt{N} }}{\sqrt{e^{(k)}_{\mathsf{sum}}}K\sqrt{K}(1-\Lambda^{(k)})}\sum_{n\in \mathcal{N}} \left(\frac{1}{{D}^{(k)} } \right)^2\frac{1}{e^{(k)}_n}
\left[ \frac{1}{{B}_n^{(k)}} \left(\sum_{j=1}^{{S}^{(k)}_{n}} \widetilde{\sigma}^{(k)}_{n,j}{S}^{(k)}_{n,j} \right)^2-\sum_{j=1}^{{S}^{(k)}_{n}}{S}^{(k)}_{n,j} {\left(\widetilde{\sigma}^{(k)}_{n,j}\right)^2}\right].
 \end{aligned} 
 \end{equation}
 To have a tractable solution, we assume that the relative size of the strata to the size of the local dataset is upper bounded throughout the learning period, and let $s_{n,j}$ denote the upper bound of the relative size of  stratum $\mathcal{S}^{(k)}_{n,j}$ to the local dataset, i.e., $\frac{{S}^{(k)}_{n,j}}{\widehat{D}^{(k)}_{n}}\leq s_{n,j}, ~\forall k$. Accordingly, we have   
 \begin{equation}
 \sum_{j=1}^{{S}^{(k)}_{n}}{S}^{(k)}_{n,j} {\left(\widetilde{\sigma}^{(k)}_{n,j}\right)^2} \leq  \sum_{j=1}^{{S}^{(k)}_{n}} s_{n,j}\widehat{D}^{(k)}_n \left(\widetilde{\sigma}^{(k)}_{n,j}\right)^2= \widehat{D}^{(k)}_n \underbrace{\sum_{j=1}^{{S}^{(k)}_{n}} s_{n,j} \left(\widetilde{\sigma}^{(k)}_{n,j}\right)^2}_{\triangleq \widehat{Z}^{(k)}_n},
 \end{equation}
\begin{equation}
   \left(\sum_{j=1}^{{S}^{(k)}_{n}} \widetilde{\sigma}^{(k)}_{n,j}{S}^{(k)}_{n,j}\right)^2 \leq \left(\sum_{j=1}^{{S}^{(k)}_{n}} \widetilde{\sigma}^{(k)}_{n,j}{s}_{n,j}\widehat{D}^{(k)}_n\right)^2=\left(\widehat{D}^{(k)}_n\right)^2\underbrace{\left(\sum_{j=1}^{{S}^{(k)}_{n}} \widetilde{\sigma}^{(k)}_{n,j}{s}_{n,j}\right)^2}_{\triangleq  \widetilde{Z}^{(k)}_n}.
\end{equation} 
% Define $\widetilde{Z}^{(k)}_n=\left(\sum_{j=1}^{{S}^{(k)}_{n}} \widetilde{\sigma}^{(k)}_{n,j}{S}^{(k)}_{n,j}\right)^2$, $\widehat{Z}^{(k)}_n=\sum_{j=1}^{{S}^{(k)}_{n}}{S}^{(k)}_{n,j} {\Bigg(\widetilde{\sigma}^{(k)}_{n,j}\Bigg)^2}$, $C^{(k)}= \frac{\zeta_1+\Lambda^{(k)}}{\zeta_1}$.
Thus the expression in~\eqref{eq:gen_conv_neymanApp1} can be written as follows:
% \begin{equation}\label{eq:gen_conv_neymanApp2}
 \begin{align}
 & \frac{1}{K} \hspace{-1mm} \sum_{k=0}^{K-1} \mathbb E\left[\Vert \nabla F^{(k)}(\mathbf{w}^{({k})}) \Vert^2\right]\leq
 \frac{2\sqrt{\widehat{e}_{\mathsf{max}}}\left({F}^{(-1)}(\mathbf{w}^{(0)})\right)}{\alpha \overline{e}_{\mathsf{min}}\sqrt{N K}(1-\Lambda_{\mathsf{max}})}+\sum_{k=0}^{K-1}\frac{2\sqrt{e^{(k)}_{\mathsf{sum}}}\Omega^{(k+1)} \Delta^{(k+1)} }{\alpha e^{(k)}_{\mathsf{avg}}\sqrt{N K}(1-\Lambda^{(k)})}\hspace{-16mm}
  \nonumber \\&+\sum_{k=0}^{K-1} \frac{1}{(1-\Lambda^{(k)})}\vast({ \frac{8 \beta^2\Theta^2\alpha^2N}{{e^{(k)}_{\mathsf{sum}} K^2}} }\sum_{n\in \mathcal{N}}\frac{ e_n^{(k)}}{{D}^{(k)} \widehat{{D}}^{(k)}_n} 
   %%%%%%%%%%%%%
    \Bigg[ \frac{1}{{B}_n^{(k)}} \left(\widehat{D}^{(k)}_n\right)^2\widetilde{Z}^{(k)}_n
    -\widehat{D}^{(k)}_n\widehat{Z}^{(k)}_n\Bigg]
     +{8\zeta_2 \frac{\alpha^2N}{{e^{(k)}_{\mathsf{sum}} K^2}}\beta^2 \left(e_{\mathsf{max}}^{(k)}\right)^2}
     \vast)\nonumber \\&
     +\sum_{k=0}^{K-1}\frac{4 e^{(k)}_{\mathsf{avg}} \alpha \Theta^2 {\beta \sqrt{N} }}{\sqrt{e^{(k)}_{\mathsf{sum}}}K\sqrt{K}(1-\Lambda^{(k)})}\sum_{n\in \mathcal{N}} \left(\frac{1}{{D}^{(k)} } \right)^2\frac{1}{e^{(k)}_n}
\Bigg[ \frac{1}{{B}_n^{(k)}} \left(\widehat{D}^{(k)}_n\right)^2\widetilde{Z}^{(k)}_n
    -\widehat{D}^{(k)}_n\widehat{Z}^{(k)}_n\Bigg],
 \end{align} 
%  \end{equation}
which can be written as follows:
 \begin{equation}\label{eq:gen_conv_neyman3}
\footnotesize
\hspace{-10mm}
% \resizebox{.6\linewidth}{!}{$
\begin{aligned}
 &\Xi\triangleq  \sum_{k=0}^{K-1} \mathbb E\left[\Vert \nabla F^{(k)}(\mathbf{w}^{({k})}) \Vert^2\right] \leq  \underbrace{\left(\min_{1\leq k\leq K}\left\{\sum_{n\in \mathcal{N}}\frac{\widehat{{D}}^{(k)}_{n}e^{(k)}_{n}}{{D}^{(k)} }\right\}\right)^{-1} \left(\max_{1\leq k\leq K} \left\{\sum_{n\in\mathcal{N}} e_n^{(k)}\right\}\right)^{1/2}}_{(a)}  \underbrace{\frac{2\left(F^{(-1)}(\mathbf{w}^{(0)})\right)}{\alpha \sqrt{N K}(1-\Lambda_{\mathsf{max}})}}_{(b)}
 \hspace{-20mm}
 \\&+\sum_{k=0}^{K-1} \underbrace{\left(\sum_{n\in \mathcal{N}}\frac{\widehat{{D}}^{(k)}_{n}e^{(k)}_{n}}{{D}^{(k)} }\right)^{-1} \left(\sum_{n\in\mathcal{N}} e_n^{(k)}\right)^{1/2} }_{(c)}\underbrace{\frac{2\Omega^{(k+1)} \Delta^{(k+1)} }{\alpha \sqrt{N K}(1-\Lambda^{(k)})}}_{(d)}
  \\&+\sum_{k=0}^{K-1} \underbrace{\frac{1}{(1-\Lambda^{(k)})}}_{(e)}\vast(\underbrace{\frac{8 \beta^2\Theta^2 \alpha^2 N}{K^2} \left(\sum_{n\in\mathcal{N}} e_n^{(k)}\right)^{-1} \sum_{n\in \mathcal{N}}\frac{e_n^{(k)}\widehat{{D}}^{(k)}_n}{{D}^{(k)}}
   %%%%%%%%%%%%%
   \frac{1}{{B}_n^{(k)}} \widetilde{Z}_n^{(k)}}_{(f)}
    \\& +\underbrace{{\frac{8\zeta_2\beta^2\alpha^2 N}{K^2} \left(\sum_{n\in\mathcal{N}} e_n^{(k)}\right)^{-1}   \left( \max_{n\in \mathcal{N}}\{e_{n}^{(k)}\}\right)^2}}_{(g)}
     \vast)
     \\&+\sum_{k=0}^{K-1}\underbrace{\frac{4 \Theta^2 {\beta \alpha \sqrt{N} }}{\sqrt{K} K(1-\Lambda^{(k)})} 
     \left(\sum_{n \in \mathcal{N}}  \frac{\widehat{{D}}^{(k)}_{n}e^{(k)}_{n}}{{D}^{(k)} }\right)\left(\sum_{n\in\mathcal{N}} e_n^{(k)}\right)^{-1/2}
     }_{(h)}\underbrace{\sum_{n\in \mathcal{N}} \left(\frac{\widehat{{D}}^{(k)}_n}{{D}^{(k)} \sqrt{e^{(k)}_n}} \right)^2  \frac{1}{{B}_n^{(k)}} \widetilde{Z}_n^{(k)}}_{(i)}\\&
 -\sum_{k=0}^{K-1} \underbrace{\frac{1}{(1-\Lambda^{(k)})}\frac{8 \beta^2\Theta^2 \alpha^2 N}{K^2} \left(\sum_{n\in\mathcal{N}} e_n^{(k)}\right)^{-1} }_{(j)} \underbrace{\sum_{n\in \mathcal{N}}\frac{e_n^{(k)}}{{D}^{(k)}} \widehat{Z}_n^{(k)}}_{(m)}\\
 &-\sum_{k=0}^{K-1}\underbrace{\frac{4 \Theta^2 {\beta \alpha \sqrt{N} }}{\sqrt{K} K(1-\Lambda^{(k)})} 
     \left(\sum_{n\in\mathcal{N}}\frac{\widehat{{D}}^{(k)}_{n}e^{(k)}_{n}}{{D}^{(k)} }\right)\left(\sum_{n\in\mathcal{N}} e_n^{(k)}\right)^{-1/2}}_{(n)}\underbrace{\sum_{n\in \mathcal{N}} \left(\frac{\sqrt{\widehat{{D}}^{(k)}_n}}{{D}^{(k)} \sqrt{e^{(k)}_n}} \right)^2\widehat{Z}_n^{(k)}}_{(o)},
 \end{aligned}
%  $}
 \end{equation}
 which is explicitly written as the optimization variables. In the following, we break down the expression into different parts and carefully investigate the characteristics of each part with respect to the optimization variables.
 \begin{itemize}
     \item Term $(a)$ is in the form of a fraction with  a non-posynomial denominator and a posynomial in the denominator, which is not a posynomial. Term $(b)$ is a constant and thus a posynomial. Thus, term $(a)$ and $(b)$ combined are in the form of a ratio with a posynomial denominator which is not a posynomial. 
     \item Term $(c)$ has a non-posynomial denominator and a posynomial numerator, and thus is not a posynomial. Term $(d)$ is a posynomial. Thus, term $(c)$ and $(d)$ combined are in the form of a ratio with a posynomial in the denominator which is not a posynomial. 
     \item Term $(e)$ is a constant and thus a posynomial. Term $(f)$ is a ratio between a sum of monomials which is a posynomial and a posynomial, and thus it is not a posynomial. Also, term $(g)$ is a ratio between a max of monomials and a posynomial, which is not a posynomial. Thus, both combinations of term $(e)$ with $(f)$ and term $(e)$ with $(g)$ are not posynomials.
     \item Term $(h)$ is a ratio between a posynomial and square root of a posynomial, which is not a posynomial. Term $(i)$ is a sum of monomials which is a posynomial. Thus, combination of term $(h)$ with $(i)$ is not a posynomial.
     \item Term $(j)$ is a ratio between two posynomials, which is not a posynomial. Term $(m)$ is a sum of monomials which is a posynomial. Thus, combination of term $(j)$ and $(m)$ is a posynomial. 
     \item Term $(n)$ a ratio between a posynomial and square root of a posynomial, which is not a posynomial. Term $(o)$ is a sum of monomials which is a posynomial. Thus, combination of term $(n)$ and $(o)$ is not a posynomial.
 \end{itemize}
 
 Note that combination of two terms $(j)\&(m)$ is multiplied with a negative sign which makes the term non-posynomial even if the terms inside were posynomials. The same holds for the combination of terms $(n)\&(o)$. 
 
 In summary, we are faced with two challenges: first, the existence of division by posynomials and square root of posynomials and maximum of monomials inside the expressions; second, the existence of negative sign in the expression. In the following, we aim to alleviate these issues via introducing new constraints, using arithmetic-geometric mean to approximate the posynomial denominators with monomials, and some algebraic manipulations.

 \textbf{Introducing new constrains and exploiting the arithmetic-geometric mean:} We first incorporate $e^{(k)}_{\mathsf{sum}}$ and $e^{(k)}_{\mathsf{avg}}$ as the optimization variables, which makes the expressions simpler and more tractable. Then, to ensure that they take their desired values in the solution, we re-write their definitions as follows:
 \begin{align}
     & \frac{e^{(k)}_{\mathsf{sum}}}{\sum_{n\in\mathcal{N}} e_n^{(k)}}=1,\label{conmide1}\\
     & \frac{e^{(k)}_{\mathsf{avg}}}{\sum_{n\in \mathcal{N}}\frac{\widehat{{D}}^{(k)}_{n}e^{(k)}_{n}}{{D}^{(k)} }}=1\label{conmide2}.
 \end{align}
 However, the above two expressions are as the ratio of a monomial and a posynomial, which is not a posynomial. We thus exploit the arithmetic-geometric mean inequality in Lemma~\ref{Lemma:ArethmaticGeometric} and approximate the denominator of the above two constraints as below:
  \begin{align}
       R(\bm{x})= \sum_{n\in \mathcal{N}}e^{(k)}_{n}  \geq \widehat{R}(\bm{x};\ell)\triangleq \prod_{n\in\mathcal{N}}\left(\frac{e^{(k)}_{n}  R([\bm{x}]^{\ell-1})    }{\left[e^{(k)}_{n} \right]^{[\ell-1]}} \right)^{\frac{\left[e^{(k)}_{n} \right]^{[\ell-1]}}{R([\bm{x}]^{\ell-1})   }},
     \end{align}
  \begin{align}
        & V(\bm{x})= \sum_{n\in \mathcal{N}}\widehat{{D}}^{(k)}_{n}e^{(k)}_{n}  
        = \sum_{n\in \mathcal{N}} \sum_{m \in \mathcal{N}} {D}^{(k)}_{m} \varrho_{m,n}^{(k)} e^{(k)}_{n}  \nonumber \\
        & \geq 
        \widehat{V}(\bm{x};\ell)
        \triangleq 
        \prod_{n\in\mathcal{N}} \prod_{m \in \mathcal{N}}
        \left(\frac{ \varrho_{m,n}^{(k)} e^{(k)}_{n} {V}([\bm{x}]^{\ell-1})} 
        {\left[ \varrho_{m,n}^{(k)} e^{(k)}_{n} \right]^{[\ell-1]}} \right)
        ^{\frac{ \mathcal{D}^{(k)}_{m} \left[   \varrho_{m,n}^{(k)} e^{(k)}_{n} \right]^{[\ell-1]}}{{V}([\bm{x}]^{\ell-1}) }}.
     \end{align}
     As a result, constraints~\eqref{conmide1} and~\eqref{conmide2} can be expressed as follows:
% \begin{tcolorbox}[ams align]
%      & \frac{e^{(k)}_{\mathsf{sum}}}{\widehat{R}(\bm{x};\ell)}=1,\label{conmide1F}\\
%      & \frac{{D}^{(k)}e^{(k)}_{\mathsf{avg}}}{\widehat{V}(\bm{x};\ell)}=1\label{conmide2F},
%  \end{tcolorbox}
 
 \begin{tcolorbox}[ams align]
     & \left({e^{(k)}_{\mathsf{sum}}}\right)^{-1}\sum_{n\in\mathcal{N}} e_n^{(k)}\leq 1,\label{conmide1F1}\\
     & \frac{(A^{\mathsf{Sum},(k)})^{-1}}{\left({e^{(k)}_{\mathsf{sum}}}\right)^{-1}\widehat{R}(\bm{x};\ell)}\leq 1,\label{conmide1F2}\\
     &A^{\mathsf{Sum},(k)}\geq 1,\\
     & \left({e^{(k)}_{\mathsf{avg}}}\right)^{-1}\left[\sum_{n\in \mathcal{N}}\frac{\widehat{{D}}^{(k)}_{n}e^{(k)}_{n}}{{D}^{(k)} }\right]\leq 1\label{conmide2F1},\\
        & \frac{{D}^{(k)}(A^{\mathsf{Avg},(k)})^{-1}}{\left({e^{(k)}_{\mathsf{avg}}}\right)^{-1}\widehat{V}(\bm{x};\ell)}\leq 1\label{conmide2F2},\\
        &A^{\mathsf{Avg},(k)}\geq 1,
 \end{tcolorbox}
\noindent  where $A^{\mathsf{Sum},(k)}$ and $A^{\mathsf{Avg},(k)}$ are added with penalty terms to the objective function to ensure $A^{\mathsf{Sum},(k)} \downarrow 1$ and $A^{\mathsf{Avg},(k)} \downarrow 1$.
 As a result, we have translated both~\eqref{conmide1} and~\eqref{conmide2} to inequalities on posynomials, which is acceptable in GP. 
 
 We also find it useful to add $\widehat{D}^{(k)}_n$ as optimization variable to have tractable expressions. Considering $\widehat{D}^{(k)}_n=\sum_{m\in\mathcal{N}} \varrho^{(k)}_{m,n} D^{(k)}_m$, we take a similar approach as above and exploit the arithmetic-geometric mean inequality as follows:
  \begin{align}
     & I(\bm{x})=\sum_{m\in\mathcal{N}} \varrho^{(k)}_{m,n} D^{(k)}_m \geq \widehat{I}(\bm{x};\ell)\triangleq \prod_{m\in\mathcal{N}}\left(\frac{\varrho^{(k)}_{m,n} I([\bm{x}]^{\ell-1})}{[\varrho^{(k)}_{m,n}]^{\ell-1}}\right)^{\frac{D^{(k)}_m[\varrho^{(k)}_{m,n} ]^{\ell-1}}{I([\bm{x}]^{\ell-1})}}.
 \end{align}
We finally transform this constraint as:
\begin{tcolorbox}[ams align]
     & (\widehat{D}^{(k)}_n)^{-1} \left[\sum_{m\in\mathcal{N}} \varrho^{(k)}_{m,n} D^{(k)}_m\right] \leq 1,\\&
     \frac{(A^{\mathsf{D},(k)})^{-1}}{(\widehat{D}^{(k)}_n)^{-1} \widehat{I}(\bm{x};\ell)} \leq 1,\\&
     A^{\mathsf{D},(k)} \geq 1,
 \end{tcolorbox}
 where $A^{\mathsf{D},(k)}$ is added with a penalty terms to the objective function to ensure $A^{\mathsf{D},(k)}\downarrow 1$.
 
 Then, we introduce $\widehat{e}_{\mathsf{max}}$ and $\overline{e}_{min}$ (appeared in term (a) in \eqref{eq:gen_conv_neyman3}) into the optimization variables and impose the following constraints:
 \begin{tcolorbox}[ams align]
     & {e^{(k)}_{\mathsf{sum}}}(\widehat{e}_{\mathsf{max}})^{-1}\leq 1,\label{conmide1F}\\
     & {\overline{e}_{\mathsf{min}}}({e^{(k)}_{\mathsf{avg}}})^{-1} \leq 1\label{conmide2F},
 \end{tcolorbox}
 \noindent both of which are inequality on monomials. 
 
 Furthermore, we introduce $e^{(k)}_{\mathsf{max}}$ as another optimization variable and then introduce the following constraint:
 \begin{tcolorbox}[ams align]\label{conmide3F}
    \frac{e^{(k)}_{n}}{e^{(k)}_{\mathsf{max}}} \leq 1, ~n\in\mathcal{N},
 \end{tcolorbox}
\noindent  which is an inequality on a posynomial accepted in GP. 
 
 It can be readily validated that with the procedure taken above, i.e.,  extending the optimization variables space to contain $e^{(k)}_{\mathsf{sum}}$, $e^{(k)}_{\mathsf{avg}}$, $\widehat{D}^{(k)}_n$,  $\widehat{e}_{\mathsf{max}}$ $\overline{e}_{min}$, and $e^{(k)}_{\mathsf{max}}$, and introducing the above constraints, all the terms which are recognized by under-braces in~\eqref{eq:gen_conv_neyman3} are transformed to posynomials. We then, focus on the two summations in the last two lines of~\eqref{eq:gen_conv_neyman3} which have negative coefficients. 
 
 \textbf{Handling the terms with negative coefficients:}
 We write down $\Xi$ defined in~\eqref{eq:gen_conv_neyman3} as follows:
 \begin{equation}
     \Xi = \sum_{k=0}^{K-1} \left[\sigma_{1}^{(k)}(\bm{x}) + \sigma_{2}^{(k)}(\bm{x})  + \sigma_{3}^{(k)}(\bm{x})  + \sigma_{4}^{(k)}(\bm{x}) +\sigma_{5}^{(k)}(\bm{x})  -\sigma_{6}^{(k)}(\bm{x}) -\sigma_{7}^{(k)}(\bm{x})\right],
     \end{equation}
     where $\sigma_{1}$ is the first line of~\eqref{eq:gen_conv_neyman3} (i.e.,  $(a)\&(b)$) after the above variables are introduced as optimization variables, $\sigma_{2}$ is the second line of~\eqref{eq:gen_conv_neyman3} (i.e., $(c)\&(d)$) after the above variables are introduced as optimization variables, and so forth. Note that $\sigma_{6},\sigma_{7}$ are the last two lines of~\eqref{eq:gen_conv_neyman3} which appear with negative sign. We next define:
     \begin{align}
         &\sigma_{+}^{(k)}(\bm{x})\triangleq \sigma_{1}^{(k)}(\bm{x}) + \sigma_{2}^{(k)}(\bm{x})  + \sigma_{3}^{(k)}(\bm{x})  + \sigma_{4}^{(k)}(\bm{x}) +\sigma_{5}^{(k)}(\bm{x}),\\
         &\sigma_{-,1}^{(k)}(\bm{x})\triangleq   \sigma_{6}^{(k)}(\bm{x}), \\
         &\sigma_{-,2}^{(k)}(\bm{x})\triangleq \sigma_{7}^{(k)}(\bm{x}).
     \end{align}
     
    Note that $\Xi=\sum_{k=0}^{K-1}\left[\sigma_{+}^{(k)}(\bm{x})-\sigma_{-,1}^{(k)}(\bm{x})-\sigma_{-,2}^{(k)}(\bm{x})\right]$ corresponds to the term with the weight coefficient $\gamma$ in the objective function of $\bm{\mathcal{P}}$. We then transform the term $\gamma \Xi $ in the objective function as $\gamma \sum_{k=0}^{K-1} \chi^{(k)}$, in which the auxiliary variable $\chi^{(k)}$ is the upperbound of the summand in $\Xi$  (i.e., $\sigma_{+}^{(k)}(\bm{x})-\sigma_{-,1}^{(k)}(\bm{x})-\sigma_{-,2}^{(k)}(\bm{x})\leq \chi^{(k)}, ~\forall k$), which is  added to the constraints as follows:
     \begin{align}
     &\sigma_{+}^{(k)}(\bm{x})-\sigma_{-,1}^{(k)}(\bm{x})-\sigma_{-,2}^{(k)}(\bm{x}) \leq \chi^{(k)}\\
     &\equiv  \sigma_{+}^{(k)}(\bm{x}) \leq \chi^{(k)}+\sigma_{-,1}^{(k)}(\bm{x})+\sigma_{-,2}^{(k)}(\bm{x})\\
     & \equiv \frac{\sigma_{+}^{(k)}(\bm{x})} {\chi^{(k)}+\sigma_{-,1}^{(k)}(\bm{x})+\sigma_{-,2}^{(k)}(\bm{x}) } \leq 1 \label{eq:finalObjFrac}.
     \end{align}
     To transform this constraint to an upperbound inequality on a posynomial, we need to make the numerator of the fraction on the left hand side of~\eqref{eq:finalObjFrac} a posynomial and its denominator a monomial. Note that all the terms in the numerator are posynomials, which makes the numerator a posynomial. Thus, it is sufficient to focus on the approximation of the denominator, for which we exploit the arithmetic-geometric mean inequality as follows:
     \begin{align}
    W(\bm{x})=& \chi^{(k)}+\frac{1}{(1-\Lambda^{(k)})}\frac{8 \beta^2\Theta^2 \alpha^2 N}{K^2} (e^{(k)}_{\mathsf{sum}})^{-1} \sum_{n\in \mathcal{N}}\frac{e_n^{(k)}}{{D}^{(k)} } \widehat{Z}_n^{(k)} \nonumber \\
 &+\frac{4 \Theta^2 {\beta \alpha \sqrt{N} }}{\sqrt{K} K(1-\Lambda^{(k)})} 
     \left(e^{(k)}_{\mathsf{avg}}\right)\left(e^{(k)}_{\mathsf{sum}}\right)^{-1/2}\sum_{n\in \mathcal{N}} \frac{\widehat{{D}}^{(k)}_n}{\left({D}^{(k)}\right)^2 e^{(k)}_n} \widehat{Z}_n^{(k)}
     \nonumber\\&
    \geq  \widehat{W}(\bm{x};\ell)\triangleq \left(\frac{\chi^{(k)}  {W}([\bm{x}]^{\ell-1})    }{[\chi^{(k)}]^{\ell-1} } \right)^{\frac{[\chi^{(k)}]^{\ell-1}}{{W}([\bm{x}]^{\ell-1})   }}\times \prod_{n\in\mathcal{N}}  \prod_{q=1}^{2} \left(\frac{\delta_{q}(\bm{x},n)  {W}([\bm{x}]^{\ell-1})    }{\delta_{q}([\bm{x}]^{\ell-1},n) } \right)^{\frac{\delta_{q}([\bm{x}]^{\ell-1},n)}{{W}([\bm{x}]^{\ell-1})   }},
     \end{align}
    %  This $\widehat{W}$ needs ${\frac{\chi W}{\chi}}^{\chi/W}$, I'm adding it to the code, just fyi @ALI
    %  where  $W(\bm{x})=\chi^{(k)}+\sum_{n\in\mathcal{N}}\left[\delta_{1}(\bm{x},n)+\delta_{2}(\bm{x},n)+\delta_{3}(\bm{x},n)+\delta_{4}(\bm{x},n)\right]$ with
    where
     \begin{align}
       & \delta_{1}(\bm{x},n)= \frac{1}{(1-\Lambda^{(k)})}\frac{8 \beta^2\Theta^2 \alpha^2 N}{K^2} \frac{e_n^{(k)}}{e^{(k)}_{\mathsf{sum}}{D}^{(k)}} \widehat{Z}_n^{(k)},\\&
     \delta_{2}(\bm{x},n)= \frac{4 \Theta^2 {\beta \alpha \sqrt{N} }}{K\sqrt{K}(1-\Lambda^{(k)})} 
 \frac{e^{(k)}_{\mathsf{avg}}\widehat{{D}}^{(k)}_n}{  \left(e^{(k)}_{\mathsf{sum}}\right)^{1/2}\left({D}^{(k)}\right)^2 e^{(k)}_n} \widehat{Z}_n^{(k)},\\&\widehat{Z}^{(k)}_n=\sum_{j=1}^{S^{(k)}_n}{s}_{n,j} {\left(\widetilde{\sigma}^{(k)}_{n,j}\right)^2}.
     \end{align}
     We finally convert~\eqref{eq:finalObjFrac} to the following constraint:
     \begin{tcolorbox}[ams align]\label{conmide3F}
    \frac{\sigma_{+}^{(k)} (\bm{x})} {\widehat{W}(\bm{x};\ell)} \leq 1,
 \end{tcolorbox}
     which is an inequality on the ratio between a posynomial and a monomial and thus a posynomial.

     We finally obtain the optimization problem presented on the next page as our final formulation, which is an iterative-based optimization, where at each iteration, the optimization admits the GP format (coefficients $p_1,p_2,p_3,p_4,p_5,p_6,p_7,p_8 \gg 1$ in the objective function are constant penalty coefficients):
     
     \newpage
     \vspace{-9mm}
     {\footnotesize
 \begin{tcolorbox}[ams align]
     &\hspace{-4mm} \footnotesize (\bm{\mathcal{P}'}): \min \frac{1}{K}\sum_{k=0}^{K-1} \hspace{-.5mm} \left[c_1 \left( \sum_{n\in \mathcal{N}}E^{(k)}_n\right)+c_2 \left( T^{\mathsf{D},(k)}+T^{\mathsf{L},(k)}+T^{\mathsf{M},(k)}+T^{\mathsf{U},(k)}\right)\hspace{-.5mm}+ c_3 \chi^{(k)}\right]\hspace{-.5mm}\nonumber \\
     &\hspace{10mm}+ p_1 A^{\mathsf{ML}} \hspace{-.5mm}+p_2 A^{\mathsf{DO}} +\hspace{-.5mm}\sum_{k=0}^{K-1} \left[p_3 A^{\mathsf{MO},(k)}\hspace{-.5mm}+p_4 A^{\mathsf{D},(k)} \hspace{-.5mm}+p_5 A^{\mathsf{Sum},(k)}\hspace{-.5mm}+p_6 A^{\mathsf{Avg},(k)}+p_7 A^{\mathsf{B_1},(k)}+p_8 A^{\mathsf{B_2},(k)}\right]
     \label{prob:INITIAL} \hspace{-4mm} \\[-1em]
     & \textrm{s.t.}\nonumber\\
     & \frac{\sigma_{+}^{(k)} (\bm{x}) }{\widehat{W}(\bm{x};\ell)} \leq 1,\\
      & \left({e^{(k)}_{\mathsf{sum}}}\right)^{-1}\sum_{n\in\mathcal{N}} e_n^{(k)}\leq 1,\label{conmide1F1}\\
     & \frac{(A^{\mathsf{Sum},(k)})^{-1}}{\left({e^{(k)}_{\mathsf{sum}}}\right)^{-1}\widehat{R}(\bm{x};\ell)}\leq 1,\label{conmide1F2}\\
      & \left(\widehat{D}^{(k)}_n\right)^{-1} \left[\sum_{m\in\mathcal{N}} \varrho^{(k)}_{m,n} D^{(k)}_m\right] \leq 1\\&
     \frac{(A^{\mathsf{D},(k)})^{-1}}{\left(\widehat{D}^{(k)}_n\right)^{-1} \widehat{I}(\bm{x};\ell)} \leq 1,\\
     & \left({e^{(k)}_{\mathsf{avg}}}\right)^{-1}\left[\sum_{n\in \mathcal{N}}\frac{\widehat{{D}}^{(k)}_{n}e^{(k)}_{n}}{{D}^{(k)} }\right]\leq 1\label{conmide2F1},\\
        & \frac{{D}^{(k)}(A^{\mathsf{Avg},(k)})^{-1}}{\left({e^{(k)}_{\mathsf{avg}}}\right)^{-1}\widehat{V}(\bm{x};\ell)}\leq 1\label{conmide2F2},\\
      &{e^{(k)}_{n}}\left({e^{(k)}_{\mathsf{max}}}\right)^{-1} \leq 1, ~n\in\mathcal{N},\\
       & {e^{(k)}_{\mathsf{sum}}}\left(\widehat{e}_{\mathsf{max}}\right)^{-1}\leq 1,\\
     & {\overline{e}_{\mathsf{min}}}\left({e^{(k)}_{\mathsf{avg}}}\right)^{-1} \leq 1,\\
      & \left(T^{\mathsf{ML}}\right)^{-1} \left(\sum_{k=1}^{K} T^{\mathsf{Tot},(k)}+\Omega^{(k)}\right) \leq 1,\\
 & \frac{(A^{\mathsf{ML}})^{-1}}{(T^{\mathsf{ML}})^{-1} 
 \widehat{H}(\bm{x};\ell)} \leq 1,\\
     & \left(T^{\mathsf{D},(k)}\right)^{-1}\varrho^{(k)}_{m,n} {D}^{(k)}_m b^{\mathsf{D}} /r_{m,n} \leq 1 ,~  m,n\in\mathcal{N},\\
   &\left(T^{\mathsf{L},(k)}\right)^{-1}e^{(k)}_n{a_n B^{(k)}_n}/{f^{(k)}_n} \leq 1, ~ n\in\mathcal{N},\\
     &   {\left(T^{\mathsf{M},(k)} \right)^{-1}\varphi_{m,n} {M}b^{\mathsf{G}} }/{r_{m,n}} \leq 1, ~ m,n\in \mathcal{N},\\
       &  \left( T^{\mathsf{U},(k)} \right)^{-1}   {Mb^{\mathsf{G}}\varphi_{n,n}}/{r_n} \leq 1,  n\in\mathcal{N},\\
    %  &D^{\mathsf{F},(k)}_n=\sum_{m\in \mathcal{N}} \varrho^{(k)}_{m,n}{D}^{(k)}_m ,~~ n\in\mathcal{N},\\ 
    & \sum_{m\in \mathcal{N}}\varrho^{(k)}_{n,m} \leq 1,\\
 & \frac{(A^{\mathsf{DO}})^{-1}}{\widehat{G}(\bm{x};\ell)} \leq 1,\\
      & \sum_{m\in \mathcal{N}}\varphi^{(k)}_{n,m} \leq 1,\\
 & \frac{(A^{\mathsf{MO},(k)})^{-1}}{\widehat{J}(\bm{x};\ell)} \leq 1,\\
%  & A^{\mathsf{MO},(k)}\geq 1, A^{\mathsf{Avg},(k)}\geq 1,A^{\mathsf{DO}}\geq 1,A^{\mathsf{ML}}\geq 1,\\
 &\left(A^{\mathbf{B_1},(k)}\right)^{-1}\left(\varphi_{n,n}^{(k)}\sum_{m \in \mathcal{N} \setminus \{n\}} \varphi_{n,m}^{(k)}+1\right) \leq 1,~~n\in\mathcal{N},\\&
 \frac{\sum_{m \in \mathcal{N} \setminus \{n\}} \varphi_{m,n}^{(k)}  +1}{\widehat{L}(\bm{x};\ell)} \leq 1, ~~n\in\mathcal{N},\\&
 %%%%%%%%%%%%%%%%%%%%
 %%%%%%%%%%%%%%%%%%%%
 %%%%%%%%%%%%%%%%%%%% FEASIBILITY
    A^{\mathsf{MO},(k)}\hspace{-.5mm}, A^{\mathsf{Avg},(k)}\hspace{-.5mm},A^{\mathsf{D},(k)}\hspace{-.5mm},A^{\mathsf{Sum},(k)}\hspace{-.5mm},A^{\mathbf{B_1},(k)}\hspace{-.5mm},A^{\mathbf{B_2},(k)},A^{\mathsf{DO}},A^{\mathsf{ML}}\geq 1,  \frac{f^{\mathsf{min}}_n}{f^{(k)}_n}\leq 1,~~\frac{ f^{(k)}_n}{f^{\mathsf{max}}_n}\leq 1,~~ 1\leq B_n^{(k)}\leq \widehat{D}^{(k)}_n,~~ n\in\mathcal{N} \\
     & {\widehat{D}}_n^{(k)}\hspace{-.5mm},{e}_n^{(k)}\hspace{-.5mm},{f}_n^{(k)}\hspace{-.5mm},T^{\mathsf{D},(k)}\hspace{-.5mm},T^{\mathsf{L},(k)}\hspace{-.5mm},T^{\mathsf{M},(k)}\hspace{-.5mm},T^{\mathsf{U},(k)}\hspace{-.5mm},\Omega^{(k)}\hspace{-.5mm},e^{(k)}_{\mathsf{sum}}\hspace{-.2mm},e^{(k)}_{\mathsf{avg}}\hspace{-.2mm},e^{(k)}_{\mathsf{max}}\hspace{-.1mm},\chi^{(k)}, \overline{e}_{\mathsf{min}},\widehat{e}_{\mathsf{max}},\varrho^{(k)}_{n,m},\varphi^{(k)}_{n,m} > 0, ~~n,m\in \mathcal{N},\label{prob:FINAL}\\
      \cline{1-2}
     &\hspace{-5mm}\textrm{\textit{\textbf{Variables}}:}\nonumber\\
     &\scriptsize \hspace{-4mm} K,\big\{\bm{\widehat{D}}^{(k)}\hspace{-.5mm},\mathbf{e}^{(k)}\hspace{-.5mm},\mathbf{f}^{(k)}\hspace{-.5mm},\mathbf{B}^{(k)}\hspace{-.5mm},\bm{\varrho}^{(k)}\hspace{-.5mm}, \bm{\varphi}^{(k)}\hspace{-.5mm},T^{\mathsf{D},(k)}\hspace{-.5mm},T^{\mathsf{L},(k)}\hspace{-.5mm},T^{\mathsf{M},(k)}\hspace{-.5mm},T^{\mathsf{U},(k)}\hspace{-.5mm},\Omega^{(k)}\hspace{-.5mm},e^{(k)}_{\mathsf{sum}}\hspace{-.2mm},e^{(k)}_{\mathsf{avg}}\hspace{-.2mm},e^{(k)}_{\mathsf{max}}\hspace{-.1mm},\chi^{(k)}\hspace{-.5mm},A^{\mathsf{MO},(k)}\hspace{-.5mm}, A^{\mathsf{Avg},(k)}\hspace{-.5mm},A^{\mathsf{D},(k)}\hspace{-.5mm},A^{\mathsf{Sum},(k)}\hspace{-.5mm},A^{\mathbf{B_1},(k)}\hspace{-.5mm},A^{\mathbf{B_2},(k)}\big\}_{k=1}^{K},\nonumber \\
     &\hspace{-4mm}A^{\mathsf{DO}},A^{\mathsf{ML}},\overline{e}_{\mathsf{min}},\widehat{e}_{\mathsf{max}} \nonumber \hspace{-10mm} 
 \end{tcolorbox}
 }
 
%  \noindent    

 \section{Numerical Evaluation of the Optimality and Complexity of Algorithm~\ref{alg:cent}}\label{app:NumOpt}

To numerically assess the optimality obtained by the solver of our method, we conduct an exhaustive search over the minibatch size, one of the critical variables in our problem. 
% For an explanation of the complexity and optimality of algorithm 1, please see \textbf{R1.6}. Herein, we detail our result for the optimality of algorithm 1. 
The results are presented in Fig.~\ref{fig:optimality_minib}  for the same settings as in Fig.~\ref{fig:test_obj_value} of the manuscript, with $N = 5$ devices, where we fix the value of all optimization variables except minibatch size at one device.
% , and assume that one can violate energy and timing constraints as our optimization variables are tightly coupled together (i.e., changing the minibatch size with fixed other variables can result in timing or energy constraint violations). 
We can see that our result, indicated by the red circle and arrow, achieves the optimal objective function value for three choices of energy importance (i.e., $c_1$ in the objective of our optimization problem $\bm{\mathcal{P}}$).

To assess computational complexity, we measure the run time of our algorithm in Fig.~\ref{fig:complexity_timing}, again for the same settings as Fig.~\ref{fig:test_obj_value}, with varying numbers of devices (the computation times are obtained upon conducting 15 iterations in Alg.~\ref{alg:cent} since as seen from Fig.~\ref{fig:test_obj_value} the algorithm converges after that point). The run time of our method increases in a near quadratic relationship as the number of devices increases, which we expect based on the space complexity of the optimization discussed in Sec. IV.

\begin{figure}[h!]
    \centering
    \includegraphics[width=0.4\textwidth]{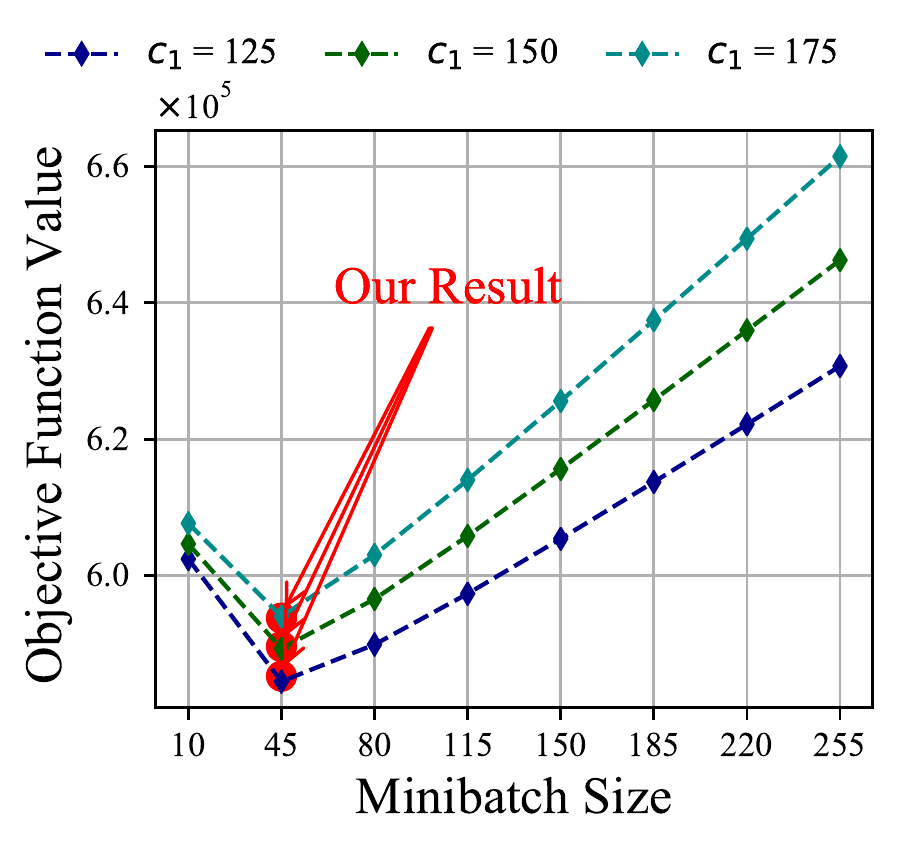}
    \vspace{-2mm}
    \caption{Numerical evaluation of the optimality of the obtained solution of Alg. 1 under three different objective function configurations. The objective function value is shown as the mini-batch size of one device is varied. Our result from the optimization  is indicated by the red circle and arrow, which gives the smallest objective function values.}
    \label{fig:optimality_minib}
\end{figure}

\begin{figure}[h!]
    \centering
    \includegraphics[width=0.4\textwidth]{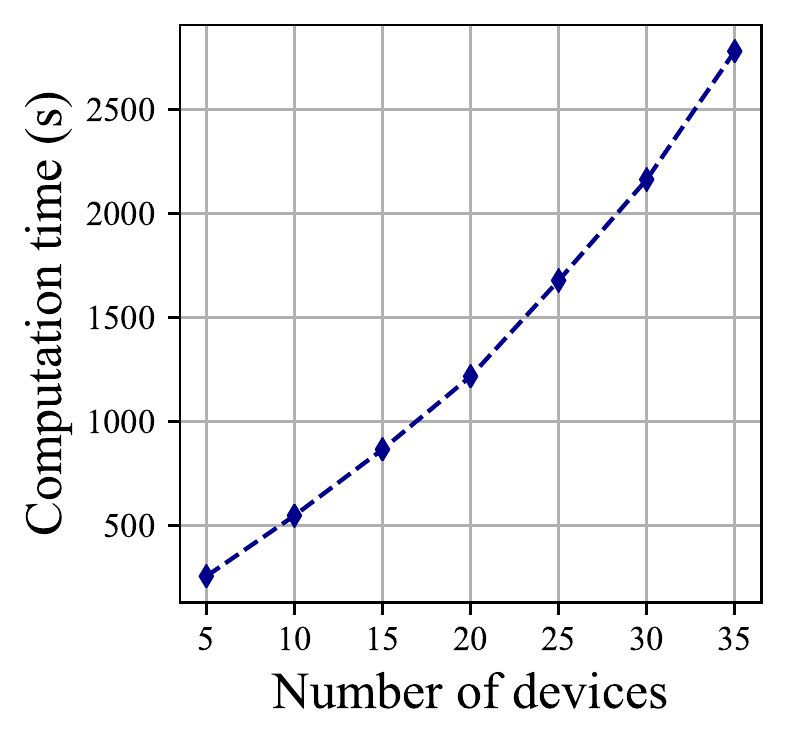}
    \vspace{-2mm}
    \caption{Numerical evaluation of the complexity  Alg. 1 in terms of the run time  for different number of devices.}
    \label{fig:complexity_timing}
\end{figure}

\vspace{-0mm}

\end{document}